\documentclass[9.5pt,journal,compsoc,cspaper]{IEEEtran}
\usepackage{enumitem}

\usepackage{amsmath,amsfonts,bm}

\def\1{\bm{1}}

\def\vd{{\bm{d}}}

\def\vf{{\bm{f}}}
\def\vg{{\bm{g}}}

\def\vl{{\bm{l}}}

\DeclareMathAlphabet{\mathsfit}{\encodingdefault}{\sfdefault}{m}{sl}
\SetMathAlphabet{\mathsfit}{bold}{\encodingdefault}{\sfdefault}{bx}{n}

\usepackage{url}            
\usepackage{booktabs}       
\usepackage{amsfonts}       
\usepackage{nicefrac}       
\usepackage{microtype}      

\usepackage{subfigure}
\usepackage{dsfont}
\usepackage{amsthm}
\usepackage{mathrsfs}
\usepackage{algorithm}
\usepackage{algpseudocode}
\usepackage{bbm}
\usepackage{mathtools}
\usepackage{stackengine}
\usepackage{epsfig}
\usepackage{graphicx}
\usepackage{amsmath}
\usepackage{amssymb}
\usepackage{pifont}
\usepackage{wrapfig}
\usepackage{enumitem}
\usepackage{makecell}
\usepackage{multirow}
\usepackage{bibunits}

\usepackage{xspace}
\makeatletter
\DeclareRobustCommand\onedot{\futurelet\@let@token\@onedot}
\def\@onedot{\ifx\@let@token.\else.\null\fi\xspace}

\def\eg{\emph{e.g}\onedot} 
\def\ie{\emph{i.e}\onedot}

\def\wrt{w.r.t\onedot} 

\makeatother

\newtheorem{proposition}{Proposition}
\newtheorem{theorem}{Theorem}
\newtheorem{definition}{Definition}
\newtheorem{lemma}{Lemma}
\newtheorem{corollary}{Corollary}
\newtheorem{assumption}{Assumption}

\ifCLASSOPTIONcompsoc
  \usepackage[nocompress]{cite}
\else
  \usepackage{cite}
\fi

\begin{document}


\title{Partial Distribution Matching via \\ Partial Wasserstein Adversarial Networks}

\author{Zi-Ming Wang, Nan Xue, Ling Lei, Rebecka Jörnsten, Gui-Song Xia
\IEEEcompsocitemizethanks{
\IEEEcompsocthanksitem Zi-Ming Wang and Rebecka Jörnsten are with the Department of Mathematical Sciences, University of Chalmers, Gothenburg, 41296, Sweden (email: zimingwa@chalmers.se, jornsten@chalmers.se).
\IEEEcompsocthanksitem Nan Xue is with Ant Group, Hangzhou, Zhejiang, 310013, China (email:xuenan@ieee.org).
\IEEEcompsocthanksitem Ling Lei is with the School of Mathematics and Statistics, Wuhan University, Wuhan, 430072, China (email: leiling@whu.edu.cn).
\IEEEcompsocthanksitem Gui-Song Xia is with the School of Computer Science, National Engineering Research Center for Multimedia Software (NERCMS), and Institute for Mathematics and Artificial Intelligence, Wuhan University, Wuhan, 430072, China (email: guisong.xia@whu.edu.cn).}
\thanks{Corresponding author: Gui-Song Xia}
}

\IEEEtitleabstractindextext{%
\begin{abstract}
    This paper studies the problem of \emph{distribution matching} (DM), 
    which is a fundamental machine learning problem seeking to robustly align two probability distributions. 
    Our approach is established on a relaxed formulation, 
    called partial distribution matching (PDM),
    which seeks to match a fraction of the distributions instead of matching them completely.
    We theoretically derive the Kantorovich-Rubinstein duality for the {\em partial Wasserstain-$1$} (PW) discrepancy,
    and develop a {\em partial Wasserstein adversarial network} (PWAN) that efficiently approximates the PW discrepancy based on this dual form.
    Partial matching can then be achieved by optimizing the network using gradient descent.
    Two practical tasks, 
    point set registration and partial domain adaptation are investigated, 
    where the goals are to partially match distributions in 3D space and high-dimensional feature space respectively. 
    The experiment results confirm that the proposed PWAN effectively produces highly robust matching results, performing better or on par with the state-of-the-art methods.
\end{abstract}

}

\maketitle

\IEEEdisplaynontitleabstractindextext

\IEEEpeerreviewmaketitle

\IEEEraisesectionheading{\section{Introduction}\label{sec:introduction}}

\IEEEPARstart{D}{istribution} matching (DM) is a fundamental machine learning task with many applications.
As illustrated in Fig.~\ref{DM_PDM_D},
given two probability distributions,
the goal of DM is to match one distribution to the other.
For example,
in generative modelling,
to describe the observed data,
a parametric model distribution is matched to the data distribution.
Recently advanced DM algorithms~\cite{pmlr-v70-arjovsky17a, song2019generative, ho2020denoising} 
are well known for being able to handle highly complicated distributions.
However, 
these methods
still have difficulty in handling
distributions that are contaminated by outliers,
because their attempt to \emph{completely} match the ``dirty'' distributions will inevitably lead to biased results.

A natural way to address this issue is to consider the \emph{partial} distribution matching (PDM) task as a relaxation of the DM task:
given two unnormalized distributions of mass,
\ie, the total mass of each distribution is not necessarily $1$,
the goal of PDM is to match a certain fraction of mass of the two distributions.
An example of PDM is presented in Fig.~\ref{DM_PDM_P}.
In contrast to DM,
PDM allows to omit the fraction of distribution that does not match,
thus it naturally tolerates outliers better.
In other words, 
it is more suitable for applications with dirty data.

\begin{figure}[t!]
  \centering
  \subfigure[
  Given two normalized distributions,
  DM aims to match the source distribution (red) to the reference distribution (blue).
    ]{
    \label{DM_PDM_D}
    \includegraphics[width=0.9\linewidth]{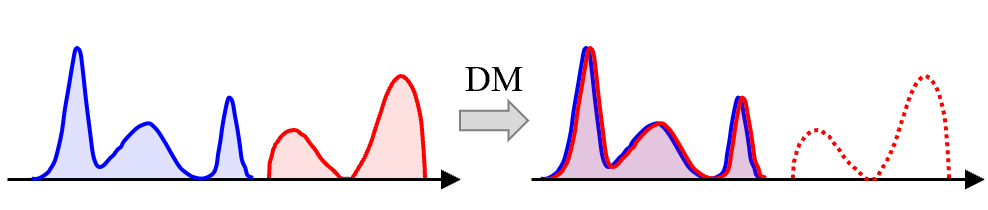}
  }
  \subfigure[
  Given two unnormalized distributions,
  PDM aims to match a certain fraction of the source distribution (red) to the reference distribution (blue).
    ]{
    \label{DM_PDM_P}
    \includegraphics[width=0.9\linewidth]{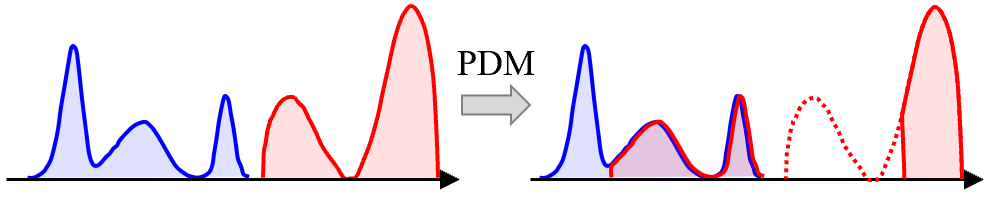}
  }
\caption{Comparison between DM and PDM.
DM aims to completely match two distributions,
but PDM only aims to match a certain fraction of them. 
}
\label{DM_PDM}
\end{figure}

Several PDM methods have been investigated,
but they tend to be computationally expensive for large-scale PDM problems,
because they generally rely on the correspondence between distributions~\cite{chizat2018scaling,chapel2020partial, mukherjee2021outlier}.
To address this issue,
we propose an efficient PDM algorithm, 
called partial Wasserstein adversarial network (PWAN), 
based on deep neural networks and the optimal transport (OT) theory~\cite{villani2016optimal}.
Our key idea is to partially match the distributions by minimizing their partial Wasserstein-1 (PW) discrepancy~\cite{figalli2010the} which is approximated by a neural network.
Specifically,
we first derive the Kantorovich-Rubinstein (KR) dual form~\cite{villani2003topics} for the PW discrepancy,
and show that its gradient can be explicitly computed via its potential.
These results then allow us to compute the PW discrepancy by approximating the potential using a deep neural network,
and minimize it by gradient descent.
Compared to the existing PDM methods,
the proposed PWAN is correspondence-free,
thus it is more efficient for large-scale PDM tasks.
In addition,
PWAN is a direct generalization of the well-known DM algorithm Wasserstein GAN (WGAN)~\cite{pmlr-v70-arjovsky17a} for PDM tasks.

To show the effectiveness of the proposed PWAN,
we evaluate it on two practical tasks which we formulate as PDM problems: 
point set registration~\cite{myronenko2006non} and partial domain adaptation~\cite{cao2022big, cao2018partial}.
Point set registration seeks to align two point sets representing incomplete shapes  in 3D space.
We regard this task as a PDM problem by treating point sets as discrete distributions  in 3D space;
Partial domain adaptation seeks to align the source domain to a fraction of the reference domain.
We formulate this task as a PDM problem in the high dimensional space of the extracted features.

These two tasks are challenging because they both require matching large-scale distributions containing a high ratio of outliers.
Specifically,
the distributions may contain up to $10^6$ samples,
and the outlier ratio can be up to $90\%$ in adaptation tasks, 
and $60\%$ in registration tasks.
Thus, 
the DM methods and existing correspondence-based PDM methods are difficult to apply.
In addition,
the datasets in adaptation tasks are usually too large to be processed as a whole,
thus mini-batch approaches are generally necessary.

We apply PWAN in its canonical form to registration tasks,
and develop an efficient coherence~\cite{myronenko2006non} regularizer to handle the non-rigid deformations of point sets.
As for adaptation tasks,
we extend PWAN to its mini-batch form, 
which does not introduce mini-batch errors like existing mini-batch OT methods~\cite{fatras2021unbalanced,nguyen2022improving,fatras2020learning},
\ie, it can produce accurate results even with small batch sizes.
In addition,
by utilizing the assumption of the clean source domain,
\ie, outliers only exist in the reference domain,
we derive a simplified PWAN formulation,
which is easier to train,
and can be seen as a direct extension of the classic closed-set adaptation methods~\cite{ganin2016domain,shen2018wasserstein}.
Our numerical experiments show that PWAN can produce highly robust matching results,
and perform better than or is at least comparable with the state-of-the-art methods in both tasks.

In summary, 
the contributions of this work are:
\begin{itemize}[leftmargin=4mm]
  \item[-] 
  Theoretically, 
  we derive the Kantorovich-Rubinstein (KR) duality of the partial Wasserstein-1 (PW) discrepancy.
  We further study its differentiability and present a qualitative description of the solution to the KR dual form.
  \item[-] 
  Based on the theoretical results,
  we propose a novel algorithm, 
  called Partial Wasserstein Adversarial Network (PWAN),
  for partial distribution matching (PDM).
  PWAN approximates distance-type or mass-type PW divergences by neural networks,
  thus it can be efficiently trained using gradient descent.
  We further extend PWAN to its mini-batch form,
  which does not introduce mini-batch errors.
  Notably, 
  PWAN is a generalization of the well-known Wasserstein GAN.
  \item[-] 
  We apply PWAN to point set registration,
  where we partially align discrete distributions representing point sets in 3D space.
  We evaluate PWAN on both non-rigid and rigid point set registration tasks,
  and show that PWAN exhibits strong robustness against outliers,
  and can register point sets accurately even when they are dominated by noise points or are only partially overlapped.
  \item[-] 
  We apply the mini-batch version PWAN to partial domain adaptation,
  where we align the source domain to a fraction of the reference domain.
  By exploiting the assumption of the clean source domain,
  we obtain a simplified PWAN formulation,
  which 
  generalizes the classic adversarial closed-set domain adaptation methods.
  We evaluate our approach on four benchmark datasets,
  and show that it can effectively align domains 
  even with a high outlier ratio.
\end{itemize}

The rest of this work is organized as follows: 
Sec.~\ref{Sec_related} reviews some related works. 
Sec.~\ref{Sec_Preliminaries} recalls the definitions of PW divergences which are the major tools used in PWAN.
Sec.~\ref{Sec_proposed_method} derives the formulation of PWAN, 
and provides a concrete training algorithm. Some connections between PWAN and WGAN are also discussed.
Sec.~\ref{Sec_application_1} applies PWAN to point set registration,
and Sec.~\ref{Sec_application_2} applies PWAN to partial domain adaptation.
Sec.~\ref{Sec_conclusion} finally draws some conclusions.

This work extends our earlier work~\cite{wang2022partial} in the following ways:
\begin{itemize}[leftmargin=4mm]
  \item[-]  We complete the point set registration experiments by applying PWAN to rigid registration tasks.
  \item[-]  We extend PWAN to its mini-batch form,
  which does not introduce mini-batch errors like existing mini-batch OT methods~\cite{fatras2020learning,fatras2021unbalanced,nguyen2022improving}.
  We numerically show that mini-batch PWAN can produce accurate results regardless of batch sizes.
  \item[-]  
  We apply the mini-batch version PWAN to partial domain adaptation tasks,
  where we derive a simplified PWAN formulation by exploiting the assumption of the task that the source domain is clean.
  Our formulation 
  is a direct generalization of the classic adversarial closed-set domain adaptation methods~\cite{shen2018wasserstein,ganin2016domain}.
  We experimentally demonstrate its effectiveness on four benchmark datasets.
\end{itemize}

\section{Related Works}
\label{Sec_related}

\subsection{Computational Optimal Transport}
PWAN is developed based on OT theory~\cite{villani2016optimal, figalli2010the, caffarelli2010free},
which is a powerful tool for comparing distributions.
Various types of numerical OT solvers have been proposed~\cite{bai2023sliced,bonneel2019spot,bonneel2015sliced, makkuva2020optimal, mukherjee2021outlier, schmitzer2019stabilized, chapel2020partial, bonneel2015sliced, schmitzer2019framework}.
A well-known type of solver is the entropy-regularized solver~\cite{benamou2015iterative, cuturi2013sinkhorn, chizat2018scaling},
which iteratively estimates the transport plan, \ie, the correspondence, between distributions.
This type of solver is flexible as it can handle a wide range of cost functions~\cite{villani2016optimal},
but it is generally computationally expensive for large-scale distributions.
Some efficient mini-batch-based approximations~\cite{fatras2021unbalanced,fatras2020learning} have been proposed to alleviate this issue,
but these methods are known to suffer from mini-batch errors, 
thus extra modifications are usually needed~\cite{nguyen2022improving}.

PWAN belongs to the Wasserstein-1-type OT solver,
which is dedicated to the special OT problem where the cost function is a metric.
This type of solver exploits the KR duality of the Wasserstein-1-type metrics,
thus it is generally highly efficient.
A representative method in this class is WGAN~\cite{pmlr-v70-arjovsky17a},
which approximates the Wasserstein-1 distance using neural networks. 
PWAN is a direct generalization of WGAN,
because it approximates the PW divergence, 
which is a generalization of the Wasserstein-1 distance,
under the same principle as WGAN.
In addition,
\cite{lellmann2014imaging, schmitzer2019framework} studied the KR duality of various Wasserstein-1-type divergences, 
including the distance-type PW divergence considered in this work.
Our work completes these works~\cite{lellmann2014imaging, schmitzer2019framework} in the sense that we additionally study the mass-type PW divergence,
and we provide a continuous approximation of PW discrepancies using neural networks,
which is suitable for a broad range of applications.

\subsection{Point Set Registration}
Point set registration is a classic PDM task that seeks to match discrete distributions, 
\ie, point sets.
This task has been extensively studied for decades,
and various methods have been proposed~\cite{besl1992a, zhang1994iterative, chui2000a, myronenko2006non, maiseli2017recent, vongkulbhisal2018inverse, hirose2021a}.
PWAN is related to the probabilistic registration methods,
which solve the registration task as a DM problem.
Specifically,
these methods first smooth the point sets as Gaussian mixture models (GMMs), 
and then align the distributions via minimizing the robust discrepancies between them.
Coherent point drift (CPD)~\cite{myronenko2006non} and its variants~\cite{hirose2021a} are well-known examples in this class,
which minimize Kullback-Leibler (KL) divergence between the distributions. 
Other robust discrepancies,
including $L_2$ distance~\cite{ma2013robust, jian2011robust}, 
kernel density estimator~\cite{tsin2004a} and scaled Geman-McClure estimator~\cite{zhou2016fast} have also been studied.

The proposed PWAN has two major differences from the existing probabilistic registration methods.
First,
it directly processes the point sets as discrete distributions instead of smoothing them to GMMs,
thus it is more concise and natural.
Second,
it solves a PDM problem instead of a DM problem,
as it only requires matching a fraction of points,
thus it is more robust to outliers.

\subsection{Partial Domain Adaptation}
Partial domain adaptation~\cite{cao2018partial, cao2022big, kuhnke2019deep} is a recent extension of the classic closed-set domain adaptation task~\cite{flamary2016optimal,ganin2016domain}.
The goal of partial adaptation is to train a classifier for the source data when the source and reference data are sampled from different domains and the source label space is a subset of the reference label space.
This task was first formally addressed by~\cite{cao2018partial},
where a re-weighting procedure was used to discard some samples as outliers,
and a DM model was then used to align the rest of the data distributions.
Other DM methods, 
such as maximum mean discrepancy network~\cite{li2020deep},
and more advanced re-weighting procedures~\cite{zhang2018importance, guo2022selective, cao2022big, gu2024adversarial} have been studied.
To eliminate the need for re-weighting procedures,
some PDM methods,
\ie, mini-batch OT methods~\cite{fatras2021unbalanced, nguyen2022improving},
have been investigated.

The proposed PWAN is conceptually close to the existing mini-batch OT methods~\cite{fatras2021unbalanced, nguyen2022improving} because they are also based on the OT theory.
However,
in contrast to these methods,
PWAN does not suffer from mini-batch errors~\cite{nguyen2022improving} as we will show in Sec.~\ref{Sec_algorithm}, 
and it explicitly uses the clean source domain assumption,
\ie, the source data is fully utilized.
On the other hand,
PWAN is a direct generalization of the WGAN-based closed-set adaptation methods~\cite{shen2018wasserstein,ganin2016domain}.
This connection will be discussed in detail in Sec.~\ref{Sec_app_PDA_alg}.
Note that apart from PDM and DM,
other formulations have been used for adaptation tasks,
such as contrastive learning~\cite{he2023addressing}, 
reinforcement learning~\cite{wu2022reinforced},
invariant feature learning~\cite{sahoo2023select},
and clustering~\cite{yang2021exploiting}.

\section{Preliminaries on PW Divergences}
\label{Sec_Preliminaries}
This section introduces the major tools used in this work,
\ie,
two types of PW divergences:
the partial-mass Wasserstein-1 discrepancy~\cite{figalli2010the, caffarelli2010free} and the bounded-distance Wasserstein-1 discrepancy~\cite{lellmann2014imaging, bogachev2007measure}.
We refer readers to~\cite{villani2016optimal} for a more complete introduction to the OT theory.

Let $\alpha, \beta \in \mathcal{M}_+(\Omega)$ be the reference and the source distribution of mass,
where $\Omega$ is a compact metric space, 
and $\mathcal{M}_+(\Omega)$ is the set of non-negative measures defined on $\Omega$.
Denote $m_\beta=\beta(\Omega)$ and $m_\alpha=\alpha(\Omega)$ the total mass of $\beta$ and $\alpha$ respectively.
The partial-mass Wasserstein-1 discrepancy $\mathcal{L}_{M, m}(\alpha, \beta)$ is defined as 
\begin{equation}
        \label{POT_primal}
        \mathcal{L}_{M, m}(\alpha, \beta)=\inf_{\pi \in \Gamma_{m}(\alpha, \beta)} \int_{\Omega \times \Omega} \vd(x,y) d\pi(x,y),
\end{equation}
where $\Gamma_{m}(\alpha, \beta)$ is the set of $\pi \in \mathcal{M}_+(\Omega \times \Omega)$ satisfying
\begin{equation}
    \pi(A \times \Omega) \leq \alpha(A), \quad \pi( \Omega \times A) \leq \beta(A) \quad and \quad \pi( \Omega \times \Omega) \geq m  \nonumber
\end{equation}
for all measurable set $A \subseteq \Omega$ and $\vd$ is the metric in $\Omega$.

In the complete OT case,
\ie, $m=m_\beta=m_\alpha$,
$\mathcal{L}_{M, m}$ is the Wasserstein-$1$ metric $\mathcal{W}_{1}$,
which is also called the earth mover's distance.
According to the Kantorovich-Rubinstein (KR) duality~\cite{kantorovich2006on},
$\mathcal{W}_{1}$ can be equivalently expressed as
\begin{equation}
    \label{KR-dual}
    \mathcal{W}_{1}(\alpha, \beta)=\sup_{\vf \in Lip(\Omega) } \int_{\Omega} \vf d\alpha - \int_{\Omega} \vf d\beta.
\end{equation}
where $Lip(\Omega)$ represents the set of Lipschitz-1 functions defined on $\Omega$.
\eqref{KR-dual} is exploited in WGAN~\cite{pmlr-v70-arjovsky17a} to efficiently compute $\mathcal{W}_{1}$,
where $\vf$ is approximated by a neural network.

Several methods have been proposed to generalize~\eqref{KR-dual} to unbalanced OT~\cite{chizat2018scaling, schmitzer2019stabilized}.
A unified framework of these approaches can be found in~\cite{schmitzer2019framework}.
Amongst these generalizations,
we considered the bounded-distance Wasserstein-1 discrepancy $\mathcal{L}_{D,h}$,
which can be regarded as~\eqref{POT_primal} with a soft mass constraint:
\begin{equation}
        \label{POT_primal_lag}
        \mathcal{L}_{D,h}(\alpha, \beta)=\inf_{ \pi \in \Gamma_0(\alpha, \beta) } \int_{\Omega \times \Omega} \vd(x,y) d\pi(x,y) - h m(\pi).
\end{equation}
Note that $\mathcal{L}_{D,h}$ can be interpreted as finding an optimal plan $\pi$ with a \emph{distance threshold $h$}.
Importantly,
$\mathcal{L}_{D,h}$ is known to have an equivalent form~\cite{lellmann2014imaging,schmitzer2019framework}
\begin{equation}
    \label{KR}
    \mathcal{L}_{D,h}(\alpha, \beta)=\sup_{ \substack{\vf \in Lip(\Omega) \\ -h \leq \vf \leq 0} }\int_{\Omega} \vf d\alpha - \int_\Omega \vf d\beta - h m_{\beta}.
\end{equation}
We call~\eqref{KR-dual} and~\eqref{KR} \emph{KR form}s,
and the solution to KR forms \emph{potential}s.
Note that we use letters $M$ and $D$ in our notation $\mathcal{L}_{M,m}$ and $\mathcal{L}_{D,h}$ to indicate that they are \textbf{M}ass and \textbf{D}istance type PW divergences.

\section{The Proposed Method}
\label{Sec_proposed_method}
We present the details of the proposed PWAN in this section.
We first formulate the PDM problem in Sec.~\ref{Sec_problem_formulation}.
Then we present an efficient approach for the problem in Sec.~\ref{Sec_Opt_PW} and~\ref{Sec_algorithm}.
We finally discuss the connections between PWAN and the well-known WGAN~\cite{pmlr-v70-arjovsky17a} in Sec.~\ref{Sec_connection_WGAN} to provide a deeper understanding of PWAN.

\subsection{Problem Formulation}
\label{Sec_problem_formulation}
Let $\alpha \in \mathcal{M}_+(\Omega)$, 
$\beta \in \mathcal{M}_+(\Omega')$ be the reference and the source distribution of mass,
where $\Omega'$ and $\Omega$ are two compact subsets of Euclidean spaces.
Let $\mathcal{T}_\theta: \Omega' \rightarrow \Omega$ denote a differentiable transformation parametrized by $\theta$,
and $\beta_\theta = (\mathcal{T}_{\theta})_{\#}\beta \in \mathcal{M}_+(\Omega)$ denote the push-forward of $\beta$ by $\mathcal{T}_\theta$.
Our goal is to \emph{partially} match $\beta_\theta$ to $\alpha$ by solving
\begin{equation}
    \label{Objctive}
    \min_{\theta}\mathcal{L}(\alpha, \beta_\theta) + \mathcal{C}(\theta),
\end{equation}
where $\mathcal{L}$ represents $\mathcal{L}_{M,m}$ or $\mathcal{L}_{D,h}$,
which measures the dissimilarity between $\beta_\theta$ and $\alpha$,
and $\mathcal{C}$ is a differentiable regularizer that reduces the ambiguity of solutions.

\subsection{Partial Matching via PW Divergences}
\label{Sec_Opt_PW}

To see the effectiveness of framework~\eqref{Objctive},
we present a toy example of matching discrete distributions in Fig.~\ref{C}.
Let $Y=\{y_j\}_{j=1}^r$ and $X=\{x_i\}_{i=1}^q$ be the source and reference 2-D point sets,
$\beta=\sum_{y_j \in Y} \delta_{y_j}$ and $\alpha=\sum_{x_i \in X} \delta_{x_i}$ be the corresponding discrete distributions,
where $\delta$ is the Dirac function.
We seek to find a transformation $\mathcal{T}_\theta$ that partially matches $Y$ to $X$ by solving~\eqref{Objctive}.

By expressing $\mathcal{L}_{M,m}$~\eqref{POT_primal} and $\mathcal{L}_{D,h}$~\eqref{POT_primal_lag} in their primal forms,
we can write \eqref{Objctive} as
\begin{equation}
    \label{our_primal}
    \min_{\theta} \sum_{i,j} \pi_{ij} \vd(x_i, \mathcal{T}_\theta(y_j)) + \mathcal{C}(\theta) + const,
\end{equation}
where $const$ is a constant that is not related to $\theta$, 
and $\pi \in \mathbb{R}^{q \times r}$ is the solution to~\eqref{POT_primal} or~\eqref{POT_primal_lag} obtained via linear programming.
We represent the non-zero elements in $\pi$ by line segments in Fig.~\ref{C}.
As can be seen,
the value of~\eqref{our_primal} only depends on the distances between the corresponding point pairs that are within the mass or distance threshold,
thus minimizing~\eqref{our_primal} will only align these point pairs subjecting to the regularizer $\mathcal{C}$,
while omitting all other points, 
\ie, the hollow points.
In other words,
$\mathcal{L}_{M,m}$ and $\mathcal{L}_{D,h}$ provide two ways to control the degree of alignment of distributions based on the mass or distance criteria.
\begin{figure}[htb!]
  \centering
  \subfigure[$\mathcal{L}_{M,6}$]{
    \includegraphics[width=0.45\linewidth]{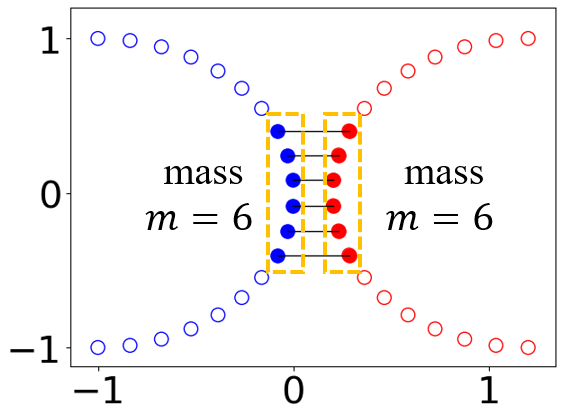}
  }
  \subfigure[$\mathcal{L}_{D,0.9}$]{
    \includegraphics[width=0.45\linewidth]{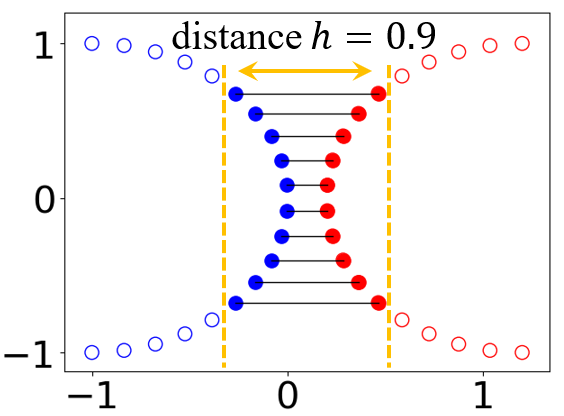}
  }
\caption{The computed correspondence $\pi$ between $\alpha$ (blue) and $\beta_\theta$ (red).
}
\label{C}
\end{figure}

However,
although the solution to~\eqref{our_primal} indeed partially matches the distributions,
it is intractable for large-scale discrete distributions or continuous distributions,
due to the high computational cost of solving for $\pi$ using linear programming in each iteration.

To address this challenge,
we propose to avoid computing $\pi$ by optimizing $\mathcal{L}_{D,h}$ or $\mathcal{L}_{M,m}$ in their KR forms as an alternative to their primal forms.
To this end,
we first present the KR forms for $\mathcal{L}_{D,h}$ and $\mathcal{L}_{M,m}$,
and then show that they are differentiable.
As for $\mathcal{L}_{D,h}$,
the KR form is known in~\eqref{KR},
and we further show that it is valid to compute its gradient.
We also show that a similar statement for $\mathcal{L}_{M,m}$ holds.
Specifically,
we have the following two theorems:
\begin{theorem}
  \label{Theorem1}
  $\mathcal{L}_{D,h}(\alpha, \beta)$ can be equivalently expressed as~\eqref{KR}.
  There is a solution $\vf^*:\Omega \rightarrow \mathbb{R}$ to problem~\eqref{KR}.
  If $\mathcal{T}_\theta$ satisfies a mild assumption,
  then $\mathcal{L}_{D,h}(\alpha, \beta_\theta)$ is continuous \wrt $\theta$, and is differentiable almost everywhere.
  Furthermore,
  we have
  \begin{equation}
      \label{grad_dis}
      \nabla_\theta \mathcal{L}_{D,h}(\alpha, \beta_\theta) = - \int_\Omega \nabla_\theta (\vf^* \circ \mathcal{T}_\theta) d\beta
  \end{equation}
  when both sides are well-defined.
\end{theorem}
\begin{theorem}
    \label{Theorem2}
    $\mathcal{L}_{M,m}(\alpha, \beta)$ can be equivalently expressed as
    \begin{equation}
        \label{PW-m}
        \mathcal{L}_{M, m}(\alpha, \beta)=\sup_{\substack{\vf \in Lip(\Omega), h \in \mathbb{R}_+ \\ -h \leq \vf \leq 0 }}  \int_\Omega \vf d(\alpha -\beta) +h(m-m_{\beta}).
    \end{equation}
    There is a solution $\vf^*:\Omega \rightarrow \mathbb{R}$ to problem~\eqref{PW-m}.
    If $\mathcal{T}_\theta$ satisfies a mild assumption,
    then $\mathcal{L}_{M,m}(\alpha, \beta_\theta)$ is continuous \wrt $\theta$, and is differentiable almost everywhere.
    Furthermore,
    we have
    \begin{equation}
        \label{grad_mass}
        \nabla_\theta \mathcal{L}_{M,m}(\alpha, \beta_\theta) = - \int_\Omega \nabla_\theta (\vf^* \circ \mathcal{T}_\theta) d\beta
    \end{equation}
    when both sides are well-defined.
\end{theorem}

Theorem~\ref{Theorem1} and~\ref{Theorem2} immediately allow us to approximate $\mathcal{L}_{D,h}$ or $\mathcal{L}_{M,m}$ using neural networks.
Specifically,
let $\vf_{w,h}$ be a neural network satisfying $-h \leq \vf_{w,h} \leq 0$ and $\|\nabla \vf_{w,h} \| \leq 1$,
where $w$ and $h \in \mathbb{R}_{+}$ are parameters of the network $\vf_{w,h}$,
and $\|\nabla \vf_{w,h} \|$ is the gradient of $\vf$ \wrt the input.
We can compute $\mathcal{L}_{M,m}(\alpha, \beta_\theta)$ and $\mathcal{L}_{D,h}(\alpha, \beta_\theta)$ according to~\eqref{KR} and~\eqref{PW-m}:
    \begin{align}
      \label{loss_m}
      \mathcal{L}_{M,m}(\alpha, \beta_\theta)= \max_{w,h} \int_{\Omega} \vf_{w,h}d\alpha  -  & \int_{\Omega}  (\vf_{w,h} \circ \mathcal{T}_{\theta}) d\beta \nonumber \\
      & + h(m-m_{\beta}), 
    \end{align}
and
\begin{equation}
    \label{loss_d}
    \mathcal{L}_{D,h}(\alpha, \beta_\theta) = \max_{w} \int_{\Omega} \vf_{w,h}d\alpha  -  \int_{\Omega}  (\vf_{w,h} \circ \mathcal{T}_{\theta}) d\beta - hm_\beta
\end{equation}
using gradient descent and back-propagation.

Using the approximated KR forms of $\mathcal{L}_{M,m}$~\eqref{loss_m} and $\mathcal{L}_{M,m}$~\eqref{loss_d},
we can finally rewrite~\eqref{Objctive} as
\begin{equation}
  \label{our_kr}
  \min_\theta  \bigl( - \vf_{w,h}^*(\mathcal{T}_{\theta}(y_j)) + \mathcal{C}(\theta) + const \bigr),
\end{equation}
where $\vf_{w,h}^*$ is the solution to $\mathcal{L}_{M,m}$~\eqref{loss_m} or $\mathcal{L}_{D,h}$~\eqref{loss_d}.
To see the connection between this KR-based formulation~\eqref{our_kr} and the primal formulation~\eqref{our_primal},
we present an example in Fig.~\ref{app_Fish_complete}.
In contrast to~\eqref{our_primal},
which matches a fraction of $\beta_\theta$ to the corresponding points in $\alpha$ (1-st row),
\eqref{our_kr} is correspondence-free, 
and it seeks to move $\beta_\theta$ along the direction of $\nabla \vf_{w,h}^*$.
Note that the points with zero gradients (hollow points in the 3-rd row) are omitted by~\eqref{our_kr},
and the same group of points (hollow points in the 1-st row) are also omitted by~\eqref{our_primal},
which indicates the consistency of these two formulations.
Further numerical results in Appx. B.1 confirm that these two forms are indeed consistent.
More discussions of the property of the potential are presented in Sec.~\ref{Sec_connection_WGAN} and Appx. A.2.

\begin{figure*}[htb!]
	\centering
    \begin{minipage}[b]{0.02\linewidth}
        \small 
        $\pi$ \\ \vspace{24mm}
        $\vf^*$ \\ \vspace{22mm}
        $||\nabla \vf^*||$ \\ \vspace{14mm}
    \end{minipage}
    \hspace{4.8mm}
    \begin{minipage}[b]{0.89\linewidth}
        \subfigure[$\mathcal{L}_{M,25}$ or $\mathcal{L}_{D,0.64}$]{
            \begin{minipage}[b]{0.24\linewidth}
                \includegraphics[width=1\linewidth]{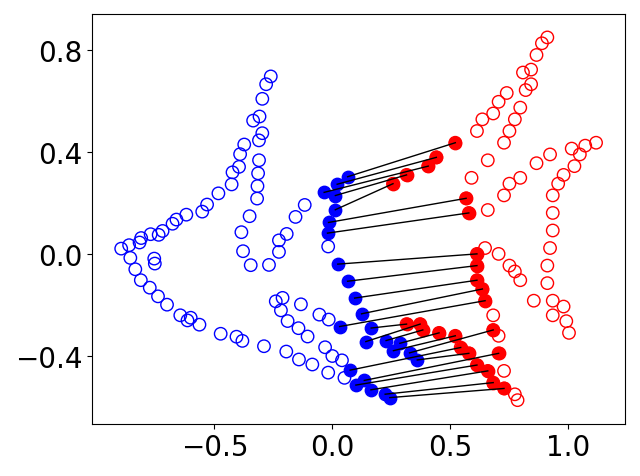} \\\vspace{-2ex}
                \includegraphics[width=1\linewidth]{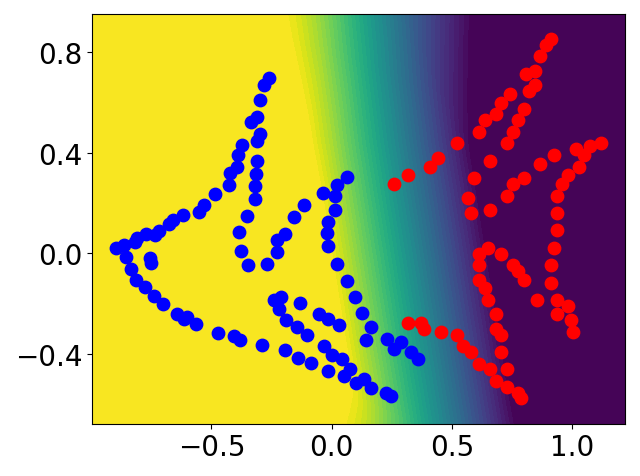} \\\vspace{-2ex}
                \includegraphics[width=1\linewidth]{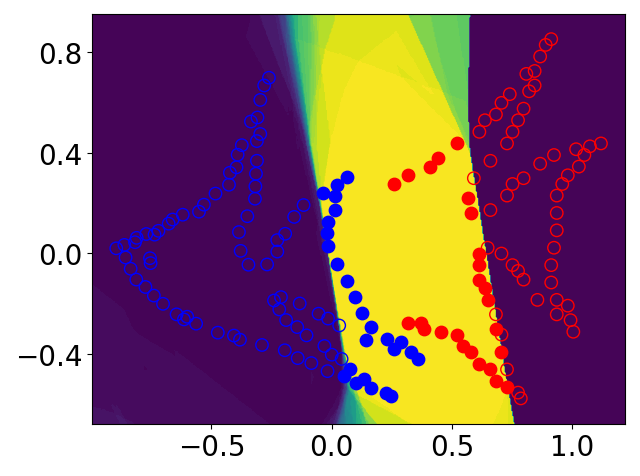} \\\vspace{-2ex}
            \end{minipage}
            \label{app_Fish_complete_1}
        }
        \hspace{-4.3mm}
        \subfigure[$\mathcal{L}_{M,50}$ or $\mathcal{L}_{D,1.09}$]{
            \begin{minipage}[b]{0.24\linewidth}
                \includegraphics[width=1\linewidth]{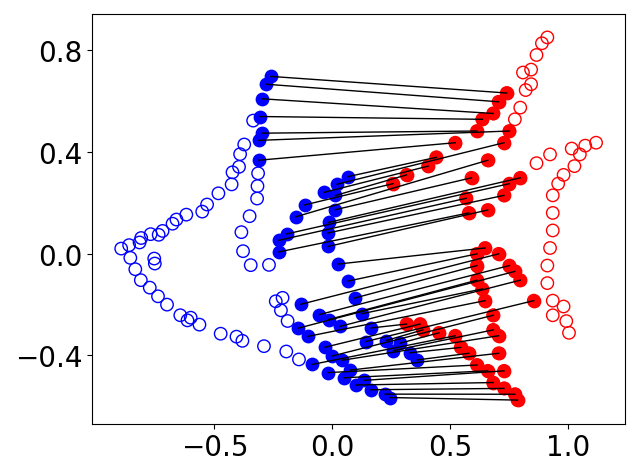} \\\vspace{-2ex}
                \includegraphics[width=1\linewidth]{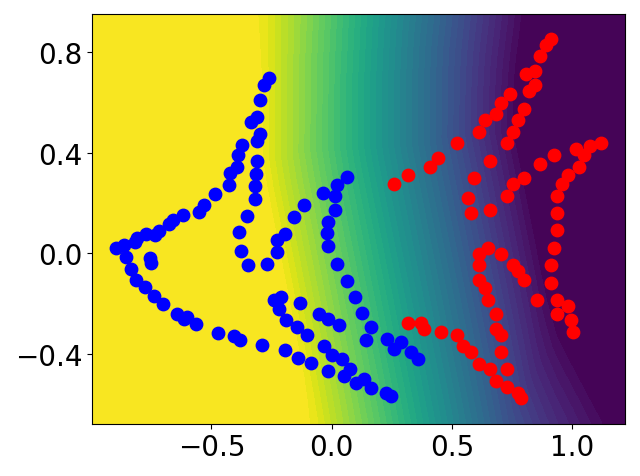} \\\vspace{-2ex}
                \includegraphics[width=1\linewidth]{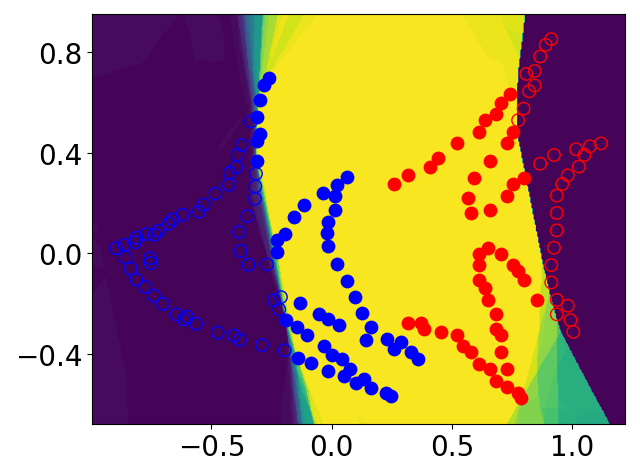} \\\vspace{-2ex}
            \end{minipage}
            \label{app_Fish_complete_2}
        }
        \hspace{-4.3mm}
        \subfigure[$\mathcal{L}_{M,78}$ or $\mathcal{L}_{D,5}$]{
            \begin{minipage}[b]{0.24\linewidth}
                \includegraphics[width=1\linewidth]{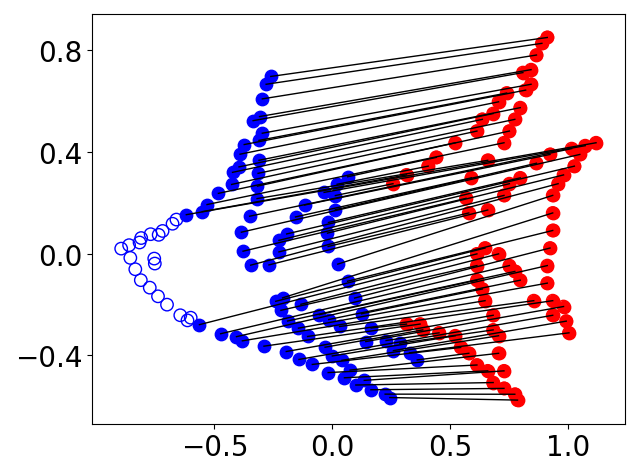} \\\vspace{-2ex}
                \includegraphics[width=1\linewidth]{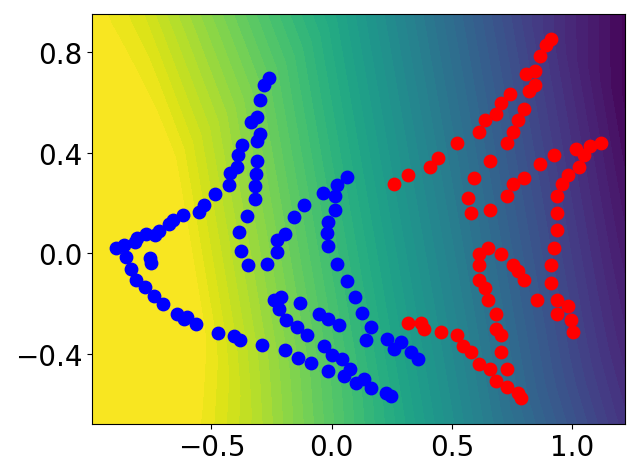} \\\vspace{-2ex}
                \includegraphics[width=1\linewidth]{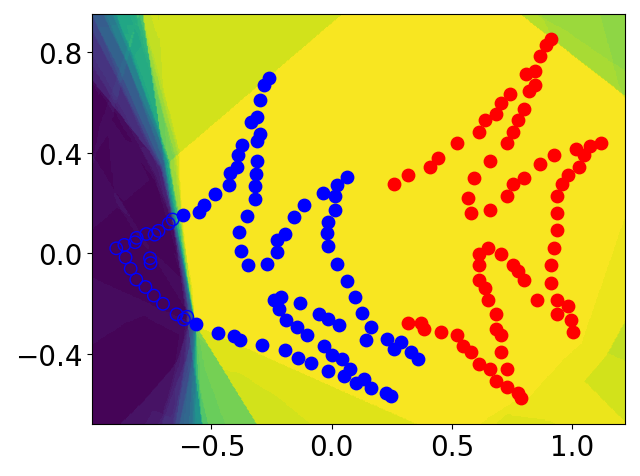} \\\vspace{-2ex}
            \end{minipage}
            \label{app_Fish_complete_3}
        }
        \hspace{-4.3mm}
        \subfigure[$\mathcal{W}_{1}$, $\mathcal{L}_{M,98}$ or $\mathcal{L}_{D,100}$]{
            \begin{minipage}[b]{0.24\linewidth}
                \includegraphics[width=1\linewidth]{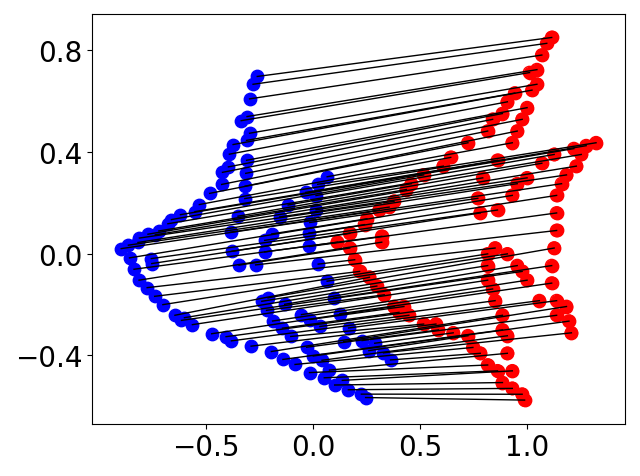} \\\vspace{-2ex}
                \includegraphics[width=1\linewidth]{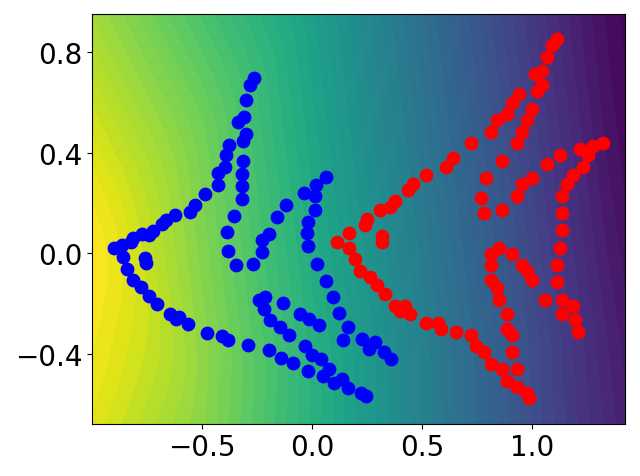} \\\vspace{-2ex}
                \includegraphics[width=1\linewidth]{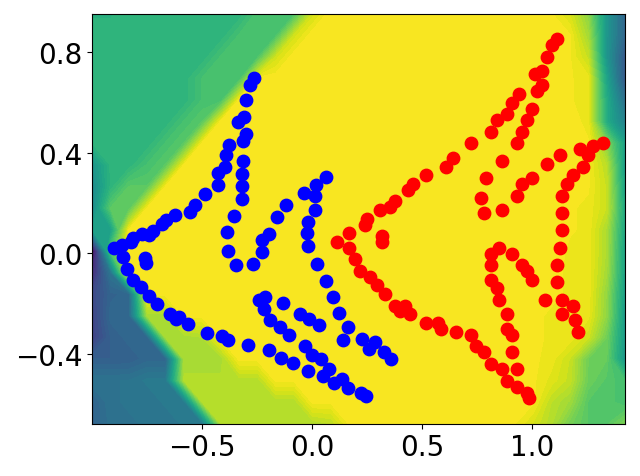} \\\vspace{-2ex}
            \end{minipage}
            \label{app_Fish_complete_4}
        }
    \end{minipage}
    \hspace{-8.4mm}
    \begin{minipage}[b]{0.051\linewidth}
        \includegraphics[width=1\linewidth]{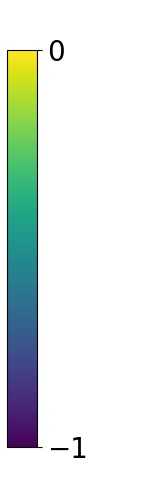} \\\vspace{-3.4ex}
        \includegraphics[width=1\linewidth]{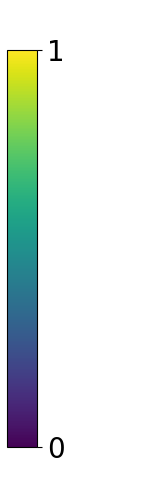} \\ \vspace{-2.0ex}
    \end{minipage}
\vspace{-2mm}
  \caption{
    Comparison of the primal form and our approximated KR form on discrete distributions $\alpha$ (blue) and $\beta_\theta$ (red).
    The solutions to these two forms are presented in the 1-st and 2-nd row respectively,
    and the gradients of the potentials are presented in the 3-rd row.
  }
	\label{app_Fish_complete}
\end{figure*}

\subsection{Algorithm}
\label{Sec_algorithm}
We now develop a concrete PDM algorithm based on approximated KR forms discussed in Sec.~\ref{Sec_Opt_PW}.
We use the following three techniques to facilitate efficient computations.
First,
we need to ensure that the neural network $\vf_{w,h}$ is a valid potential function,
\ie, bounded and Lipschitz-1.
We ensure the boundedness by defining $\vf_{w,h}(x)=\textit{clip}(-|\vf_w(x)|, -h, 0 )$,
where $\vf_w$ is a multi-layer perception (MLP) with learnable parameter $w$,
and $\textit{clip}(\cdot)$ is the clip function,
\ie, we take the negative absolute value of the output of $\vf_w(x)$,
and then clip it to the interval $[-h, 0]$.
Note that here we use $\textit{clip}(\cdot)$ instead of the more commonly used bounded activation functions like $\textit{sigmoid}$ or $\textit{tanh}$,
because it allows $\vf_{w,h}$ to attain the exact maximum and minimum values $0$ and $-h$ (Prop.~6 and Cor.~4 in the appendix),
\ie, the brightest and darkest colors in the 2-nd row in Fig.~\ref{app_Fish_complete}.
To ensure the Lipschitzness,
we add a gradient penalty~\cite{gulrajani2017improved} $\textit{GP}(\vf) = \max_{\hat{x} \in \Omega}\{||\nabla_{\hat{x}}\vf(\hat{x})||^2, 1\}$ to the training loss defined below.

Second,
to handle continuous or large-scale distributions,
we learn $\vf_{w,h}$ in a mini-batch manner.
Specifically,
in each iteration,
we sample $\tilde{q}$ and $\tilde{r}$ i.i.d. samples $\{x_i\}_{i=1}^{\tilde{q}}$ and $\{y_i\}_{i=1}^{\tilde{r}}$ from the normalized distributions $\frac{1}{m_\alpha} \alpha$ and $\frac{1}{m_\beta} \beta$ respectively,
and approximate $\alpha$ and $\beta$ by their empirical distributions: 
$\tilde{\alpha}=\frac{m_\alpha}{\tilde{q}}\sum_{i=1}^{\tilde{q}} \delta_{x_i}$ and  $\tilde{\beta}=\frac{m_\beta}{\tilde{r}}\sum_{i=1}^{\tilde{r}} \delta_{y_i}$.

Finally,
instead of solving for the optimal $\vf^*_{w,h}$,
we only update $\vf_{w,h}$ for a few steps in each iteration to reduce the computation cost.
In summary,
by combining all three techniques,
we optimize
\begin{equation}
    \label{adv_object}
    \min_\theta \max_{\tilde{w}} L(\alpha, \beta_\theta; \tilde{w}) + \mathcal{C}(\theta)
\end{equation}
by alternatively updating $\tilde{w}$ and $\theta$ using gradient descent,
where 
\begin{align}
  \label{loss_m_app}
  L = \frac{m_\alpha}{\tilde{q}}  \sum_{i=1}^{\tilde{q}} & \vf_{w,h}(x_i)  -  \frac{m_\beta}{\tilde{r}} \sum_{j=1}^{\tilde{r}} \vf_{w,h}(\mathcal{T}_{\theta}(y_j))  \nonumber \\
  & +  h(m-m_{\beta}) - GP(\vf_{w,h})
\end{align}
and $\tilde{w} = \{w, h\}$ for $\mathcal{L}_{M,m}$;
\begin{align}
  \label{loss_d_app}
  L = \frac{m_\alpha}{\tilde{q}} \sum_{i=1}^{\tilde{q}} \vf_{w,h}(x_i)  -  \frac{m_\beta}{\tilde{r}}    \sum_{j=1}^{\tilde{r}} &  \vf_{w,h}(\mathcal{T}_{\theta}(y_j)) \nonumber \\
  &- GP(\vf_{w,h})
\end{align}
and $\tilde{w} = \{w\}$ for $\mathcal{L}_{D,h}$.
The gradient at each iteration can be computed as
\begin{equation}
  \label{gradient_app}
  \nabla_\theta L = - \frac{m_\beta}{\tilde{r}}\sum_{j=1}^{\tilde{r}}  \nabla_\theta \vf_{w,h}(\mathcal{T}_{\theta}(y_i)).
\end{equation}

We call our method partial Wasserstein adversarial network (PWAN),
and call $\vf_{w,h}$ potential network.
The detailed algorithm is presented in Alg.~\ref{our_algorithm}.
For clearness,
we refer to PWAN based on $\mathcal{L}_{M,m}$ and $\mathcal{L}_{D,h}$ as mass-type PWAN (m-PWAN) or distance-type PWAN (d-PWAN) respectively.
These two types of PWAN are generally not equivalent despite their close relation as described in lemma~5 in the appendix.
Specifically,
for each $(\alpha, \beta_\theta, m)$,
there exists an $h^*$ such that the solution to $\mathcal{L}_{D,h^*}(\alpha, \beta_\theta)$ recovers to that of $\mathcal{L}_{M,m}(\alpha, \beta_\theta)$,
but optimizing $\mathcal{L}_{M,m}(\alpha, \beta_\theta)$ is generally not equivalent to optimizing $\mathcal{L}_{D,h}(\alpha, \beta_\theta)$ with any fixed $h$,
because the corresponding $h^*$ depends on $\beta_\theta$,
which varies during the optimization process.
A special case where m-PWAN and d-PWAN are equivalent are presented in corollary~\ref{corollary-P}.

\begin{figure}[t!]
  \centering
  \includegraphics[width=0.4\linewidth]{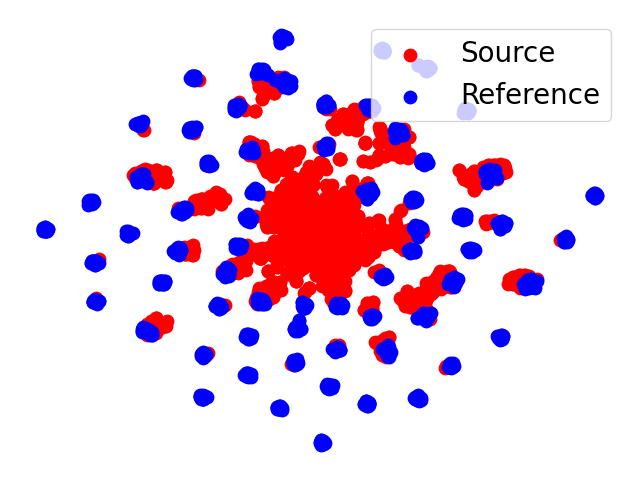}
  \includegraphics[width=0.45\linewidth]{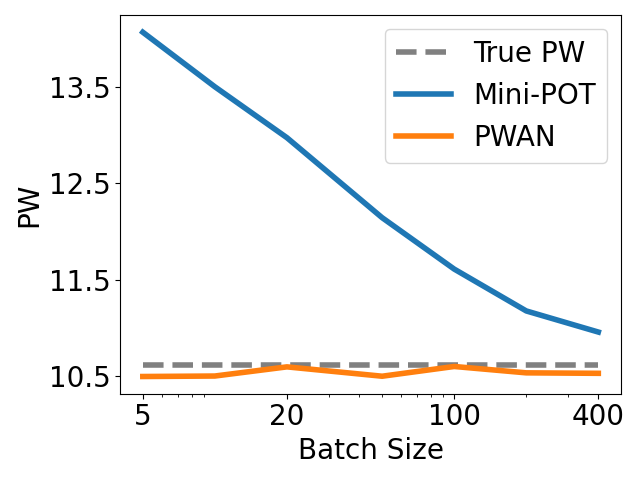}
\vspace{-2mm}
\caption{Comparison between PWAN and mini-POT under various batch sizes.
$\alpha$ and $\beta$ are visualized on the left panel.
}
\vspace{-4mm}
\label{Algo-comparison}
\end{figure}

We finally remark that one important property of PWAN is that unlike the popular mini-batch OT method mini-POT~\cite{fatras2021unbalanced},
PWAN theoretically does not introduce mini-batch errors,
\ie, it does not require large batch sizes to obtain accurate results.
This is due to the fact that $\vf_{w,h}$ is shared across all samples instead of being computed independently for each batch.
We numerically illustrate this property in Fig.~\ref{Algo-comparison},
where we compute $\mathcal{L}_{M,m}$ using PWAN and mini-POT with various batch sizes.
We also compute the true PW discrepancy value using linear programming on the whole dataset for comparison.
As can be seen,
PWAN produces consistent values under different batch sizes despite some randomness brought by the stochastic gradient descent process, 
and the relative error remains low (less than $1\%$).
In contrast,
mini-POT suffers from large relative errors (about $30\%$) when the batch size is small,
and very large batch sizes (at least $400$) are required to achieve comparable accuracy with PWAN.
The result demonstrates the advantage of mini-batch PWAN especially when the batch size is small.
Further comparison with the improved mini-POT~\cite{nguyen2022improving} and more details of this example can be found in Appx.~B.2.

\begin{algorithm}[th]
	\caption{PDM using PWAN}
  \label{our_algorithm}
	\begin{algorithmic}
    \Require reference distribution $\alpha$, source distribution $\beta$, transform $\mathcal{T}_\theta$, network $\vf_{w,h}$, network update frequency $u$,  
    training type (``\textit{mass}'' or ``\textit{distance}''), mass threshold $m$ or distance threshold $h$, 
    training step $T$, regularizer $\mathcal{C}_\theta$,
    batch size $\tilde{q}$ and $\tilde{r}$ 
	\Ensure learned $\theta$
    \If {training type = ``\textit{mass}''}
      \State  $\tilde{w} \leftarrow  (w, h)$; \quad $L \leftarrow$~\eqref{loss_m_app}
    \ElsIf {training type = ``\textit{distance}''}
      \State  $\tilde{w} \leftarrow (w)$; \quad $L \leftarrow$~\eqref{loss_d_app}
    \EndIf
	\For {$t= 1, \ldots, T$}
        \State  Obtain a mini-batch of $\tilde{q}$ samples $\{x_i\}_{i=1}^{\tilde{q}}$ from $\frac{1}{m_\alpha} \alpha$.
        \State  Obtain a mini-batch of $\tilde{r}$ samples $\{y_i\}_{i=1}^{\tilde{r}}$ from $\frac{1}{m_\beta} \beta$.
        \For {$= 1, \ldots, u$} 
            \State  Update $\tilde{w}$ by ascending the gradient $ \partial L / \partial \tilde{w} $. 
        \EndFor
        \State Compute the gradient $ \partial L / \partial \theta $ according to~\eqref{gradient_app}. 
        \State Compute the gradient $\partial \mathcal{C}(\theta) /\partial \theta$.
        \State Update $\theta$ by descending the gradient $ \partial L / \partial  \theta + \partial \mathcal{C} / \partial \theta $.
    \EndFor
	\end{algorithmic}
\end{algorithm}

\subsection{Connections with WGAN}
\label{Sec_connection_WGAN}
PWAN includes the well-known WGAN as a special case,
because the objective function of WGAN, \ie, $\mathcal{W}_1$ is a special case of that of PWAN, \ie, $\mathcal{L}_{M,m}$ and $\mathcal{L}_{D,h}$,
and they both approximate the KR forms using neural networks.
Specifically,
we have the following corollary. 
\begin{corollary}
  \label{corollary_WGAN}
  When $m_\alpha=m_\beta=m$, 
  $\mathcal{L}_{M,m}$ can be equivalently written as $\mathcal{W}_1$~\eqref{KR-dual}.
  In addition, 
  when $m_\alpha=m_\beta$ and $h > diam(\Omega)$,
  $\mathcal{L}_{D,h}$ can be equivalently written as $\mathcal{W}_1 - h m_\beta$.
\end{corollary}
In other words,
both m-PWAN and d-PWAN become WGAN when we seek to match the whole distribution.

However,
despite its close relations with WGAN,
PWAN has a unique property that WGAN does not have:
PWAN can automatically omit a fraction of data in the training process.
This property can be theoretically explained by observing the learned potential function $\vf^*$.
As for PWAN,
$||\nabla \vf^*||=1$ or $||\nabla \vf^*|| =0$ on the input data (Corollary~4 and~3 in the appendix),
\ie, the data points with $0$ gradients will be omitted during training,
because they do not contribute to the update of $\theta$ in~\eqref{gradient_app}.
While for WGAN,
$||\nabla \vf^*||=1$ on the input data (Corollary 1 in~\cite{gulrajani2017improved}),
\ie, no data point will be omitted.
This property suggests that PWAN can be used as a drop-in enhancement of WGAN to handle ``dirty'' datasets.
An illustrative example in generative modeling can be found in Appx.~B.3.

Finally,
we note that PWAN has an intuitive adversarial explanation as an analogue of WGAN:
Let $\alpha$ and $\beta_\theta$ be the distribution of real and fake data respectively.
During the training process,
$\vf_{w,h}$ is trying to discriminate $\alpha$ and $\beta_\theta$ by assigning each data point a realness score in range $[-h, 0]$.
The points with the highest score $0$ or the lowest score $-h$ are regarded as the ``realest'' or ``fakest'' points respectively.
Meanwhile, 
$\mathcal{T}_\theta$ is trying to cheat $\vf_{w,h}$ by gradually moving the fake data points to obtain higher scores.
However, 
it does not tend to move the ``fakest'' points,
as their scores cannot be changed easily.
The training process ends when $\mathcal{T}_\theta$ cannot make further improvements.

\section{Applications \uppercase\expandafter{\romannumeral1}: Point Set Registration}
\label{Sec_application_1}
This section applies Alg.~\ref{our_algorithm} to point set registration tasks.
We present the details of the algorithms in Sec.~\ref{Sec_app_PSR},
and the numerical results in Sec.~\ref{Sec_experiments_PS}.

\subsection{Aligning Point Set Using PWAN}
\label{Sec_app_PSR}
Recall the definition for point set registration in Sec.~\ref{Sec_Opt_PW}:
Given $\beta=\sum_{y_j \in Y} \delta_{y_j}$ and $\alpha=\sum_{x_i \in X} \delta_{x_i}$,
where $Y=\{y_j\}_{j=1}^r$ and $X=\{x_i\}_{i=1}^q$ are the source and reference point sets in 3D space.
We seek to find an optimal transformation $\mathcal{T}$ that matches $Y$ to $X$ by solving problem~\eqref{Objctive}.

First of all,
we define the non-rigid transformation $\mathcal{T}_{\theta}$ parametrized by $\theta = (A, t, V)$ as
\begin{equation}
    \label{our_transformation}
    \mathcal{T}_{\theta}(y_j) = y_jA + t + V_j,
\end{equation}
where $y_j \in \mathbb{R}^{1\times3}$ represents the coordinate of point $y_j \in Y$,
$A \in \mathbb{R}^{3\times3}$ is the linear transformation matrix,
$t\in \mathbb{R}^{1\times3}$ is the translation vector,
$V \in \mathbb{R}^{r\times3}$ is the offset vector of all points in $Y$,
and $V_j$ represents the $j$-th row of $V$.

Then we define the coherence regularizer $\mathcal{C}(\theta)$ similar to~\cite{myronenko2006non},
\ie,
we enforce the offset vector $v_j$ to vary smoothly.
Formally,
let $\mathbf{G} \in \mathbb{R}^{r \times r}$ be a kernel matrix,
\eg, the Gaussian kernel $\mathbf{G}_\rho(i,j) = e^{-||y_i-y_j||^2/\rho}$,
and $\sigma \in \mathbb{R}_{+}$.
We define
\begin{equation}
   \label{bending}
    \mathcal{C}(\theta) = \lambda \, Tr(V^T (\sigma \mathcal{I} + \mathbf{G}_\rho)^{-1}V),
\end{equation}
where $\lambda \in \mathbb{R}_+$ is the strength of constraint,
$\mathcal{I}$ is the identity matrix,
and $Tr(\cdot)$ is the trace of a matrix.

Since the matrix inversion $(\sigma \mathcal{I} + \mathbf{G}_\rho)^{-1}$ is computationally expensive,
we approximate it via the Nystr\"{o}m method~\cite{williams2000using},
and obtain the gradient
\begin{equation}
    \label{bending_grad}
   \frac{\partial \mathcal{C}(\theta)}{\partial V} \approx 2 \lambda\sigma^{-1} ( V - \sigma^{-1} \mathbf{Q} (\Lambda^{-1} + \sigma^{-1} \mathbf{Q}^T \mathbf{Q})^{-1}\mathbf{Q}^T V),
\end{equation}
where $\mathbf{G}_\rho \approx \mathbf{Q} \Lambda \mathbf{Q}^T$,
$\mathbf{Q} \in \mathbb{R}^{r\times k}$, 
$\Lambda \in \mathbb{R}^{k\times k}$ is a diagonal matrix,
and $k \ll r$.
Note the computational cost of~\eqref{bending_grad} is only $O(r)$ if we regard $k$ as a constant.
The detailed derivation is presented in Appx.~B.4.
Note that $\frac{\partial \mathcal{C}(\theta)}{\partial A}  =\frac{\partial \mathcal{C}(\theta)}{\partial t} =0$ since $\mathcal{C}(\theta)$ is not related to $A$ and $t$.

Finally,
since it can be shown that $\mathcal{T}_{\theta}$ satisfies the assumption required by Theorem~\ref{Theorem1} and Theorem~\ref{Theorem2} (Proposition~9 in the appendix),
it is valid to use Alg.~\ref{our_algorithm}.
We do not adopt the mini-batch setting,
\ie,
we sample all points in the sets in each iteration: $\tilde{q}=q$ and $\tilde{r}=r$,
and set $m_\alpha=q$ and $m_\beta=r$.
For rigid registration tasks,
\ie,
when $\mathcal{T}_{\theta}$~\eqref{our_transformation} is required to be a rigid transformation,
we simply fix $V=0$ ($\mathcal{C}(\theta)=0$),
and define $A$ to be a rotation matrix parametrized by a quaternion.

Due to the approximation error of neural networks,
the optimization process often has difficulty converging to the exact local minimum.
To address this issue,
when the objective function no longer decreases,
we replace the objective function by
\begin{equation}
  \label{fine_tune}
   \min_{\theta} \sum_{j=1}^r s_{j} ||x_{N(j)} - \mathcal{T}_\theta(y_{j})|| + \mathcal{C}(\theta)
\end{equation}
where $x_{N(j)}$ is the closest point to $\mathcal{T}_\theta(y_{j})$,
and $s_j$ is the threshold for point $\mathcal{T}_\theta(y_{j})$ defined as
\begin{equation}
  s_{j} =\begin{cases}
      1 & if \; ||x_{N(j)} - \mathcal{T}_\theta(y_{j})|| \leq h \\
      0 & else
  \end{cases}
\end{equation}
for $\mathcal{L}_{D,h}$,
and 
\begin{equation}
  s_{j} =\begin{cases}
      1 & if \;  j \in \textit{Topk}(||\{x_{N(i)} - \mathcal{T}_\theta(y_{i})||\}_{i=1}^r, m) \\
      0 & else
  \end{cases}
\end{equation}
for $\mathcal{L}_{M,m}$,
where $\textit{Topk}(\cdot, m)$ presents the smallest $m$ elements in $\cdot$.
We obtain $\mathcal{T}_\theta$ by optimizing~\eqref{fine_tune} until convergence using gradient descent,
which generally only takes a few steps.
Note that this approach is reasonable because when the point sets are sufficiently well aligned,
the closest point pairs are likely to be the corresponding point pairs,
thus~\eqref{fine_tune} can be regarded as an approximation of~\eqref{our_primal}.

\subsection{Experiments on Point Set Registration}
\label{Sec_experiments_PS}
We experimentally evaluate the proposed PWAN on point set registration tasks. 
We first present a toy example to highlight the robustness of the PW discrepancies in Sec.~\ref{Sec_exp_Toy}.
Then we compare PWAN against the state-of-the-art methods on synthesized data in Sec.~\ref{Sec_exp_acc},
and discuss its scalability in Sec.~\ref{Sec_exp_eff}.
We further evaluate PWAN on real datasets in Sec.~\ref{Sec_exp_real}.
Finally,
we present the results of rigid registration in Sec.~\ref{Sec_rigid}.

\subsubsection{Comparison of Several Discrepancies}
\label{Sec_exp_Toy}

\begin{figure}[t!]
  \centering
  \vspace{-2mm}
  \begin{minipage}[b]{1\linewidth}
      \centering
      \subfigure[Experimental setting.]{
          \includegraphics[width=0.8\linewidth]{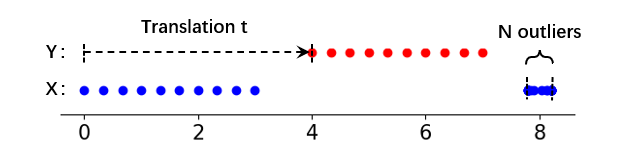}
          \label{sub1}
  }
  \vspace{-2mm}
  \end{minipage}
  \subfigure[$L_2$]{
      \includegraphics[width=0.45\linewidth]{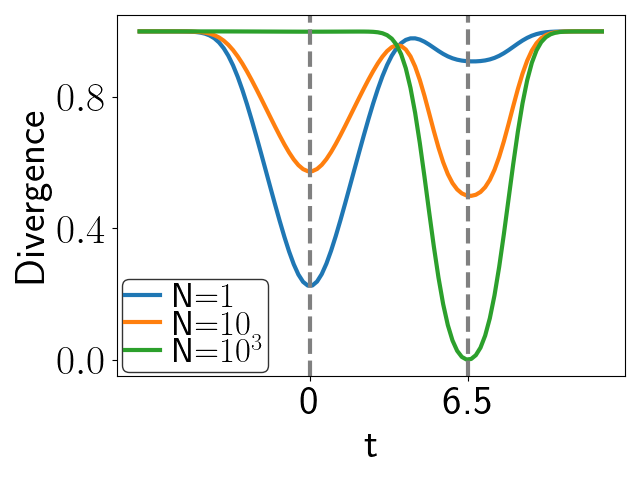}
      \label{sub2}
  }
  \hspace{-2.4mm}
  \subfigure[$KL$]{
      \includegraphics[width=0.45\linewidth]{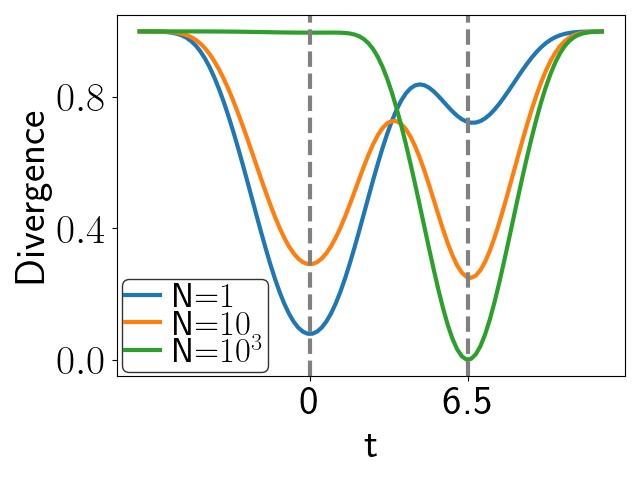}
  }
  \hspace{-2.4mm}
\subfigure[$\mathcal{L}_{M,10}$]{
  \includegraphics[width=0.45\linewidth]{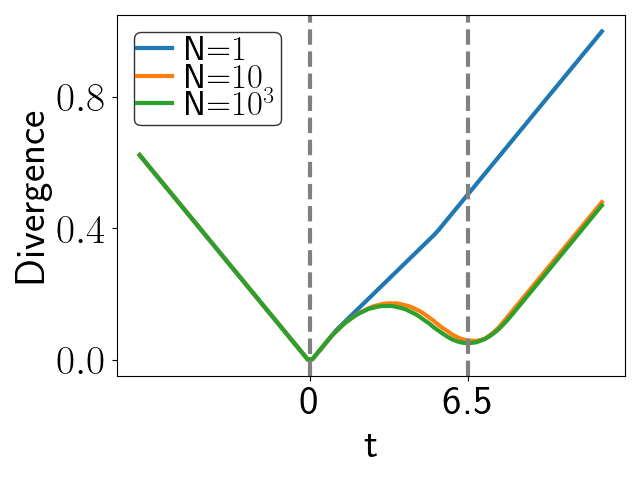}
  }
  \hspace{-2.4mm}
\subfigure[$\mathcal{L}_{D,2}$]{
      \includegraphics[width=0.45\linewidth]{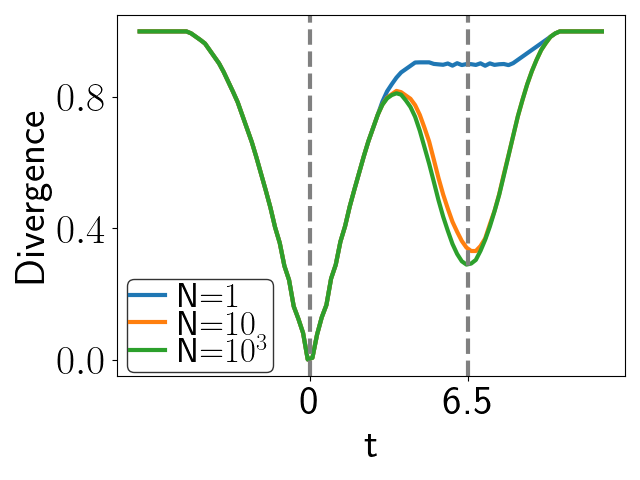}
      \label{sub5}
  }
\vspace{-2.4mm}
\caption{Comparison of different discrepancies on a toy point set.
}
\vspace{-2mm}
\label{metric2}
\end{figure}

To provide an intuition of the PW divergence $\mathcal{L}_{M,m}$ and $\mathcal{L}_{D,h}$ used in this work,
we compare them against two representative robust discrepancies, 
\ie, KL divergence~\cite{hirose2021a} and $L_2$ distance~\cite{jian2011robust},
that are commonly used in registration tasks on a toy 1-dimensional example. 
More discussion can be found in Appx.~B.6.

We construct the toy point sets $X$ and $Y$ shown in Fig.~\ref{sub1},
where we first sample $10$ equi-spaced data points in interval $[0,3]$,
then we define $Y$ by translating the data points by $t$,
and define $X$ by adding $N$ outliers in a narrow interval $[7.8,8.2]$ to the data points.
For $N=\{1, 10, 10^3\}$,
we record four discrepancies: KL, $L_2$, $\mathcal{L}_{M,10}$ and $\mathcal{L}_{D,2}$ between $X$ and $Y$ as a function of $t$,
and present the results from Fig.~\ref{sub2} to Fig.~\ref{sub5}.

By construction,
the optimal (perfect) alignment is achieved at $t=0$.
However,
due to the existence of outliers,
there are two alignment modes, \ie, $t=0$ and $t=6.5$ in all settings,
where $t=6.5$ represents the biased result.
As for KL and $L_2$,
the local minimum $t=0$ gradually vanishes and they gradually bias toward $t=6.5$ as $N$ increases.
In contrast,
$\mathcal{L}_{M,10}$ and $\mathcal{L}_{D,2}$ do not suffer from this issue,
because the local minimum $t=0$ remains deep, \ie, it is the global minimum, regardless of the value of $N$.

The result suggests that compared to KL and $L_2$,
$\mathcal{L}_{M,m}$ and $\mathcal{L}_{D,h}$ exhibit stronger robustness against outliers,
as the optimal alignment can always be achieved by an optimization process, 
such as gradient descent, 
regardless of the ratio of outliers.

\subsubsection{Evaluation on Synthesized Data}
\label{Sec_exp_acc}

This section evaluates the performance of PWAN on synthesized data.
We use three synthesized datasets shown in Fig.~\ref{syn_dataset}: bunny, armadillo and monkey.
The bunny and armadillo datasets are from the Stanford Repository~\cite{Sbunny},
and the monkey dataset is from the internet.
These shapes consist of $8,171$, $106,289$ and $7,958$ points respectively.
We create deformed sets following~\cite{hirose2021acceleration},
and evaluate all registration results via mean square error (MSE).

We compare the performance of PWAN with four state-of-the-art methods:
CPD~\cite{myronenko2006non}, GMM-REG~\cite{jian2011robust}, BCPD~\cite{hirose2021a} and TPS-RPM~\cite{chui2000a}.
We evaluate all methods in the following two settings:
\begin{itemize}[leftmargin=3mm]
  \item[-] Extra noise points: We first sample $500$ random points from the original and the deformed sets as the source and reference sets respectively.
  Then we add uniformly distributed noise to the reference set, 
  and normalize both sets to mean $0$ and variance $1$.
  We vary the number of outliers from $100$ to $600$ at an interval of $100$,
  \ie,
  the outlier/non-outlier ratio varies from $0.2$ to $1.2$ at an interval of $0.2$.
  \item[-] Partial overlap: We first sample $1000$ random points from the original and the deformed sets as the source and reference sets respectively.
  We then intersect each set with a random plane,
  and retain the points on one side of the plane and discard all points on the other side.
  We vary the retaining ratio $s$ from $0.7$ to $1.0$ at an interval of $0.05$ for both the source and the reference sets,
  \ie, the minimal overlap ratio $(2s-1)/s$ varies from $0.57$ to $1$.
\end{itemize}
We evaluate m-PWAN with $m=500$ in the first setting,
while we evaluate m-PWAN with $m=(2s-1)*1000$ and d-PWAN with $h=0.05$ in the second setting.
More detailed experiment settings are presented in Appx.~B.5.
We run all methods $100$ times and report the median and standard deviation of MSE in Fig.~\ref{compare_fig}.

Fig.~\ref{com_sub1} presents the results of the first experiment.
The median error of PWAN is generally comparable with TPS-RPM,
and is much lower than the other methods.
In addition,
the standard deviations of the results of PWAN are much lower than that of TPS-RPM.
Fig.~\ref{com_sub2} presents the results of the second experiment.
Two types of PWANs perform comparably, 
and they outperform the baseline methods by a large margin in terms of both median and standard deviations when the overlap ratio is low,
while all methods perform comparably when the data is fully overlapped.
These results suggest the PWAN are more robust against outliers than baseline methods,
and can effectively handle partial overlaps.
More results are provided in Appx.~B.7.

\begin{figure}[tb]
  \vspace{-1mm}
\centering
    \includegraphics[width=0.3\linewidth]{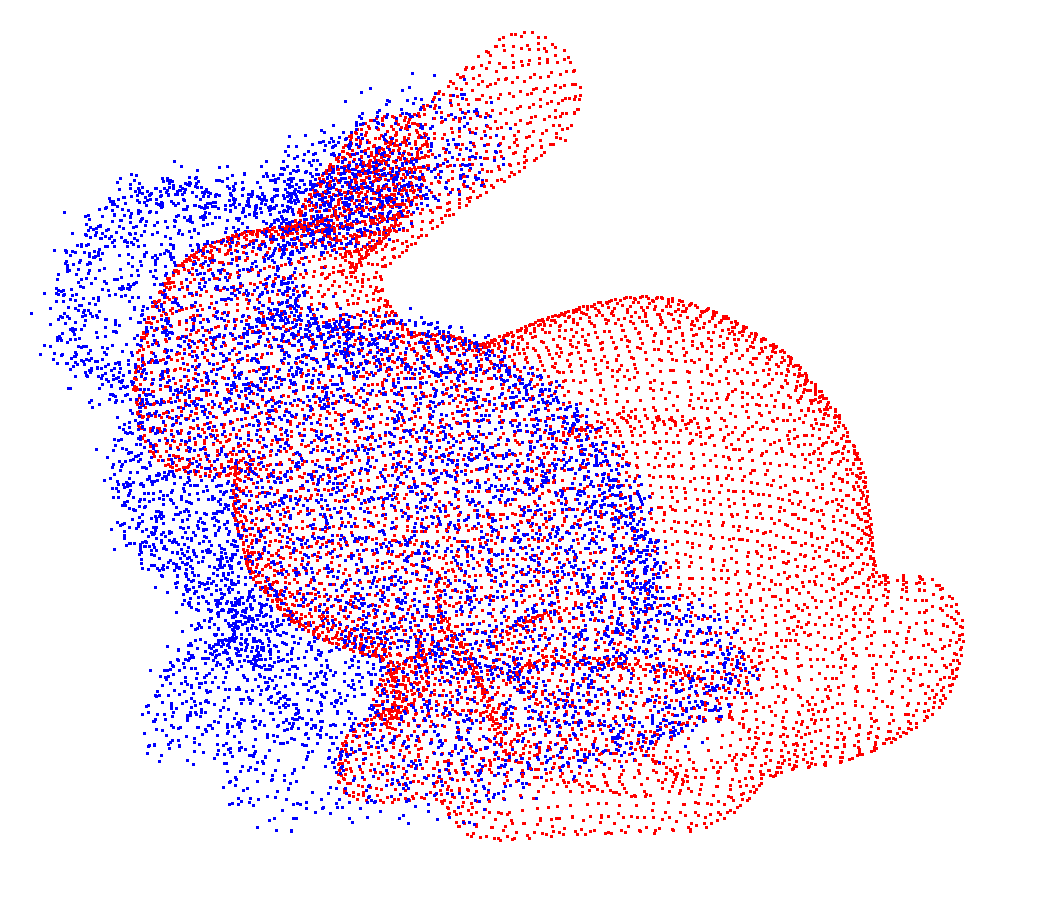}
    \includegraphics[width=0.3\linewidth]{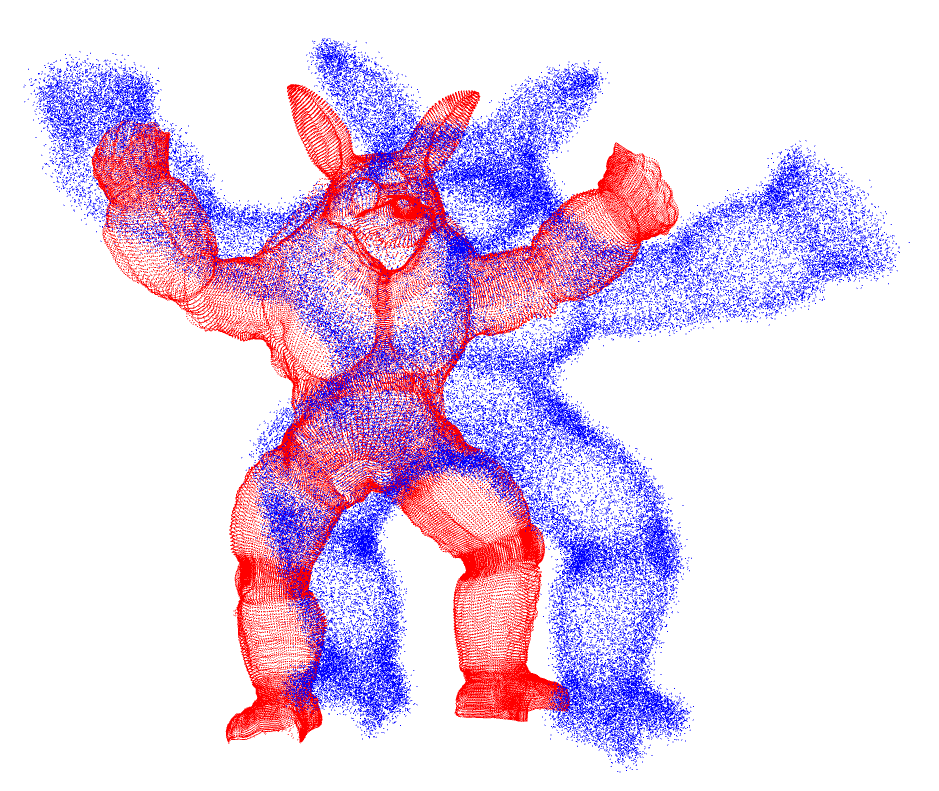}
    \includegraphics[width=0.3\linewidth]{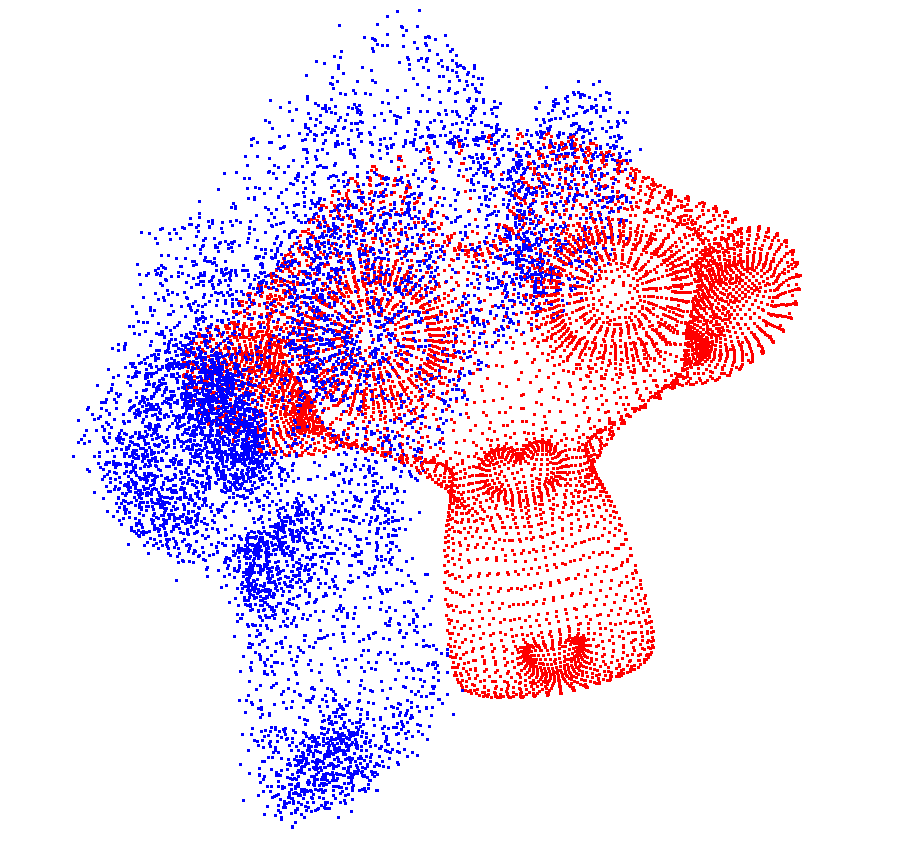}
    \vspace{-3mm}
\caption{The synthesized datasets used in our experiments.
}
\label{syn_dataset}
\vspace{-2mm}
\end{figure}

\begin{figure*}[t!]
	\centering
    \subfigure[Evaluation of robustness against noise points.]{
      \includegraphics[width=0.3\linewidth]{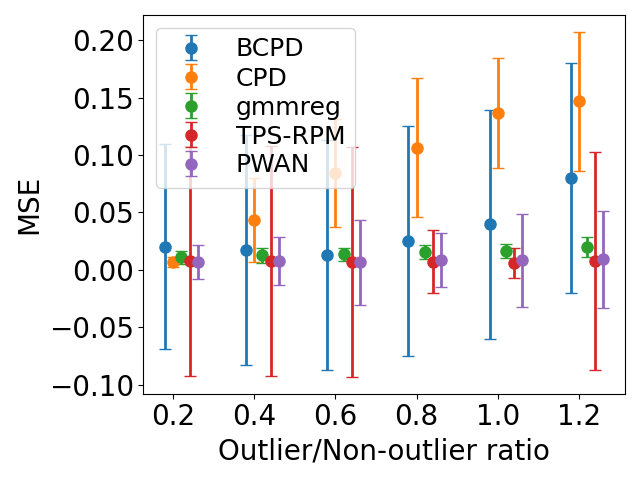}
      \includegraphics[width=0.3\linewidth]{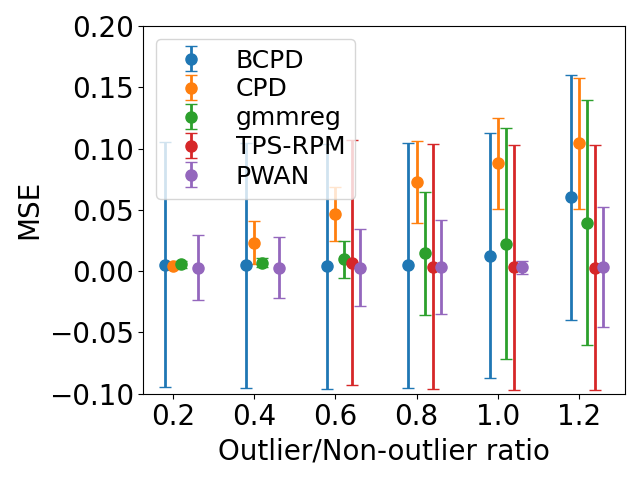}
      \includegraphics[width=0.3\linewidth]{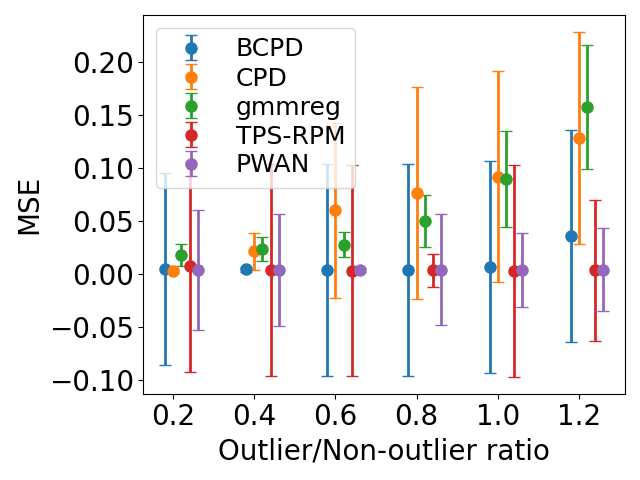}
      \label{com_sub1}
  }
    \vspace{-2mm}
    \subfigure[Evaluation of robustness against partial overlap.]{
        \includegraphics[width=0.3\linewidth]{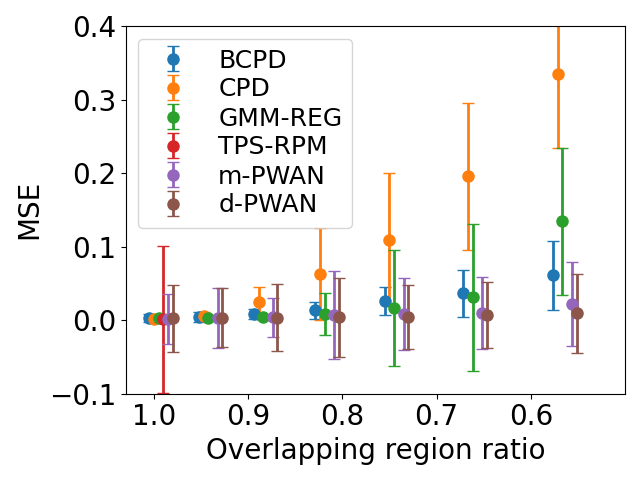}
        \includegraphics[width=0.3\linewidth]{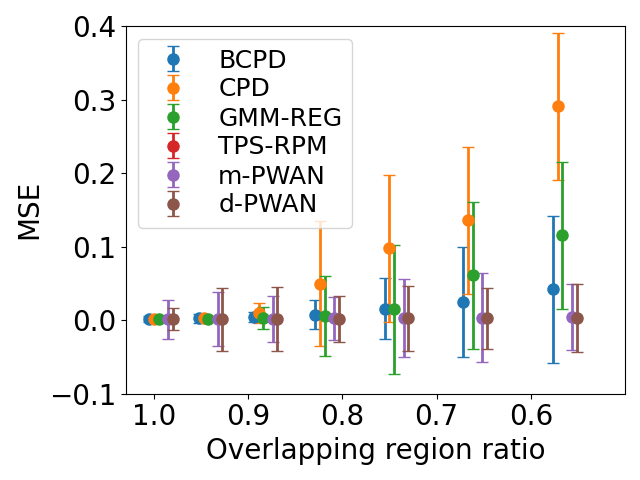}
        \includegraphics[width=0.3\linewidth]{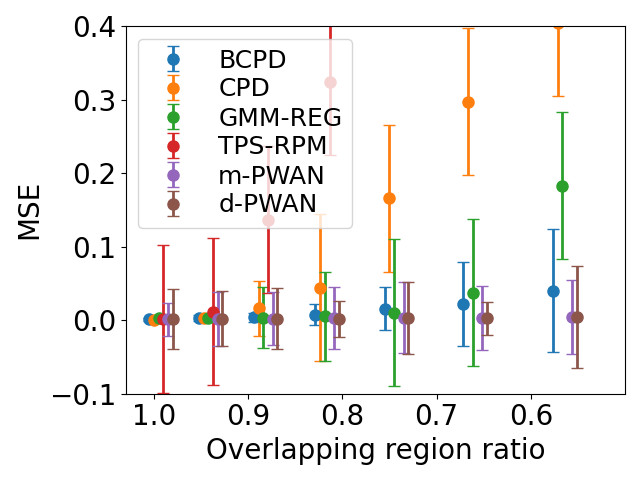}
        \label{com_sub2}
    }
  \caption{Registration accuracy of the bunny (left), monkey (middle) and armadillo (right) datasets.
  The error bars represent the medians and standard deviations of MSE.
  }
	\label{compare_fig}
\end{figure*}

To provide an intuition of the learned potential network of PWAN, 
we visualize the gradient norm $||\nabla \vf_{w,h}||$ on a pair of partially overlapped armadillo point sets at the end of the registration process in Fig.~\ref{vis_outlier_larma}.
As can be seen,
$||\nabla \vf_{w,h}||$ is close to $0$ in non-overlapping regions,
such as the nose and the left arm,
suggesting that the points in these regions are successfully omitted by PWAN.

\begin{figure}[htb!]
  \vspace{-5mm}
  \centering
  \begin{minipage}{\linewidth}
    \subfigure[Initial step]{
      \includegraphics[width=0.33\linewidth]{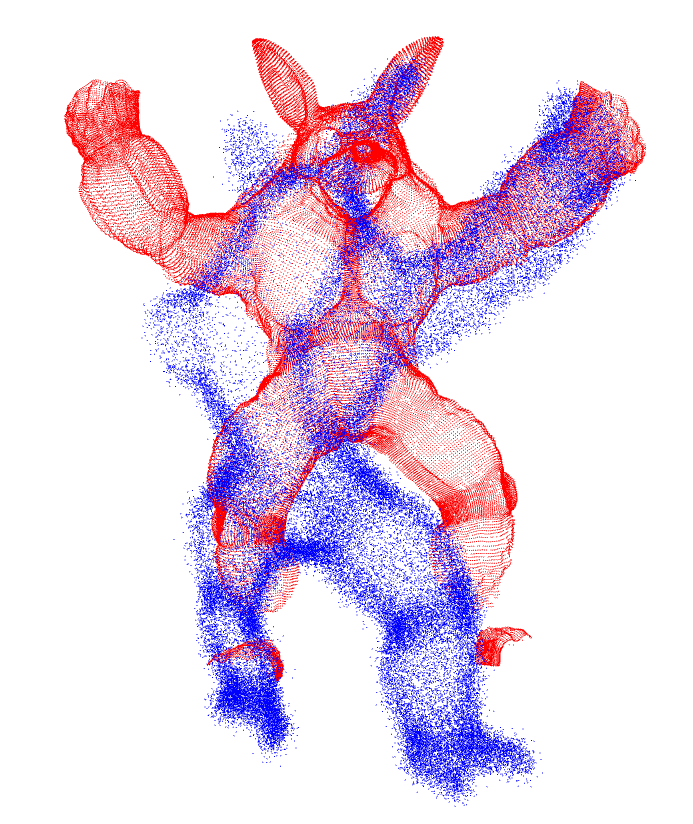}
    }
    \hspace{-5mm}
    \subfigure[Final step]{
      \includegraphics[width=0.33\linewidth]{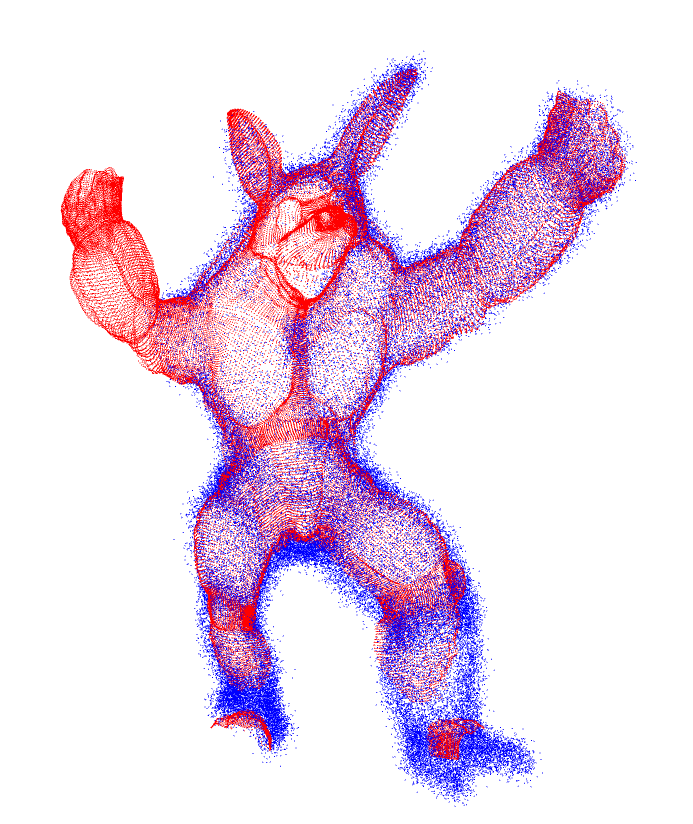}
    }
    \hspace{-5mm}
    \subfigure[$||\nabla \vf_{w,h}||$]{
      \includegraphics[width=0.33\linewidth]{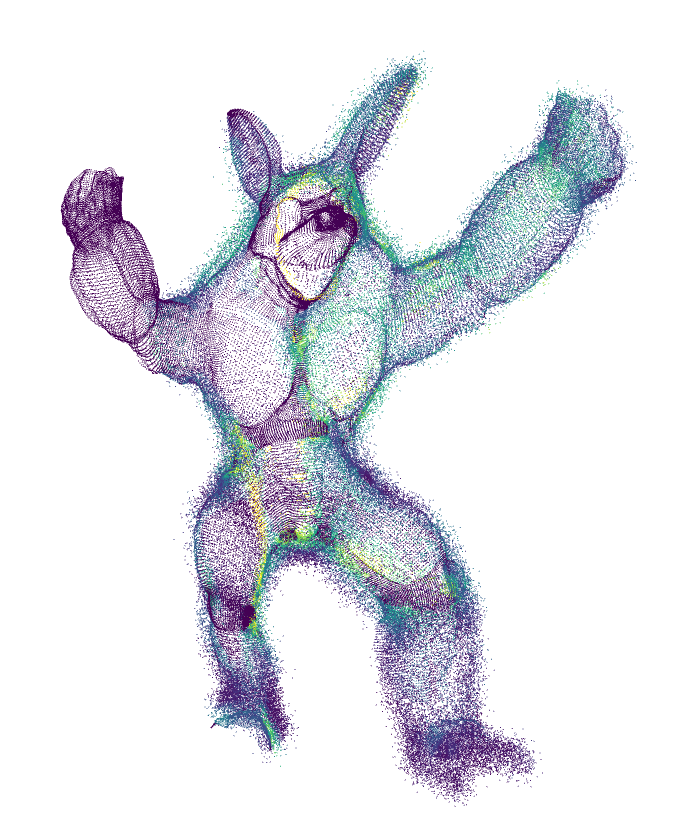}
    }
  \end{minipage}
\vspace{-3mm}
\caption{Visualization of the learned potential network on a pair of partially overlapped armadillo point sets.
$||\nabla \vf_{w,h}||$ is lower (darker) in non-overlapping regions,
suggesting that the points in those regions are successfully omitted.
}
\label{vis_outlier_larma}
\end{figure}

\subsubsection{Evaluation of the Efficiency}
\label{Sec_exp_eff}

To evaluate the efficiency of PWAN,
we first need to investigate the influence of the parameter $u$,
\ie,
the network update frequency,
which is designed to control the tradeoff between efficiency and effectiveness.
To this end,
we register a pair of bunny datasets consisting of $2000$ points with varying $u$ and report the results on the left panel of Fig.~\ref{Scalability}.
As can be seen,
the median and standard deviation of MSE decrease as $u$ increases,
and the computation time increases proportionally with $u$,
which verifies the effectiveness of $u$.

We then benchmark the computation time of different methods on a computer with two Nvidia GTX TITAN GPUs and an Intel i7 CPU.
We fix $u=20$ for PWAN.
We sample $q=r$ points from the bunny shape,
where $q$ varies from $10^3$ to $7 \times 10^5$. 
PWAN is run on the GPU while the other methods are run on the CPU.
We also implement a multi-GPU version of PWAN where the potential network is updated in parallel.  
We run each method $10$ times and report the mean of the computation time on the right panel of Fig.~\ref{Scalability}.
As can be seen,
BCPD is the fastest method when $q$ is small,
and PWAN is comparable with BCPD when $q$ is near $10^6$.
In addition,
the 2-GPU version PWAN is faster than the single GPU version,
and it is faster than BCPD when $q$ is larger than $5 \times 10^5$.

\begin{figure}
      \centering
      \includegraphics[width=0.45\linewidth]{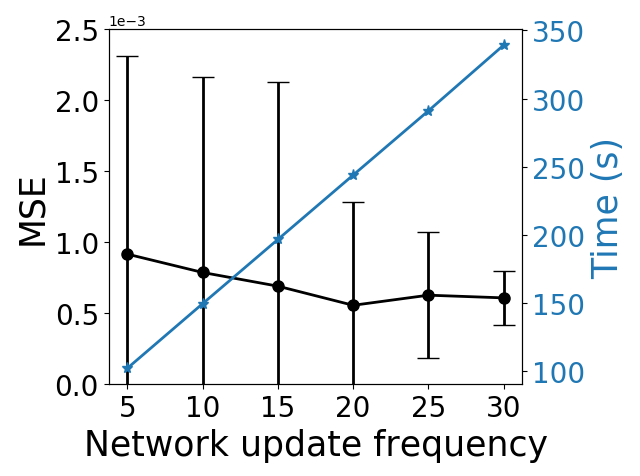}
      \includegraphics[width=0.45\linewidth]{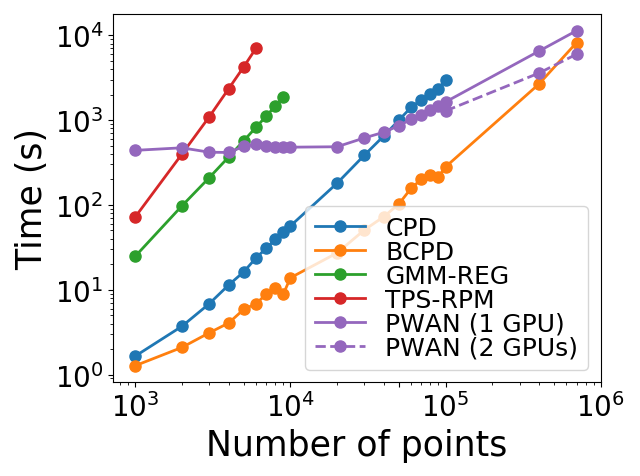}
  \vspace{-2mm}
  \caption{Scalability of PWAN.
  Left: $u$ controls the tradeoff between efficiency and effectiveness.
  Right:
  The speed of PWAN is comparable with that of BCPD when $q$ is near $10^6$.
  }
\vspace{-2mm}
\label{Scalability}
\end{figure}

\subsubsection{Evaluation on Real Data}
\label{Sec_exp_real}
To demonstrate the capability of PWAN in handling datasets with non-artificial deformations,
we evaluate it on the space-time faces dataset~\cite{zhang2008spacetime} and human shape dataset~\cite{bodydataset}.

The human face dataset~\cite{zhang2008spacetime} consists of a time series of point sets sampled from a real human face.
Each face consists of $23,728$ points.
We use the faces at time $i$ and $i + 20$ as the source and the reference set,
where $i=1,...,20$.
All point sets in this dataset are completely overlapped.
An example of our registration result is presented in Fig.~\ref{real}.
The registration results are quantitatively compared with CPD and BCPD in Tab.~6 in the appendix,
where PWAN outperforms both baseline methods.

We further evaluate PWAN on the challenging human body dataset~\cite{bodydataset}.
We apply PWAN to both complete and incomplete human shapes.
An example of our registration results is shown in Fig.~\ref{real_human},
and more details can be found in Appx.~B.8.
As can be seen,
although PWAN is not designed to handle articulated deformations as~\cite{ge2014non},
it can still produce good full-body registration ($1$-st row) and natural partial alignment ($2$-nd row).
In addition,
it is worth noticing that the partial alignment produces a novel shape that does not exist in the original dataset,
which suggests the potential of PWAN in shape editing.
Nevertheless,
we observe some degree of misalignment near complicated articulated structures, 
such as fingers.
This issue might be addressed by incorporating mesh-level constraints~\cite{fan2022coherent} in the future.

\begin{figure}[tb!]
  \vspace{-1mm}
\centering
    \includegraphics[width=0.25\linewidth]{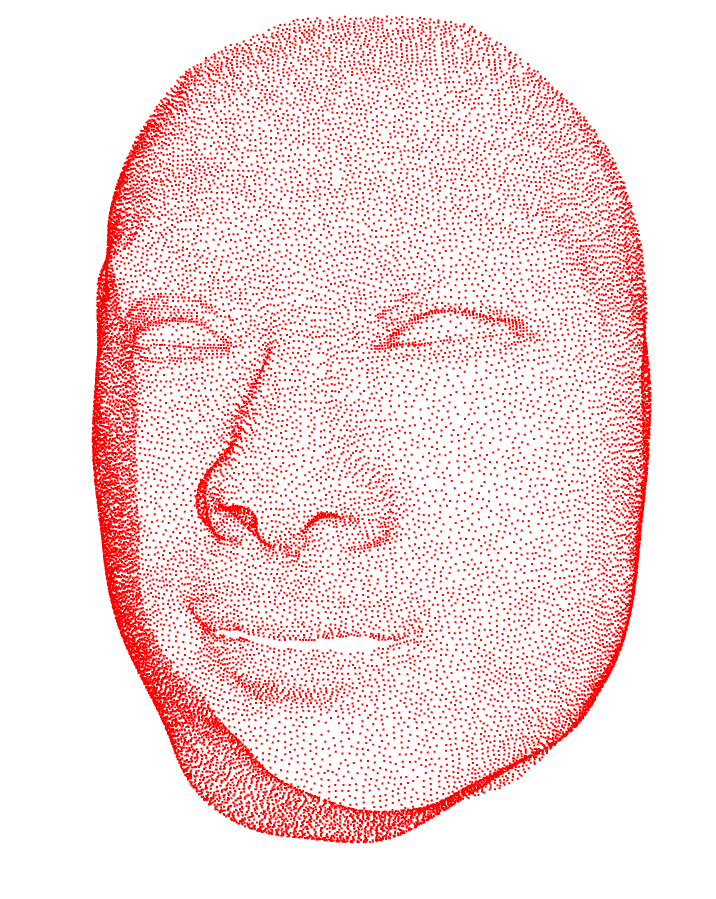}
    \hspace{3mm}
    \includegraphics[width=0.25\linewidth]{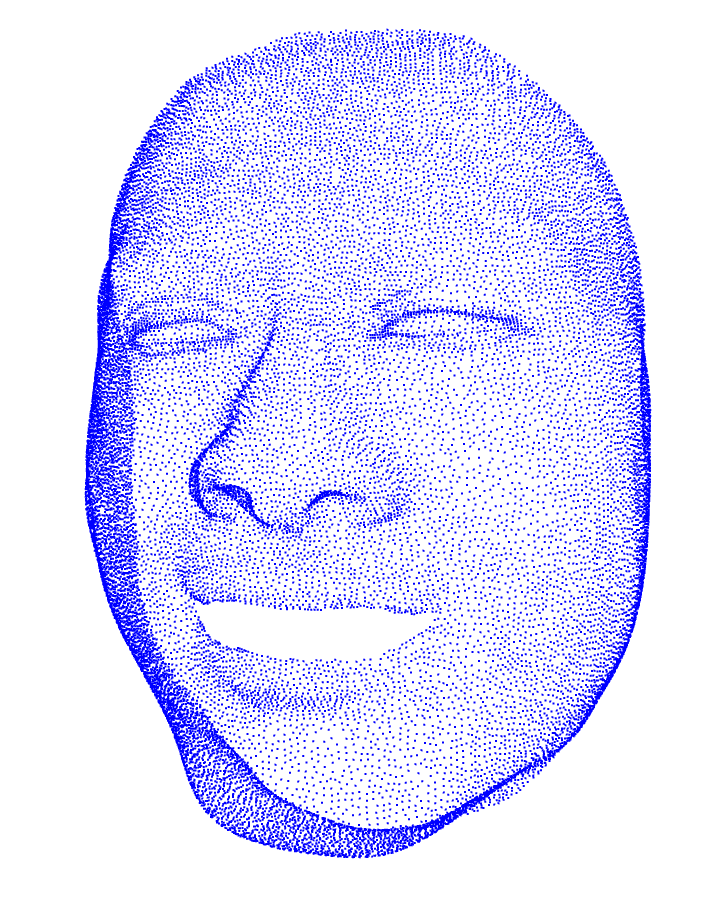}
    \hspace{3mm}
    \includegraphics[width=0.25\linewidth]{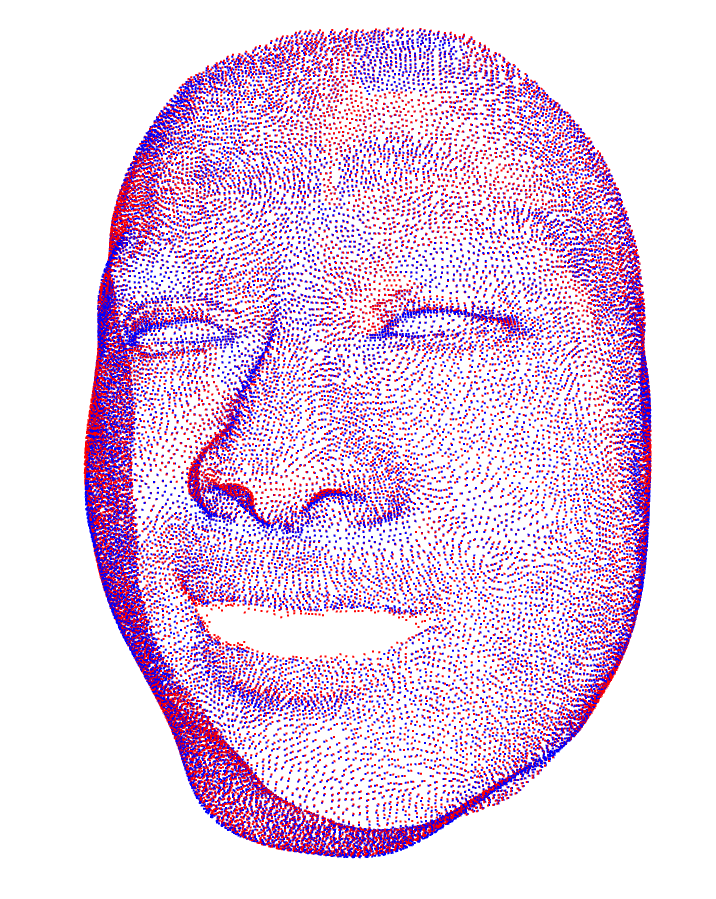}
\vspace{-2mm}
\caption{An example of our registration results on the human face dataset. 
The aligned point sets (right) is obtained by matching the $1$-st frame (left) to the $21$-st frame (middle).
}
\label{real}
\vspace{-2mm}
\end{figure}

\begin{figure}[htb!]
	\centering
      \begin{minipage}[b]{0.3\linewidth}
          \centering
          \includegraphics[width=0.92\linewidth]{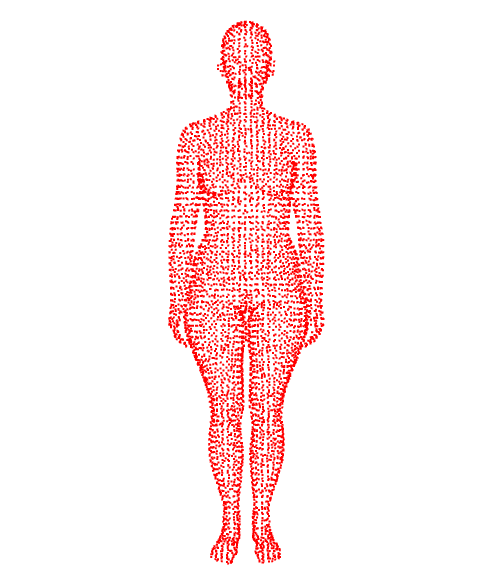}\\
          \includegraphics[width=1\linewidth]{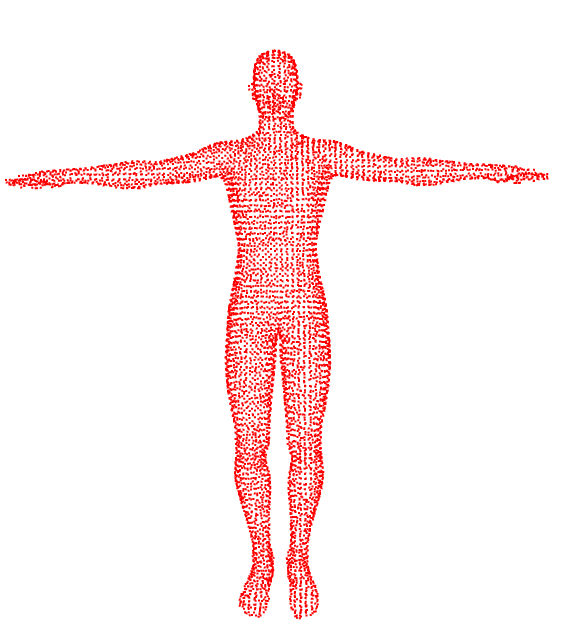}
      \end{minipage}
      \begin{minipage}[b]{0.3\linewidth}
          \centering
          \includegraphics[width=1\linewidth]{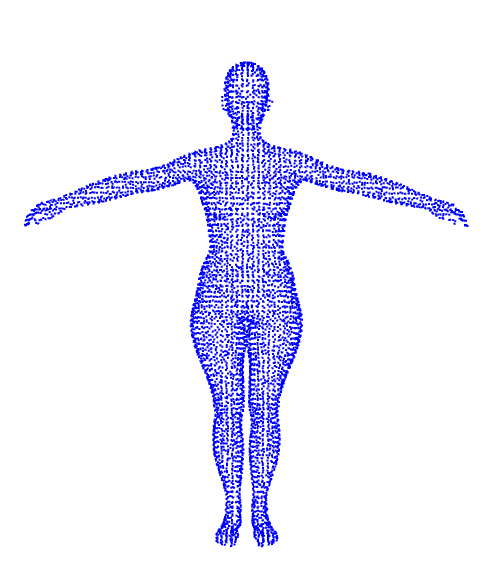}\\
          \includegraphics[width=1\linewidth]{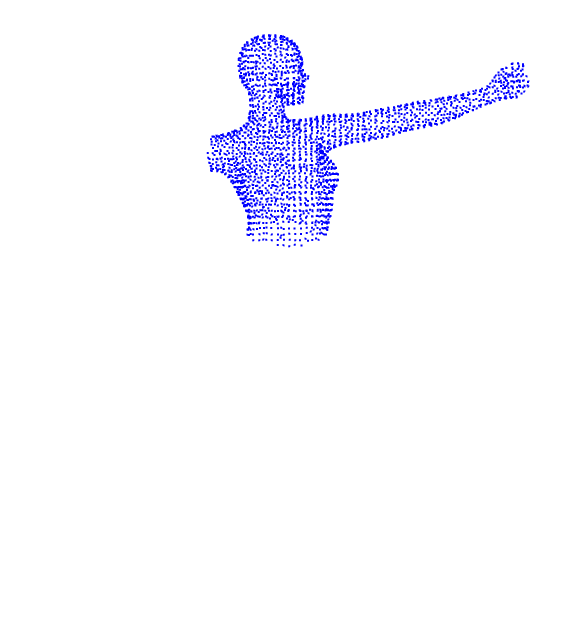} \\
      \end{minipage}
      \begin{minipage}[b]{0.3\linewidth}
          \centering
          \includegraphics[width=1\linewidth]{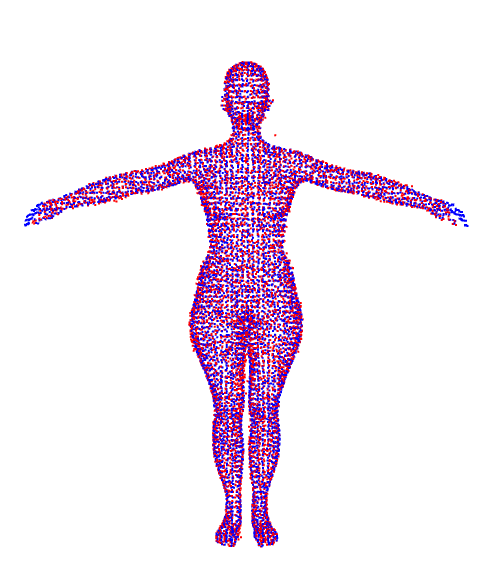}\\
          \includegraphics[width=1\linewidth]{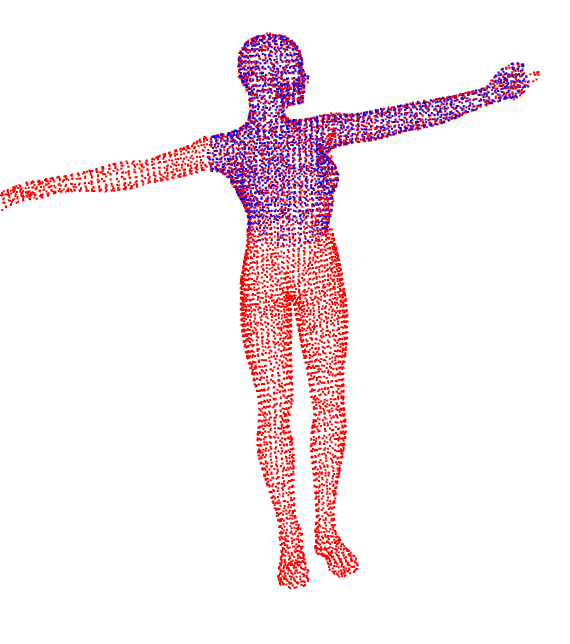} \\
      \end{minipage}
    \vspace{-2mm}
    \caption{
      Examples of our registration results on the human body dataset. 
      We present the result of registering complete shapes ($1$-st row) and incomplete shapes ($2$-nd row),
      where the source point set (left) is aligned to the reference set (middle).
    }
    \label{real_human}
    \vspace{-2mm}
\end{figure}

\subsubsection{Results of Rigid Registration}
\label{Sec_rigid}
PWAN can be naturally applied to rigid registration as a special case of non-rigid registration.
We evaluate PWAN on rigid registration using two datasets: ASL~\cite{pomerleau2012challenging} and UWA~\cite{mian2006three}.
We construct three registration tasks: outdoor, indoor and object,
where the outdoor task consists of two scenes from ASL: mountain and wood-summer,
the indoor task consists of two scenes from ASL: stair and apartment,
and the object task consists of two scenes from UWA: parasaurolophus and T-rex.
For each of the $6$ scenes,
we register the $i$-th view to the $(i + 1)$-th view,
where $i=1,...,10$.
We report the rotation error $\frac{180}{\pi} \arccos (\frac{1}{2} (Tr(R\tilde{R}^T)-1))$,
where $R$ and $\tilde{R}$ are the estimated and true rotation matrices respectively.
Since UWA dataset does not provide ground true poses,
we run the global search algorithm Go-ICP~\cite{yang2015go} on UWA dataset and use its result as the ground truth.

We pre-process all point sets by normalizing them to be inside the cubic $[-1, 1]^3$,
and down-sampling them using a voxel grid filter with grid size $0.08$.
We compare d-PWAN and m-PWAN against BCPD~\cite{hirose2021a}, CPD~\cite{myronenko2006non}, ICP with trimming~\cite{chetverikov2005robust} (m-ICP) and ICP with distance threshold~\cite{besl1992a, Zhou2018} (d-ICP),
and sopt~\cite{bai2023sliced}.
We set $h=0.05$ for d-PWAN,
$m=0.8\min(q,r)$ for m-PWAN, and fix $u=25$.

A quantitative comparison of the registration result is presented in Tab.~\ref{rigid_quantitative},
and more details are presented in Appx.~B.9.
As can be seen,
m-PWAN and d-PWAN outperform all baseline methods in all but the object class,
where BCPD performs slightly better,
and m-PWAN achieves better results than d-PWAN in terms of standard deviation.

\begin{table}[ht!]
	\begin{center}
	  \caption{Quantitative results of rigid registration. We report the median and standard deviation of rotation errors.}
	  \setlength{\tabcolsep}{2.0pt}
    \label{rigid_quantitative}
      \begin{tabular}{c c c c} 
		  \hline
      & ASL indoor & ASL outdoor & UWA object \\
    \hline
	  \centering
    BCPD & 0.35 (10.37) & 8.08 (8.54) & \textbf{0.10 (0.26)} \\
    CPD & 0.79 (28.46) & 5.89 (3.95) & 0.17 (61.99) \\
    SOPT & 11.03 (28.91) & 30.01 (11.45) & 34.9 (6.55) \\
    m-ICP & 0.54 (22.46) & 13.36 (12.37) & 1.96 (6.43) \\
    d-ICP & 6.62 (18.85) & 14.86 (8.20) & 21.17 (6.65) \\
    m-PWAN (Ours) & \textbf{0.26 (14.22)} & \textbf{2.00 (3.62)} & 0.10 (0.35) \\
    d-PWAN (Ours) & 0.26 (22.44) & 2.01 (3.63) & 0.11 (5.60) \\
    \hline
	  \end{tabular}
	\end{center}
\end{table}

\section{Application \uppercase\expandafter{\romannumeral2}: Partial Domain Adaptation}
\label{Sec_application_2}

This section applies PWAN to partial domain adaptation tasks.
After presenting the formulation of partial domain adaptation in Sec.~\ref{Sec_app_PDA},
and the details of the algorithm in Sec.~\ref{Sec_app_PDA_alg},
we show the results of our numerical experiments in Sec.~\ref{Sec_experiment_PDA}.

\subsection{A PDM Formulation of Partial Domain Adaptation}
\label{Sec_app_PDA}
Consider an image classification problem defined on the image space $I$ and the label space $L$.
Let $\mathbf{R}$ and $\mathbf{S}$ be the reference and source distributions over $I \times L$,
$(\cdot)|_I$ be their marginal distributions on $I$,
and $L_{(\cdot)}$ be their label sets.
For example,
$\mathbf{R}_I$ is the reference image distribution,
and $L_R$ consists of all possible labels of $\mathbf{R}$.
Given the labelled reference dataset $\mathbf{D}_R=\{(x_i, l_i)\}_{i=1}^q \sim (\mathbf{R})^q$
and the unlabeled source dataset $\mathbf{D}_S=\{(y_i)\}_{i=1}^r \sim (\mathbf{S}_I)^r$,
where $L_S$ is assumed to be an unknown subset of $L_R$,
the goal of partial domain adaptation is to obtain an accurate classifier $\eta$ for $\mathbf{D}_S$.
Note that following~\cite{ganin2016domain,cao2022big},
we assume $\eta = \mathcal{D}_{\nu} \circ \mathcal{T}_\theta$,
where $\mathcal{T}_\theta$ is a feature extractor with parameter $\theta$,
$\mathcal{D}_{\nu}$ is a label predictor with parameter $\nu$.

We regard the partial adaptation task as a PDM task,
where the goal is to match $\mathbf{S}_I$ to the fraction of $\mathbf{R}_I$ belonging to class $L_S$ in the feature space.
This PDM task is special in that it assumes that the outliers,
\ie, the data belonging to class $L_R \setminus L_S$,
do not exist in $\mathbf{S}_I$.
We utilize this assumption in our formulation by not omitting any data in $\mathbf{S}_I$ during matching.
Formally,
let $\beta_\theta=(\mathcal{T}_\theta)_{\#}\mathbf{S}_I$ and $\alpha=(\mathcal{T}_{\widehat{\theta}})_{\#}\mathbf{R}_I$ be the feature distributions of $\mathbf{S}_I$ and $\mathbf{R}_I$ respectively,
where $\widehat{\theta}$ denotes a frozen copy of $\theta$ which does not allow back-propagation.
We minimize $\mathcal{L}_{M,1}(\alpha, \beta_\theta)$,
where we set $m_\alpha \geq m_\beta=m=1$ to enforce that the whole $\beta_\theta$ is matched to a fraction of $\alpha$.
By combining the PDM formulation with the classifier $\eta$,
our complete formulation can be expressed as
\begin{equation}
  \label{PDA_goal}
  \min_{\theta, \nu} \sum_{(x_i, l_i) \in \mathbf{R}}  \mu_i \textit{CE}(\eta(x_i), l_i )  + \lambda_1 \mathcal{L}_{M,1}(\alpha, \beta_\theta)
                                                                                        + \lambda_2 \mathcal{C}(\eta),
\end{equation}
where $\mathcal{C}(\eta)$ is the entropy regularizer~\cite{grandvalet2004semi},
$\textit{CE}$ is the cross-entropy loss,
$\mu_i=\frac{1}{|\mathbf{D_S|}} \sum_{y \in \mathbf{D}_S}\mathds{1}( \eta(y)=l_i)$
is the re-weighting coefficient of class $l_i$~\cite{liang2020balanced},
where the indicator function $\mathds{1}(\cdot)$ equals $1$ if $(\cdot)$ is true, 
and $0$ otherwise.
Note that $\mu$ is designed to reduce the influence of outliers on the training of the classifier $\eta$,
\ie, $\mu_i$ is expected to be smaller if $x_i$ is in the outlier class.
In practice,
we update $\mu$ periodically during the training process.
Also note that existing OT-based partial adaptation methods~\cite{nguyen2022improving,fatras2021unbalanced} do not utilize the clean $\mathbf{S}_I$ property,
because they generally omit samples in both $\mathbf{R}_I$ and $\mathbf{S}_I$.

\subsection{PWAN for Partial Domain Adaptation}
\label{Sec_app_PDA_alg}
To apply PWAN to~\eqref{PDA_goal},
we first verify that the feed-forward network $\mathcal{T}_\theta$ is a valid transformation for PWAN,
\ie, it satisfies the assumption required by Theorem~\ref{Theorem2} (Proposition~10 in the appendix).
Then we need to determine the parameter $m_\alpha$, 
which specifies the alignment ratio $1/m_\alpha$ of the reference data.
If the outlier ratio $r$ of $\alpha$ is known as a prior,
we can simply set $m_\alpha \approx 1 / (1-r)$ to ensure that the correct ratio of data is aligned.
However, $r$ is usually unknown in practice,
so we first heuristically set $m_\alpha$ as $\tilde{m}_\alpha=|L_R| / |\tilde{L_S}|$,
where $\tilde{L_S} = \{l \in L_R | w_l > \frac{1}{|L_R|} \}$ is the estimated $L_S$.
This choice of $m_\alpha$ is intuitive because it encourages aligning $\tilde{L_S}$ classes of data between $\mathbf{D_R}$ and $\mathbf{D_S}$.
Then,
to further avoid biasing toward outliers,
we gradually increase  $m_\alpha$ during training.
Formally,
we set 
\begin{equation}
  \label{m_alpha_pda}
  m_\alpha= \tilde{m}_\alpha \times s^v
\end{equation}
at training step $v$,
where 
$s >1$ is the annealing factor.
An explanation of the effect of $m_\alpha$ is presented in Appx.~B.10.

Furthermore,
we observed that the KR dual of $\mathcal{L}_{M,m}$~\eqref{PW-m} takes a particularly simple form in this case~\eqref{PDA_goal},
where the scalar variable $h$ does not appear.
In other words, 
we do not need to optimize the scalar $h$,
and the bounded constraint $\vf \geq -h$ of $\vf$ can be dropped.
Specifically,
we have the following straightforward corollary.
\begin{corollary}
  \label{corollary-P}
  When $m_\alpha \geq m_\beta=m$, 
  $\mathcal{L}_{M,m}$ can be equivalently written as
  \begin{equation}
    \label{simple_1}
    \mathcal{L}_{P}(\alpha, \beta) = \sup_{\vf \in Lip(\Omega), \vf \leq 0 }  \int_\Omega \vf d\alpha -   \int_\Omega \vf d\beta.
  \end{equation}
  In addition,
  $\mathcal{L}_{D,h}$ can be written as $\mathcal{L}_{P}-hm_\beta$ when $m_\alpha \geq m_\beta$ and $h \geq \textit{diam}(\Omega)$.
\end{corollary}

By inserting~\eqref{simple_1} into~\eqref{PDA_goal} and taking a mini-batch form,
we can finally write the objective function as:
\begin{align}
  \label{obj_PDA}
  \min_{\theta, \nu} \max_{w} & \sum_{i=1}^{\tilde{q}} \mu_i \textit{CE}(\eta(x_i), l_i )  +  \lambda_1 \Bigl(  \frac{m_\alpha}{\tilde{q}} \sum_{i=1}^{\tilde{q}}  \vf_w(\mathcal{T}_{\hat{\theta}} (x_i) ) \nonumber \\ 
             -  &\frac{1}{\tilde{q}}\sum_{i=1}^{\tilde{q}}  \vf_w(\mathcal{T}_\theta(y_i)) - GP(\vf_{w}) \Bigr) +  \lambda_2 \mathcal{C}(\eta), 
\end{align}
where $\{(x_i, l_i)\}_{i=1}^{\tilde{q}}$ is a random batch of $\mathbf{D}_R$, 
$\{(y_i)\}_{i=1}^{\tilde{q}}$ is a random batch of $\mathbf{D}_S$,
$m_\alpha$ is set according to~\eqref{m_alpha_pda},
and $\vf_w \leq 0$ is the potential network with parameter $w$.
The overall network architecture is presented in Fig.~\ref{PDA_net}.
By abusing notations,
we also call~\eqref{obj_PDA} PWAN.

We end this subsection with a remark that PWAN~\eqref{obj_PDA} can be seen as a direct generalization of the classic adversarial closed-set domain adaptation methods~\cite{shen2018wasserstein,ganin2016domain},
where the only differences are the constraints of $m_\alpha$ and $\vf_w$:
$m_\alpha \geq 1$ and $\vf_w \leq 0$ in PWAN,
while $m_\alpha = 1$ and there is no sign constraint for $\vf_w$ in~\cite{shen2018wasserstein}.
As we will show in the experiment section,
this generalization alone,
without complicated sampling strategies~\cite{zhang2018importance, guo2022selective, cao2022big, gu2024adversarial},
allows PWAN to produce accurate results in the presence of a high ratio of outliers.

\begin{figure}[htb!]
  \centering
    \includegraphics[width=0.95\linewidth]{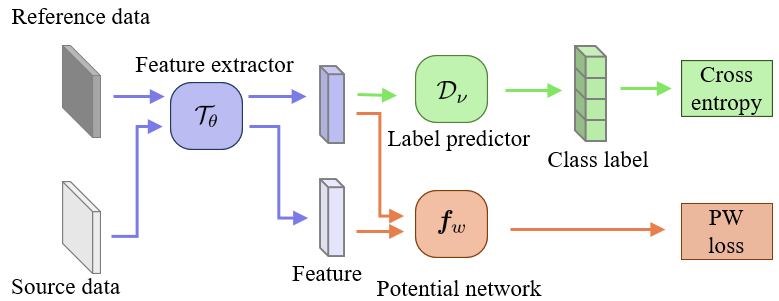}
    \vspace{-2mm}
\caption{The proposed PWAN for partial domain adaptation.
}
\label{PDA_net}
\end{figure}

\subsection{Experiments on Partial Domain Adaptation}
\label{Sec_experiment_PDA}
We experimentally evaluate PWAN on partial domain adaptation tasks.
After describing the experimental settings in Sec.~\ref{Sec_PDA_setting},
we first present an ablation study in Sec.~\ref{Sec_PDA_ablation} to explain the effectiveness of our formulation.
Then we compare the performance of PWAN with the state-of-the-art methods in Sec.~\ref{Sec_PDA_comparison}.
We finally discuss parameter selection in Sec.~\ref{Sec_PDA_analysis}.

\subsubsection{Experimental Settings}
\label{Sec_PDA_setting}
We construct PWAN as follows:
the potential net $\vf_w$ is a fully connected network with $2$ hidden layers ($\textit{feature}$ $\rightarrow$ $256$ $\rightarrow$ $256$ $\rightarrow$ $1$);
Following~\cite{ganin2016domain},
the feature extractor $\mathcal{T}_\theta$ is defined as a ResNet-50~\cite{he2016deep} network pre-trained on ImageNet~\cite{russakovsky2015imagenet},
and the label predictor $\mathcal{D}_\nu$ is a fully connected network with $1$ bottleneck layer ($\textit{feature}$ $\rightarrow$ $256$ $\rightarrow$ $\textit{label}$).
We update $\{\theta, \nu\}$ using Adam optimizer~\cite{kingma2014adam} and update $\{w\}$ using RMSprop optimizer~\cite{Tieleman2012},
and manually set the learning rate to $1e^{-4}(1+10i)^{-0.75}$ following~\cite{ganin2016domain}, 
where $i$ is the training step.
We fix $\lambda_1=0.05$ and $\lambda_2=0.1$ for all experiments.
More detailed settings are presented in Appx.~B.11.

We consider the following four datasets in our experiments:
\begin{itemize}[leftmargin=3mm]
  \item[-] {OfficeHome}~\cite{venkateswara2017deep}: A dataset consisting of four domains: Artistic (A), Clipart (C), Product (P) and real-world (R). 
  Following~\cite{cao2022big},
  we consider all $12$ adaptation tasks between these $4$ domains.
  We keep all $65$ classes of the reference domain and keep the first $25$ classes (in alphabetic order) of the source domain.
  \item[-] {VisDa17}~\cite{peng2018visda}: A large-scale dataset for synthesis-to-real adaptation, 
  where the reference domain consists of abundant rendered images, 
  and the source domain consists of real images belonging to the same classes.
  We keep all $12$ classes of the reference domain and keep the first $6$ classes (in alphabetic order) of the source domain.
  \item[-] {ImageNet-Caltech}: A challenging large-scale dataset for general-to-specify adaptation.
  The reference domain is ImageNet~\cite{russakovsky2015imagenet} ($1000$ classes),
  and the source domain consists of data from Caltech~\cite{griffin2007caltech} belonging to the $84$ classes shared by Caltech and ImageNet.
  \item[-] {DomainNet}~\cite{peng2019moment}: A recently proposed challenging large-scale dataset. 
  Following~\cite{saito2019semi},
  we adopt four domains: Clipart (C), Painting (P), Real (R) and Sketch (S).
  We consider all $12$ adaptation tasks on this dataset.
  We keep all $126$ classes of the reference domain and keep the first $40$ classes (in alphabetic order) of the source domain.
\end{itemize}
For clearness,
for OfficeHome and DomainNet,
we name each task $\textit{RS}$-$\textit{D}$,
where $\textit{D}$ represents the dataset,
$\textit{R}$ and $\textit{S}$ represent the reference and source domain respectively.
For example, 
$AC$-OfficeHome represents the task of adapting domain A to domain C on OfficeHome.
We omit the dataset name if it is clear from the text.

\subsubsection{Ablation Study}
\label{Sec_PDA_ablation}

\begin{figure}[tb!]
  \centering
  \vspace{-2mm}
  \subfigure[$\lambda_1=\lambda_2=0$]{
    \includegraphics[width=0.45\linewidth]{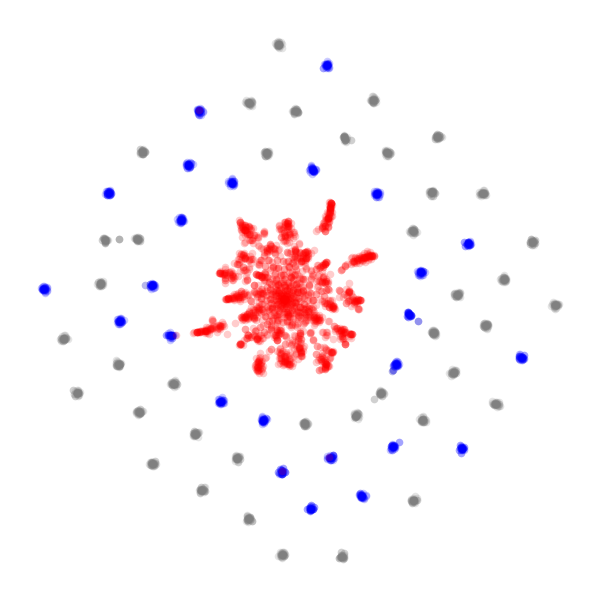}
  }
  \subfigure[$\lambda_1=1, m_\alpha=1 ; \lambda_2=0$]{
    \includegraphics[width=0.45\linewidth]{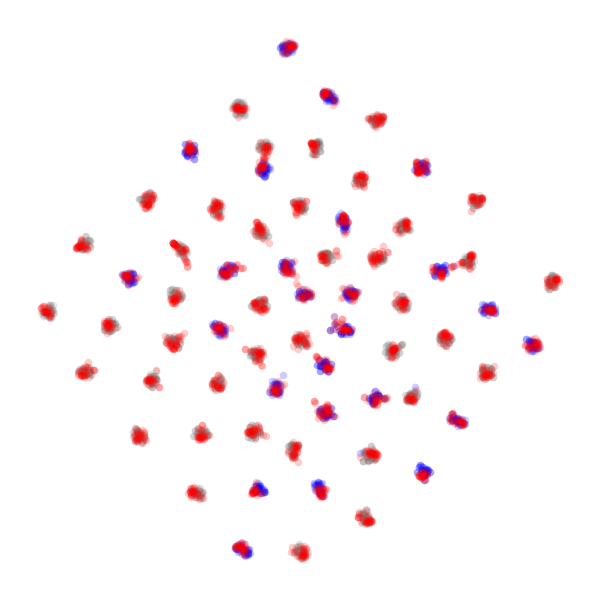}
  }
  \subfigure[$\lambda_1=1, m_\alpha=10 ; \lambda_2=0$]{
    \includegraphics[width=0.45\linewidth]{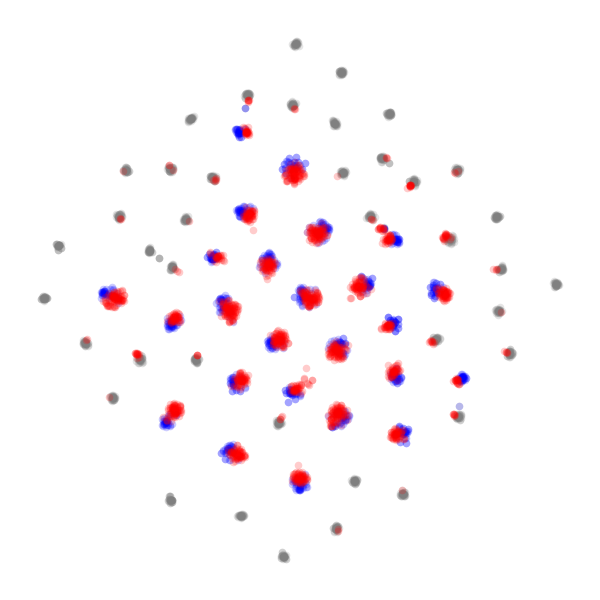}
  }
  \subfigure[$\lambda_1=1, m_\alpha=10; \lambda_2=0.1$]{
    \includegraphics[width=0.45\linewidth]{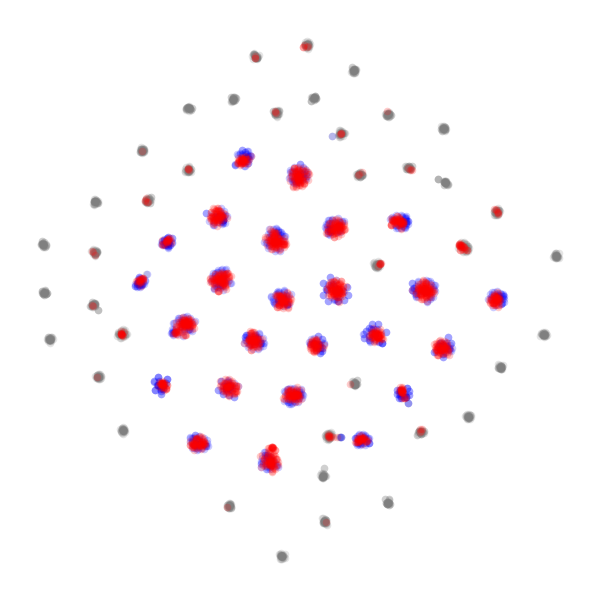}
  }

\caption{t-SNE visualization of the learned source and reference features,
where blue and grey points represent non-outlier and outlier reference features respectively,
and red points represent the source features.
PWAN ($\lambda_1 \neq 0, m_\alpha > 1$) can align the source features to a fraction of the reference features while avoiding biasing toward outliers.
}
\vspace{-2mm}
\label{PDA_vis}
\end{figure}

To explain the effectiveness of formulation~\eqref{obj_PDA},
we conduct an ablation experiment on AR-OfficeHome.
We train PWAN with different values of $\{\lambda_1, m_\alpha, \lambda_2\}$,
and visualize the features extracted by $\mathcal{T}_\theta$ in Fig.~\ref{PDA_vis} using t-SNE~\cite{van2008visualizing}.
For clearness,
we fix $\mu_i=1$ and $s=1$ in this experiment.

When a classifier is trained on the reference data without the aid of PWAN ($\lambda_1=\lambda_2=0$),
the reference and the source feature distributions are well-separated,
which suggests that the trained classifier cannot be used for the source data.
PWAN ($\lambda_1>0$) can alleviate this issue by aligning these two distributions.
However,
complete alignment ($m_\alpha=1$) leads to degraded performance due to serious negative transfer, 
\ie, most of the source features are biased toward the outlier features.
In contrast,
by aligning a lower ratio of the reference features,
PWAN with larger $m_\alpha$ ($m_\alpha=10$) can largely avoid the negative transfer effect,
\ie, most of the outliers are omitted.
In addition,
the entropy regularizer ($\lambda_2 >0$) further improves the alignment as it pushes the source features to their nearest class centers. 
The results suggest that our formulation~\eqref{obj_PDA} can indeed align the feature distributions while avoiding negative transfer,
which makes the trained classifier suitable for the source data.

\begin{table*}[t!]
  \footnotesize
	\begin{center}
	  \caption{Results (accuracy \%) of partial domain adaptation on OfficeHome.
    We additionally report PWAN with four extra techniques (PWAN + A): label smoothing, 
    complement objective regularizer, 
    neighborhood reciprocity clustering,
    and $\alpha$-power.}
    \vspace{-3mm}
	  \setlength{\tabcolsep}{1.3pt}
    \label{Office_Home_tab}
      \begin{tabular}{c c c c c c c c c c c c c c} 
		  \hline
           & AC &	AP	& AR	& CA	& CP	& CR	& PA	& PC	& PR	& RA	& RC	& RP & Avg \\
		\hline
	  \centering
    Ref Only	& 42&	67&	79.2&	56.8&	55.9&	65.4&	59.3&	35.5&	75.5&	68.7&	43.4&	74.5&	60.2 \\
    DANN~\cite{ganin2016domain}	&46.6	&45.8&	57.5&	37.3&	32.6&	40.5&	40.2&	39.4&	55.4&	54.5&	44.8&	57.6&	46.0 \\
    PADA~\cite{cao2018partial}	&48.9&	66.9&	81.6&	59.1&	55.3&	65.7&	65&	41.6&	81.1&	76&	47.6&	82&	64.2 \\
    SAN++~\cite{cao2022big} &61.25 &81.57 &88.57 &72.82 &76.41 &81.94 &74.47& 57.73 &87.24 &79.71& 63.76& 86.05& 75.96\\
    BAUS~\cite{liang2020balanced}	&60.6&	83.1&	88.3&	71.7&	72.7&	83.4&	75.4&	61.5&	86.5&	79.2&	62.8&	86.0&	75.9 \\
    DPDAN~\cite{hu2020discriminative} & 59.4 & --- & 79.0& ---& ---& ---& ---& --- &81.7& 76.7& 58.6& 82.1 & --- \\
    SHOT~\cite{liang2020we}	&62.9 (1.9)&	87.0 (0.5)&	92.5 (0.1)&	75.8 (0.3)&	77.0 (1.2)&	86.2 (0.6)&	77.7 (0.8)&	62.8 (0.6)	&90.4 (0.5)&	81.8 (0.1)&	65.5 (0.4)&	86.2 (0.1)&	78.8 (0.1) \\
    ADV~\cite{gu2021adversarial}	&62.3 (0.4)	&82.3 (2.9)	&91.5 (0.5)&	77.3 (0.9)&	76.6 (2.4)&	84.9 (2.6)&	79.8 (0.8)&	63.4 (0.7)&	90.1 (0.5)&	81.6 (1.1)&	65.0 (0.6)&	86.6 (0.5)&	78.5 (0.4) \\
    IDSP~\cite{li2022partial} & 60.8 & 80.8 & 87.3 &69.3 &76.0 &80.2& 74.7& 59.2& 85.3& 77.8& 61.3 &85.7& 74.9 \\
    SLM~\cite{sahoo2023select}  & 61.1 & 84.0 & 91.4 & 76.5 & 75.0 & 81.8 &74.6 &55.6 &87.8 &82.3 &57.8& 83.5& 76.0 \\
    SPDA~\cite{guo2022selective} &63.1 &87.8 &90.1 &77.2& 75.4 &85.6 &81.4& 62.4& 90.5 &82.6 &69.5& 88.2& 79.5 \\
    RAN~\cite{wu2022reinforced} & 63.3 & 83.1 & 89.0 & 75.0 & 74.5 & 82.9&  78.0 & 61.2&  86.7&  79.9 & 63.5 & 85.0 & 76.8 \\
    CLA~\cite{yang2022contrastive} &66.7 &85.6 &90.9 &75.6 &76.9& 86.8& 78.8 &67.4 &88.7 &81.7& 66.9 &87.8& 79.5 \\
    ARPM~\cite{gu2024adversarial} & 66.6 (2.8)&	87.2 (1.5)&	93.1 (0.5)&	78.1 (4.2)&	78.6 (2.0)&	84.8 (2.2)&	75.8 (0.4)&	68.5 (1.1)&	90.3 (2.6)&	86.5 (0.6)&	69.6 (2.4)&	89.2 (0.2) & \textbf{80.7 (0.6)} \\
		PWAN (Ours) &63.5 (3.3)&	83.2 (2.7)&	89.3 (0.2)&	75.8 (1.5)&	75.5 (2.5)&	83.3 (0.1)&	77.0 (2.4)&	61.1 (1.5)&	86.9 (0.7)&	79.9 (0.6)&	65.0 (3.8)&	86.2 (0.2)	&77.2 (1.6)\\
    PWAN+A  (Ours)	& 65.4 (0.2) &	88.0 (1.8) &	89.9 (0.3) &	79.2 (1.1) &	78.0 (1.1) &	88.0 (0.5) &	80.5 (1.6)& 66.2 (0.5) &	88.6 (0.8)&	81.8 (0.5)&	70.2 (3.4) &	90.1 (0.6)&	80.5 (0.3)\\
    \hline
	  \end{tabular}
	\end{center}
	\vspace{-4mm}
\end{table*}

\subsubsection{Results of Partial Domain Adaptation}
\label{Sec_PDA_comparison}

We compare the performance of PWAN with the state-of-the-art algorithms: DPDAN~\cite{hu2020discriminative} (2020, ECCV), 
PADA~\cite{cao2018partial} (ECCV, 2018), BAUS~\cite{liang2020balanced} (ECCV, 2020), 
SHOT~\cite{liang2020we} (ICML, 2020), 
ADV~\cite{gu2021adversarial} (NIPS, 2021),  
SAN++~\cite{cao2022big} (TPAMI, 2023), 
IDSP~\cite{li2022partial} (TPAMI, 2023),
SLM~\cite{sahoo2023select} (WACV, 2023), 
SPDA~\cite{guo2022selective} (BMVC, 2022), 
RAN~\cite{wu2022reinforced} (TCSVT, 2023)
CLA~\cite{yang2022contrastive} (TNNLS, 2023) and ARPM~\cite{gu2024adversarial} (IJCV, 2024).
We also report the performance of DANN~\cite{ganin2016domain} (JMLR, 2016), 
which can be seen as a special case of PWAN for closed-set domain adaptation.
We evaluate the performances of all algorithms by the accuracy on the source domain.
We run the code of ARPM, ADV and SHOT with $3$ different random seeds and report the mean and standard deviation.
The results of other algorithms are adopted from ARPM~\cite{gu2024adversarial}.

\begin{figure}[tb!]
  \centering
  \subfigure[Learned features]{
    \label{Vis_Norm_a}
    \begin{minipage}[b]{0.36\linewidth}
        \centering
        \includegraphics[width=1\linewidth]{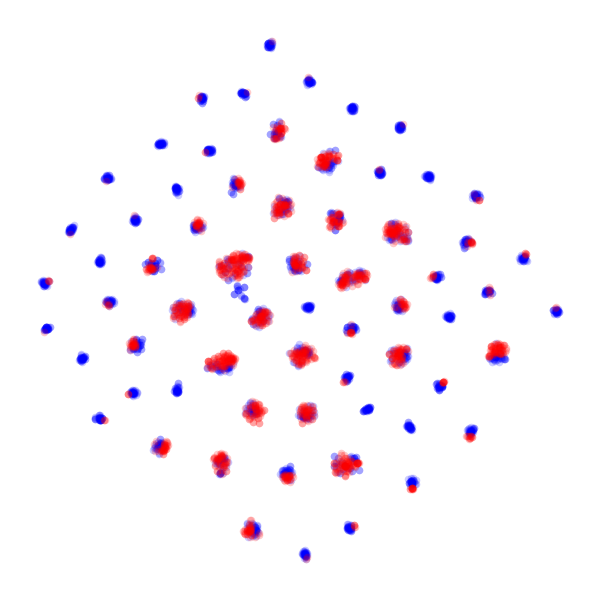}
    \end{minipage}
    }
    \hspace{3mm}
    \subfigure[$\|\nabla \vf_w \|$]{
        \label{Vis_Norm_b}
        \begin{minipage}[b]{0.36\linewidth}
            \centering
            \includegraphics[width=1\linewidth]{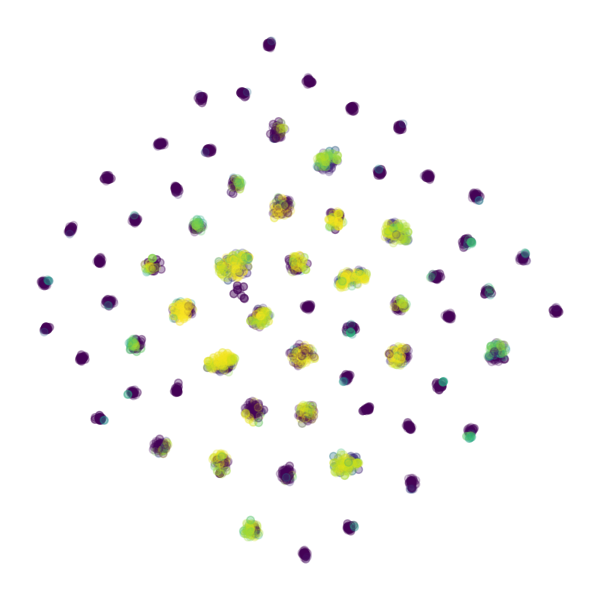}
        \end{minipage}
    }
    \subfigure[Classification result]{
        \label{Vis_Norm_c}
        \begin{minipage}[b]{0.45\linewidth}
            \centering
            \includegraphics[width=1\linewidth]{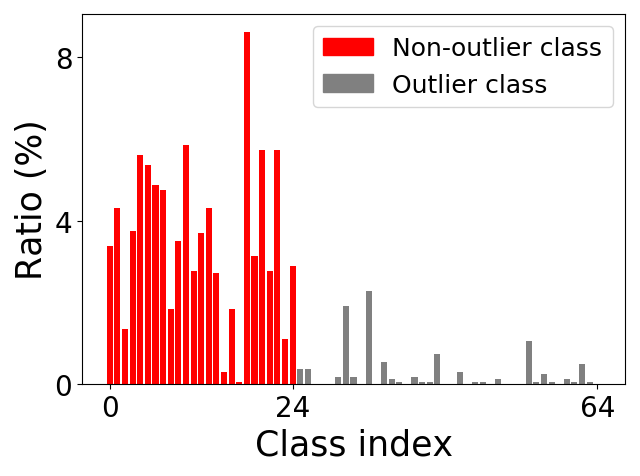}
        \end{minipage}
    }
    \subfigure[Training process]{
        \label{Vis_Norm_d}
        \begin{minipage}[b]{0.45\linewidth}
            \centering
            \includegraphics[width=1\linewidth]{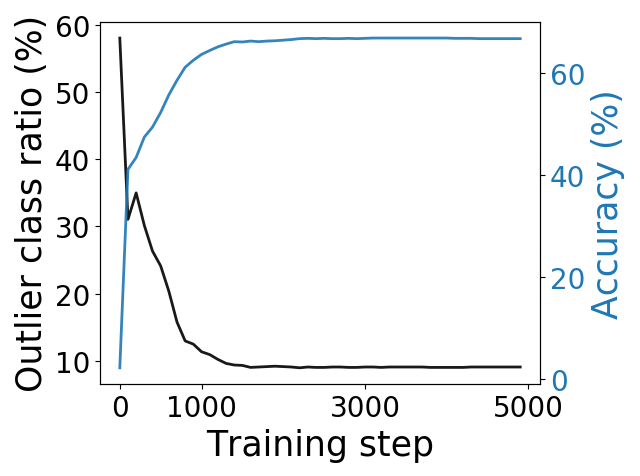}
        \end{minipage}
    }
\caption{More detailed results of PWAN on AC-OfficeHome.
\subref{Vis_Norm_a} Visualization of the learned features. 
The red and blue points represent the source and reference features respectively.
\subref{Vis_Norm_b} PWAN automatically omits a fraction of reference features (the dark points).
\subref{Vis_Norm_c} Most of the source data ($92\%$) is correctly classified into non-outlier classes, \ie,
the negative transfer effect is largely avoided.
\subref{Vis_Norm_d} During the training process, 
the test accuracy increases as the ratio of biased source data decreases.
}
\label{Vis_Norm}
\end{figure}

For a fair comparison,
we additionally report the results of PWAN with four advanced techniques developed for classification-type tasks (PWAN + A) on OfficeHome and DomainNet datasets:
complement objective regularizer~\cite{chen2019complement}, 
label smoothing~\cite{szegedy2016rethinking},
$\alpha$-power~\cite{gu2024adversarial},
and neighborhood reciprocity clustering~\cite{yang2021exploiting},
as they are also used in previous methods such as ADV, SHOT and BAUS and ARPM.
An ablation study of the effectiveness of each of these techniques is shown in Tab.~9 in the appendix.
The results on OfficeHome are reported in Tab.~\ref{Office_Home_tab}.
We observe that PWAN with these techniques (PWAN+A) performs comparably with ARPM,
and they outperform all other methods.

To show the effectiveness of PWAN,
we present more training details of PWAN on AC-OfficeHome in Fig.~\ref{Vis_Norm}.
As shown in Fig.~\ref{Vis_Norm_a} and Fig.~\ref{Vis_Norm_b},
the learned source features are matched to a fraction of the reference features,
while the other reference features are successfully omitted,
\ie, the gradient norm is low on these features.
Classification results in Fig.~\ref{Vis_Norm_c} show that most of the source data (about 92\%) are correctly classified into non-outlier classes,
which suggests that the negative transfer effect is largely avoided.
In addition,
the training process is shown in Fig.~\ref{Vis_Norm_d},
where we observe that classification accuracy (on the source domain) gradually increases as the biased ratio decreases during the training process.
The results show that our training process can effectively narrow the domain gap while avoiding the negative transfer effect.

\begin{table*}[tb!]
	\begin{center}
	  \caption{Results (accuracy \%) of partial domain adaptation on DomainNet.
    We additionally report PWAN with four extra techniques (PWAN + A): label smoothing, 
    complement objective regularizer, 
    neighborhood reciprocity clustering,
    and $\alpha$-power.}
    \vspace{-2mm}
	  \setlength{\tabcolsep}{1.3pt}
    \label{DomainNet_tab}
      \begin{tabular}{c c c c c c c c c c c c c c} 
		  \hline
           & CP&	CR&	CS&	PC&	PR&	PS&	RC&	RP&	RS&	SC&	SP&	SR & Avg \\
		\hline
	  \centering
    Ref Only	&41.2&	60.0&	42.1&	54.5&	70.8&	48.3&	63.1&	58.6&	50.2&	45.4&	39.3&	49.7&	51.9 \\
    DANN~\cite{ganin2016domain}	&27.8&	36.6&	29.9&	31.7&	41.9&	36.5&	47.6&	46.8&	40.8&	25.8&	29.5&	32.7&	35.6 \\
    PADA~\cite{cao2018partial}	&22.4	&32.8&	29.9&	25.7&	56.4&	30.4&	65.2&	63.3&	54.1&	17.4&	23.8&	26.9&	37.4\\ 
    BAUS~\cite{liang2020balanced}	&42.8 	&54.7& 	53.7& 	64.0& 	76.3& 	64.6& 	79.9& 	74.3& 	74.0& 	50.3& 	42.6& 	49.6& 	60.6 \\
    ADV~\cite{gu2021adversarial}	&54.1 (3.7)	&72.0 (0.4)&	55.4 (1.5) &	68.8 (0.5)	& 78.9 (0.2)	&75.4 (0.6)	& 77.4 (0.6)&	72.3 (0.5)	&70.4 (1.0)	&58.4 (0.6)	&53.4 (1.6)	&65.3 (0.6)	& 66.8 (0.2) \\
    ARPM~\cite{gu2024adversarial} & 65.9 (7.7)&	77.2 (10.7)&	65.6 (6.6)&	78.4 (2.8)&	84.4 (4.4)&	81.5 (1.2)&	84.5 (1.0)&	77.2 (7.2)&	79.6 (1.4)&	60.7 (3.8)&	64.4 (0.1)&	72.0 (1.2)&	74.3 (0.06) \\
    PWAN  (Ours)	&54.4 (0.9)&	74.1 (1.0)	&58.7 (2.5)&	65.3 (0.8)&	81.4 (0.5)&	73.0 (0.6)&	78.0 (0.5)&	73.4 (0.8)&	70.8 (0.7)&	51.6 (2.5)&	55.1 (0.6)&	66.2 (3.3)&	66.8 (0.5) \\
    PWAN+A (Ours) & 71.1 (1.5)&	80.2 (1.8)&	66.8 (0.9)&	77.5 (0.8)&	84.0 (1.0)&	80.1 (0.6)&	83.6 (0.6)&	77.3 (0.2)&	75.2 (1.4)&	62.4 (2.7)&	66.2 (1.4)&	71.5 (1.6)&	\textbf{74.7 (0.6)} \\
    \hline
	  \end{tabular}
	\end{center}
\end{table*}

\begin{table}[tb!]
	\begin{center}
	  \caption{Results (accuracy \%) of partial domain adaptation on ImageNet-Caltech and VisDa17.}
    \vspace{-2mm}
	  \setlength{\tabcolsep}{2.0pt}
    \label{IC_V_tab}
      \begin{tabular}{c c c} 
		  \hline
           & ImageNet-Clatech & VisDa17 \\
		\hline
	  \centering
    Ref Only  & 70.6 & 60.0  \\
    DANN~\cite{ganin2016domain} & 71.4 & 57.1  \\
    PADA~\cite{cao2018partial} & 75 & 66.8 \\
    SAN++~\cite{cao2022big}& 83.34 & 63.0 \\
    BAUS~\cite{liang2020balanced} & 84 &  69.8 \\
    DPDAN~\cite{hu2020discriminative} & ---& 65.2 \\
    SHOT~\cite{liang2020we} & 83.8 (0.2) & --- \\
    ADV~\cite{gu2021adversarial} & 85.4 (0.2) & 80.1 (7.9) \\
    ARPM~\cite{gu2024adversarial} & 84.6 & 83.1 (0.5) \\
    PWAN  (Ours) & \textbf{86.0 (0.5)} & \textbf{84.8 (6.1)} \\
    \hline
	  \end{tabular}
	\end{center}
	\vspace{-4mm}
\end{table}

\begin{figure*}[tb!]
  \centering
  \includegraphics[width=0.3\linewidth]{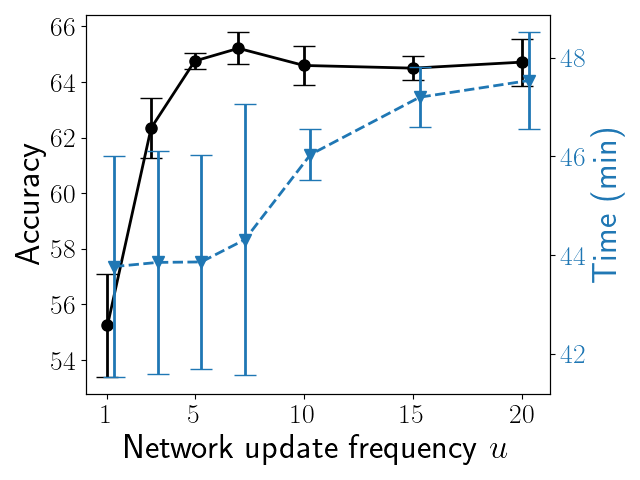}
  \includegraphics[width=0.3\linewidth]{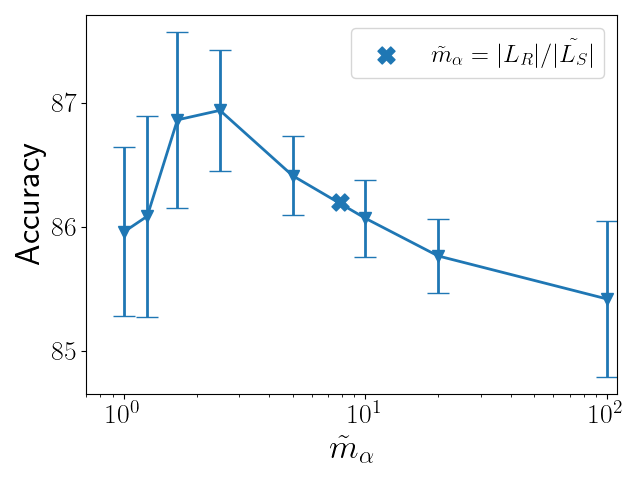}
  \includegraphics[width=0.3\linewidth]{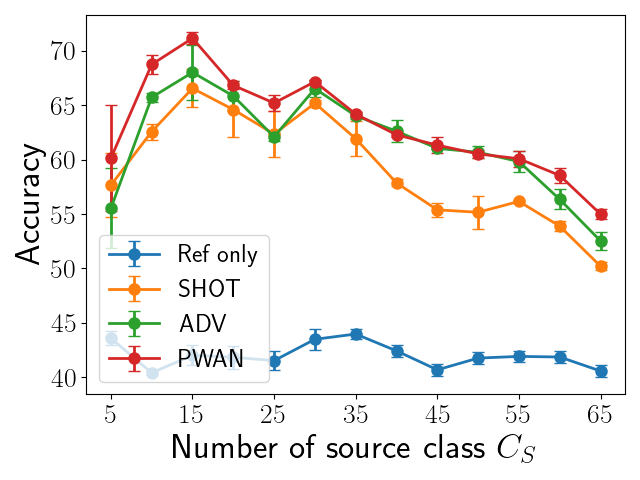}
\caption{Effect of parameter $m_\alpha$ and $u$.
Left: The performance and training time of PWAN as functions of $u$.
Middle: The performance of PWAN on ImageNet-Caltech with different values of $\tilde{m}_\alpha$.
Right: Comparison of performance on AC-OfficeHome with different ratios of outlier class.
}
\label{PDA_num_class}
\end{figure*}

The results on large-scale datasets VisDa17, ImageNet-Caltech and DomainNet are reported in Tab.~\ref{IC_V_tab} and Tab.~\ref{DomainNet_tab}.
On VisDa17,
PWAN outperforms all other baseline methods by a large margin.
As shown in Fig.~30 in the appendix,
PWAN can discriminate all data almost perfectly except for the ``knife'' and ``skateboard'' classes which are visually similar.
On ImageNet-Caltech,
PWAN also outperforms all other baseline methods,
which suggests that PWAN can successfully handle datasets dominated by outliers (more than $90\%$ of classes in ImageNet are outlier classes) without complicated sampling techniques~\cite{gu2021adversarial} or self-training procedures~\cite{cao2022big}.
On DomainNet,
PWAN with extra techniques (PWAN + A) performs slightly better than ARPM~\cite{gu2024adversarial},
and outperforms all other methods by a large margin.

\subsubsection{Effect of Parameters}
\label{Sec_PDA_analysis}

In this section,
we investigate the effects of the two major hyper-parameters involved in PWAN, 
\ie, $m_\alpha$ and $u$, 
so that they can be intuitively chosen in practical applications without hyper-parameter searching~\cite{you2019towards}.

Parameter $u$ is designed to control the tradeoff between the effectiveness and efficiency of PWAN.
To see the practical effects of $u$,
we run PWAN on AC-OfficeHome with varying $u$, 
and report the test accuracy and the corresponding training time in the left panel in Fig.~\ref{PDA_num_class}.
We observe that increasing $u$ leads to more accurate results and longer training time,
and the accuracy saturates when $u$ is larger than $10$.
Moreover,
as $u$ increases from $1$ to $20$,
the training time only increases by $10\%$.
This result suggests that the cost of training the potential network is small compared to training the classifier,
thus relatively large $u$ can be chosen to achieve better performance without causing a heavy computational burden.

Parameter $m_\alpha$~\eqref{m_alpha_pda} is designed to control the alignment ratio,
\ie,
larger $m_\alpha$ indicates that a lower ratio of reference feature will be aligned.
To investigate the effects of $m_\alpha$,
we run PWAN on ImageNet-Caltech with varying $\tilde{m}_\alpha$.
We update PWAN for $T=48000$ steps,
and we fix $s=\exp(10/T)$,
\ie, $m_\alpha=10\tilde{m}_\alpha$ at step $T$.
The results are reported in the middle panel of Fig.~\ref{PDA_num_class},
where we observe that overly large or small $\tilde{m}_\alpha$ leads to poor performance,
but the variation of performance is not large (about $2\%$) when $\tilde{m}_\alpha$ is in the range of $[10^0, 10^2]$.
This result suggests that the performance of PWAN is stable if a reasonable value of $\tilde{m}_\alpha$ is selected.

To further reduce the difficulty of selecting $\tilde{m}_\alpha$,
we suggested choosing $\tilde{m}_\alpha=|L_R| / \tilde{|L_S|}$ in Sec.~\ref{Sec_app_PDA}.
To see the effectiveness of this selection,
we compute the suggested $\tilde{m}_\alpha$ for ImageNet-Caltech and mark it in the middle panel of Fig.~\ref{PDA_num_class},
where we see that although the suggested value is not optimal,
it still leads to reasonable performance.
We further report the performance of PWAN using the suggested $\tilde{m}_\alpha$ on AC-OfficeHome with varying outlier ratio,
where we keep the first $C_S$ classes (in alphabetic order) in the source domain and vary $C_S$ from $5$ to $60$,
\ie,
the ratio of outlier class varies from $92\%$ to $0\%$.
The results are shown in the right panel of Fig.~\ref{PDA_num_class},
where we observe that 
the selected $\tilde{m}_\alpha$ can indeed handle different outlier ratios,
and it enables PWAN to outperform the baseline methods.

\section{Conclusion}
\label{Sec_conclusion}
In this work,
we propose PWAN for PDM tasks.
To this end,
we first derive the KR form and the gradient for the PW divergence.
Based on these results,
we approximate the PW divergence by a neural network and optimize it via gradient descent.
We apply PWAN to point sets registration and partial domain adaptation tasks,
and show that PWAN achieves favorable results on both tasks due to its strong robustness against outliers.

There are several issues that need further study.
First, 
when applied to point set registration,
the computation time of PWAN is still relatively high.
A promising approach to accelerate PWAN is to use a forward generator network as in~\cite{sarode2019pcrnet},
or to take an optimize-by-learning strategy~\cite{donti2021dc3}.
Second,
it is interesting to explore PWAN in other practical tasks,
such as outlier detection~\cite{xia2022gan},
and other types of domain adaptation tasks~\cite{kuhnke2019deep, li2021synthetic, panareda2017open}.

\ifCLASSOPTIONcompsoc
  \section*{Acknowledgments}
\else
  \section*{Acknowledgment}
\fi

This work was funded by the National Natural Science Foundation of China (NSFC) grants under contracts No. 62325111 and U22B2011,
and Knut and Alice Wallenberg Foundation.
The computation was partially enabled by resources provided by the National Academic Infrastructure for Supercomputing in Sweden (NAISS) at C3SE partially funded by the Swedish Research Council through grant agreement No. 2022-06725.



\bibliography{a}
\bibliographystyle{IEEEtran}

\clearpage

\clearpage
\appendices


\onecolumn
\section{Theoretical results}
We first derive the KR duality of $\mathcal{L}_{M,m}$ in Sec.~\ref{app_Sec_formulation},
and then characterize its potential in Sec.~\ref{app_sec_property_potential}.
We finally discuss the differentiability in Sec.~\ref{app_sec_diff_potential}.

\subsection{KR Duality of $\mathcal{L}_{M,m}$}
\label{app_Sec_formulation}
In this subsection,
we derive the KR formulation of $\mathcal{L}_{M,m}$.
We use the following notations:
\begin{itemize}[leftmargin=4mm]
    \item [-] $(\Omega, d)$: a metric $d$ associated with a compact metric space $\Omega$. 
    For example, $\Omega$ is a closed cubic in $\mathbb{R}^3$ in point set registration tasks,
    and $\Omega$ is a closed cubic in $\mathbb{R}^p$ in partial domain adaptation tasks,
    where $p$ is the dimension of features.
    \item [-] $C(\Omega)$: the set of continuous bounded functions defined on $\Omega$ equipped with the supreme norm.
    \item [-] $Lip(\Omega) \subseteq C(\Omega)$: the set of 1-Lipschitz function defined on $\Omega$.
    \item [-] $\mathcal{M}(X)$: the space of Radon measures on space $X$.
    \item [-] $\pi_\#^1$: The marginal of $\pi$ on its first variable. Similarly, $\pi_\#^2$ represents the marginal of $\pi$ on its second variable.
    \item [-] Given a function $\mathbf{F}$: $X \rightarrow \mathbb{R} \cup +\infty$, 
            the Fenchel conjugate of $\mathbf{F}$ is denoted as $\mathbf{F}^*$ and is given by: 
            \begin{equation}
                \mathbf{F}^*(x^*) = \sup_{x \in X} <x, x^*> - \mathbf{F}(x), \quad \forall x^* \in X^*
            \end{equation}
            where $X^*$ is the dual space of $X$ and $< \cdot >$ is the dual pairing.
\end{itemize}

Recall the definition of $\mathcal{L}_{M,m}$ and $\mathcal{L}_{D,h}$:
for $\alpha, \beta \in \mathcal{M}_+(\Omega)$,
\begin{equation}
        \label{app_POT_primal}
        \mathcal{L}_{M, m}(\alpha, \beta)=\inf_{\pi \in \Gamma_{m}(\alpha, \beta)} \int_{\Omega \times \Omega} c(x,y) d\pi(x,y),
\end{equation}
where $c: \Omega \times \Omega \rightarrow \mathbb{R}^+ $ is a continuous cost function,
and $\Gamma_{m}(\alpha, \beta)$ is the set of non-negative measure $\pi$ defined on $\Omega \times \Omega$ satisfying
\begin{equation}
    \pi(A \times \Omega) \leq \alpha(A), \quad \pi( \Omega \times A) \leq \beta(A) \quad and \quad \pi( \Omega \times \Omega) \geq m  \nonumber
\end{equation}
for all measurable set $A \subseteq \Omega$.
For ease of notations,
we abbreviate $m_\alpha=\alpha(\Omega)$,
$m_\beta=\beta(\Omega)$ and $m(\pi)=\pi( \Omega \times \Omega)$.
We also define
\begin{equation}
    \label{app_POT_primal_lag}
    \mathcal{L}_{D,h}(\alpha, \beta)=\inf_{ \pi \in \Gamma_0(\alpha, \beta) } \int_{\Omega \times \Omega} c(x,y) d\pi(x,y) - h m(\pi).
\end{equation}
where $h>0$ is the Lagrange multiplier.

We first derive the Fenchel-Rockafellar dual of $\mathcal{L}_{M,m}(\alpha, \beta)$.
\begin{proposition}[Dual form of $\mathcal{L}_{M,m}$]
    \label{app_dual_m}
    \eqref{app_POT_primal} can be equivalently expressed as 
    \begin{equation}
        \label{app_dual_m_eq}
        \mathcal{L}_{M,m}(\alpha, \beta)=\sup_{(\vf, \vg, h) \in \mathbf{R}} \int_\Omega \vf d\alpha + \int_\Omega \vg d\beta + mh.
    \end{equation}
    where the feasible set $\mathbf{R}$ is 
    \begin{equation}
        \label{app_admissible_R}
        \mathbf{R}=\Bigl\{ (\vf, \vg, h) \in C(\Omega) \times C(\Omega) \times \mathbb{R}_+ | \vf \leq 0,\; \vg \leq 0 ,\;  c(x,y)-h-\vf(x) - \vg(y) \geq 0, \forall x,y \in \Omega \Bigr\}
    \end{equation}
    In addition,
    the infimum in \eqref{app_POT_primal} is attained. 
\end{proposition}

\begin{proof}
    We prove this proposition via Fenchel-Rockafellar duality.
    We first define space $E$: $C(\Omega) \times C(\Omega) \times \mathbb{R}$,
    space $F$: $C(\Omega \times \Omega)$,
    and a linear operator $\mathcal{A}$: $E \rightarrow F$ as 
    \begin{equation}
        \mathcal{A}(\vf,\vg,h): (x, y) \rightarrow \vf(x) + \vg(y) +h; \quad \forall \vf, \vg \in C(\Omega), \; \forall h \in \mathbb{R}, \; \forall x, y \in \Omega.
    \end{equation}
    Then we introduce a convex function $\mathbf{H}$: $F \rightarrow \mathbb{R} \cup +\infty $ as
    \begin{equation}
        \mathbf{H}(u) =\begin{cases}
            0 &\; if \; u \geq -c \\
            +\infty &\; else
        \end{cases}
    \end{equation}
    and $\mathbf{L}$: $E \rightarrow \mathbb{R} \cup +\infty $ as
    \begin{equation}
        \mathbf{L}(\vf, \vg, h) =\begin{cases}
            \int \vf d\alpha + \int \vg d\beta +hm &\; if \; \vf \geq 0, \; \vg \geq 0,\; h \leq 0 \\
            +\infty &\; else
        \end{cases}
    \end{equation}
    We can check when $\vf \equiv \vg \equiv 1$ and $h=-1$,
    $\mathbf{H}$ is continuous at $\mathcal{A}(\vf,\vg,h)$.
    Thus by Fenchel-Rockafellar duality,
    we have 
    \begin{equation}
        \label{app_RK_dual}
        \inf_{(\vf,\vg,h) \in E} \mathbf{H}(\mathcal{A}(\vf,\vg,h)) + \mathbf{L}(\vf,\vg,h) = \sup_{\pi \in \mathcal{M}(\Omega \times \Omega)} -\mathbf{H}^*(-\pi) - \mathbf{L}^*(\mathcal{A}^*\pi)
    \end{equation}

    We first compute $\mathbf{H}^*(-\pi)$ and $\mathbf{L}^*(\mathcal{A}^*\pi)$ on the right-hand side of~\eqref{app_RK_dual}.
    For arbitrary $\pi \in \mathcal{M}(\Omega \times \Omega)$,
    we have
    \begin{align}
        &\mathbf{H}^*(-\pi) \nonumber \\
        & = \sup_{u\in F} \Bigl\{ \int (-u)d\pi - \mathbf{H}(u) \Bigr\} \nonumber \\
        & = \sup_{u\in F} \Bigl\{    \int (-u)d\pi| \; u(x,y) \geq -c(x, y), \; \forall (x,y) \in \Omega \times \Omega  \Bigr\} \nonumber \\
        & = \sup_{u\in F} \Bigl\{    \int u d\pi| \; u(x,y) \leq c(x, y), \; \forall (x,y) \in \Omega \times \Omega \Bigr\} \nonumber
    \end{align}
    It is easy to see that if $\pi$ is a non-negative measure,
    then this supremum is $\int c d\pi $,
    otherwise it is $+\infty$.
    Thus
    \begin{equation}
        \mathbf{H}^*(-\pi) = \begin{cases}
            \int c(x,y) d\pi(x,y) & if \; \pi \in \mathcal{M}_+(\Omega \times \Omega) \\
            + \infty & else
        \end{cases}
    \end{equation}
    Similarly,
    we have
    \begin{align}
        &\mathbf{L}^*(\mathcal{A}^*\pi)  \nonumber \\
        &=\sup_{(\vf,\vg,h) \in E} \Bigl\{ <(\vf,\vg,h), A^*\pi > -\mathbf{L}(\vf,\vg,h) \Bigr\} \nonumber \\
        &=\sup_{(\vf,\vg,h) \in E} \Bigl\{ <A(\vf,\vg,h), \pi > -(\int \vf d\alpha + \int \vg d\beta +hm) | \vf,\vg \geq 0, \; h\leq 0 \Bigr\} \nonumber \\
        &=\sup_{(\vf,\vg,h) \in E} \Bigl\{\int \vf d (\pi_\#^1-\alpha) + \int \vg d(\pi_\#^2-\beta) +h ( \pi(\Omega \times \Omega) -m ) | \vf,\vg \geq 0, \; h\leq 0 \Bigr\} \nonumber
    \end{align}
    If $(\alpha-\pi_\#^1)$ and $(\beta -\pi_\#^2)$ are non-negative measures,
    and $\pi(\Omega \times \Omega) -m \geq 0$, 
    this supremum is $0$,
    otherwise it is $+\infty$.
    Thus
    \begin{equation}
        \mathbf{L}^*(\mathcal{A}^*\pi) = \begin{cases}
            0 & if (\alpha-\pi_\#^1) \in \mathcal{M}_+(\Omega), \; (\beta-\pi_\#^2) \in \mathcal{M}_+(\Omega), \; \pi(\Omega \times \Omega) \geq m\\
            + \infty & else
        \end{cases}
    \end{equation}
    In addition,
    the left-hand side of~\eqref{app_RK_dual} reads
    \begin{align}
        &\inf_{(\vf,\vg,h) \in E} \mathbf{H(\mathcal{A}(\vf,\vg,h))} + \mathbf{L}(\vf,\vg,h) \nonumber \\
        &= \inf_{(\vf,\vg,h) \in E} \Bigl\{ \int \vf d\alpha + \int \vg d\beta +hm |\; \vf,\vg \geq 0, \; h \leq 0, \; \vf(x) + \vg(y) +h \geq -c(x,y), \forall x, y \in \Omega \Bigr\} \nonumber \\ 
        &= -\sup_{(\vf,\vg,h) \in E} \Bigl\{ \int \vf d\alpha + \int \vg d\beta +hm |\; \vf,\vg \leq 0, \; h \geq 0, \; \vf(x) + \vg(y) +h \leq c(x,y), \forall x, y \in \Omega \Bigr\} \nonumber 
    \end{align}
    Finally,
    by inserting these terms into~\eqref{app_RK_dual},
    we have
    \begin{equation}
        \sup_{\substack{(\vf,\vg,h) \in E \\ \vf,\vg \leq 0, h \geq 0 \\ \vf(x) + \vg(y) + h \leq c(x,y) \forall x, y \in \Omega }} \int \vf d\alpha + \int \vg d\beta +hm = \inf_{\substack{\pi \in \mathcal{M}_+(\Omega \times \Omega) \\ (\alpha-\pi_\#^1) \in \mathcal{M}_+(\Omega), (\beta-\pi_\#^2) \in \mathcal{M}_+(\Omega) \\ \pi(\Omega \times \Omega) \geq m}  } \int c(x,y) d \pi(x,y), \nonumber
    \end{equation}
    which proves~\eqref{app_dual_m_eq}.
    
    In addition,
    we can also check right-hand side of~\eqref{app_RK_dual} is finite,
    since we can always construct independent coupling $\widetilde{\pi}=\frac{m}{\alpha(\Omega) \beta(\Omega)} \alpha \otimes  \beta $,
    such that $\widetilde{\pi} \in \mathcal{M}_+(\Omega \times \Omega)$,
    $\widetilde{\pi}_\#^1 = \frac{m}{\alpha(\Omega)} \alpha \leq \alpha$,
    $\widetilde{\pi}_\#^2 = \frac{m}{\beta(\Omega)} \beta \leq \beta $
    and $\widetilde{\pi}(\Omega \times \Omega) = m$.
    Thus the Fenchel-Rockafellar duality suggests the infimum is attained.
\end{proof}

Similarly,
we can derive the Fenchel-Rockafellar dual form of $\mathcal{L}_{D,h}$.
\begin{proposition}[Dual form of $\mathcal{L}_{D,h}$]
    \label{app_dual_h}
    \eqref{app_POT_primal_lag} can be equivalently expressed as
    \begin{equation}
        \label{app_dual_h_eq}
        \mathcal{L}_{D,h}(\alpha, \beta)=\sup_{(\vf, \vg) \in \mathbf{R}(h) } \int_\Omega \vf d\alpha + \int_\Omega \vg d\beta.
    \end{equation}
    where the feasible set is 
    \begin{equation}
        \label{app_admissible_Rh}
        \mathbf{R}(h)=\Bigl\{ (\vf, \vg) \in C(\Omega) \times C(\Omega) | \vf \leq 0,\; \vg \leq 0 ,\;  c(x,y)-h-\vf(x) - \vg(y) \geq 0, \forall x,y \in \Omega \Bigr\}
    \end{equation}
    In addition,
    the infimum in \eqref{app_POT_primal_lag} is attained. 
\end{proposition}

By comparing Proposition~\ref{app_dual_h} with Proposition~\ref{app_dual_m},
we can see that $\mathcal{L}_{D,h}$ and $\mathcal{L}_{M,m}$ are related by
\begin{align}
    \label{app_L-M}
    \mathcal{L}_{M,m} &= \sup_{(\vf, \vg, h) \in \mathbf{R}} \int_\Omega \vf d\alpha + \int_\Omega \vg d\beta + mh \nonumber \\
    & =\sup_{h\in \mathbb{R}_+}   \mathcal{L}_{D,h} + mh.
\end{align}
Therefore,
when the cost function $c$ is the distance $\vd$,
we obtain the KR form of $\mathcal{L}_{M,m}$ by inserting~\eqref{KR} into~\eqref{app_L-M}.
\begin{proposition}[KR form of $\mathcal{L}_{M,m}$]
    \label{app_m-PW}
    When $c(x, y) = \vd(x, y)$,
    \eqref{app_POT_primal} can be reformulated as 
    \begin{equation}
        \label{app_final_m_eq1}
        \mathcal{L}_{M, m}(\alpha, \beta)=\sup_{\substack{\vf \in Lip(\Omega), h \in \mathbf{R}_+ \\ -h \leq \vf \leq 0 }} \int_\Omega \vf d\alpha - \int_\Omega \vf d\beta +h(m-m_{\beta})
    \end{equation}
    or equivalently as 
    \begin{equation}
        \label{app_final_m_eq2}
        \mathcal{L}_{M, m}(\alpha, \beta)=\sup_{\vf \in Lip(\Omega), \vf \leq 0 } \int_\Omega \vf d\alpha - \int_\Omega \vf d\beta - \inf(\vf)(m-m_{\beta}).
    \end{equation}
\end{proposition}

\begin{proof}
For equation~\eqref{app_final_m_eq2},
note that given an $\vf \in C(\Omega) $,
the optimal $h$ is simply $-\inf(\vf) < +\infty$ ($\vf$ is continuous and $\Omega$ is compact).
So we can replace $h$ by $-\inf(\vf)$ in~\eqref{app_final_m_eq1} to obtain~\eqref{app_final_m_eq2}.
\end{proof}

For simplicity,
we define some functionals associated with $\mathcal{L}_{M,m}$ and $\mathcal{L}_{D,h}$.
\begin{definition}[$\mathbf{L}_{M,m}$ and $\overline{\mathbf{L}_{M,m}}$]
    Define 
\begin{equation}
    \label{app_Functional_M}
    \mathbf{L}_{M,m}^{\alpha, \beta}(\vf, h)=\begin{cases}
        \int_\Omega \vf d(\alpha - \beta) + h(m-m_{\beta}) & h \in \mathbb{R}_+ \; , \; \vf \in Lip(\Omega) \; ,\; -h \leq \vf \leq 0 \nonumber \\
        -\infty & else,
    \end{cases}
\end{equation}
and 
\begin{equation}
    \label{app_Functional_M2}
    \overline{\mathbf{L}_{M,m}^{\alpha, \beta}}(\vf)=\begin{cases}
        \int_\Omega \vf d(\alpha - \beta) - \inf(\vf) (m-m_{\beta}) & \vf \in Lip(\Omega) \; ,\; \vf \leq 0  \nonumber \\
        -\infty & else.
    \end{cases}
\end{equation}
\end{definition}
\begin{definition}[$\mathbf{L}_{D,h}$]
    \begin{equation}
        \mathbf{L}_{D,h}^{\alpha, \beta}(\vf)=\begin{cases}
            \int_{\Omega} \vf d\alpha - \int_\Omega \vf d\beta - h m_{\beta} & \vf \in Lip(\Omega) \;, \; -h \leq \vf \leq 0, \nonumber \\
            -\infty & else.
        \end{cases}
    \end{equation}
\end{definition}
We also define the functional associated with Wasserstein-1 metric $\mathcal{W}_1$:
\begin{definition}[$\mathbf{W}_{1}$]
\begin{equation}
    \label{app_Functional_W}
    \mathbf{W}^{\alpha, \beta}(\vf)=\begin{cases}
        \int_\Omega \vf d\alpha - \int_\Omega \vf d\beta & \vf \in Lip(\Omega), \; m_\alpha = m_\beta\\
        -\infty & else,
    \end{cases}
\end{equation}
\end{definition}

With these definitions,
we can write
\begin{equation}
    \mathcal{L}_{D,h}(\alpha, \beta)=\sup_{\vf \in C(\Omega)}\mathbf{L}_{D,h}^{\alpha, \beta}(\vf), \nonumber
\end{equation}
\begin{equation}
    \mathcal{L}_{M,m}(\alpha, \beta)=\sup_{\vf \in C(\Omega), h\in \mathbb{R}}\mathbf{L}_{M,m}^{\alpha, \beta}(\vf, h) \;\; and \;\; \mathcal{L}_{M,m}(\alpha, \beta)=\sup_{\vf \in C(\Omega)}\overline{\mathbf{L}_{M,m}^{\alpha, \beta}}(\vf),  \;  \nonumber
\end{equation}
and 
\begin{equation}
    \mathcal{W}_1(\alpha, \beta)=\sup_{\vf \in C(\Omega)} \mathbf{W}^{\alpha, \beta}(\vf). \nonumber
\end{equation}
For simplicity, we omit the notation of $\alpha$ and $\beta$ where they are clear from the text.

\subsection{Properties of KR Potentials}
\label{app_sec_property_potential}
Before we can discuss the property of the KR potential,
\ie, the maximizer of $\mathbf{L}_{M,m}$ and $\mathbf{L}_{D,h}$,
we first need to show that the potentials exist.
As for $\mathbf{L}_{D,h}$,
the existence of potentials is already known in~\cite{schmitzer2019framework}:
\begin{proposition}[\cite{schmitzer2019framework}]
    \label{app_Existence_D}
    For $\alpha, \beta \in \mathcal{M}_+(\Omega)$ and $h >0$,
    there exists an $\vf^*\in C(\Omega)$ such that $\mathcal{L}_{D,h}(\alpha, \beta) = \mathbf{L}_{D,h}^{\alpha, \beta}(\vf^*)$.
\end{proposition}
As for $\mathbf{L}_{M,m}$,
we can also prove the existence of the maximizer similarly:
\begin{proposition}
    \label{app_Existence_M}
    For $\alpha, \beta \in \mathcal{M}_+(\Omega)$ and $m > 0$,
    there exist $\vf^*\in C(\Omega)$ and $h \in \mathbb{R}$ such that $\mathcal{L}_{M,m}(\alpha, \beta) = \mathbf{L}_{M,m}^{\alpha, \beta}(\vf^*, h)$.
\end{proposition}

To prove this proposition,
we need the following lemma.
\begin{lemma}[Continuity]
    \label{app_Continuous}
    $\mathbf{L}_{M,m}^{\alpha, \beta}$ is continuous on $C(\Omega) \times \mathbb{R}$.
\end{lemma}
\begin{proof}
    Let $(\vf_n, h_n) \rightarrow (\vf, h)$ in $C(\Omega) \times \mathbb{R}$.
    Assume $\mathbf{L}_{M,m}^{\alpha, \beta}(\vf_n, h_n) > -\infty$ when $n$ is sufficiently large.
    We first check $\mathbf{L}_{M,m}^{\alpha, \beta}(\vf, h) > -\infty$ as follows.
    For arbitrary $\epsilon > 0$,
    there exists $N>0$ such that for $n > N$, 
    $h_n < h + \epsilon$,
    thus $\vf_n > - h_n > -h - \epsilon$. 
    By taking $n \rightarrow \infty$,
    we see for arbitrary $\epsilon >0$, 
    $\vf > -h - \epsilon$,
    which suggests $\vf \geq -h$.
    In addition,
    it is easy to see $\vf \leq 0$ and $h \geq 0$.
    It is also easy to see $Lip(\vf) \leq 1$ due to the closeness of $Lip(\Omega)$.
    Thus according to definition~\ref{app_Functional_M},
    we claim $\mathbf{L}_{M,m}^{\alpha, \beta}(\vf, h) > -\infty$.
    Furthermore,
    since $-h -\epsilon < \vf_n <0$,
    and $\vf_n \rightarrow \vf$,
    by dominated convergence theorem,
    we have
    \begin{align}
        & \lim_{n \rightarrow \infty} \int_\Omega \vf_n d\alpha = \int_\Omega \lim_{n \rightarrow \infty} \vf_n d\alpha = \int_\Omega \vf d\alpha, \\
        and & \lim_{n \rightarrow \infty} \int_\Omega \vf_n d\beta = \int_\Omega \lim_{n \rightarrow \infty} \vf_n d\beta = \int_\Omega \vf d\beta.
    \end{align}
    Note we also have $h_n(m-m_{\beta}) \rightarrow h(m-m_{\beta})$.
    We conclude the proof by combining these three terms and obtaining $\mathbf{L}_{M,m}^{\alpha, \beta}(\vf_n, h_n) \rightarrow \mathbf{L}_{M,m}^{\alpha, \beta}(\vf, h)$.
\end{proof}

\begin{proof}[Proof of Proposition~\ref{app_Existence_M}]
    If we can find a maximizing sequence $(\vf_n, h_n)$ that converges to $(\vf, h) \in C(\Omega) \times \mathbb{R}$,
    then Lemma~\ref{app_Continuous} suggests that $\mathcal{L}_{M,m}(\alpha, \beta) = \sup_{\vf, h} \mathbf{L}_{M,m}^{\alpha, \beta}(\vf, h)= \lim_{n \rightarrow \infty} \mathbf{L}_{M,m}^{\alpha, \beta}(\vf_n, h_n) = \mathbf{L}_{M,m}^{\alpha, \beta}(\vf, h)$,
    which proves this proposition.
    Therefore,
    we only need to show that it is always possible to construct such a maximizing sequence.
    
    Let $(\vf_n, h_n)$ be a maximizing sequence.
    We abbreviate $\max(\vf_n)=\max_{x\in \Omega}(\vf_n(x))$
    and $\min(\vf_n) =\min_{x\in \Omega}(\vf_n(x))$.
    We first assume $\vf_n$ does not have any bounded subsequence,
    then there exists $N>0$,
    such that for all $n>N$, 
    $\min(\vf_n) < -diam(\Omega)$ (otherwise we can simply collect a subsequence of $\vf_n$ bounded by $diam(\Omega)$).
    We can therefore construct $\widetilde{\vf_n} = \vf_n - (\min(\vf_n) + diam(\Omega))$ and $\widetilde{h_n} = h_n + (\min(\vf_n) + diam(\Omega))$.
    Note that $\max(\widetilde{\vf_n}) \leq \min(\widetilde{\vf_n}) + diam(\Omega) = - diam(\Omega) + diam(\Omega) = 0$,
    $\widetilde{\vf_n} + \widetilde{h_n} = \vf_n + h_n \geq 0$,
    and $\widetilde{\vf_n} \in Lip(\Omega)$,
    so $\mathbf{L}_{M,m}^{\alpha, \beta}(\widetilde{\vf_n}, \widetilde{h_n}) > -\infty$,
    and 
    \begin{align}
        \mathbf{L}_{M,m}^{\alpha, \beta}(\widetilde{\vf_n}, \widetilde{h_n})  &= \int_\Omega \widetilde{\vf_n} d(\alpha - \beta) + \widetilde{h_n}(m-m_{\beta}) \nonumber \\
        & = \int_\Omega \vf_n d(\alpha - \beta) + h_n(m-m_{\beta}) + (\min(\vf_n) - diam(\Omega)) (m - m_\alpha) \nonumber \\
        & \geq \mathbf{L}_{M,m}^{\alpha, \beta}(\vf_n, h_n), \nonumber
    \end{align}
    which suggests that $(\widetilde{\vf_n}, \widetilde{h_n})$ is a better maximizing sequence than $(\vf_n, h_n)$.
    Note that $\widetilde{\vf_n}$ is uniformly bounded by $diam(\Omega)$ because $0 \geq \widetilde{\vf_n} \geq \min(\widetilde{\vf_n})=-diam(\Omega)$.
    As a result,
    we can always assume $\widetilde{h_n}$ is also bounded by $diam(\Omega)$.
    Because otherwise we can construct $\overline{h_n}  = -\min(\widetilde{\vf_n}) \leq diam(\Omega)$,
    and it is easy to show $(\widetilde{\vf_n}, \overline{h_n})$ is a better maximizing sequence than $(\widetilde{\vf_n}, \widetilde{h_n})$.
    In summary,
    we can always find a maximizing sequence  $(\vf_n, h_n)$,
    such that both $\vf_n$ and $h_n$ are bounded by $diam(\Omega)$.

    Finally,
    since $\vf_n$ is uniformly bounded and equicontinuous, 
    $\vf_n$ converges uniformly (up to a subsequence) to a continuous function $\vf$.
    In addition,
    $h_n$ has a convergent subsequence since it is bounded.
    Therefore,
    we can always find a maximizing sequence $(\vf_n, h_n)$ that converges to some $(\vf, h) \in C(\Omega) \times \mathbb{R}$,
    which finishes the proof.
\end{proof}

\textbf{Remark} The proof of Proposition~\ref{app_Existence_M} is an analogue of Proposition 2.11 in~\cite{schmitzer2019framework}.
The difference is that in Proposition~\ref{app_Existence_M}, 
besides $\vf$, 
we additionally need to handle another variable $h$ acting as the lower bound of $\vf$.

Now we proceed to the main result in this subsection,
which states that the potential is $0$ or $-h$ on the omitted mass,
thus it has $0$ gradient on this mass.
This qualitative description reveals an interesting connection between our algorithm and WGAN~\cite{pmlr-v70-arjovsky17a}.

First,
we note that,
the potential of $\mathcal{W}_1$ can be characterized as follows.
\begin{lemma}[Potential of $\mathcal{W}_1$~\cite{gulrajani2017improved}]
    \label{app_potential_w1}
    Let $\vf^*$ be a maximizer of $\mathbf{W}_1^{\alpha, \beta}$. 
    If $\vf^*$ is differentiable, 
    and there exists a primal solution $\pi$ satisfying $\pi(x=y)=0$,
    then $\vf^*$ has gradient norm $1$ $(\alpha+\beta)$-almost surely.
\end{lemma}

Then we characterize the flatness of the potential of $\mathcal{L}_{D,h}$ in the following two lemmas.
\begin{lemma}
    \label{unimportant_M}
    Let $\pi$ be the solution to the primal form of $\mathcal{L}_{D,h}(\alpha, \beta)$.
    Assume $\alpha', \beta' \in \mathcal{M}(\Omega)$ satisfy
    \begin{equation}
        \pi_{\#}^1 \leq \alpha' \leq \alpha \quad and \quad \pi_{\#}^2 \leq \beta' \leq \beta.
    \end{equation}

    (1) $\pi$ is also the solution to the primal form of $\mathcal{L}_{D,h}(\alpha', \beta')$ and $\mathcal{W}_1(\pi_{\#}^1, \pi_{\#}^2)$,
    thus 
    \begin{equation}
        \mathcal{L}_{D,h}(\alpha, \beta)=\mathcal{L}_{D,h}(\alpha', \beta')=\mathcal{W}_1(\pi_{\#}^1, \pi_{\#}^2) - hm(\pi)
    \end{equation}

    (2) Let $\vf^*$ be a maximizer of $\mathbf{L}_{D,h}^{\alpha, \beta}$.
    Then $\vf^*$ is also a maximizer of $\mathbf{L}_{D,h}^{\alpha', \beta'}$ and $\mathbf{W}_{1}^{\pi_{\#}^1, \pi_{\#}^2}$.
\end{lemma}

\begin{proof}
    (1)
    First, it is easy to verify that $\pi$ is indeed an admissible solution to the primal form of $\mathcal{L}_{D,h}(\alpha', \beta')$.
    Then,
    notice that all the admissible solutions to the primal form of $\mathcal{L}_{D,h}(\alpha', \beta')$ are also admissible solutions to the primal form of $\mathcal{L}_{D,h}(\alpha, \beta)$.
    Finally, 
    we can conclude by contradiction:
    If there exists a better solution than $\pi$ for $\mathcal{L}_{D,h}(\alpha', \beta')$,
    then it is also better than $\pi$ for $\mathcal{L}_{D,h}(\alpha, \beta)$,
    which contradicts the optimality of $\pi$.
    Similarly,
    we can prove that $\pi$ is also the solution to $\mathcal{W}_1(\pi_{\#}^1, \pi_{\#}^2)$.
    
    (2)
    Note we have
    \begin{equation}
        \label{app_tmp_inactive}
        \mathbf{L}_{D,h}^{\alpha', \beta'}(\vf^*) = \int_\Omega \vf^* d \alpha' + \int_\Omega (-h - \vf^* ) d\beta' 
         \geq \int_\Omega \vf^* d \alpha + \int_\Omega (-h - \vf^*) d \beta 
         = \mathbf{L}_{D,h}^{\alpha, \beta}(\vf^*) = \mathcal{L}_{D,h}(\alpha, \beta),
    \end{equation}
    where the inequality holds because $\vf^* \leq 0$, $-h - \vf^* \leq 0 $, $\alpha' \leq \alpha$ and $\beta' \leq \beta$.
    According to the first part of this proof,
    we have $\mathcal{L}_{D,h}(\alpha', \beta')=\mathcal{L}_{D,h}(\alpha, \beta)$,
    thus $\mathbf{L}_{D,h}^{\alpha', \beta'}(\vf^*) \leq \mathcal{L}_{D,h}(\alpha', \beta')=\mathcal{L}_{D,h}(\alpha, \beta)$.
    By combining these two equalities,
    we conclude that $\mathbf{L}_{D,h}^{\alpha', \beta'}(\vf^*) = \mathcal{L}_{D,h}(\alpha', \beta')$,
    \ie, $\vf^*$ is a maximizer of $\mathbf{L}_{D,h}^{\alpha', \beta'}$.
    
    In addition,
    we have 
    \begin{equation}
        \mathbf{W}_{1}^{\pi_{\#}^1, \pi_{\#}^2}(\vf^*) -hm(\pi) = \int_\Omega \vf^* d \pi_{\#}^1 - \int_\Omega \vf^* d\pi_{\#}^2 -hm(\pi_{\#}^2) = \mathbf{L}_{D,h}^{\pi_{\#}^1, \pi_{\#}^2}(\vf^*). \nonumber \\
    \end{equation}
    Due to the first part of (2),
    $\vf^*$ is a maximizer of $\mathbf{L}_{D,h}^{\pi_{\#}^1, \pi_{\#}^2}$,
    thus $ \mathbf{L}_{D,h}^{\pi_{\#}^1, \pi_{\#}^2}(\vf^*) = \mathcal{L}_{D,h}(\pi_{\#}^1, \pi_{\#}^2)$.
    In addition,
    $\mathcal{L}_{D,h}(\pi_{\#}^1, \pi_{\#}^2) = \mathcal{W}_1(\pi_{\#}^1, \pi_{\#}^2) - hm(\pi)$ due to the first part of this lemma.
    Therefore we have $\mathbf{W}_{1}^{\pi_{\#}^1, \pi_{\#}^2}(\vf^*) -hm(\pi) =\mathcal{W}_1(\pi_{\#}^1, \pi_{\#}^2) - hm(\pi)$,
    which implies that $\vf^*$ is a maximizer of $\mathbf{W}_{1}^{\pi_{\#}^1, \pi_{\#}^2}$.
\end{proof}

\begin{lemma}[Flatness on omitted mass]
    \label{Flat_M}
    Let $\vf^*$ be a maximizer of $\mathbf{L}_{D,h}^{\alpha, \beta}$,
    and $\pi$ be the solution to the primal form of $\mathcal{L}_{D,h}^{\alpha, \beta}$.
    Assume that $\alpha_\mathcal{S}, \beta_\mathcal{S} \in \mathcal{M}_+(\Omega)$ satisfy
    \begin{equation}
        \pi_{\#}^1 \leq \alpha - \alpha_{\mathcal{S}} \leq \alpha \quad and \quad \pi_{\#}^2 \leq \beta - \beta_{\mathcal{S}} \leq \beta. \nonumber
    \end{equation}
    Then $\vf^*=0$ $\alpha_\mathcal{S}$-almost surely,
    and $\vf^*=-h$ $\beta_\mathcal{S}$-almost surely.
\end{lemma}
\begin{proof}
    According to Lemma~\ref{unimportant_M},
    we have $\mathcal{L}_{D,h}^{\alpha - \alpha_{\mathcal{S}}, \beta} = \mathcal{L}_{D,h}^{\alpha, \beta}$,
    and $\vf^*$ is a maximizer of $\mathbf{L}_{D,h}^{\alpha - \alpha_{\mathcal{S}}, \beta}$,
    \ie, we have
    \begin{equation}
      \mathcal{L}_{D,h}^{\alpha, \beta} = \int_\Omega \vf^* d\alpha - \int_\Omega \vf^* \beta -h m_{\beta} = \int_\Omega \vf^* d(\alpha - \alpha_{\mathcal{S}}) - \int_\Omega \vf^* \beta - hm_{\beta} = \mathcal{L}_{D,h}^{\alpha - \alpha_{\mathcal{S}}, \beta}. \nonumber
    \end{equation}
    By cleaning this equation,
    we obtain $\int_\Omega \vf^* d\alpha_{\mathcal{S}} = 0$.
    Since $\vf^* \leq 0$ and $\alpha_{\mathcal{S}}$ is a non-negative measure,
    we conclude that $\vf^*=0$ $\alpha_{\mathcal{S}}$-almost surely.
    The statement for $\beta_{\mathcal{S}}$ can be proved similarly.
\end{proof}

Finally,
we can present a qualitative description of the potential of $\mathcal{L}_{D,h}$ by decomposing the mass into transported and omitted mass:

\begin{proposition}
    \label{qualitative_des_D}
    Let $\vf^*$ be a maximizer of $\mathbf{L}_{D,h}(\alpha, \beta)$,
    and assume that there exists a primal solution $\pi$ satisfying $\pi(x=y)=0$.
    Then there exist non-negative measures $\mu$,  $\nu_\alpha \leq \alpha$ and $\nu_\beta \leq \beta$ satisfying $\mu + \nu_\alpha + \nu_\beta = \alpha + \beta$,
    such that 1) $\vf^*$ has gradient norm $1$ $\mu$-almost surely,
    2) $\vf^*=0$ $\nu_\alpha$-almost surely,
    and 3) $\vf^*=-h$ $\nu_\beta$-almost surely.
\end{proposition}

\begin{proof}
  Let $\nu_\alpha = \alpha - \pi_{\#}^1$, where $\pi$ is the solution to the primal form of $\mathcal{L}_{D,h}^{\alpha, \beta}$.
  It is easy to verify $\pi_{\#}^1 \leq \alpha - \nu_\alpha \leq \alpha$,
  thus Lemma~\ref{Flat_M} suggests that $\vf^*=0$ $\nu_\alpha$-almost surely.
  Similarly,
  let $\nu_\beta = \beta - \pi_{\#}^2$, then $\vf^*=-h$ $\nu_\beta$-almost surely.
  According to Lemma~\ref{unimportant_M},
  $\vf^*$ is a maximizer of $\mathbf{W}_{1}^{\pi_{\#}^1, \pi_{\#}^2}$,
  then Lemma~\ref{app_potential_w1} immediately suggests that $\vf^*$ has gradient norm $1$ $(\pi_{\#}^2 + \pi_{\#}^1)$-almost surely.
  We finish the proof by letting $\mu=\pi_{\#}^2 + \pi_{\#}^1$.
\end{proof}

Since $0$ and $-h$ are the upper and lower bounds of $\vf^*$ respectively,
by defining $\nu=\nu_\beta+\nu_\alpha$,
we immediately have the following corollary:
\begin{corollary}
  \label{qualitative_des_D_2}
  Let $\vf^*$ be a maximizer of $\mathbf{L}_{D,h}(\alpha, \beta)$,
  and assume that there exists a primal solution $\pi$ satisfying $\pi(x=y)=0$.
  If $\vf^*$ is differentiable,
  then there exist non-negative measures $\mu$,  $\nu$ satisfying $\mu + \nu = \alpha + \beta$,
  such that 1) $||\nabla \vf^*||=1$ $\mu$-almost surely,
  2) $\nabla \vf^*=0$ $\nu $-almost surely.
\end{corollary}

Finally,
we note that a similar statement holds for $\mathbf{L}_{M,m}$,
because the potential of $\mathbf{L}_{M,m}$ can be recovered from $\mathbf{L}_{D,h}$:
\begin{lemma}[Relations between $\mathcal{L}_{M}$ and $\mathcal{L}_{D}$]
    \label{app_Relation-m-d}
    Let $\vf^*$ be a maximizer of $\overline{\mathbf{L}_{M,m}^{\alpha, \beta}}$.
    For the fixed $h^*=-\inf(\vf^*)$,
    $\vf^*$ is also a maximizer of 
    $\mathbf{L}_{D,h^*}^{\alpha, \beta}$.
\end{lemma}

\begin{corollary}
    \label{qualitative_des_M_2}
    Let $\vf^*$ be a maximizer of $\mathbf{L}_{M,m}(\alpha, \beta)$.
    and assume that there exists a primal solution $\pi$ satisfying $\pi(x=y)=0$.
    If $\vf^*$ is differentiable,
    then there exist non-negative measures $\mu$,  $\nu$ satisfying $\mu + \nu_\alpha + \nu_\beta = \alpha + \beta$,
    such that 1) $||\nabla \vf^*||=1$ $\mu$-almost surely,
    2) $\vf^*=0$ $\nu_\alpha$-almost surely,
    and $\vf^*=-h$ $\nu_\beta$-almost surely.
    \ie, $\nabla \vf^*=0$ $\nu_\alpha + \nu_\beta$-almost surely.
\end{corollary}

\subsection{Differentiability}
  \label{app_sec_diff_potential}

We now consider the differentiability of the KR forms.
We have the following two propositions.
\begin{assumption}
    \label{assumption_1}
    For every $\tilde{\theta}$,
    there exist a neighborhood $U$ of $\tilde{\theta}$ and a constant $\Delta(\tilde{\theta})$,
    such that for $\theta, \theta' \in U$ and $y_1, y_2 \in \Omega'$,
    $||\mathcal{T}_\theta(y_1) - \mathcal{T}_{\theta'}(y_2)|| \leq \Delta(\tilde{\theta}) (||\theta - \theta'|| + ||y_2 - y_1|| )$.
\end{assumption}

\begin{proposition}[Differentiability of $\mathcal{L}_{M,m}$]
  \label{app_Smooth_M_con}
  If $\mathcal{T}_\theta$ satisfies assumption~\ref{assumption_1},
  then $\mathcal{L}_{M,m}(\alpha, \beta_\theta ) $ is continuous \wrt $\theta$,
  and is differentiable almost everywhere.
  Furthermore,
  we have
  \begin{equation}
      \nabla_\theta \mathcal{L}_{M,m}(\alpha, \beta_\theta ) = - \int_\Omega \nabla_\theta \vf^*(\mathcal{T}_\theta(x)) d\beta,
  \end{equation}
  when both sides are well defined.
\end{proposition}

\begin{proof}
For every $\tilde{\theta} \in \mathbb{R}^p$,
select a neighborhood $U$ of $\tilde{\theta}$ and a constant $\Delta(\tilde{\theta})$ according to assumption~\ref{assumption_1},
such that for all $\theta, \theta' \in U$ and $y, y' \in \Omega'$,
$||\mathcal{T}_\theta(y) - \mathcal{T}_{\theta'}(y')|| \leq \Delta(\tilde{\theta}) (||\theta - \theta'|| + ||y - y'|| )$.
By letting $y=y'$,
we have 
\begin{equation}
    \label{common_y}
    ||\mathcal{T}_\theta(y) - \mathcal{T}_{\theta'}(y)|| \leq \Delta(\tilde{\theta}) (||\theta - \theta'|| ).
\end{equation}

    Now we consider the transportation between $\alpha$, $\beta_{\theta}$ and $\beta_{\theta'}$.
    Let $\pi$ be the solution to the primal form of $\mathcal{L}_{M,m}(\alpha, \beta_\theta)$,
  and $\pi'$ be the solution to the primal form of $\mathcal{L}_{M,m}(\alpha, \beta_{\theta'})$.
  Let $\beta_1=\pi_{\#}^2 \leq \beta_\theta$,
  $\alpha_2=\pi_{\#}^{'1} \leq \alpha$ and $\beta_2=\pi_{\#}^{'2} \leq \beta_{\theta'}$.
  Since $\mathcal{W}_1$ is a metric, 
  by triangle inequality,
  we have 
  \begin{equation}
    \label{proof_m}
    \mathcal{W}_1(\alpha_2, \beta_1) \leq \mathcal{W}_1(\alpha_2, \beta_2) + \mathcal{W}_1(\beta_2, \beta_1)
  \end{equation}
  According to Lemma~\ref{unimportant_M},
  we have $\mathcal{W}_1(\alpha_2,\beta_2)=\mathcal{L}_{M,m}(\alpha, \beta_{\theta'})$.
  By the optimality of $\mathcal{L}_{M,m}(\alpha, \beta_{\theta})$,
  we have $\mathcal{L}_{M,m}(\alpha, \beta_{\theta}) \leq \mathcal{W}_1(\alpha_2, \beta_1)$.
  In addition,
  we have 
  \begin{equation}
    \mathcal{W}_1(\beta_2, \beta_1) \leq \mathcal{W}_1(\beta_\theta, \beta_{\theta'}) \leq \int_\Omega \vd(\mathcal{T}_{\theta}(y) - \mathcal{T}_{\theta'}(y)) d\beta(y) \leq  m_\beta \Delta(\tilde{\theta}) (||\theta - \theta'|| ), \nonumber
  \end{equation}
  where the first inequality holds because $\beta_2 \leq \beta_{\theta'}$ and $\beta_1 \leq \beta_{\theta}$,
  and the third inequality holds because of~\eqref{common_y}.
  By inserting these inequalities and equations back into~\eqref{proof_m},
  we obtain
  \begin{equation}
    \mathcal{L}_{M,m}(\alpha, \beta_{\theta}) \leq \mathcal{L}_{M,m}(\alpha, \beta_{\theta'}) + m_\beta \Delta ||\theta - \theta' ||, \nonumber
  \end{equation}
  which proves $\mathcal{L}_{D,h}(\alpha, \beta_\theta)$ is locally Lipschitz \wrt $\theta$,
  therefore Radamacher's theorem states that it is differentiable almost everywhere.

  Finally,
  since the maximizer of the KR form of $\mathcal{L}_{M,m}(\alpha, \beta_\theta)$ exists according to Proposition~\ref{app_Existence_M},
  following the envelope theorem and the arguments in~\cite{pmlr-v70-arjovsky17a},
  we can conclude
  \begin{equation}
    \nabla_\theta \mathcal{L}_{M,m}(\alpha, \beta_\theta ) = -\nabla_\theta \int_\Omega \vf^*(\mathcal{T}_\theta(x)) d\beta = -\int_\Omega \nabla_\theta  \vf^*(\mathcal{T}_\theta(x)) d\beta, \nonumber
  \end{equation}
  when both sides of the equation are well-defined.
\end{proof}

\begin{proposition}[Differentiability of $\mathcal{L}_{D,h}$]
    \label{app_Smooth_D_con}
    If $\mathcal{T}_\theta$ satisfies assumption~\ref{assumption_1},
    then $\mathcal{L}_{D,h}(\alpha, \beta_\theta ) $ is continuous \wrt $\theta$,
    and is differentiable almost everywhere.
    Furthermore,
    we have
    \begin{equation}
        \nabla_\theta \mathcal{L}_{D,h}(\alpha, \beta_\theta ) = - \int_\Omega \nabla_\theta \vf^*(\mathcal{T}_\theta(x)) d\beta,
    \end{equation}
    when both sides are well defined.
\end{proposition}

\begin{proof}
    Similar to the proof of Proposition~\ref{app_Smooth_M_con},
    for every $\tilde{\theta} \in \mathbb{R}^p$,
    we select a neighborhood $U$ of $\tilde{\theta}$ and a constant $\Delta(\tilde{\theta})$, 
    such that for all $\theta, \theta' \in U$ and $y \in \Omega'$ inequality~\eqref{common_y} holds.

    Define $\textit{KR}_{h}(\alpha, \beta) = \mathcal{L}_{D,h}(\alpha, \beta) + \frac{h}{2} (m_\alpha + m_\beta)$.
    It is known that $\textit{KR}_{h}$ is a metric~\cite{lellmann2014imaging, bogachev2007measure},
    thus by triangle inequality of the KR metric,
    \begin{align}
        \label{app_tmp_1}
        |\mathcal{L}_{D,h}(\alpha, \beta_\theta) - \mathcal{L}_{D,h}(\alpha, \beta_{\theta'})| =& |\Bigl(  \textit{KR}_{h}(\alpha, \beta_\theta) - \frac{h}{2} (m_\alpha + m_{\beta_\theta}) \Bigr) -  \Bigl(  \textit{KR}_{h}(\alpha, \beta_{\theta'}) - \frac{h}{2} (m_\alpha + m_{\beta_{\theta'}} )   \Bigr) | \nonumber \\ 
        =& |\textit{KR}_{h}(\alpha, \beta_\theta) - \textit{KR}_{h}(\alpha, \beta_{\theta'}) | \leq \textit{KR}_{h}(\beta_{\theta'} \beta_\theta)
    \end{align}
    for arbitrary $\theta, \theta' \in \mathbb{R}^d$.
    Note that the second equality holds because $m_{\beta_{\theta'}} = m_{\beta_\theta} = m_{\beta}$.
    In addition,
    we have
    \begin{align}
        \textit{KR}_{h}(\beta_{\theta'}, \beta_\theta)= \sup_{\substack{\vf \in Lip(\Omega) \\ -\frac{h}{2} \leq \vf \leq \frac{h}{2}} } \int_\Omega \vf d(\beta_{\theta'}-\beta_\theta) \leq \sup_{\vf \in Lip(\Omega) } \int_\Omega \vf d(\beta_{\theta'}-\beta_\theta)  = \mathcal{W}_1(\beta_{\theta'}, \beta_\theta) & \leq \int_\Omega \vd(\mathcal{T}_{\theta}(y) - \mathcal{T}_{\theta'}(y)) d\beta(y) \nonumber \\
        & \leq  m_\beta \Delta(\tilde{\theta}) (||\theta - \theta'|| ), \nonumber
    \end{align}
    where the last inequality holds because of~\eqref{common_y}.
    Therefore,
    we have
    \begin{equation}
        \mathcal{L}_{D,h}(\alpha, \beta_\theta) - \mathcal{L}_{D,h}(\alpha, \beta_{\theta'}) \leq \textit{KR}_{h}(\beta_{\theta'}, \beta_\theta) \leq m_\beta \Delta ||\theta - \theta' ||, \nonumber
    \end{equation}
    which proves $\mathcal{L}_{D,h}(\alpha, \beta_\theta)$ is locally Lipschitz \wrt $\theta$,
    and Radamacher's theorem states that it is differentiable almost everywhere.
    The rest of the proof is similar to that of Proposition~\ref{app_Smooth_M_con}.
\end{proof}

For clearness,
we summarize the results and prove the theorems in our work.
\begin{proof}[\textbf{Proof of Theorem~\ref{Theorem1}}]
    The KR formulation of $\mathcal{L}_{D,h}$ and the existence of its solution is proved in~\cite{schmitzer2019framework}.
    The gradient of $\mathcal{L}_{D,h}(\alpha, \beta_\theta)$ is derived in Proposition~\ref{app_Smooth_D_con}.
\end{proof}
\begin{proof}[\textbf{Proof of Theorem~\ref{Theorem2}}]
    The KR formulation of $\mathcal{L}_{M,m}$ is given in Proposition~\ref{app_m-PW},
    and the existence of the optimizer is proved in Proposition~\ref{app_Existence_M}.
    The gradient of $\mathcal{L}_{M,m}(\alpha, \beta_\theta)$ is derived in Proposition~\ref{app_Smooth_M_con}.
\end{proof}

We finally verify that the parametrized transformation $\mathcal{T}_\theta$~\eqref{our_transformation} satisfies assumption~\ref{assumption_1}.
\begin{proposition}[Definition~\ref{our_transformation} satisfies assumption~\ref{assumption_1}]
    \label{app_Lip_T}
    Let $y \in \Omega' \subseteq \mathbb{R}^{1 \times 3}$ be a point in 3D space.
    Define $\mathcal{T}_{\theta}(y) = yA + t + v(y)$ and $\theta=(A,t,v)$,
    where $A \in \mathbb{R}^{3 \times 3}$ is an affinity matrix,
    $t \in \mathbb{R}^{1\times3}$ is a translation vector,
    $v(y) \in \mathbb{R}^{1 \times 3}$ is the offset vectors of $y$.
    Assume that $\mathcal{T}_{\theta}$ is coherent,
    \ie, for all $y \in \Omega'$, 
    there exists a constant $J > 0$,
    such that $||\nabla_{y} v|| \leq J$ where $\nabla$ is the Jacobian matrix.
    Then,
    for every $\tilde{\theta}$,
    there exist a neighborhood $U_{\tilde{\theta}}(\delta)$ of $\tilde{\theta}$ and a constant $\Delta(\tilde{\theta})$,
    such that for $\theta, \theta' \in U_{\tilde{\theta}}(\delta)$ and $y_1, y_2 \in \Omega'$,
    $||\mathcal{T}_\theta(y_1), \mathcal{T}_{\theta'}(y_2)|| \leq \Delta(\tilde{\theta}) (||\theta - \theta'|| + ||y_2 - y_1|| )$.
\end{proposition}

\begin{proof}
    We have
    \begin{align}
        ||\mathcal{T}_\theta(y_1) - \mathcal{T}_{\theta'}(y_2) || & = ||\mathcal{T}_\theta(y_1) - \mathcal{T}_\theta(y_2) + \mathcal{T}_\theta(y_2) - \mathcal{T}_{\theta'}(y_2) || \nonumber \\
        &\leq||\mathcal{T}_\theta(y_1) - \mathcal{T}_\theta(y_2)|| + ||\mathcal{T}_\theta(y_2) - \mathcal{T}_{\theta'}(y_2) || \nonumber \\
        &= || (y_2 - y_1)A + (v(y_2)-v(y_1))|| + ||y_2(A-A') + (t-t') + (v(y_2)- v'(y_2)) || \nonumber \\
        &\leq ||y_2 - y_1|| ||\theta|| + ||\nabla_{y}v(y_t)|| ||y_2-y_1|| + ||y_2|| ||\theta-\theta'|| + ||\theta-\theta'|| + ||\theta-\theta'|| \nonumber \\
        &\leq ||y_2 - y_1|| (||\tilde{\theta}|| + \delta) + J ||y_2-y_1|| + diam(\Omega') ||\theta-\theta'|| + 2||\theta-\theta'|| \nonumber \\
        &=(||\tilde{\theta}|| + \delta) ||y_2-y_1|| + (diam(\Omega') + 2) ||\theta-\theta'|| \nonumber \\
        & \leq \max\bigl( (||\tilde{\theta}|| + \delta), (diam(\Omega') + 2) \bigr)   (||y_2-y_1|| + ||\theta-\theta'||), \nonumber
    \end{align}
    where $y_t$ is given by the mean value theorem.
    We finish the proof by letting $\Delta(\tilde{\theta})=\max\bigl( (||\tilde{\theta}|| + \delta), (diam(\Omega') + 2) \bigr)$.
\end{proof}

The statement also holds for a forward neural network:
\begin{proposition}[A forward neural network satisfies assumption~\ref{assumption_1}]
    \label{app_Lip_T_2}
    Let $y \subseteq \Omega' \subseteq \mathbb{R}^p$ be a $p$ dimensional input,
    and a forward neural network $\mathcal{T}_{\theta}$ consisting of linear and activation layers.
    For every $\tilde{\theta}$,
    there exist a neighborhood $U_{\tilde{\theta}}(\delta)$ of $\tilde{\theta}$ and a constant $\Delta(\tilde{\theta})$,
    such that for $\theta, \theta' \in U_{\tilde{\theta}}(\delta)$ and $y_1, y_2 \in \Omega'$,
    $||\mathcal{T}_\theta(y_1), \mathcal{T}_{\theta'}(y_2)|| \leq \Delta(\tilde{\theta}) (||\theta - \theta'|| + ||y_2 - y_1|| )$.
\end{proposition}

\begin{proof}
    We only discuss the case when $\mathcal{T}_{\theta}$ is differentiable.
    We have 
    \begin{equation}
        ||\mathcal{T}_\theta(y_1) - \mathcal{T}_{\theta'}(y_2) || \leq||\mathcal{T}_\theta(y_1) - \mathcal{T}_\theta(y_2)|| + ||\mathcal{T}_\theta(y_2) - \mathcal{T}_{'\theta}(y_2) ||
        \leq ||y_2 - y_1|| ||\nabla_x T_{\theta}(x) || +  ||\nabla_\theta T_{\theta''}(y_2)|| ||\theta-\theta'||, \nonumber
    \end{equation}
    where $x$ and $\theta''$ is given by the mean value theorem.
    In addition,
    by chain rule,
    we have $||\nabla_x T_{\theta}(x) || \leq \Delta_1(\theta) $,
    and $||\nabla_\theta T_{\theta}(y_2) || \leq ||y_2|| \Delta_2(\theta)$,
    where $\Delta_1$ and $\Delta_2$ are functions of $\theta$,
    and the explicit forms can be found in~\cite{pmlr-v70-arjovsky17a}.
    By combining these inequalities,
    we have 
    \begin{equation}
        ||\mathcal{T}_\theta(y_1) - \mathcal{T}_{\theta'}(y_2) || 
        \leq ||y_2 - y_1|| \Delta_1(\theta) +  diam(\Omega') \Delta_2(\theta) ||\theta-\theta'|| \leq \max\bigl( \Delta_1(\theta), diam(\Omega') \Delta_2(\theta) \bigr)  (||y_2-y_1|| + ||\theta-\theta'||), \nonumber
    \end{equation}
    which finishes our proof.
\end{proof}
\twocolumn
\section{More Experiment Details}

\subsection{More Details in Sec.~\ref{Sec_Opt_PW}}
\label{Sec_Opt_PW_app}
To demonstrate the accuracy of the proposed neural approximation,
we quantitatively compare the approximated KR forms and the primal forms in each setting in Fig.~\ref{app_Fish_complete}.
We compute the primal forms using linear programs.
The results are summarized in Tab.~\ref{Numerical_table}.
As can be seen,
our approximated KR forms are close to the true values with an average relative error less than $0.2\%$,
which is sufficiently accurate for machine learning applications.

\begin{table}[hbt!]
	\begin{center}
	  \caption{Quantitative comparison between the approximated KR forms and primal forms on the fish shape in Fig.~\ref{app_Fish_complete}}
	  \setlength{\tabcolsep}{2.0pt}
    \label{Numerical_table}
	  \begin{tabular}{c c c c  c c c c} 
		  \hline
           & $\mathcal{L}_{M,25}$ & $\mathcal{L}_{M,50}$ & $\mathcal{L}_{M,78}$ & $\mathcal{L}_{D,0.648}$ & $\mathcal{L}_{D,1.09}$ & $\mathcal{L}_{D,5}$ & $\mathcal{W}_1$\\
		\hline
	  \centering
        Primal      & 0.1354 & 0.4191  & 0.8955 & -0.0722 & -0.2795 & -4.1044 & 1.0835 \\
        KR (Ours) & 0.1352 & 0.4202  & 0.8994 & -0.0724 & -0.2791 & -4.1004 & 1.0893 \\
		\hline
	  \end{tabular}
	\end{center}
\end{table}

\subsection{More Details in Sec.~\ref{Sec_algorithm}}
\label{Sec_algorithm_app}
In Fig.~\ref{Algo-comparison},
we compute $\mathcal{L}_{M,m}(\alpha, \beta)$,
where $\alpha$ and $\beta$ are feature distributions extracted in task AC-OfficeHome in Sec.~\ref{Sec_PDA_comparison}.
They each contain $1600$ samples in the $256$-D feature space.
We set $m_\beta = m = 1$ and $m_\alpha=2$.
The mini-POT is computed using the POT library~\cite{flamary2021pot}.
Specifically,
to compute mini-POT,
we first randomly divide $\alpha$ and $\beta$ into $P$ mini-batches.
Let the $k$-th mini-batch be $\tilde{\alpha}_k$ and $\tilde{\beta}_k$,
the mini-POT is computed as
\begin{equation}
\textit{mini-POT}(\alpha, \beta) =  \frac{1}{P} \sum_{k=1}^P \mathcal{L}_{M,m}(\tilde{\alpha}_k, \tilde{\beta}_k),
\end{equation}
where $\mathcal{L}_{M,m}$ takes the primal form~\eqref{POT_primal} with an entropy regularizer.
To compute mini-batch PWAN,
we update $\vf_{w,h}$ for $1000$ epochs to make sure the training process is fully converged.
We use Adam optimizer, and the learning rate is set to $b \times 10^{-4}$ where $b$ represents the batch size.

We further consider the improved mini-POT proposed in~\cite{nguyen2022improving},
which corrects the mini-batch errors by transporting less mass in each mini-batch.
Specifically,
with a correction parameter $s\in (0,1)$,
the improved mini-POT is defined as
\begin{equation}
  \textit{mini-POT}_s(\alpha, \beta) =  \frac{1}{P} \sum_{k=1}^P \mathcal{L}_{M,sm}(\tilde{\alpha}_k, \tilde{\beta}_k).
\end{equation}
Note that only $sm$ mass, instead of $m$ mass, is transported in each batch.
The correction parameter $s$ is usually hard to choose,
and~\cite{nguyen2022improving} suggests searching for the optimal $s$, 
\eg, using a grid search on a validation set.
In this example,
we manually select $s$ to be $0.8$ or $0.9$, 
and we report the results in Fig.~\ref{Algo-comparison-app}.

As can be seen,
due to the mini-batch error, 
mini-POT is always larger than the true PW discrepancy.
By reducing the transported mass,
improved mini-POT can achieve smaller values.
However,
the optimal $s$ depends on the batch size,
and a bad choice of $s$ can lead to even worse results.
For example,
when the batch size is $10$,
$s=0.8$ leads to accurate results (the relative error is about $2\%$),
but when the batch size is $400$,
the same parameter $s=0.8$ leads to inaccurate results (the relative error is about $20\%$) due to over correction.
In other words,
if the optimal correction parameter $s$ is not known in advance,
it is hard to achieve an accurate result using improved mini-POT.

\begin{figure}[htb!]
  \centering
  \includegraphics[width=0.7\linewidth]{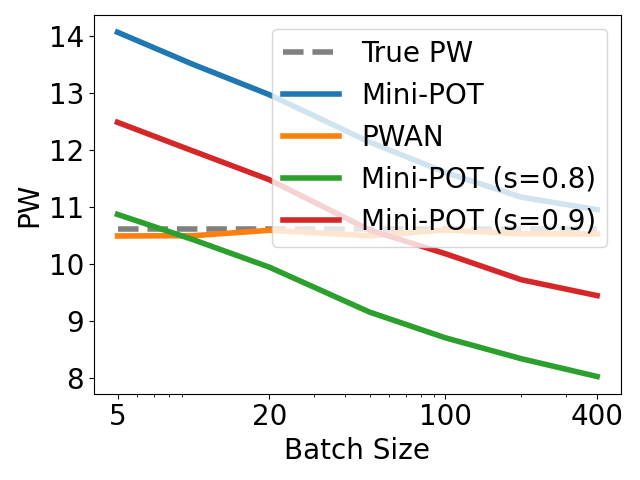}
  \caption{Comparison between PWAN and improved mini-POT on examples in Fig.~\ref{Algo-comparison}.}
  \label{Algo-comparison-app}
\end{figure}

\subsection{An Application in Generative Modelling}
\label{sec_app_WGAN}

\begin{figure*}[t!]
  \centering
  \subfigure[Random samples from the training set (left), WGAN (middle), and PWAN (right).]{
    \label{generative_vis_a}
    \begin{minipage}[b]{0.7\linewidth}
        \centering
        \includegraphics[width=0.3\linewidth]{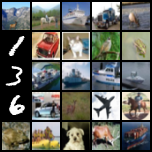}
        \includegraphics[width=0.3\linewidth]{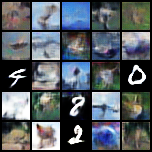}
        \includegraphics[width=0.3\linewidth]{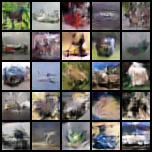}
    \end{minipage}
    }
    \subfigure[t-SNE visualization of the samples generated by WGAN (left), PWAN (middle), and a random initialized generator (right).]{
      \label{generative_vis_d}
      \begin{minipage}[b]{0.7\linewidth}
          \centering
          \includegraphics[width=0.3\linewidth]{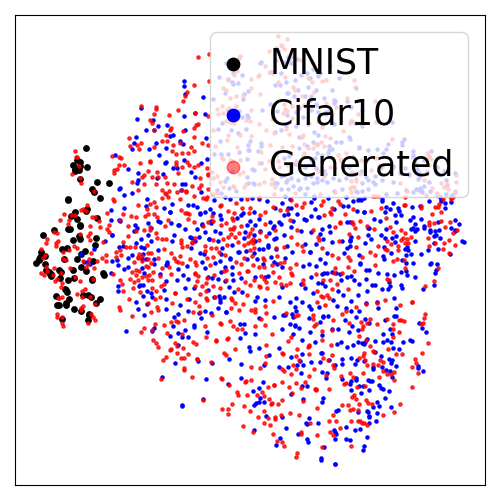}
          \includegraphics[width=0.3\linewidth]{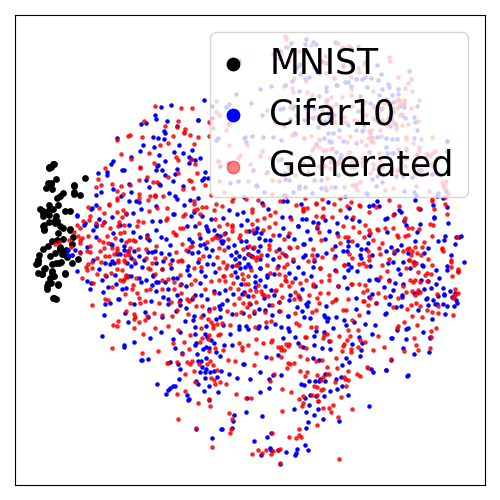}
          \includegraphics[width=0.3\linewidth]{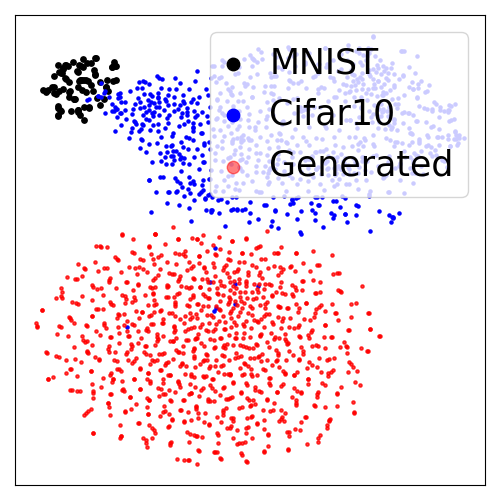}
        \end{minipage}
  }
\caption{
A comparison between PWAN and WGAN on generative modelling.  
\subref{generative_vis_a} When trained on the Cifar10-MNIST dataset,
the samples of WGAN contain a fraction of MNIST-like images,
\ie, outliers,
while the samples of PWAN are Cifar10-like.
\subref{generative_vis_d} The t-SNE visualization shows that a fraction of samples of WGAN are close to MNIST (dark points),
while most samples of PWAN are distinct from MNIST. 
Zoom in to see the details.
}
\label{Compare_WGAN}
\end{figure*}

\begin{figure*}[tb!]
  \centering
  \includegraphics[width=0.3\linewidth]{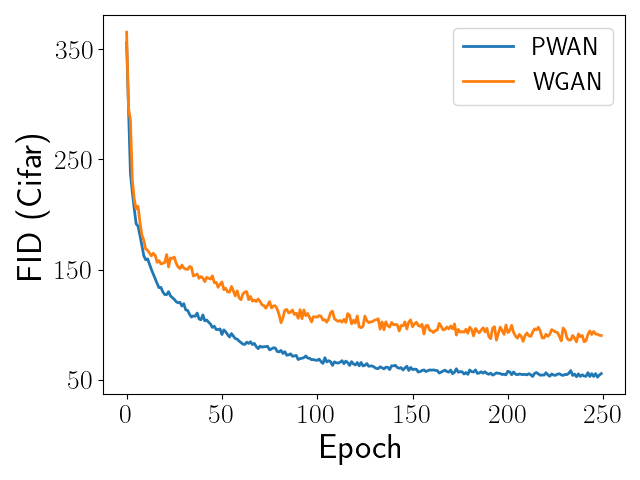}
  \includegraphics[width=0.3\linewidth]{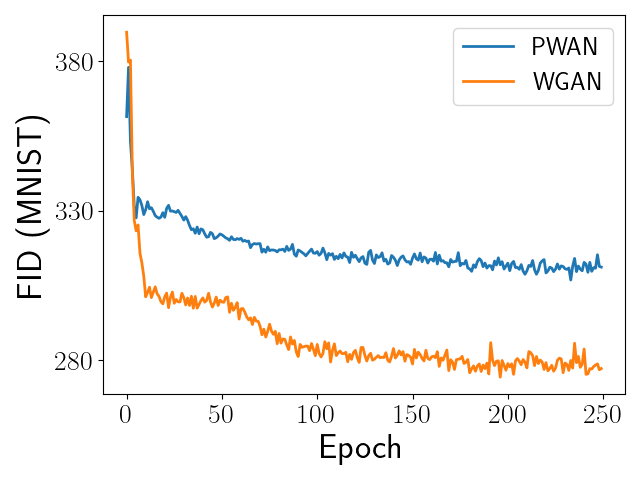}
  \includegraphics[width=0.3\linewidth]{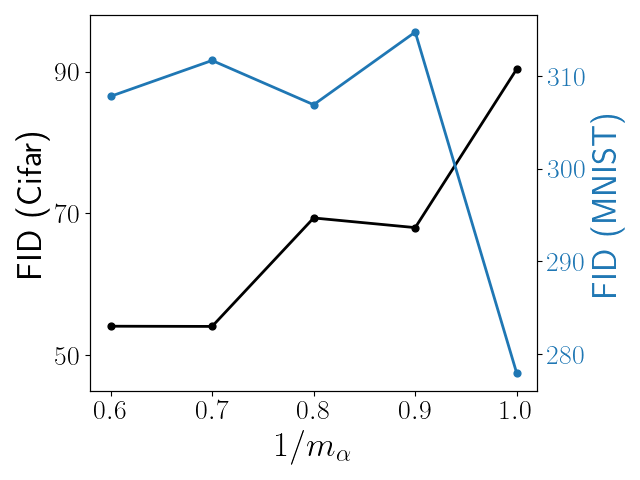}
\caption{Quantitative comparison between PWAN and WGAN on Cifar10-MNIST dataset.
Left and middle: FID on Cifar10 and MNIST.
Right: the influence of $m_\alpha$ on FID.
}
\label{Compare_WGAN_FID}
\end{figure*}

This subsection presents an application of PWAN to generative modelling:
Given an image dataset $\mathbf{D}$ contaminated by a fraction of outliers,
the goal is to generate images that are similar to the non-outliers in $\mathbf{D}$.

To achieve robustness against outliers,
we construct the following model as an extension of WGAN~\cite{gulrajani2017improved,pmlr-v70-arjovsky17a}:
\begin{equation}
  \label{generative_loss}
  \min_{\theta} \mathcal{L}_{M,1}(\alpha, \beta_\theta),
\end{equation}
where $\alpha = \frac{m_\alpha}{|\mathbf{D}|} \sum_{x \in \mathbf{D}} \delta_x$ is the data distribution with weight $m_\alpha >1$,
$|\mathbf{D}|$ is the size of dataset $\mathbf{D}$,
and $\beta_\theta = (\mathcal{T}_\theta)_{\#}\mathcal{N}$ is a push-forward of the standard Gaussian distribution $\mathcal{N}$ by a feedforward network $\mathcal{T}_\theta$,
\ie, a sample from $\beta_\theta$ can be generated by $y = \mathcal{T}_\theta(z)$,
where $z \sim \mathcal{N}$ is a Gaussian noise.
By inserting our PWAN formulation into~\eqref{generative_loss}, 
our model can be explicitly written as:
\begin{align}
  \label{generative_loss_specific}
  \min_{\mathcal{T}_\theta} \max_{\vf_w \leq 0}  \mathbb{E}_{x \sim \tilde{\alpha}} [m_\alpha \vf_w(x)]
  - \mathbb{E}_{z \sim \mathcal{N}} [\vf_w(\mathcal{T}_\theta(z))]& \nonumber \\
  - \textit{GP}(\vf_{w})&,
\end{align}
where $\vf_w$ is a feedforward network,
and $\tilde{\alpha} = \frac{1}{m_\alpha} \alpha$ is the data distribution.
Note that~\eqref{generative_loss_specific} approximately omits $\frac{m_\alpha - 1}{m_\alpha}$ data as outliers.
When $m_\alpha=1$, 
no data is omitted,
thus~\eqref{generative_loss_specific} becomes WGAN (the sign constraint $\vf_w \leq 0$ is not active).
For clarity,
we call~\eqref{generative_loss_specific} PWAN in the rest of this subsection.

To show the effectiveness of PWAN,
we consider the task of generating Cifar10~\cite{krizhevsky2009learning} images.
We construct the Cifar10-MNIST dataset $\mathbf{D}$ by adding $10\%$ MNIST~\cite{lecun2010mnist} images as outliers to the Cifar10 dataset~\cite{krizhevsky2009learning},
and we compare PWAN against WGAN on this dataset.
The network structure of WGAN and all hyper-parameters follow~\cite{gulrajani2017improved}.
For PWAN,
we set $m_\alpha=1 / 0.7$,
and take the negative absolute value to ensure the sign constraint as detailed in Sec.~\ref{Sec_algorithm}.
All other parameters are the same as those used in WGAN.

The qualitative results are presented in Fig.~\ref{generative_vis_a},
where we see that the samples of WGAN are ``dirty'' as they contain outliers,
\ie, MNIST-like images.
In contrast,
PWAN generates ``clean'' samples similar to Cifar10.
This result suggests that PWAN can effectively omit outliers during training.
Furthermore,
the robustness of PWAN is confirmed by the visualizations in Fig.~\ref{generative_vis_d},
where we observe that the samples of PWAN are close to Cifar10 and are distinct from MNIST,
while the samples of WGAN are mixed up with both MNIST and Cifar10.

A quantitative comparison is presented in Fig.~\ref{Compare_WGAN_FID},
where we observe that compared to WGAN, 
PWAN achieves lower FID on Cifar10,
and higher FID on MNIST.
In other words,
the samples of PWAN are closer to Cifar10,
and are further away from MNIST,
which is consistent with our observation in Fig.~\ref{generative_vis_d}.
In addition,
we conduct a sensitivity study of the parameter $m_\alpha$,
which controls the ratio of omitted outliers.
As shown in the right panel,
the FID on Cifar10 is robust to $m_\alpha$ when $1/m_\alpha$ is relatively small.
However,
when $1/m_\alpha$ is excessively large,
\eg, $m_\alpha$ is close to $1$,
the FID on Cifar10 increases drastically while the FID on MNIST decreases.
This is because too few outliers are omitted,
\ie, 
the samples of PWAN are aligned to MNIST.

\subsection{More Details in Sec.~\ref{Sec_app_PSR}}
\label{Sec_app_PSR_app}
To estimate the gradient of the coherence energy efficiently,
we first decompose $\mathbf{G}$ as $\mathbf{G} \approx \mathbf{Q}\Lambda \mathbf{Q}^T$ via the Nystr\"{o}m method,
where $k \ll r$, 
$\mathbf{Q} \in \mathbb{R}^{r\times k}$, 
and $\Lambda \in \mathbb{R}^{k\times k}$ is a diagonal matrix.
Then we apply the Woodbury identity to $(\sigma \mathcal{I} + \mathbf{Q}\Lambda \mathbf{Q}^T)^{-1}$ and obtain
\begin{equation}
   (\sigma \mathcal{I} + \mathbf{Q}\Lambda \mathbf{Q}^T)^{-1} = \sigma^{-1} \mathcal{I} - \sigma^{-2} \mathbf{Q} (\Lambda^{-1} + \sigma^{-1} \mathbf{Q}^T \mathbf{Q})^{-1}\mathbf{Q}^T. \nonumber
\end{equation}
As a result,
the gradient of the coherence energy can be approximated as 
\begin{align}
   \frac{\partial \mathcal{C}_\theta}{\partial V} & =  2 \lambda (\sigma \mathcal{I} + \mathbf{G})^{-1} V  \nonumber \\
   & \approx (2 \lambda) (\sigma^{-1} V - \sigma^{-2} \mathbf{Q} (\Lambda^{-1} + \sigma^{-1} \mathbf{Q}^T \mathbf{Q})^{-1}\mathbf{Q}^T V). \nonumber
\end{align}

\subsection{Detailed Experimental Settings in Sec.~\ref{Sec_experiments_PS}}
\label{Sec_experiments_PS_app}

The network used in our experiment is a 5-layer point-wise multi-layer perceptron with a skip connection.
The detailed structure is shown in Fig.~\ref{app_Exp_network}.
 
\begin{figure}[htb!]
	\centering
        \includegraphics[width=0.99\linewidth]{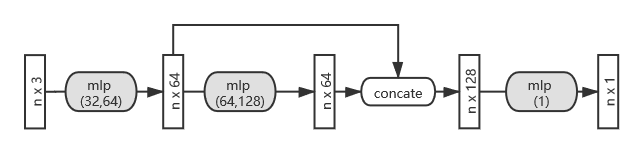}
    \vspace{-2mm}
  \caption{The structure of the network used in our experiments.
  The input is a matrix of shape $(n,3)$ representing the coordinates of all points in the set,
  and the output is a matrix of shape $(n,1)$ representing the potential of the corresponding points.
  $mlp(x)$ represents a multi-layer perceptron (mlp) with the size $x$.
  For example,
  $mlp(m,n)$ represents an mlp consisting of two layers,
  and the size of each layer is $m$ and $n$.
  We use ReLu activation function in all except the output layer.
  The activation function $\vl(x; h) = \max\{-|x|, -h\}$ is added to the output to clip the output to the interval $[-h , 0]$.
  }
	\label{app_Exp_network}
\end{figure}

We train the network $\vf_{w,h}$ using the Adam optimizer~\cite{kingma2014adam},
and we train the transformation $\mathcal{T}_{\theta}$ using the RMSprop optimizer~\cite{Tieleman2012}.
The learning rates of both optimizers are set to $10^{-4}$.
The parameters are set as follows:
\begin{itemize}[leftmargin=3mm]
  \item[-]{Experiments in Sec.~\ref{Sec_exp_acc}}: For PWAN,
  we set $(\rho, \lambda, \sigma, T)=(2, 0.01, 0.1, 2000)$;
  For TPS-RPM,
  we set $T\_finalfac=500$, $frac=1$, and $T\_init=1.5$;
  For GMM-REG,
  we set $sigma=0.5,0.2,0.02$, $Lambda=.1,.02,.01$, $max\_function\_evals=50,50,100$ and $level=3$.
  For BCPD and CPD,
  we set $(\beta, \lambda, w) = (2.0, 2.0, 0.1)$.
  \item[-]{Experiments on the human face dataset}:
  We use m-PWAN with  $(\rho, \lambda, \sigma, T)=(0.5, 5 \times 10^{-4}, 1.0, 3000)$;
  For BCPD and CPD,
  we set $(\beta, \lambda, w) = (3.0, 20.0, 0.1)$.
  \item[-]{Experiments on the human body dataset}: We use m-PWAN with $(\rho, \lambda, \sigma, T)=(1.0, 1 \times 10^{-4}, 1.0, 3000)$;
  For BCPD and CPD,
  we set $(\beta, \lambda, w) = (0.3, 2.0, 0.1)$.
  We also tried $w=0.6$ but the results are similar (not shown in our experiments).
  \item[-]{Rigid registration}: We set $(\beta, \lambda, w) = (2.0, 1e^9, 0.1)$ for BCPD.
  We use the rigid version of the CPD software.
  We set the same parameter for ICP as PWAN,
  \ie,
  we set the distance threshold $d=0.05$ for d-ICP,
  and set the trimming rate $m=0.8\min(q,r)$ for m-ICP.
\end{itemize}
The parameters for CPD and BCPD are suggested in~\cite{hirose2021a}.

We note that in our experiments, 
we determine parameter $m$ for m-PWAN by estimating the overlap ratio and setting $m$ as the number of overlapped points.
A future direction is to automatically determine the overlap ratio as in~\cite{hartley2003multiple}.
For d-PWAN,
we always assume the point sets are uniformly distributed (otherwise we can downsample the point sets using voxel filtering),
and set $h$ close to the nearest distance between points in a point set.

\subsection{More Details in Sec.~\ref{Sec_exp_Toy}}
\label{Sec_exp_Toy_app}
KL divergence and $L_2$ distance between point sets $X$ and $Y$ are formally defined as
\begin{align}
    L_2(X,Y)  =  \sum_{\substack{x_i, x_j \in X  \\ y_i, y_j \in Y} } \frac{1}{q^2} \phi(0|x_i-x_j,2\sigma) &+ \frac{1}{r^2} \phi(0|y_i-y_j,2\sigma) \nonumber \\
    & - \frac{2}{qr}\phi(0|x_i-y_j,2 \sigma), \nonumber
\end{align}
\begin{equation}
    KL(X,Y) = - \frac{1}{q} \sum_{y_j \in Y} \log \Bigl( \omega \frac{1}{q} + (1-\omega) \sum_{x_i \in X} \frac{1}{r} \phi(y_j| x_i, \sigma)    \Bigr), \nonumber
\end{equation}
where $\phi(\cdot|u,\sigma)$ is the Gaussian distribution with mean $u$ and variance $\sigma$.
For simplicity,
we set $\sigma=1$ and $\omega=0.2$ for KL and $L_2$ in Sec.~\ref{Sec_exp_Toy}.

\begin{figure*}[tb!]
  \centering
  \vspace{-4mm}
  \begin{minipage}[b]{1\linewidth}
      \centering
      \subfigure[Experimental setting.]{
          \includegraphics[width=0.4\linewidth]{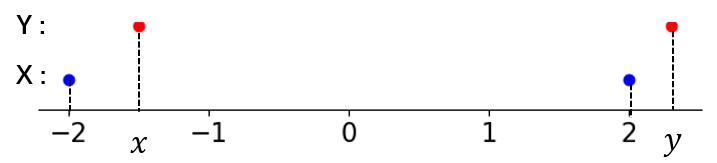}
  }
  \vspace{-2mm}
  \end{minipage}
  \begin{minipage}[b]{1\linewidth}
    \centering
  \subfigure[$L_2$]{
    \label{metric_xy_b}
    \begin{minipage}[b]{0.25\linewidth}
      \includegraphics[width=1.0\linewidth]{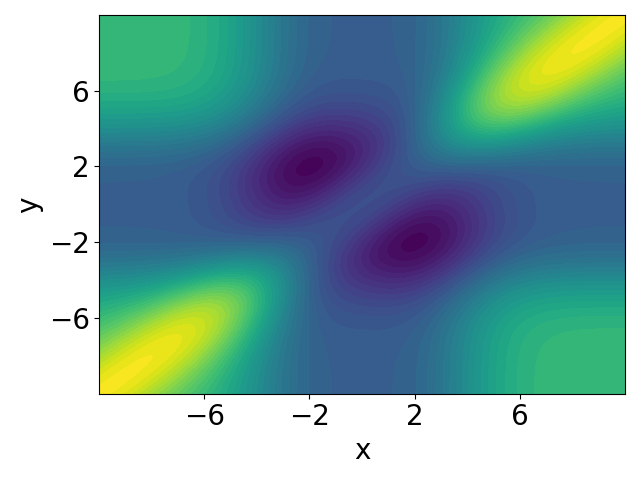}\\
      \includegraphics[width=1.0\linewidth]{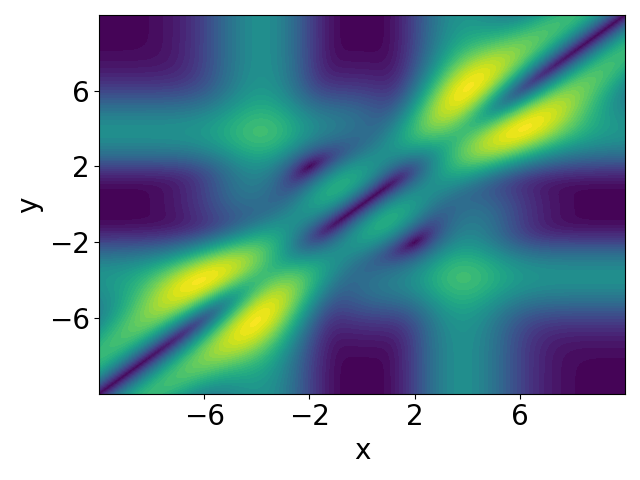}
    \end{minipage}
  }
  \subfigure[$KL$]{
    \label{metric_xy_c}
    \begin{minipage}[b]{0.25\linewidth}
      \includegraphics[width=1.0\linewidth]{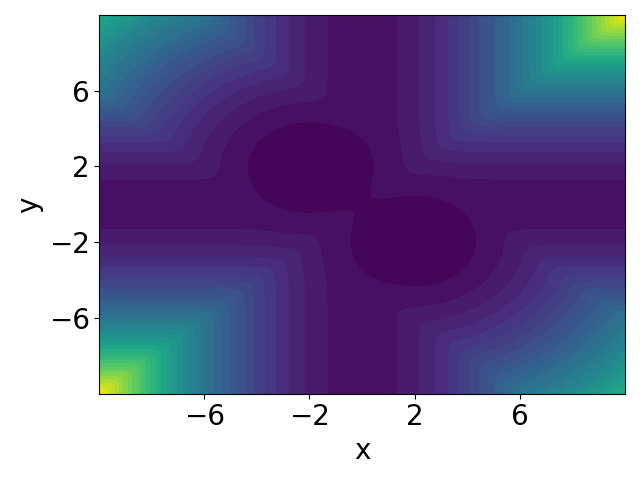}\\
      \includegraphics[width=1.0\linewidth]{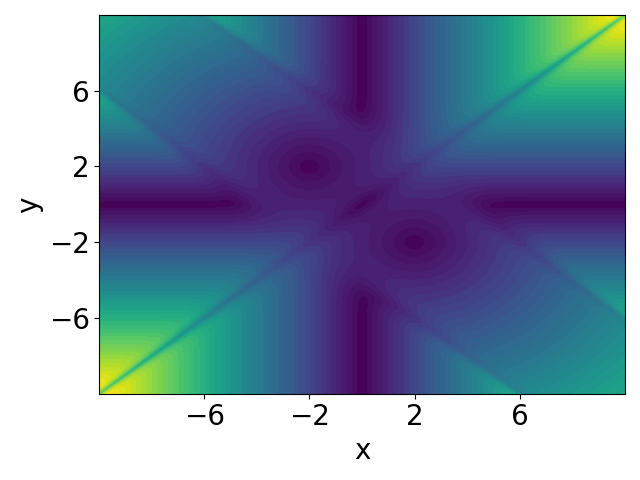}
    \end{minipage}
  }
  \subfigure[$\mathcal{W}_1$]{
    \label{metric_xy_d}
    \begin{minipage}[b]{0.25\linewidth}
      \includegraphics[width=1.0\linewidth]{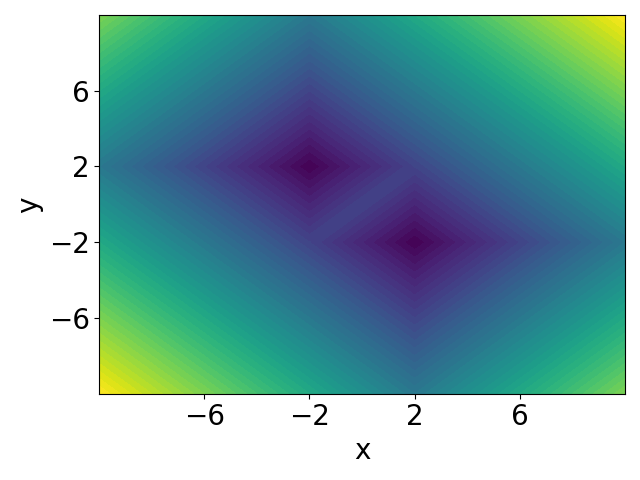}\\
      \includegraphics[width=1.0\linewidth]{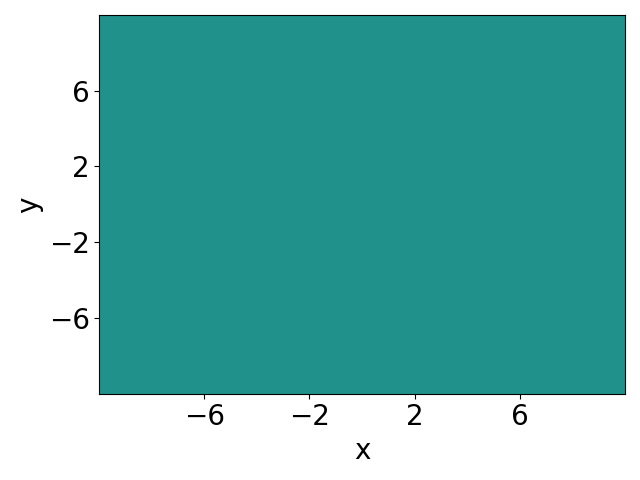}
    \end{minipage}
  }
  \begin{minipage}[b]{0.05525\linewidth}
    \includegraphics[width=1\linewidth]{Result/Fish/1.png} \\\vspace{-2.4ex}
    \includegraphics[width=1\linewidth]{Result/Fish/1.png} \\ \vspace{-0.84ex}
\end{minipage}
\end{minipage}
\caption{Comparison of different discrepancies on a pair of toy point sets.
The top and bottom row in~\subref{metric_xy_b},\subref{metric_xy_c} and~\subref{metric_xy_d} represent the discrepancies and their gradient norms respectively.
}
\label{metric_xy}
\end{figure*}

We note that compared to KL divergence and $L_2$ distance,
Wasserstein type divergence is ``smoother'', \ie, it can be optimized more easily.
To see this, 
we present a toy example comparing $\mathcal{W}_1$,  KL divergence and $L_2$ distance in Fig.~\ref{metric_xy}.
We set $\omega=0$ for KL divergence since the dataset is clean,
and we set $\sigma=1$.
We fix the reference set $X=\{-2, 2\}$,
and move the source set $Y=\{x, y\}$ in the 2D space.
We present the discrepancies between $X$ and $Y$ as a function of $(x,y)$,
and compute their respective gradients.

As can be seen,
all discrepancies have two global minima $(-2, 2)$ and $(2, -2)$ corresponding to the correct alignment.
However,
KL divergence has a suspicious stationary point $(0, 0)$,
which can trap both expectation-maximization-type and gradient-descent-type algorithms~\cite{xu2016global}.
In addition,
there exist some regions where the gradient norm of $L_2$ distance is small,
which indicates optimizing $L_2$ may be slow in these regions.
In contrast,
the gradient norm of $\mathcal{W}_1$ is constant,
thus the optimization process can easily converge to global minima.

\subsection{More Details in Sec.~\ref{Sec_exp_acc}}
\label{Sec_exp_acc_app}
We present qualitative results of the experiments in Fig.~\ref{qualtitative_outlier} and Fig.~\ref{qualtitative_partial}.
We do not show the results of TPS-RPM on the second experiment as it generally fails to converge.
As can be seen,
PWAN successfully registers the point sets in all cases,
while all baseline methods bias toward the noise points or to the non-overlapped region when the outlier ratio is high,
except for TPS-RPM which shows strong robustness against noise points comparably with PWAN in the first example.

\begin{figure*}[htb!]
	\centering
    \subfigure[Initial sets]{
        \begin{minipage}[b]{0.15\linewidth}
            \includegraphics[width=1\linewidth]{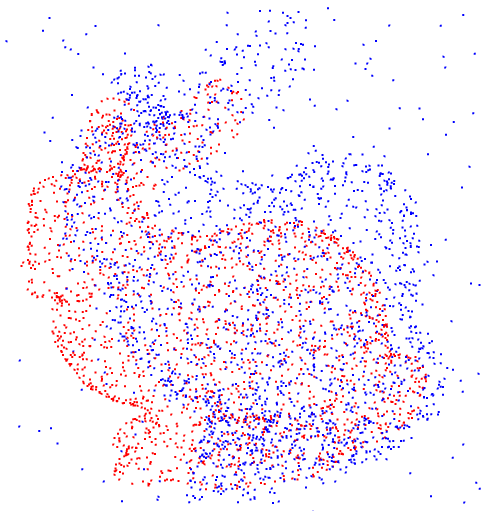}\\
            \includegraphics[width=1\linewidth]{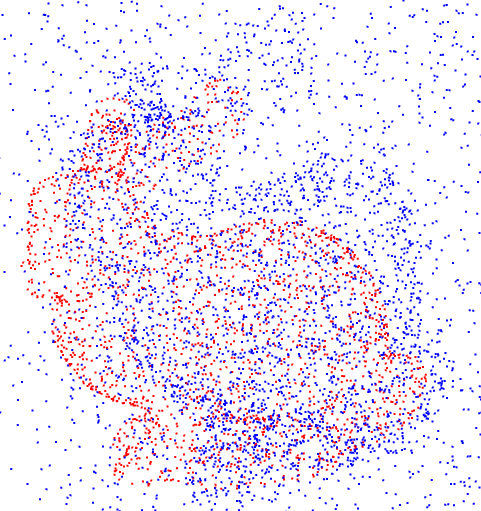}\\
            \includegraphics[width=1\linewidth]{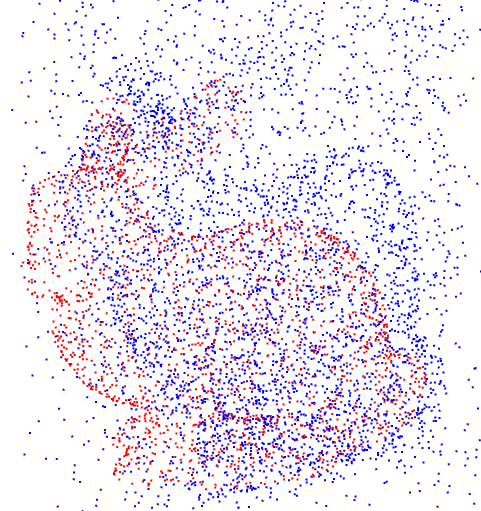}\\
        \end{minipage}
    }
    \hspace{-2mm}
    \subfigure[BCPD]{
        \begin{minipage}[b]{0.15\linewidth}
            \includegraphics[width=1\linewidth]{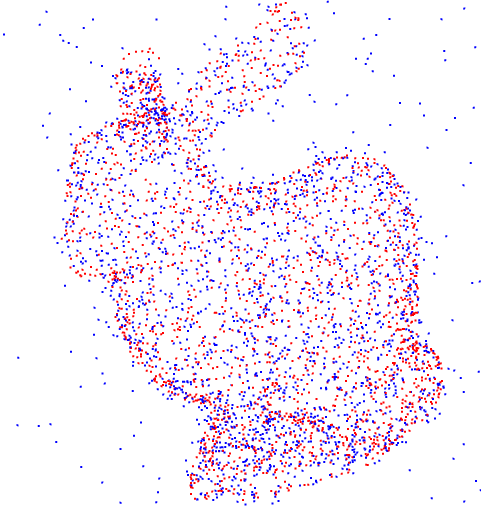}\\
            \includegraphics[width=1\linewidth]{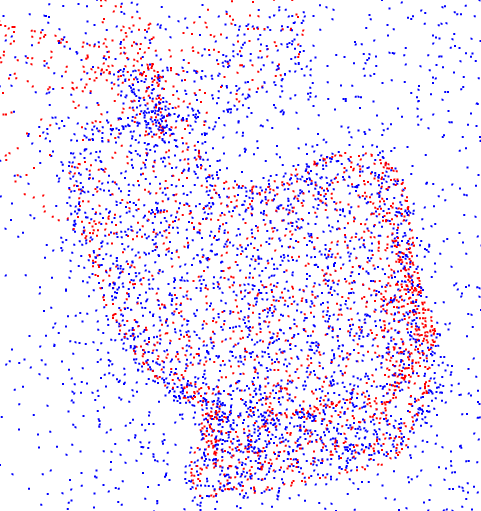}\\
            \includegraphics[width=1\linewidth]{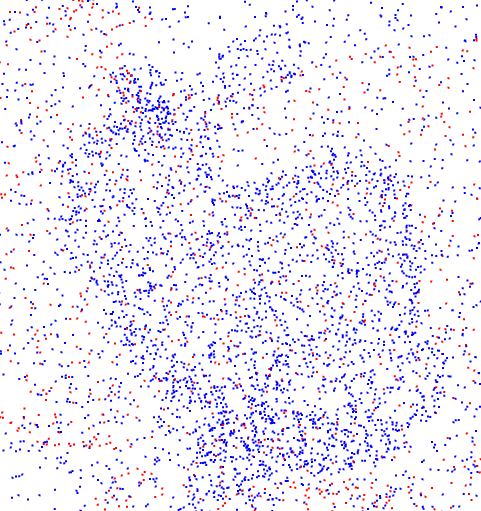}\\
        \end{minipage}
    }
    \hspace{-2mm}
    \subfigure[CPD]{
        \begin{minipage}[b]{0.15\linewidth}
            \includegraphics[width=1\linewidth]{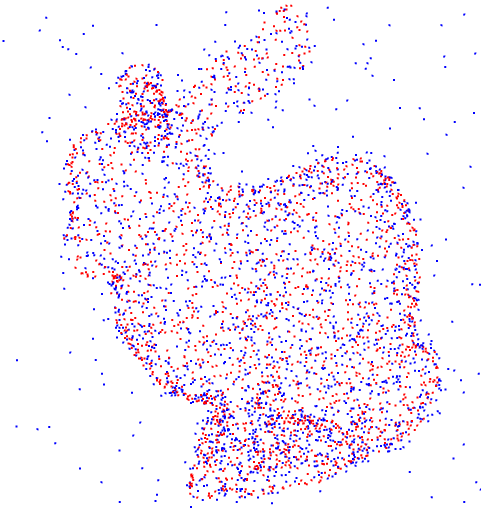}\\
            \includegraphics[width=1\linewidth]{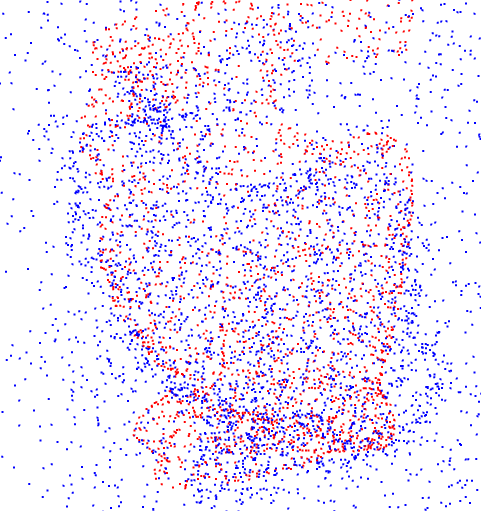}\\
            \includegraphics[width=1\linewidth]{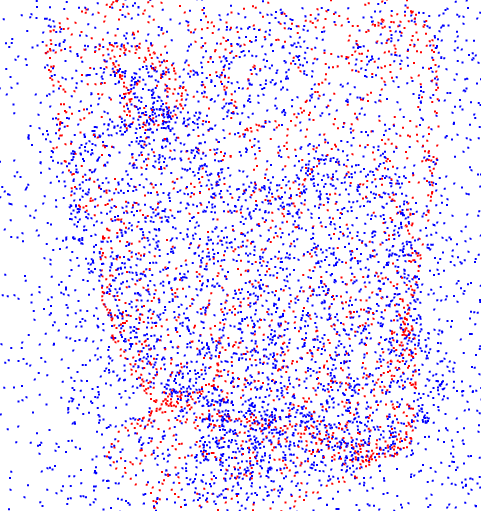}\\
        \end{minipage}
    }
    \hspace{-2mm}
    \subfigure[GMM-REG]{
        \begin{minipage}[b]{0.15\linewidth}
            \includegraphics[width=1\linewidth]{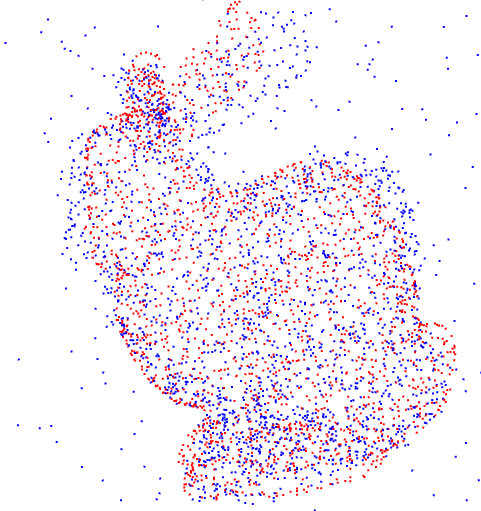}\\
            \includegraphics[width=1\linewidth]{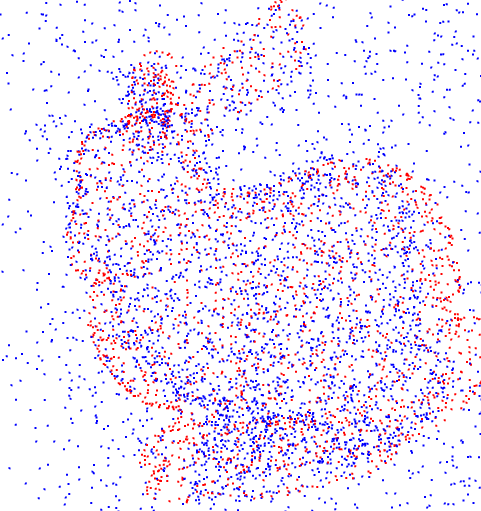}\\
            \includegraphics[width=1\linewidth]{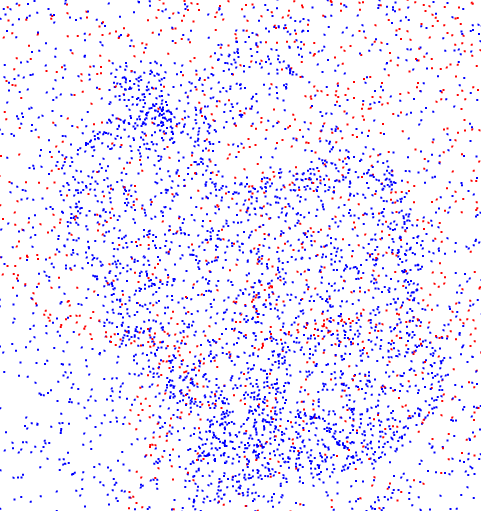}\\
        \end{minipage}
    }
    \hspace{-2mm}
    \subfigure[TPS-RPM]{
        \begin{minipage}[b]{0.15\linewidth}
            \includegraphics[width=1\linewidth]{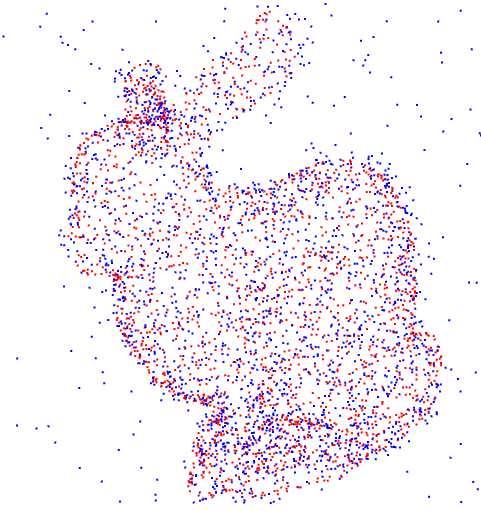}\\
            \includegraphics[width=1\linewidth]{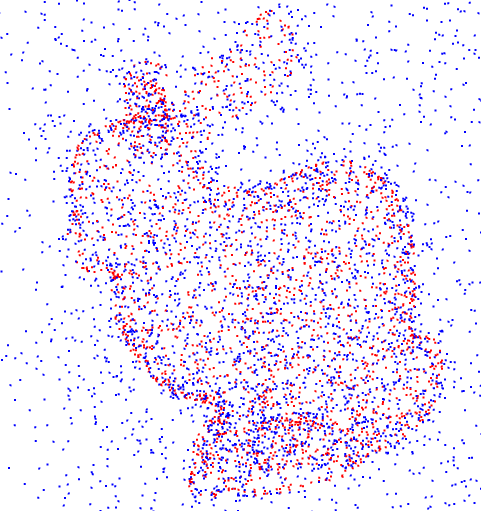}\\
            \includegraphics[width=1\linewidth]{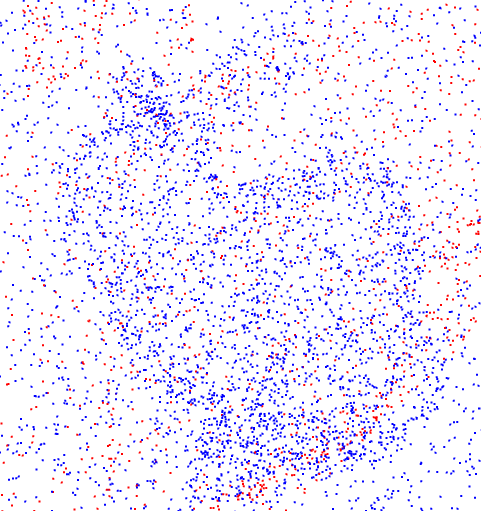}\\
        \end{minipage}
    }
    \hspace{-2mm}
    \subfigure[PWAN]{
        \begin{minipage}[b]{0.15\linewidth}
            \includegraphics[width=1\linewidth]{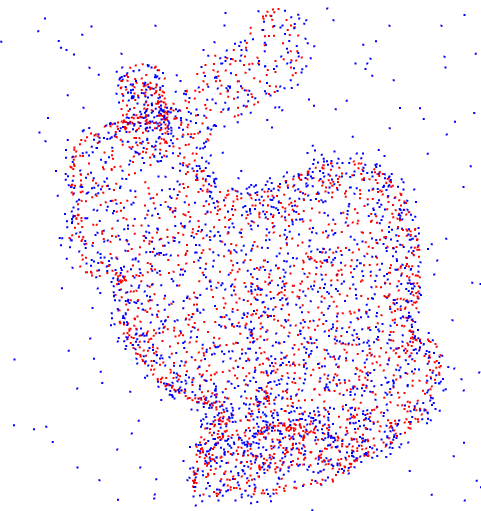}\\
            \includegraphics[width=1\linewidth]{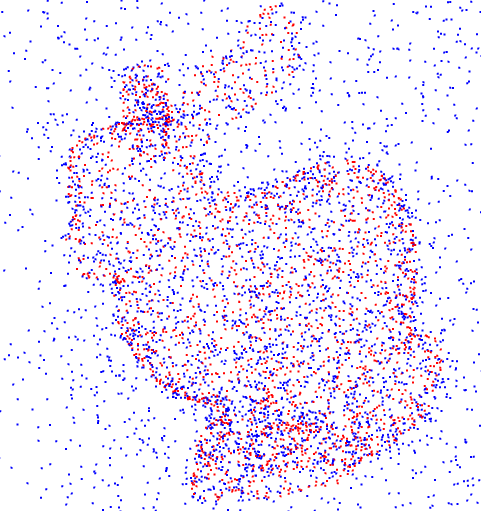}\\
            \includegraphics[width=1\linewidth]{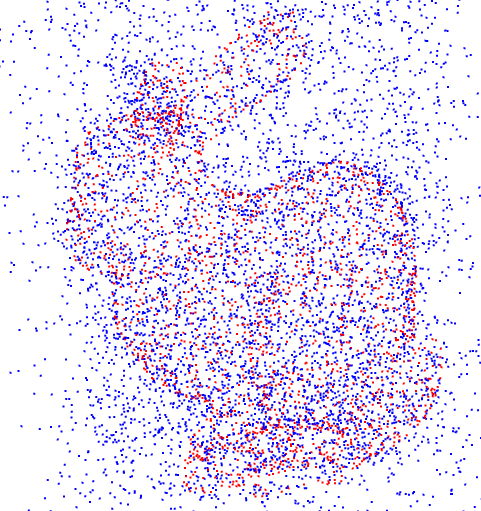}\\
        \end{minipage}
    }
    \vspace{-2mm}
    \caption{An example of registering noisy point sets. 
    The outlier/non-outlier ratios are 
    $0.2$ (1st row),
    $1.2$ (2nd row) and $2.0$ (3rd row).}
    \label{qualtitative_outlier}
\end{figure*}

\begin{figure*}[htb!]
	\centering
    \subfigure[Initial sets]{
        \begin{minipage}[b]{0.15\linewidth}
            \includegraphics[width=1\linewidth]{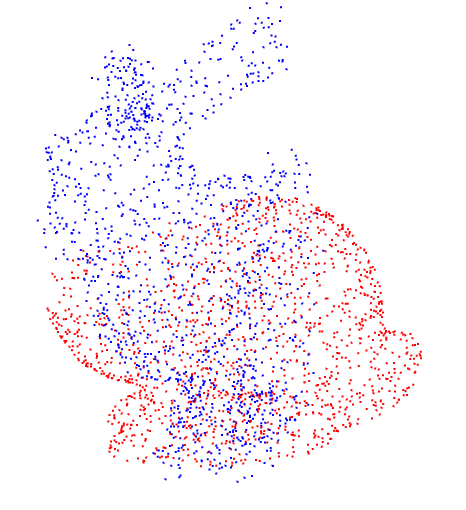}\\
            \includegraphics[width=1\linewidth]{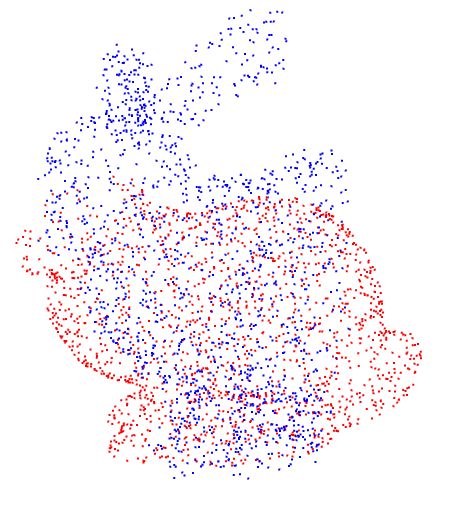}\\
            \includegraphics[width=1\linewidth]{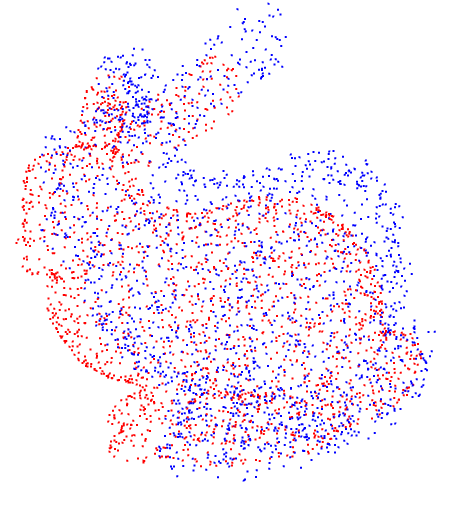}\\
        \end{minipage}
    }
    \hspace{-2mm}
    \subfigure[BCPD]{
        \begin{minipage}[b]{0.15\linewidth}
            \includegraphics[width=1\linewidth]{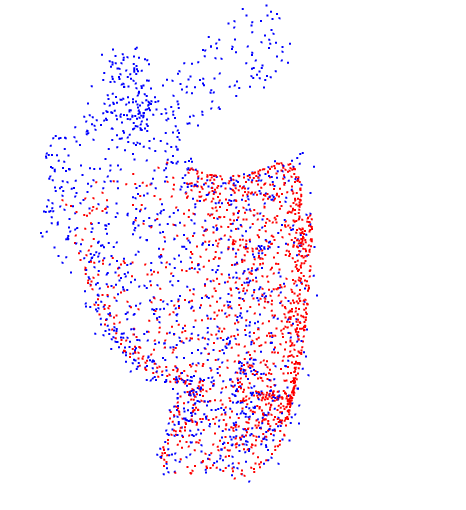}\\
            \includegraphics[width=1\linewidth]{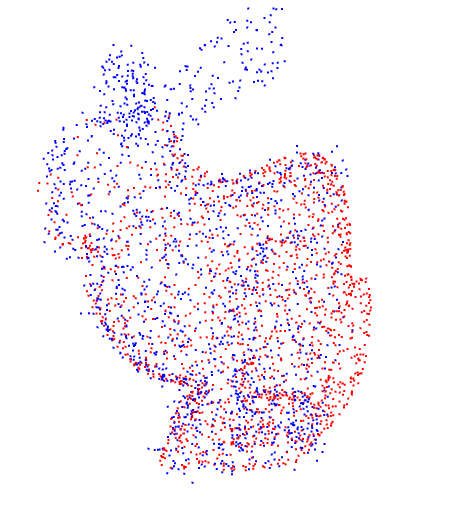}\\
            \includegraphics[width=1\linewidth]{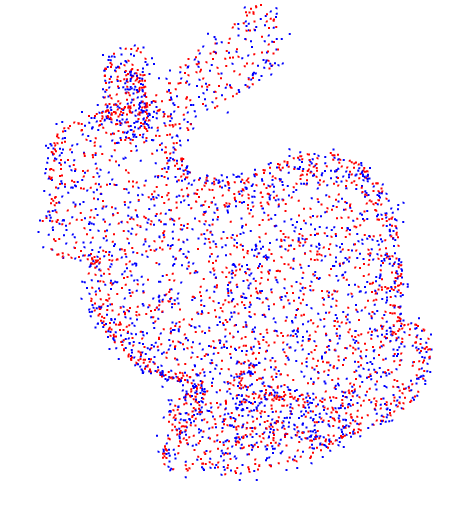}\\
        \end{minipage}
    }
    \hspace{-2mm}
    \subfigure[CPD]{
        \begin{minipage}[b]{0.15\linewidth}
            \includegraphics[width=1\linewidth]{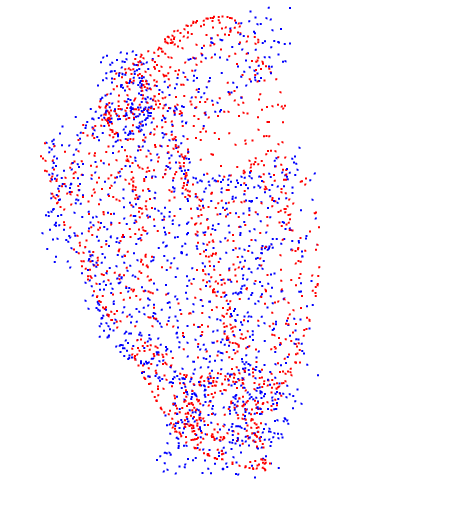}\\
            \includegraphics[width=1\linewidth]{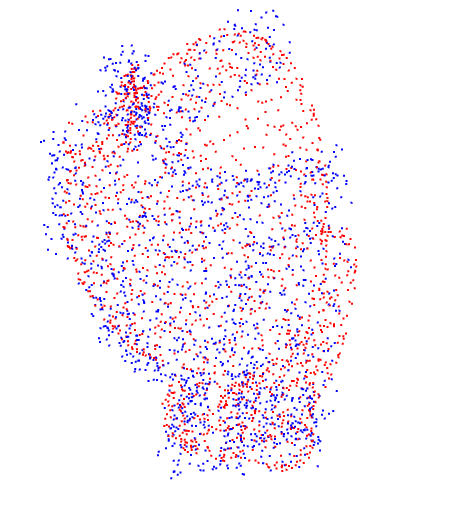}\\
            \includegraphics[width=1\linewidth]{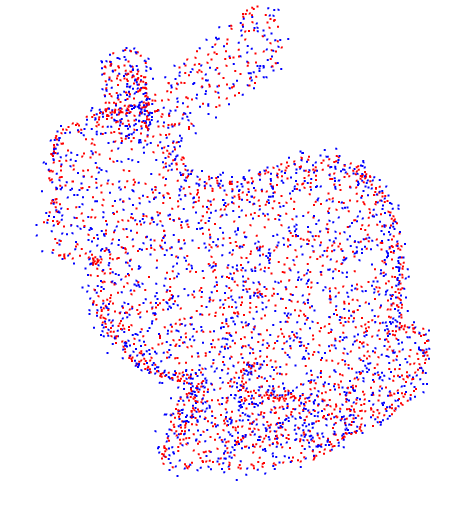}\\
        \end{minipage}
    }
    \hspace{-2mm}
    \subfigure[GMM-REG]{
        \begin{minipage}[b]{0.15\linewidth}
            \includegraphics[width=1\linewidth]{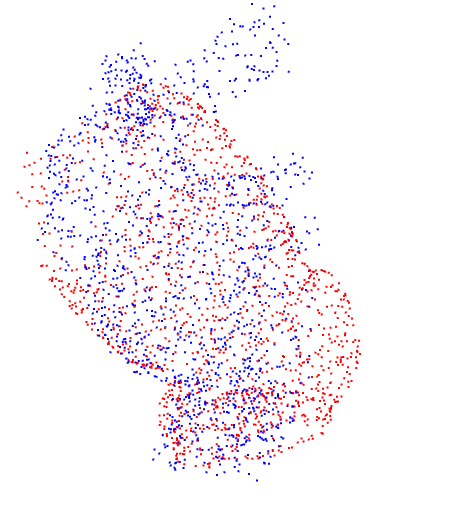}\\
            \includegraphics[width=1\linewidth]{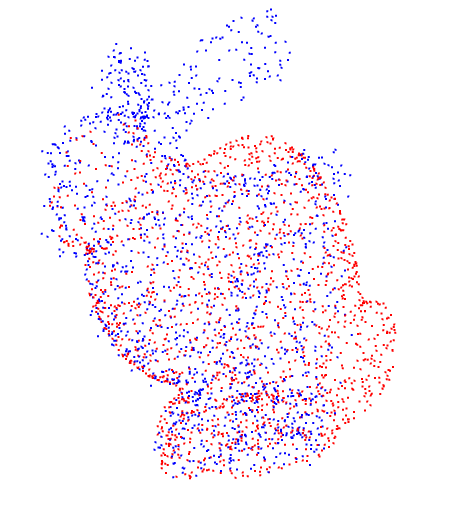}\\
            \includegraphics[width=1\linewidth]{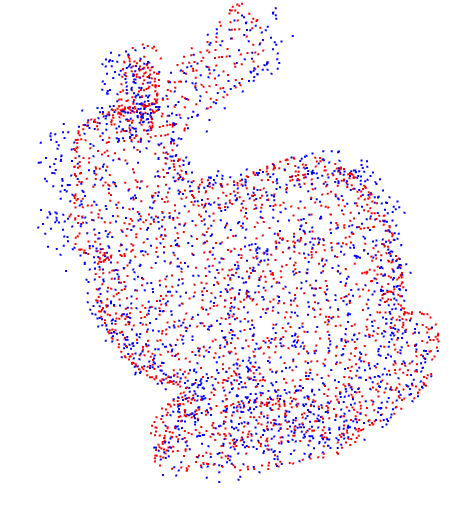}\\
        \end{minipage}
    }
    \hspace{-2mm}
    \subfigure[d-PWAN]{
        \begin{minipage}[b]{0.15\linewidth}
            \includegraphics[width=1\linewidth]{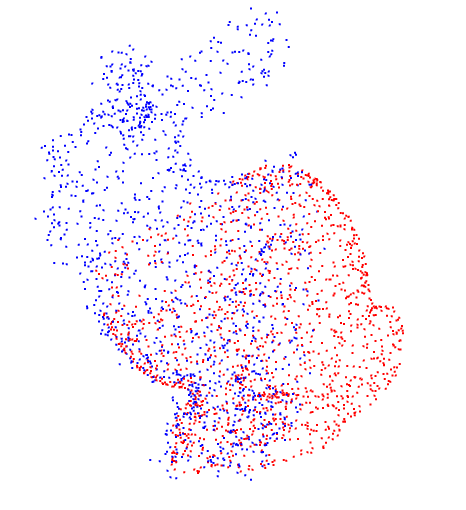}\\
            \includegraphics[width=1\linewidth]{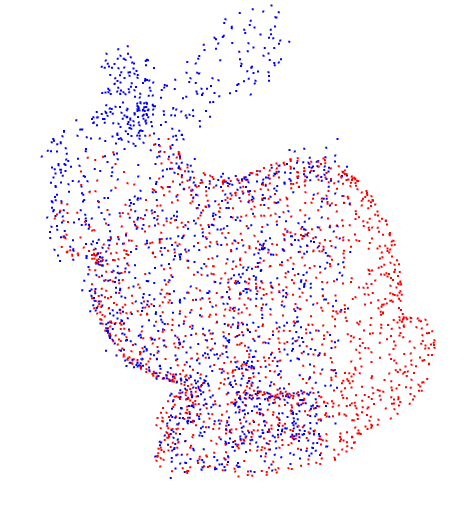}\\
            \includegraphics[width=1\linewidth]{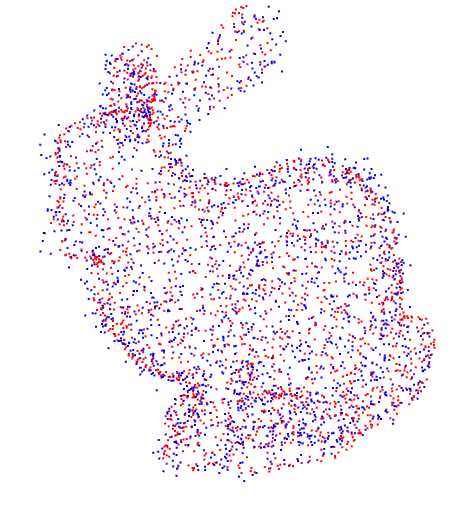}\\
        \end{minipage}
    }
    \hspace{-2mm}
    \subfigure[m-PWAN]{
        \begin{minipage}[b]{0.15\linewidth}
            \includegraphics[width=1\linewidth]{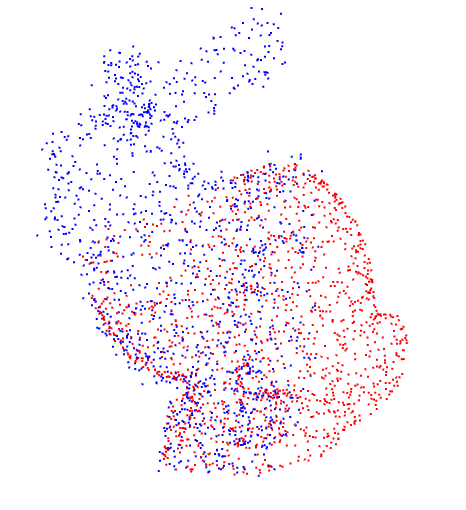}\\
            \includegraphics[width=1\linewidth]{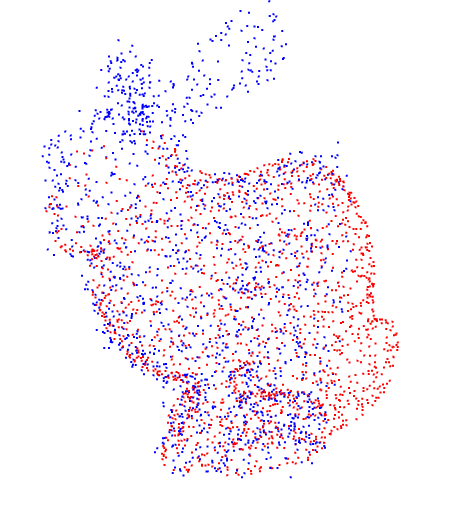}\\
            \includegraphics[width=1\linewidth]{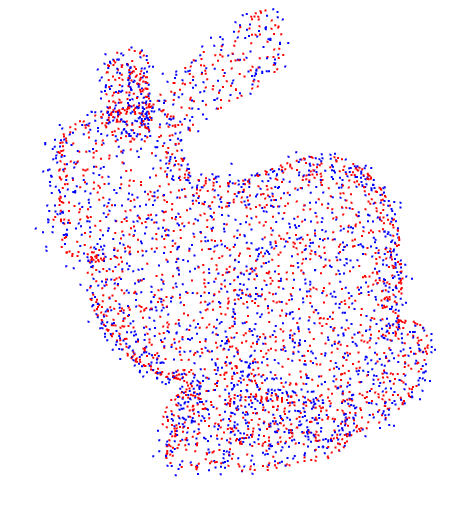}\\
        \end{minipage}
    }
    \vspace{-2mm}
    \caption{An example of registering partially overlapped point sets.
    The overlap ratios are 
    $0.57$ (1st row),
    $0.75$ (2nd row) and $1$ (3rd row).
    }
    \label{qualtitative_partial}
\end{figure*}

We further evaluate PWAN on large-scale armadillo datasets.
We compare PWAN with BCPD and CPD,
because they are the only baseline methods that are scalable in this experiment.
We present some registration results in Fig.~\ref{qualtitative_Large}.
As can be seen,
our method can handle both cases successfully,
while both CPD and BCPD bias toward outliers.

The training details of PWAN in this example are shown in Fig.~\ref{qualtitative_Large_training}.
Due to its adversarial nature,
the loss of PWAN does not decrease monotonically,
instead,
it always increases during the first few steps,
and then starts to decrease.
In addition,
the maximal gradient norm of the network (Lipschitz constant) is indeed controlled near $1$,
and the MSE decreases during the training process.
An example of the registration process is presented in Fig.~\ref{nonrigid_process}.

\begin{figure*}[htb!]
    \centering
    \subfigure[A trajectory of non-rigid registration of the armadillo dataset.]{
      \includegraphics[width=0.18\linewidth]{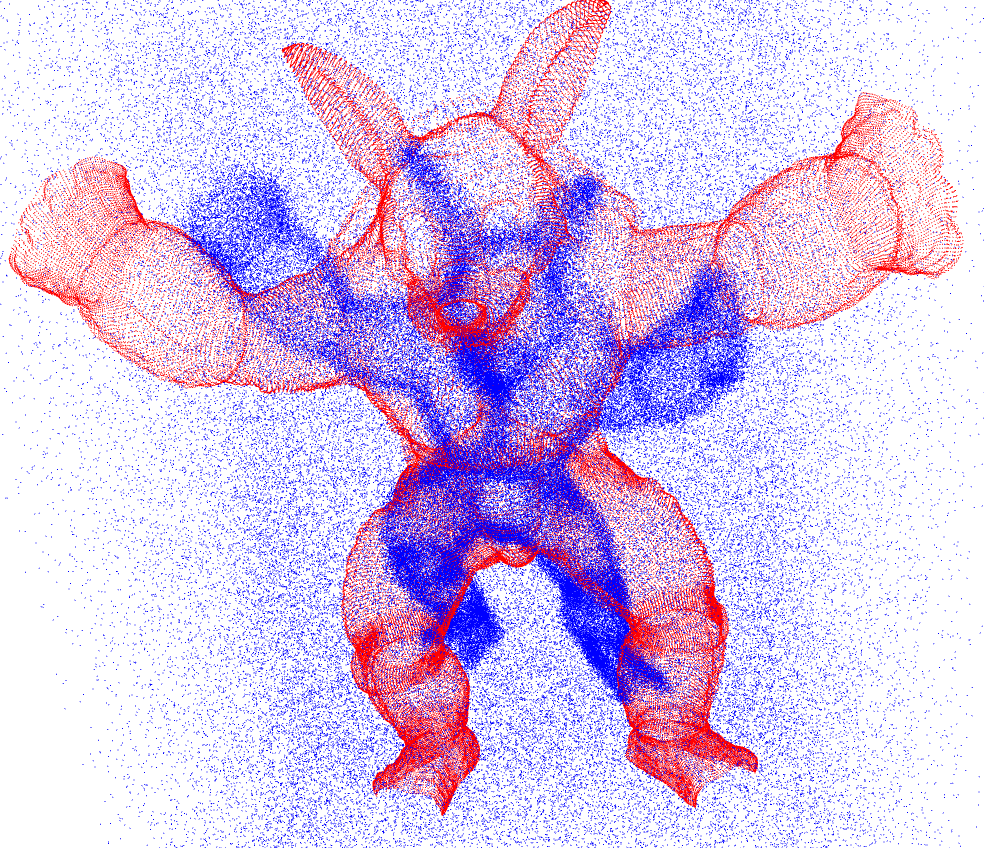}
      \includegraphics[width=0.18\linewidth]{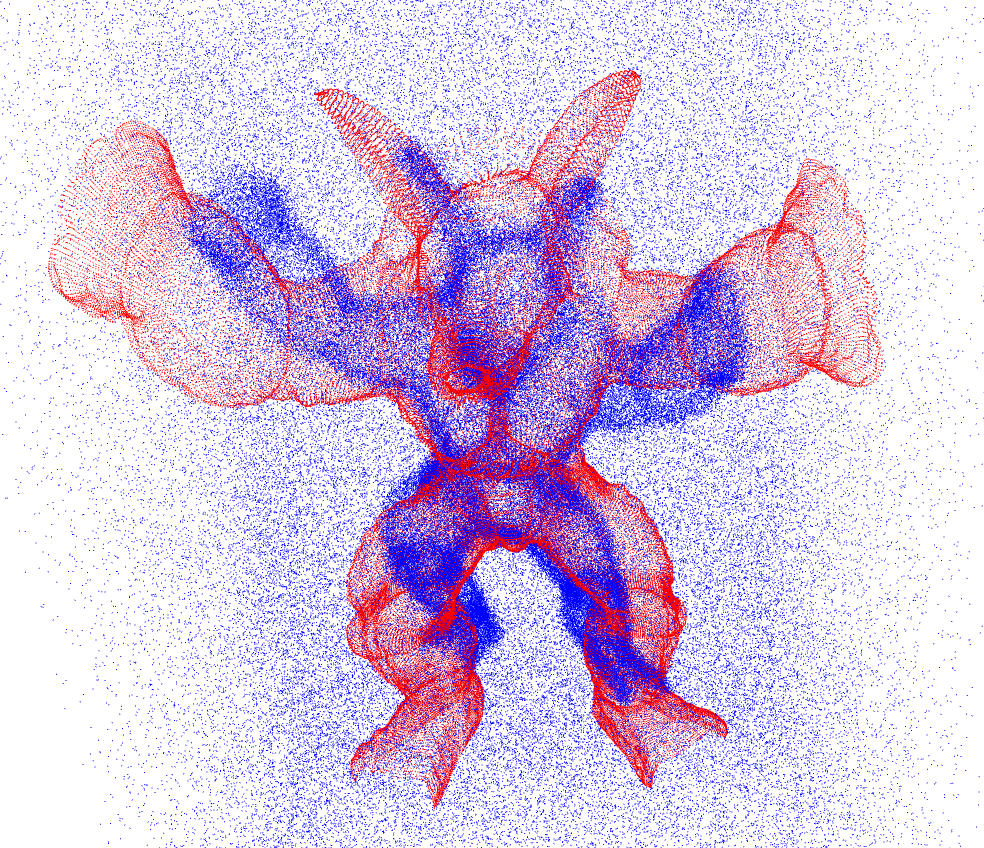}
      \includegraphics[width=0.18\linewidth]{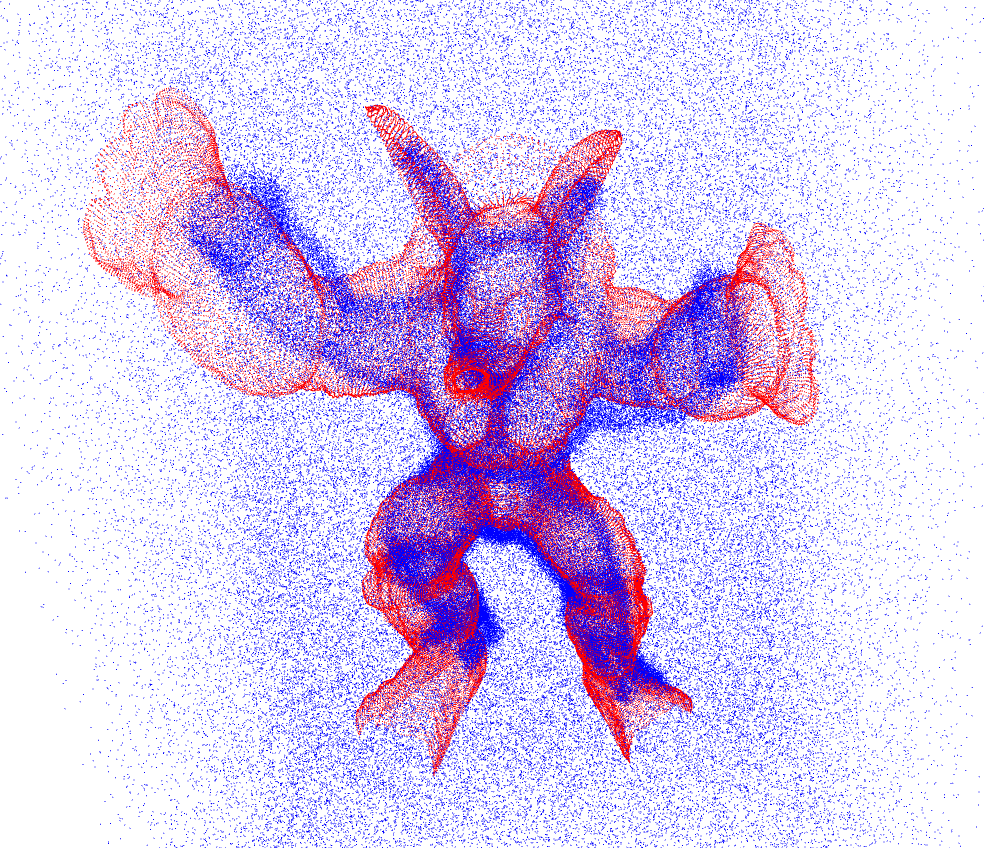}
      \includegraphics[width=0.18\linewidth]{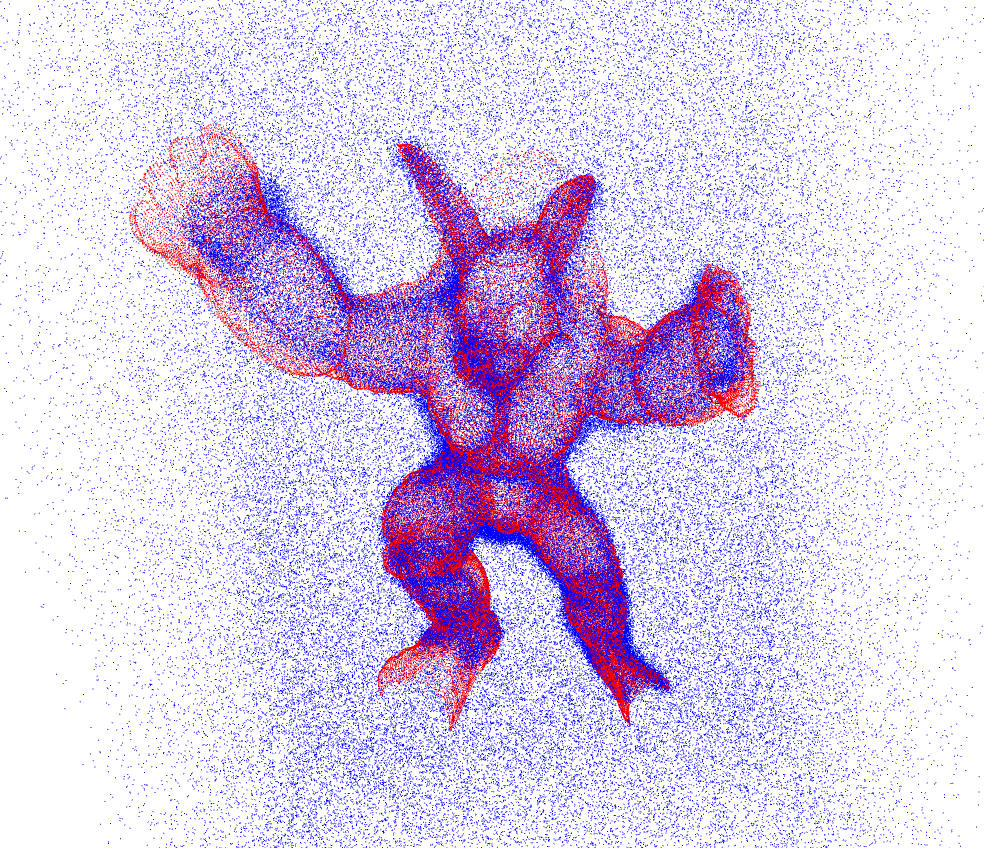}
      \includegraphics[width=0.18\linewidth]{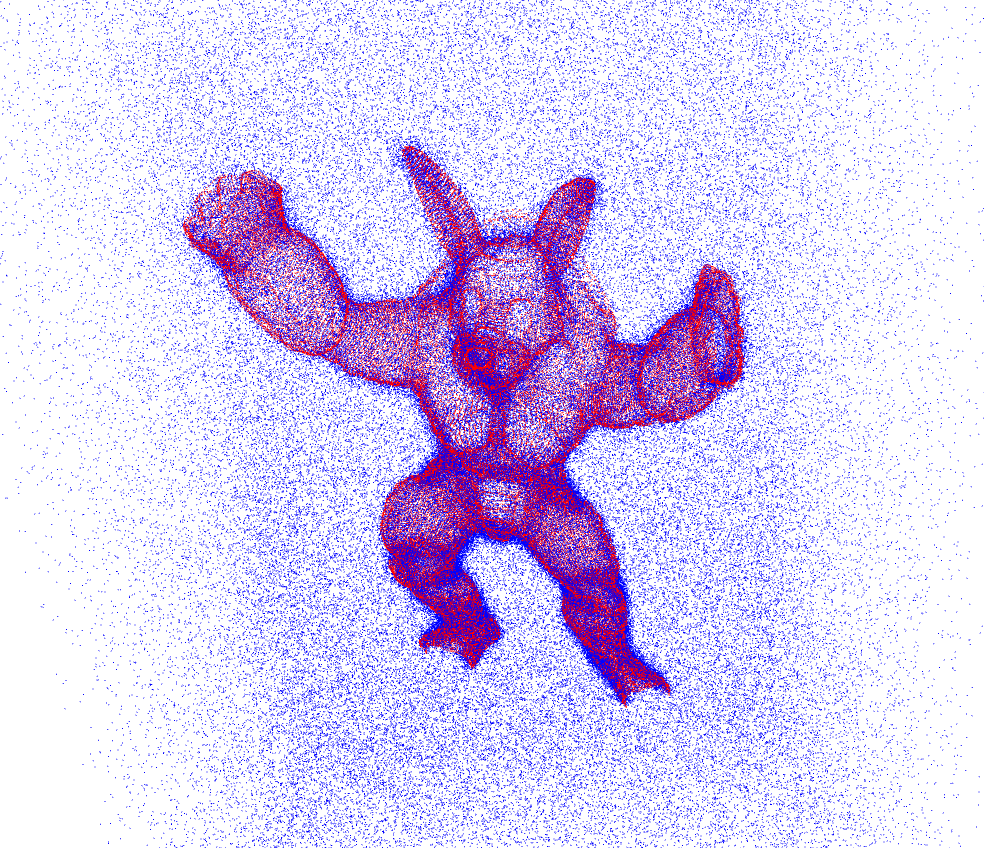}
      \label{nonrigid_process}
    }
    \subfigure[A trajectory of rigid registration of the mountain dataset.]{
      \includegraphics[width=0.18\linewidth]{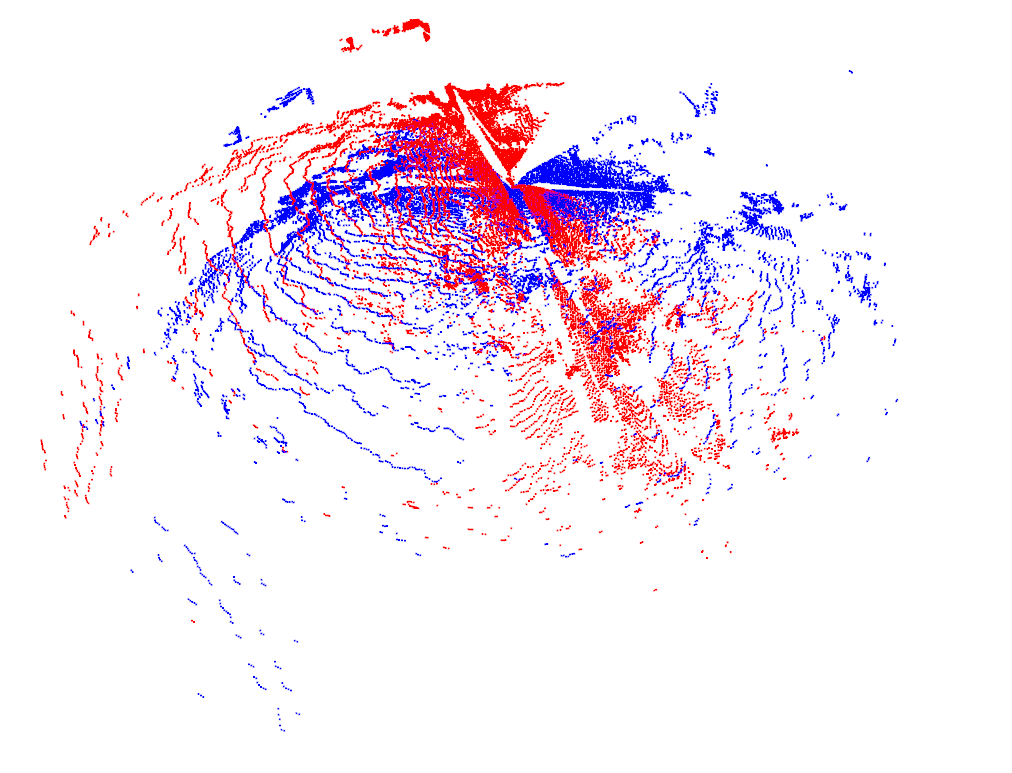}
      \includegraphics[width=0.18\linewidth]{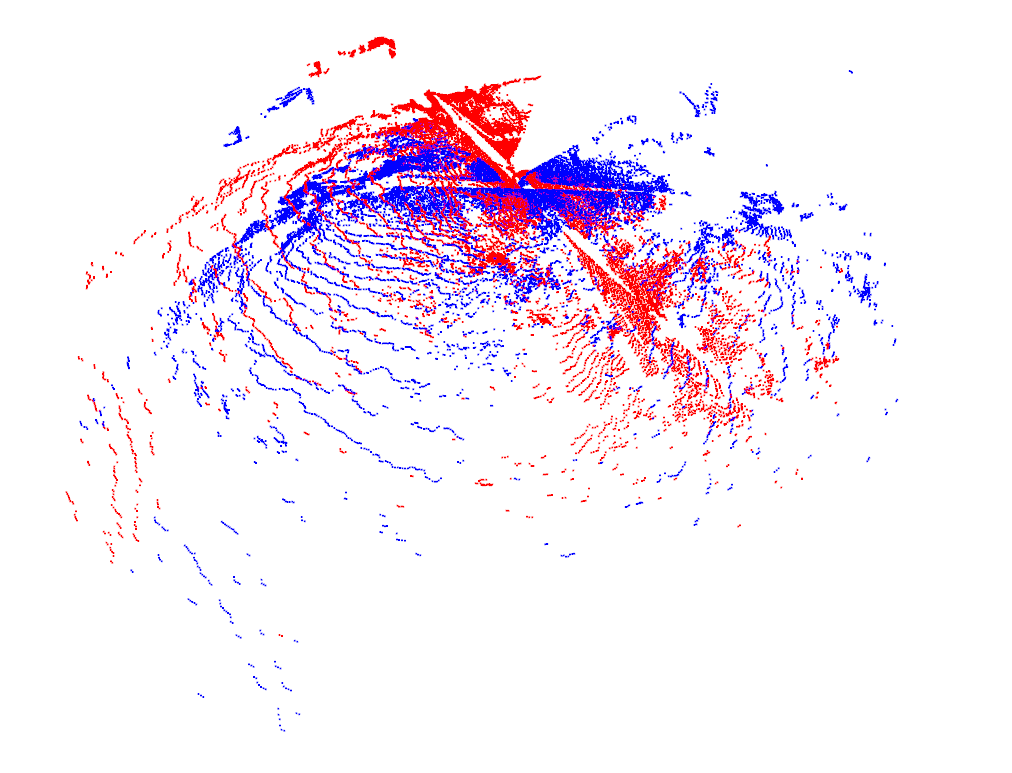}
      \includegraphics[width=0.18\linewidth]{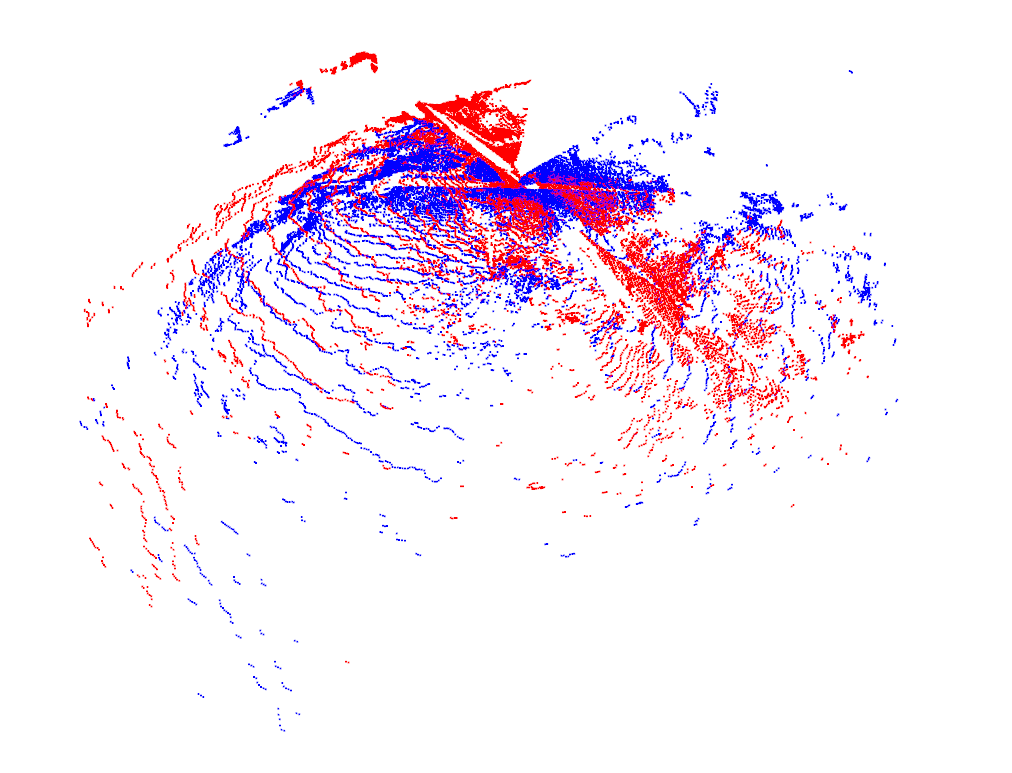}
      \includegraphics[width=0.18\linewidth]{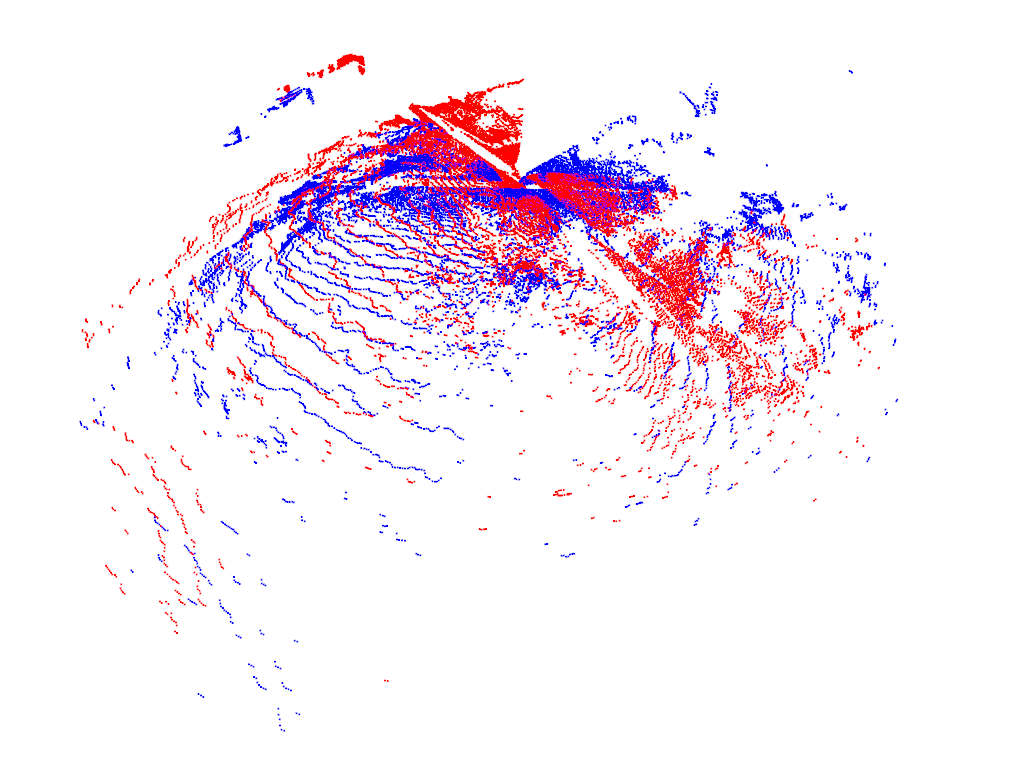}
      \includegraphics[width=0.18\linewidth]{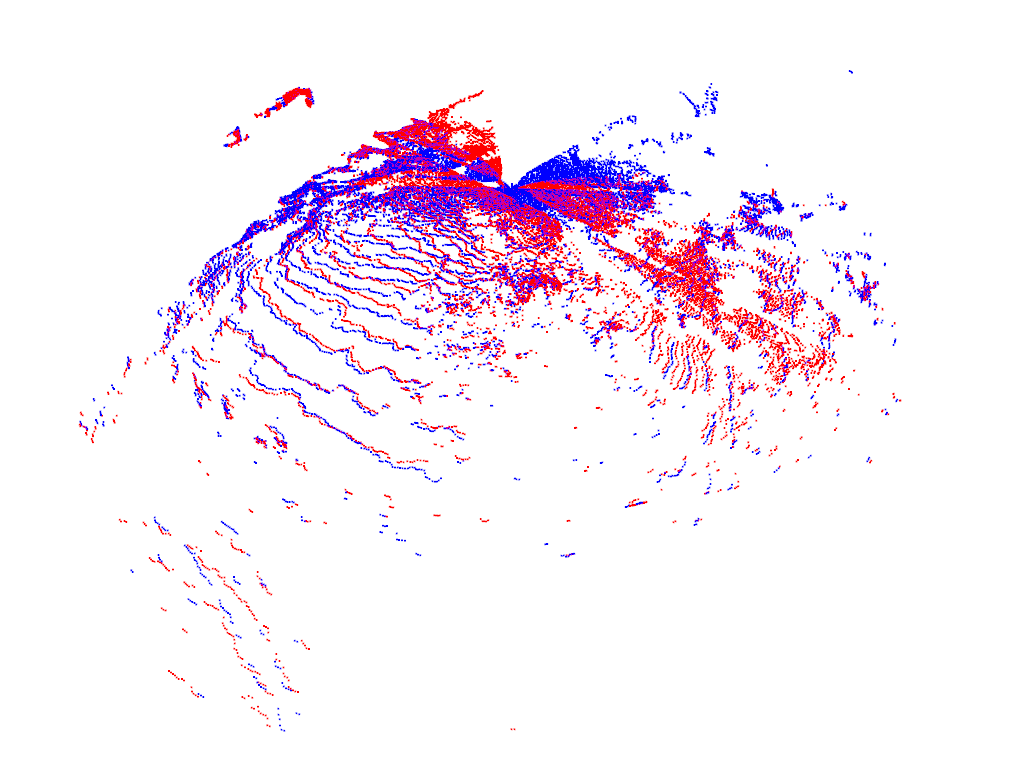}
      \label{rigid_process}
    }
      \caption{Registration trajectory. We show the process of non-rigid~\subref{nonrigid_process} and rigid registration~\subref{rigid_process} from left to right.
  }
  \label{Registration process.}
  \end{figure*}

\begin{figure*}[htb!]
	\centering  
   \subfigure[ An example of registering noisy point sets. 
   The source and reference set contain $8 \times 10^4$ and $1.76 \times 10^5$ points respectively.
    ]{
      \label{large_noise_qualitative}
    \begin{minipage}[b]{1\linewidth}
        \centering  
        \includegraphics[width=0.2\linewidth]{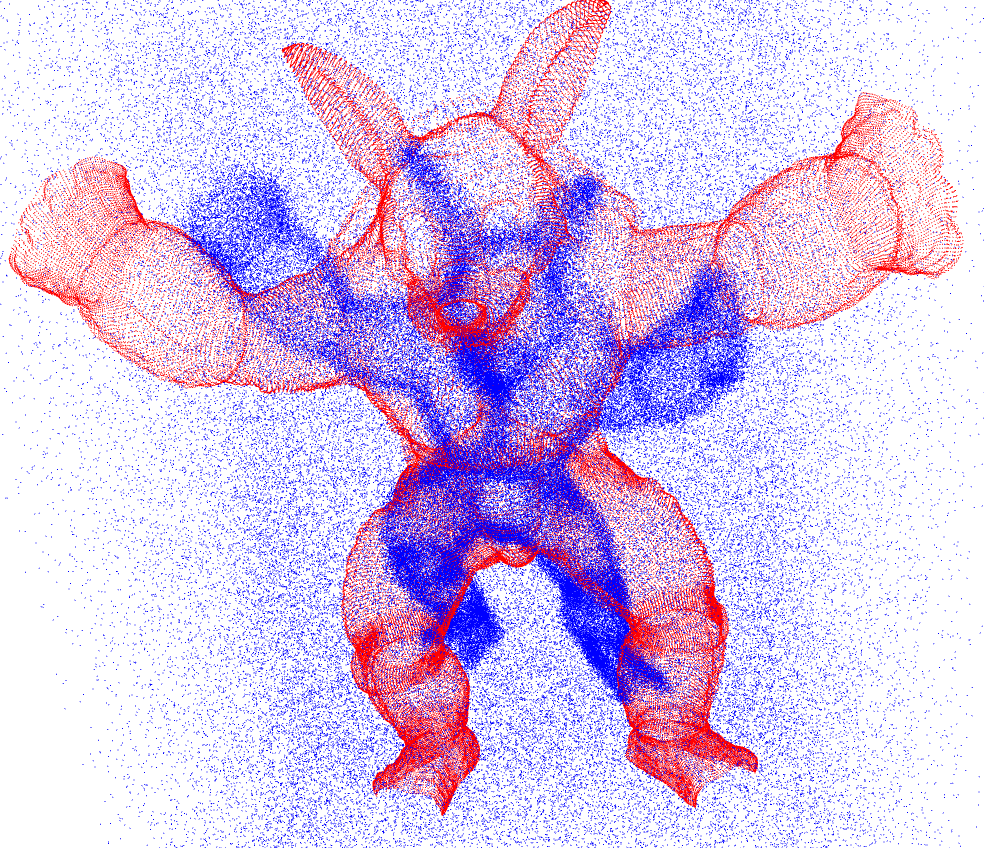}
        \includegraphics[width=0.2\linewidth]{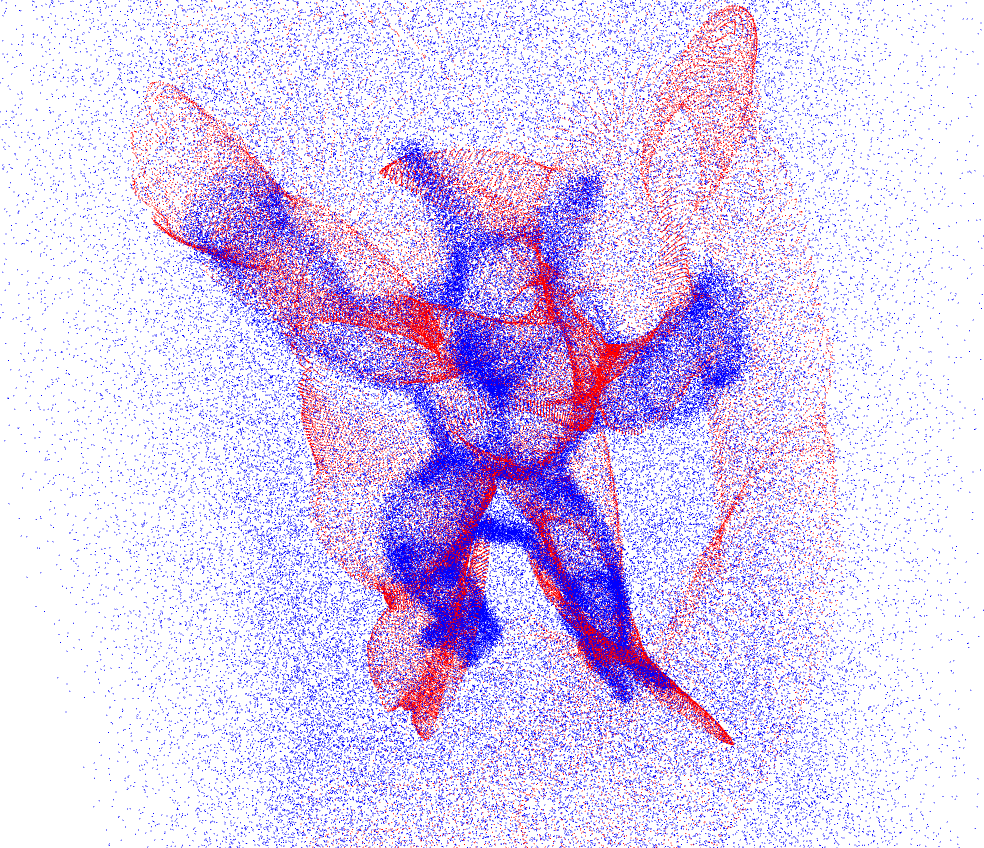}
        \includegraphics[width=0.2\linewidth]{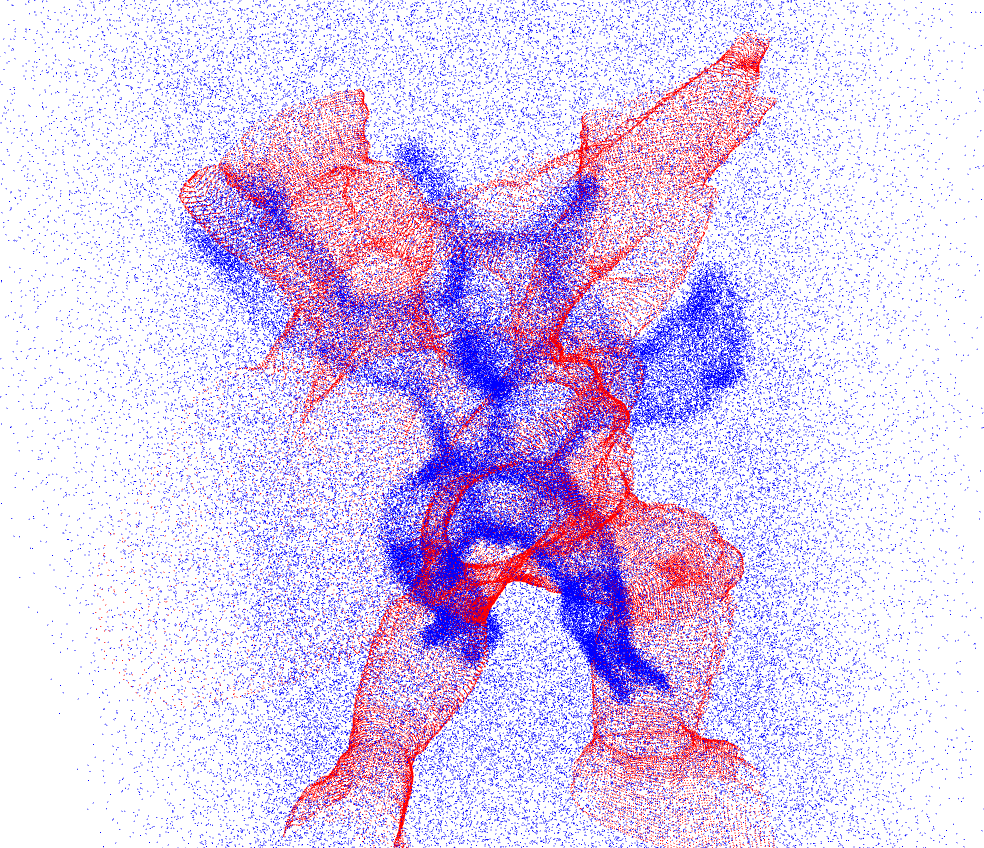}
        \includegraphics[width=0.2\linewidth]{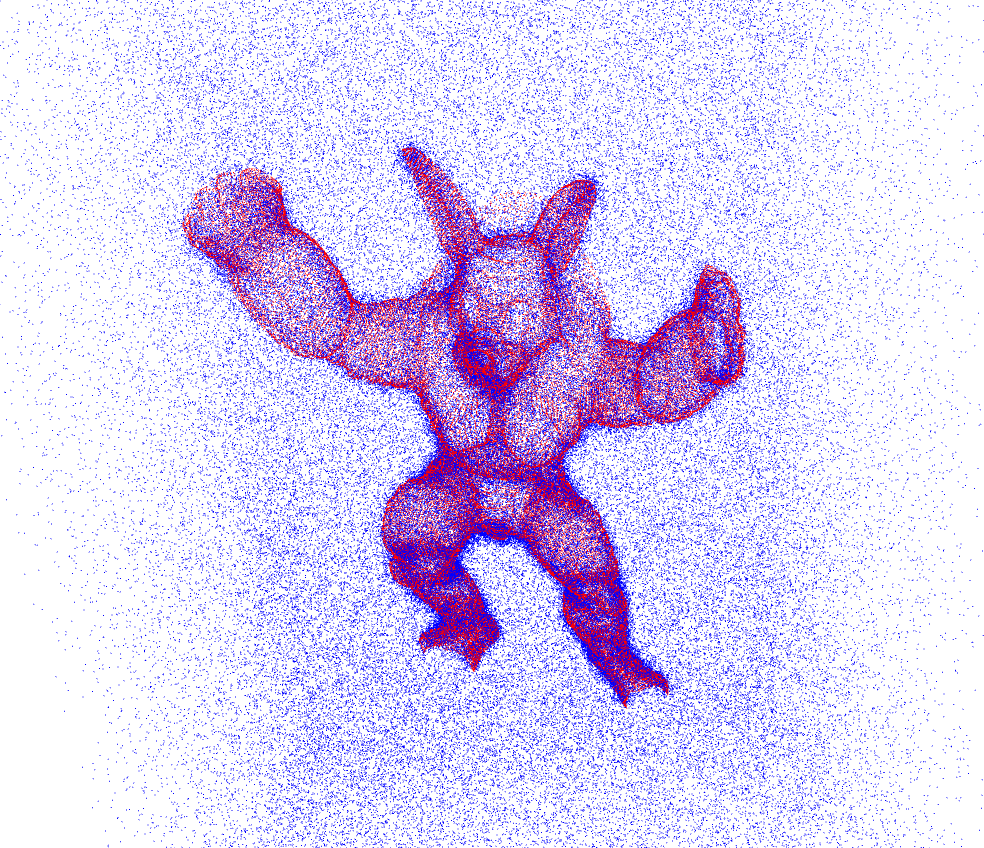}\\
        \hspace{3mm} Initial sets \hspace{23mm} BCPD  \hspace{26mm} CPD \hspace{23mm} PWAN \hspace{5mm}
    \end{minipage}
   }
   \subfigure[ An example of registering partially overlapped point sets. 
   The source and reference set both contain $7 \times 10^4$ points.
    ]{
    \begin{minipage}[b]{1\linewidth}
        \centering  
        \includegraphics[width=0.2\linewidth]{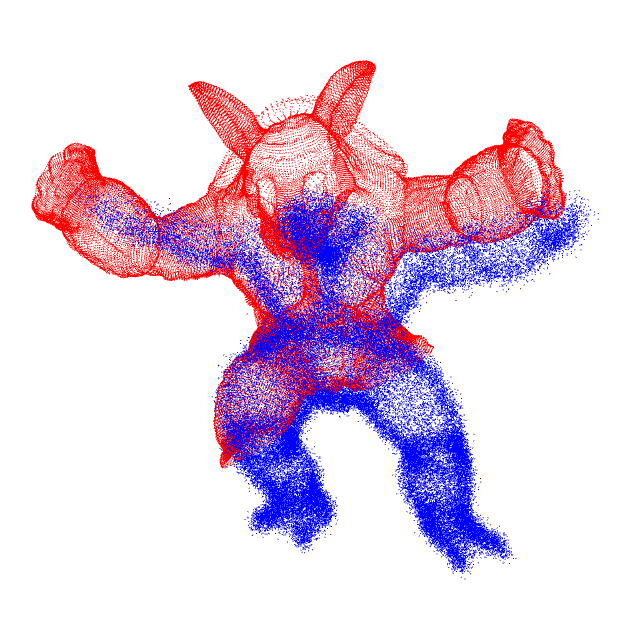}
        \hspace{-5mm}
        \includegraphics[width=0.2\linewidth]{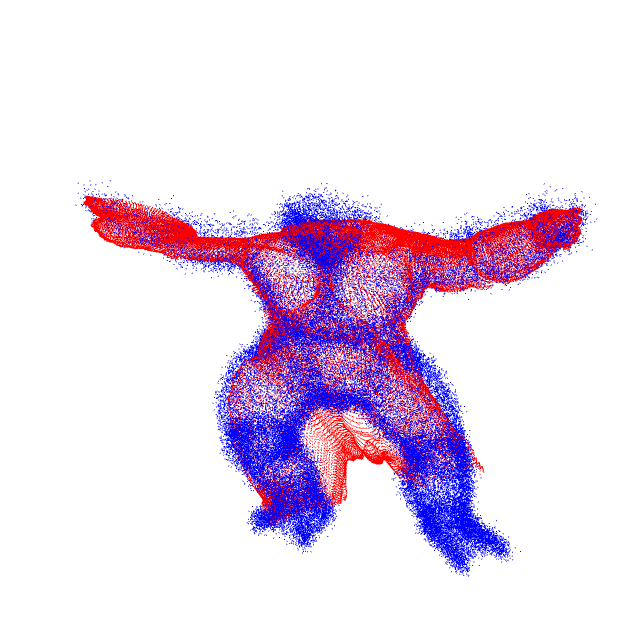}
        \hspace{-5mm}
        \includegraphics[width=0.2\linewidth]{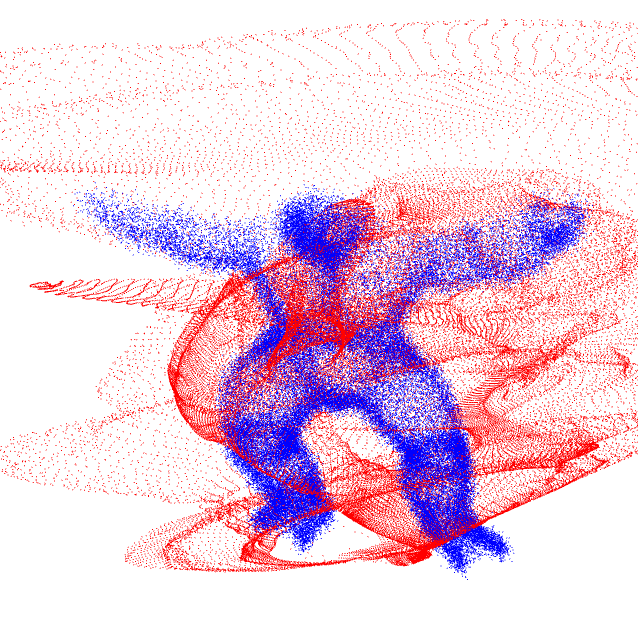}
        \hspace{-5mm}
        \includegraphics[width=0.2\linewidth]{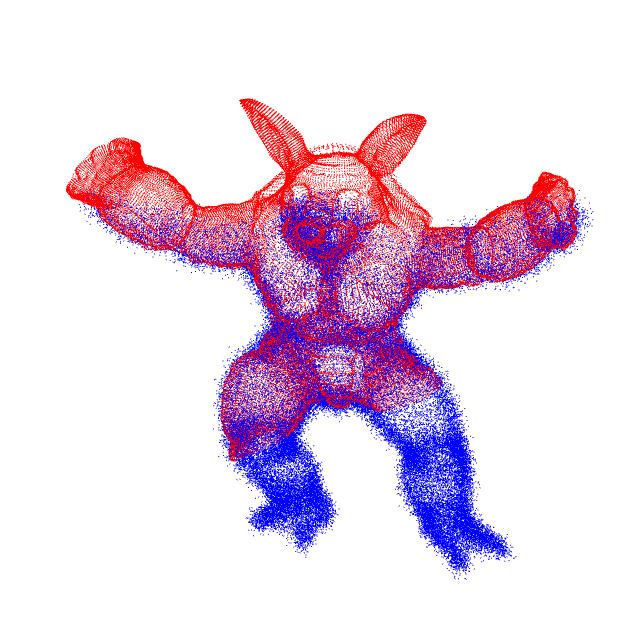}
        \hspace{-5mm}
        \includegraphics[width=0.2\linewidth]{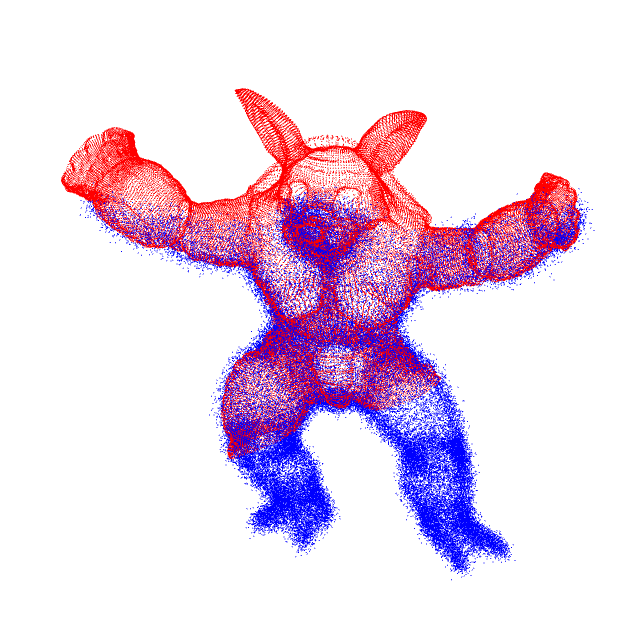} \\
        \vspace{-3mm}
        \hspace{3mm} Initial sets \hspace{20mm} BCPD  \hspace{23mm} CPD \hspace{21mm} d-PWAN \hspace{20mm} m-PWAN
    \end{minipage}
   }
\vspace{-2mm}
  \caption{Examples of registering large-scale point sets.}
	\label{qualtitative_Large}
\end{figure*}

\begin{figure}[htb!]
  \centering
  \includegraphics[width=0.52\linewidth]{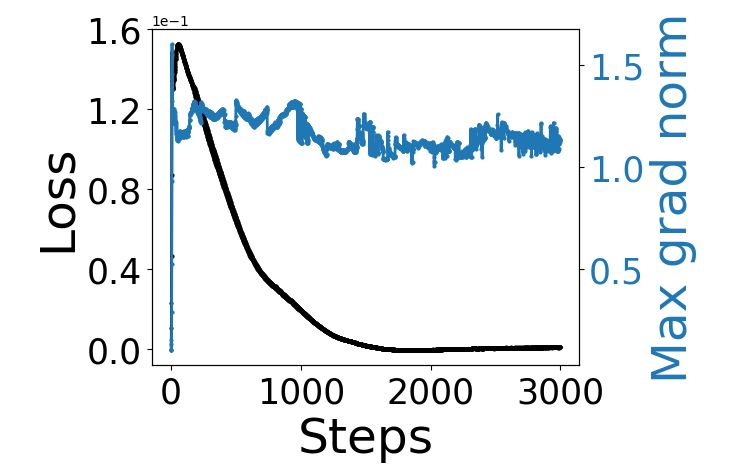}
  \hspace{-4mm}
  \includegraphics[width=0.45\linewidth]{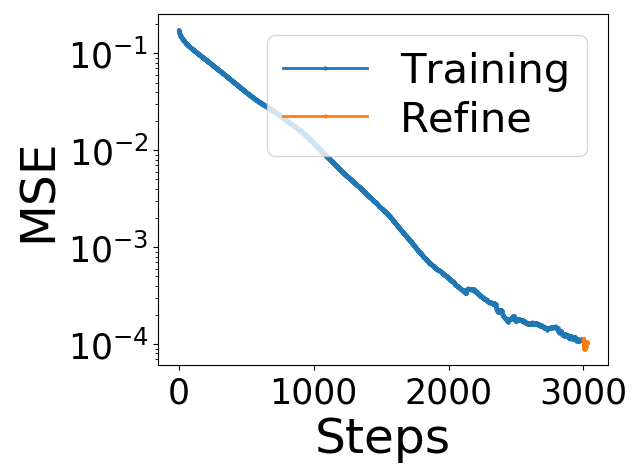}
  \caption{Training details of PWAN on the pair of point sets shown in Fig.~\ref{large_noise_qualitative}}
	\label{qualtitative_Large_training}
\end{figure}

\subsection{More Details in Sec.~\ref{Sec_exp_real}}
\label{Sec_exp_real_app}
The quantitative comparison on human face datasets is presented in Tab.~\ref{Face_tab}.
\begin{table}[h!]
	\begin{center}
        \vspace{-4mm}
	  \caption{Quantitative result of registering the space-time faces dataset. }
	  \label{Face_tab}
	  \begin{tabular}{c c c c} 
		  \hline
		   &BCPD & CPD & PWAN \\
		\hline
	  \centering
    MSE $(\times 10^{-3})$ &  1.7 (1.0)  &  0.49 (0.2)& 0.32 (0.08) \\
		\hline
	  \end{tabular}
	\end{center}
	\vspace{-4mm}
\end{table}

The human body dataset is taken from a SHREC'19 track called ``matching humans with different connectivity''~\cite{bodydataset}.
This dataset consists of $44$ shapes,
and we manually select $3$ pairs of shapes for our experiments.
To generate a point set for a shape,
we first sample $50000$ random points from the 3D mesh,
and then apply voxel grid filtering to down-sample the point set to less than $10000$ points.
The description for the generated point sets is presented in Tab.~\ref{Human_points}

\begin{table}[h!]
	\begin{center}
	  \caption{Point sets used for registration. no.$m$ represents the m-th shape in the dataset~\cite{bodydataset}.}
	  \label{Human_points}
      \begin{tabular}{ c c c c c} 
		  \hline
		    & (no.1, no.42) & (no.18, no.19) & (no.30, no.31)  \\
		\hline
	  \centering
      Size  & (5575, 5793) & (6090, 6175) & (6895, 6792)  \\
      Description & \makecell[l]{same pose \\  different person}  & \makecell[l]{different pose \\ same person} & \makecell[l]{different pose \\ different person} \\
		\hline
	  \end{tabular}
	\end{center}
\end{table}

We conduct two experiments to evaluate PWAN on registering complete and incomplete point sets respectively.
In the first experiment,
we register $3$ pairs of point sets using PWAN,
where the human shapes come from different people or/and with different poses.
In the second experiment,
we register incomplete point sets which are generated by cropping a fraction of the no.30 and no.31 point sets.
For both experiments,
we compare PWAN with CPD~\cite{myronenko2006non} and BCPD~\cite{hirose2021a},
and we only present qualitative registration results,
because we do not know the true correspondence between point sets.

The results of the first experiment are shown in Fig.~\ref{real_human_app}.
As can be seen,
PWAN can handle articulated deformations and produce good full-body registration results.
In contrast,
CPD and BCPD have difficulties aligning point sets with large articulated deformations, 
as significant registration errors are observed near the limbs.

The results of the second experiment are shown in Fig.~\ref{real_human_partial_app}.
As can be seen,
both CPD and BCPD fail in this experiment,
as the non-overlapping points are seriously biased. 
For example,
in the $3$-rd row,
they both wrongly match the left arm to the body,
which causes highly unnatural artifacts.
In contrast,
the proposed PWAN can handle the partial matching problem well,
since it successfully maintains the shape of non-overlapping regions,
which contributes to the natural registration results.

\begin{figure*}[htb!]
	\centering
  \vspace{-3mm}
    \subfigure[Source set]{
        \begin{minipage}[b]{0.18\linewidth}
            \centering
            \includegraphics[width=0.66\linewidth]{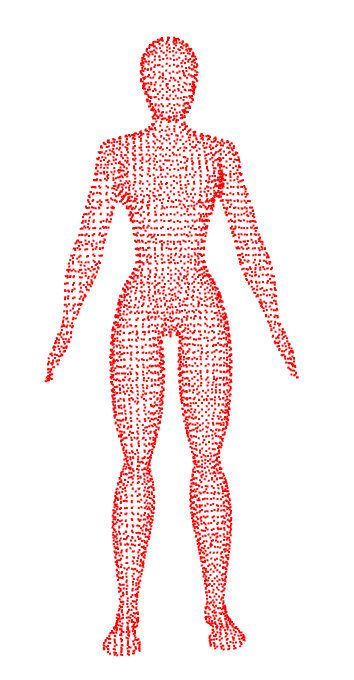}\\
            \includegraphics[width=1\linewidth]{Result/Human/1819/new_000000_transformed_npy_vis_vis3_step1_}\\
            \includegraphics[width=1\linewidth]{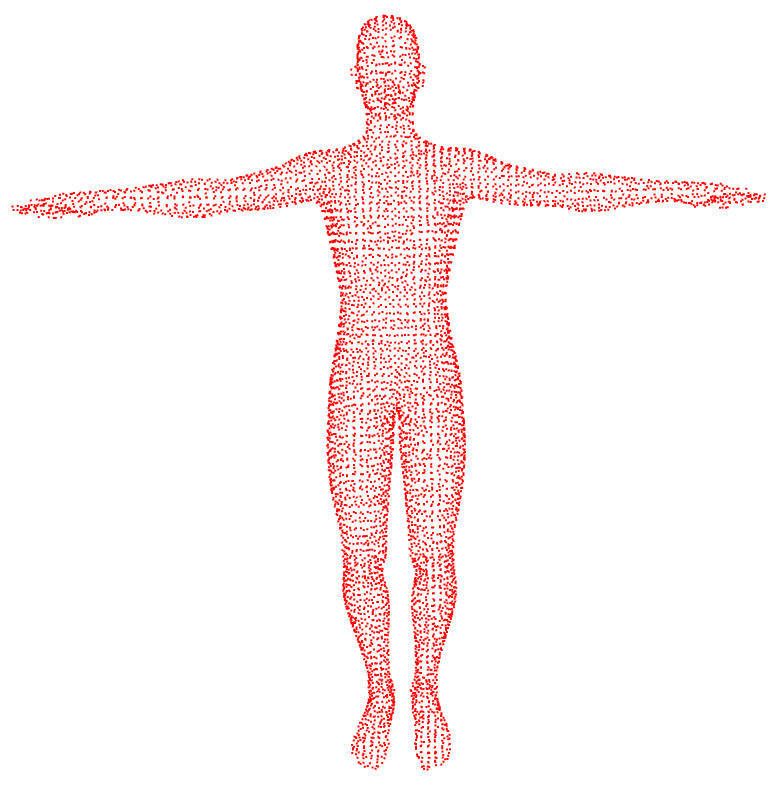}
        \end{minipage}
    }
    \subfigure[Reference set]{
        \begin{minipage}[b]{0.18\linewidth}
            \centering
            \includegraphics[width=0.66\linewidth]{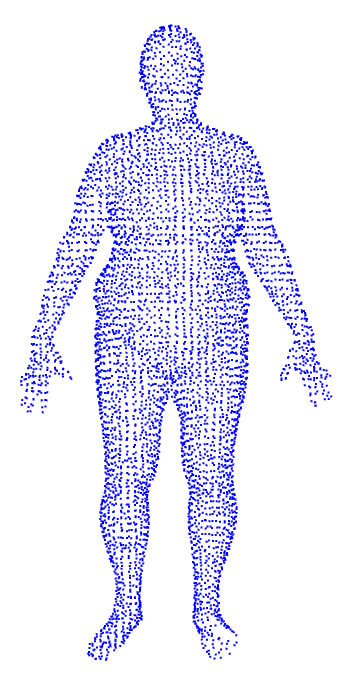}\\
            \includegraphics[width=1\linewidth]{Result/Human/1819/new_000000_transformed_npy_vis_vis3_step2_}\\
            \includegraphics[width=1\linewidth]{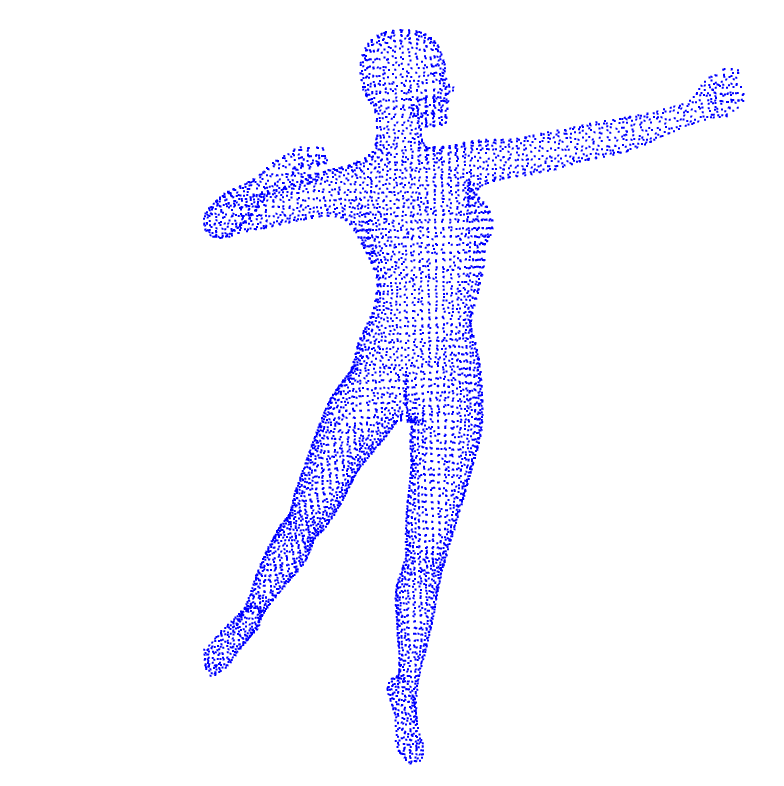}
        \end{minipage}
    }
    \subfigure[PWAN (Ours)]{
        \begin{minipage}[b]{0.18\linewidth}
            \centering
            \includegraphics[width=0.66\linewidth]{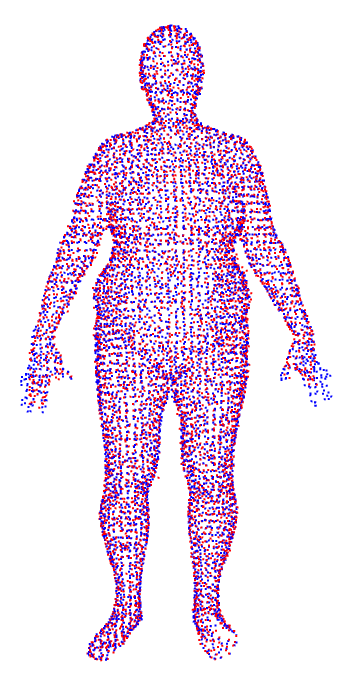}\\
            \includegraphics[width=1\linewidth]{Result/Human/1819/new_020000_transformed_npy_vis_vis3_}\\
            \includegraphics[width=1\linewidth]{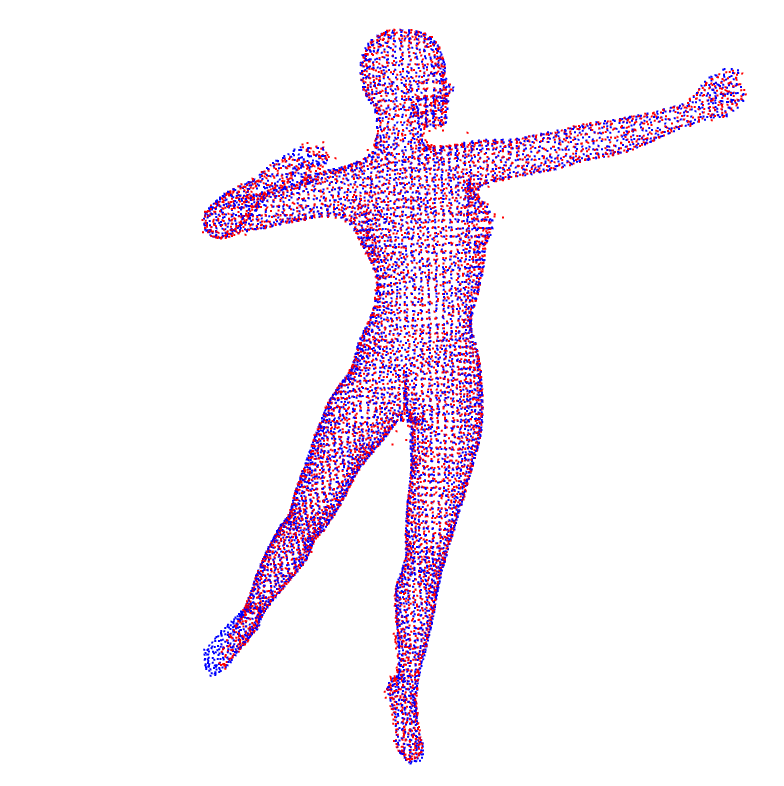}
        \end{minipage}
    }
    \subfigure[BCPD~\cite{hirose2021a}]{
        \begin{minipage}[b]{0.18\linewidth}
            \centering
            \includegraphics[width=0.66\linewidth]{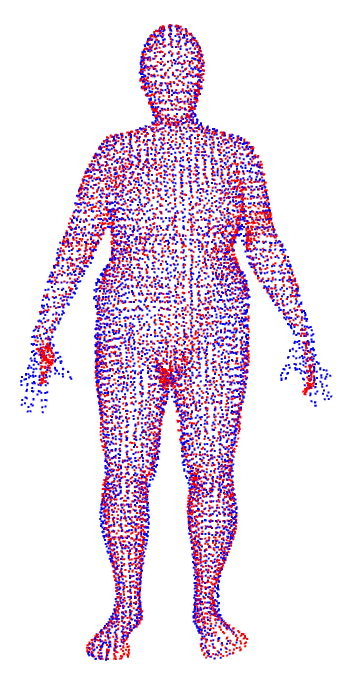} \\
            \includegraphics[width=1\linewidth]{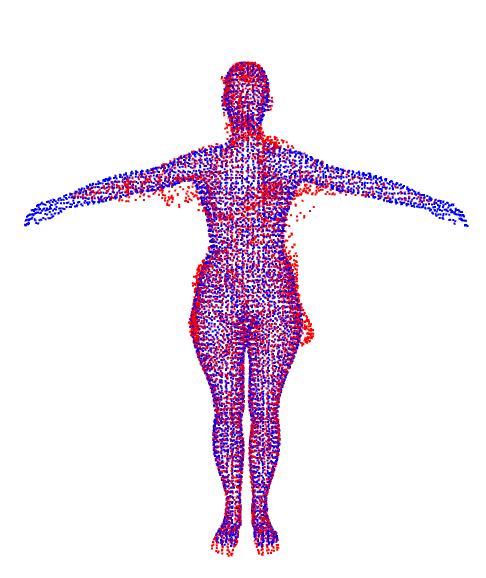} \\
            \includegraphics[width=1\linewidth]{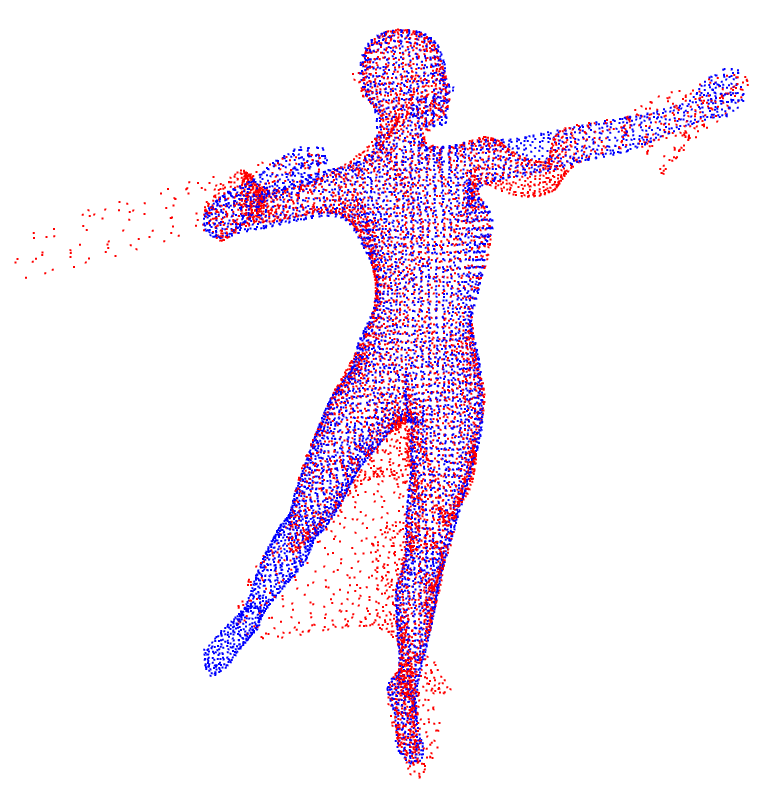} 
        \end{minipage}
    }
    \subfigure[CPD~\cite{myronenko2006non}]{
        \begin{minipage}[b]{0.18\linewidth}
            \centering
            \includegraphics[width=0.66\linewidth]{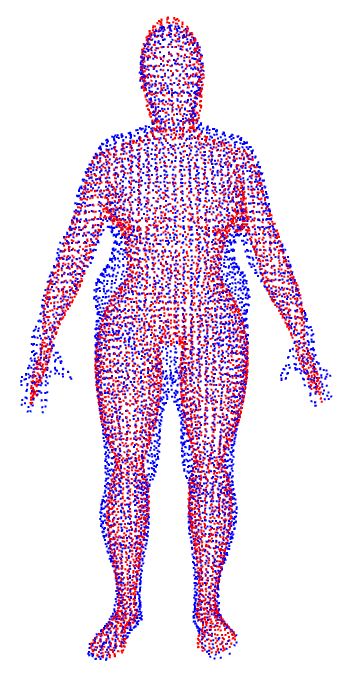} \\
            \includegraphics[width=1\linewidth]{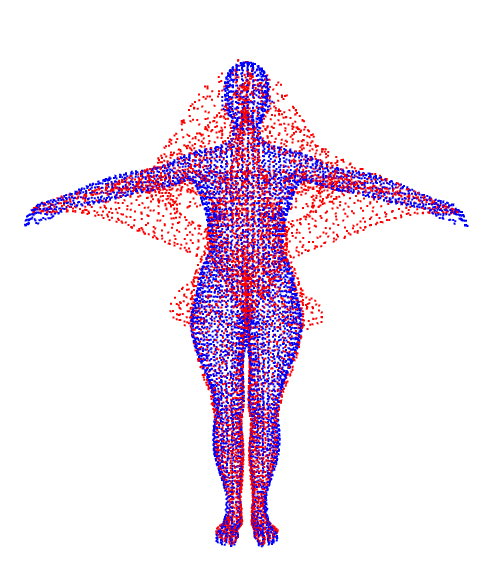} \\
            \includegraphics[width=1\linewidth]{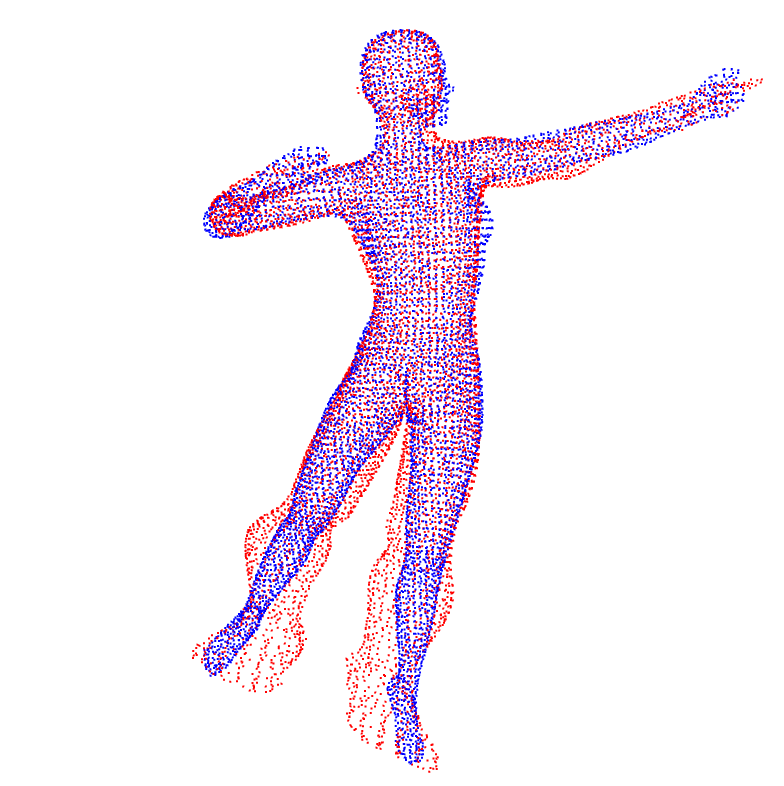}
        \end{minipage}
    }
    \vspace{-1mm}
    \caption{
    The results of registering complete point sets no.1 to no.42 (1-st row), 
    no.18 to no.19 (2-nd row), 
    and no.30 to no.31 (3-rd row). 
    }
    \vspace{-3mm}
    \label{real_human_app}
\end{figure*}

\begin{figure*}[htb!]
	\centering
  \vspace{-1mm}
    \subfigure[Source set]{
        \begin{minipage}[b]{0.18\linewidth}
            \centering
            \includegraphics[width=1\linewidth]{Result/Human/Crop3/new_000000_transformed_npy_vis_vis3_step1_}\\
            \includegraphics[width=1\linewidth]{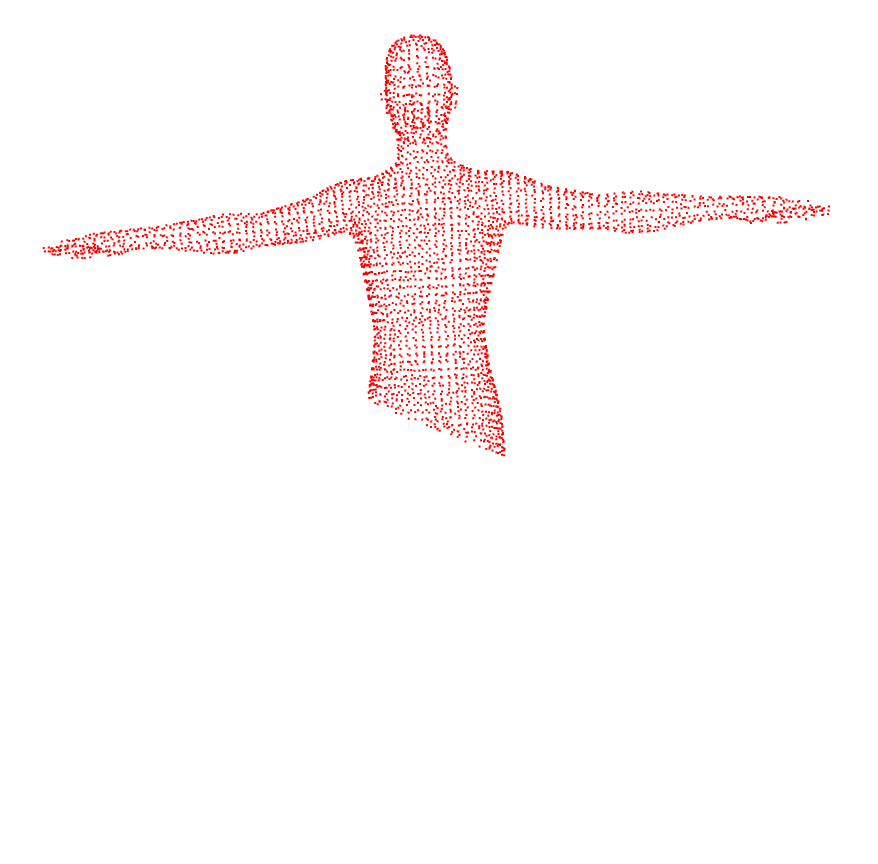} \\
            \includegraphics[width=1\linewidth]{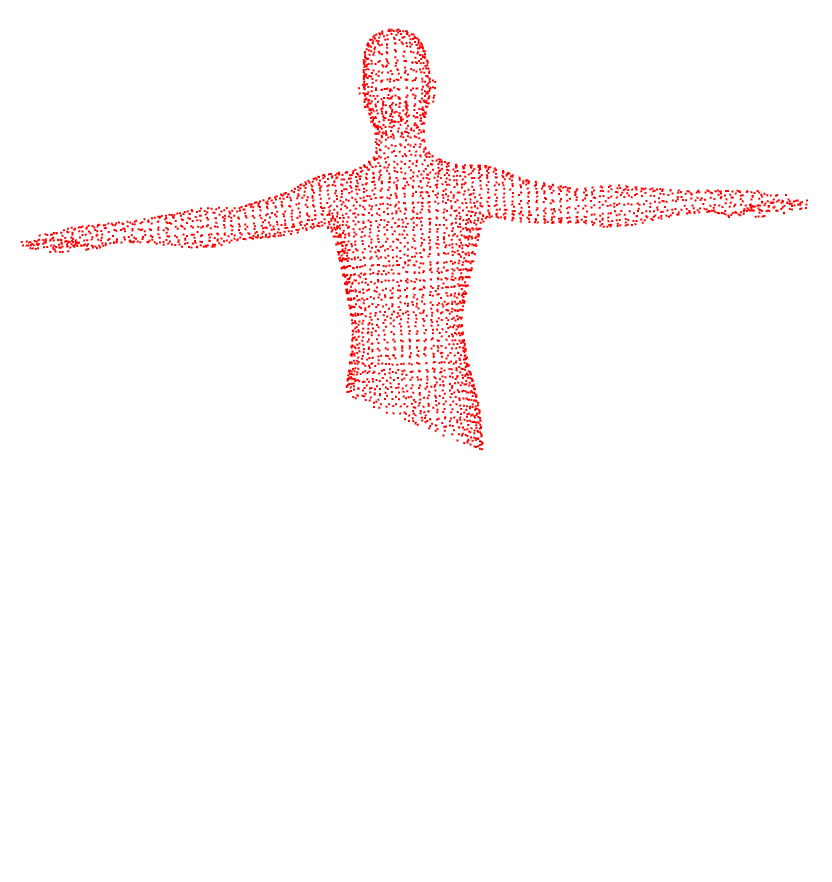} \\
        \end{minipage}
    }
    \subfigure[Reference set]{
        \begin{minipage}[b]{0.18\linewidth}
            \centering
            \includegraphics[width=1\linewidth]{Result/Human/Crop3/new_000000_transformed_npy_vis_vis3_step2_} \\
            \includegraphics[width=1\linewidth]{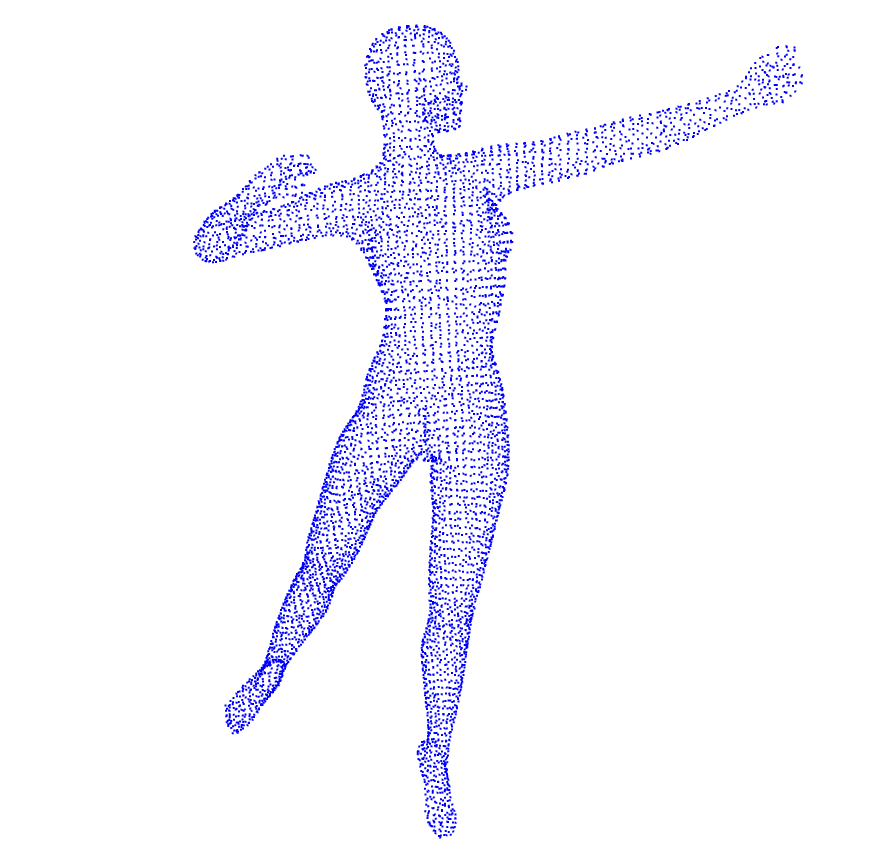} \\
            \includegraphics[width=1\linewidth]{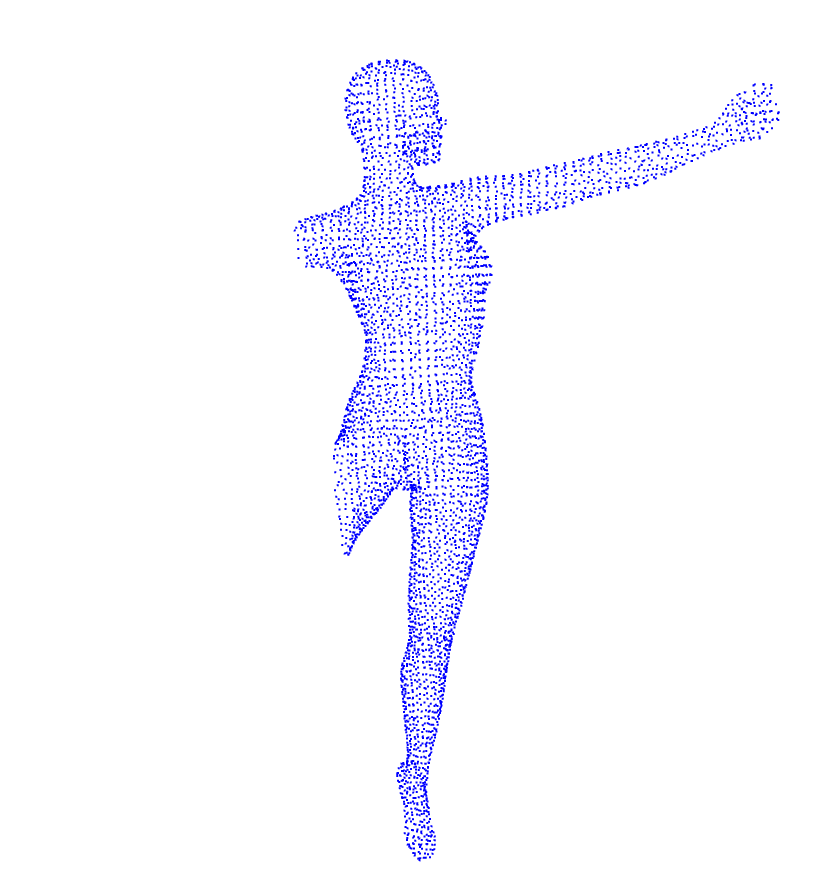} \\
        \end{minipage}
    }
    \subfigure[PWAN (Ours)]{
        \begin{minipage}[b]{0.18\linewidth}
            \centering
            \includegraphics[width=1\linewidth]{Result/Human/Crop3/new_020000_transformed_npy_vis_vis3_} \\
            \includegraphics[width=1\linewidth]{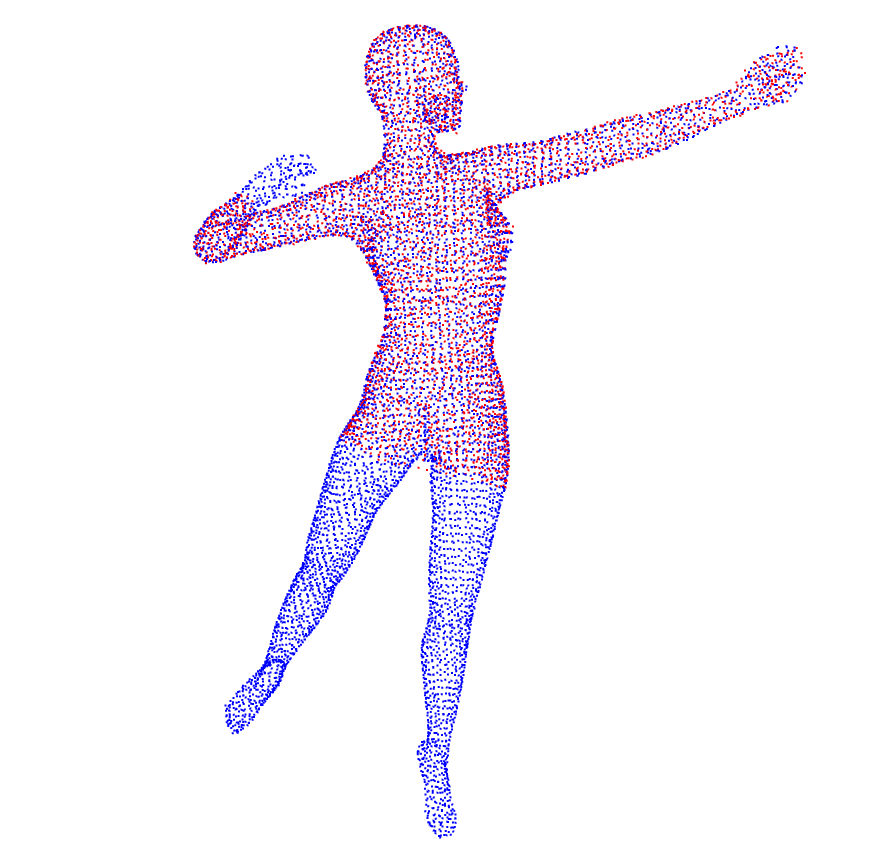} \\
            \includegraphics[width=1\linewidth]{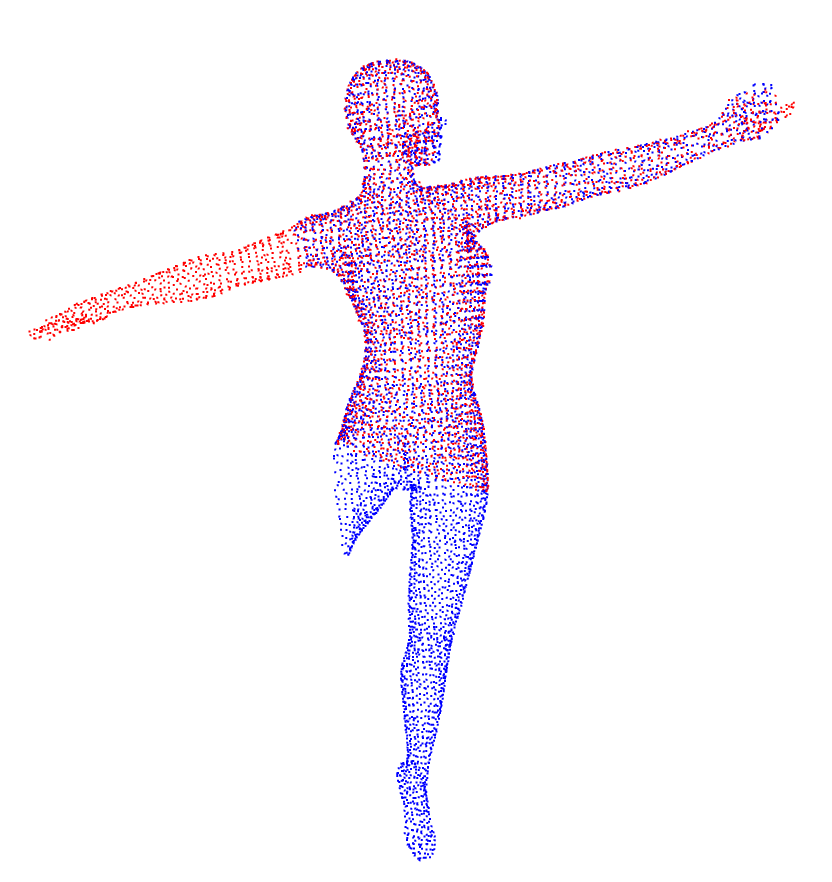} \\
        \end{minipage}
    }
    \subfigure[BCPD~\cite{hirose2021a}]{
        \begin{minipage}[b]{0.18\linewidth}
            \centering
            \includegraphics[width=1\linewidth]{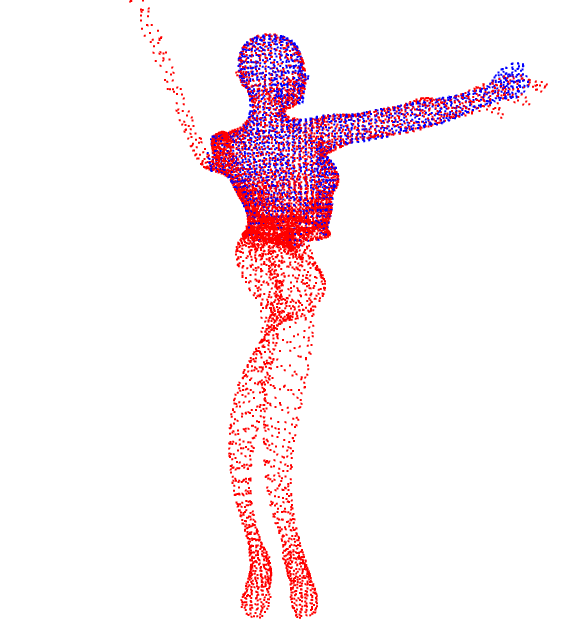} \\
            \includegraphics[width=1\linewidth]{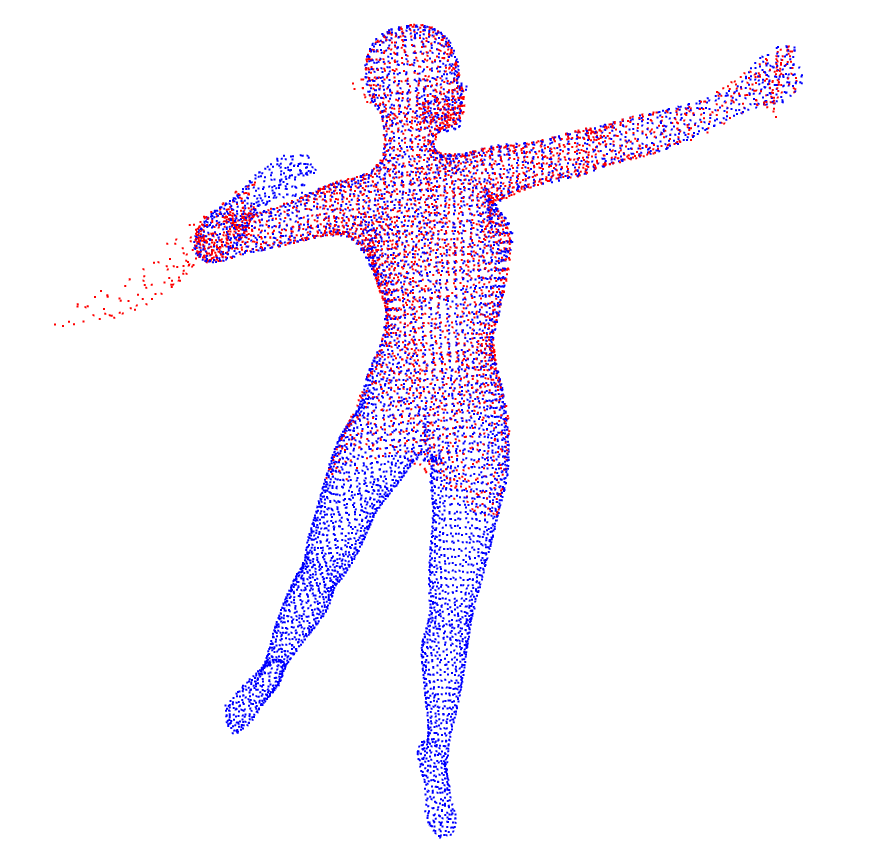} \\
            \includegraphics[width=1\linewidth]{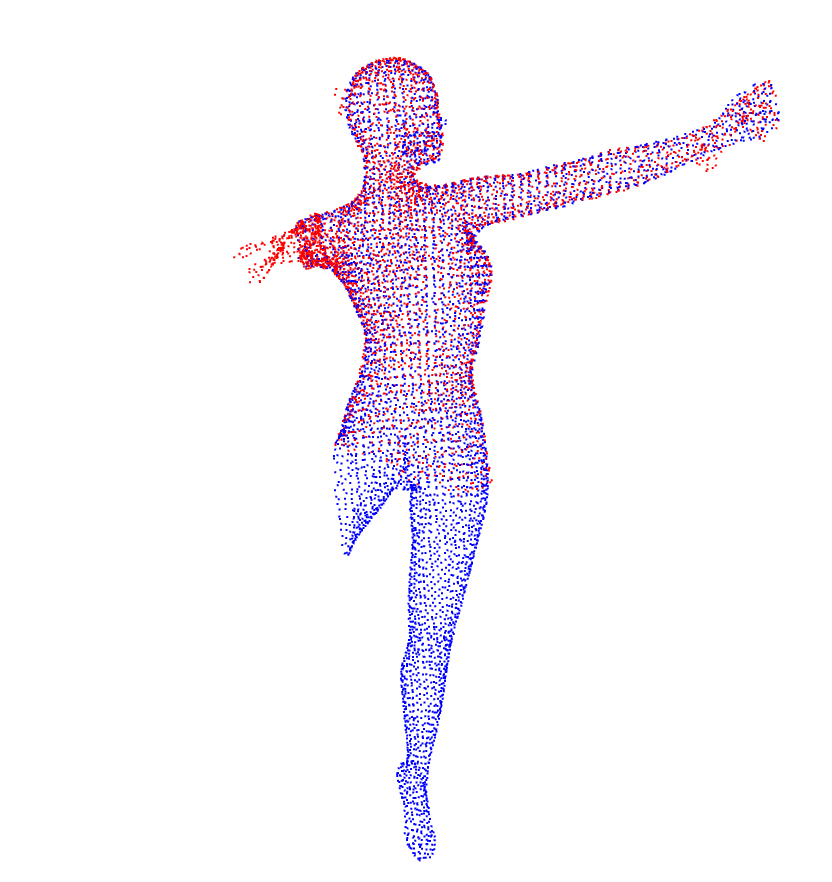} \\
        \end{minipage}
    }
    \subfigure[CPD~\cite{myronenko2006non}]{
        \begin{minipage}[b]{0.18\linewidth}
            \centering
            \includegraphics[width=1\linewidth]{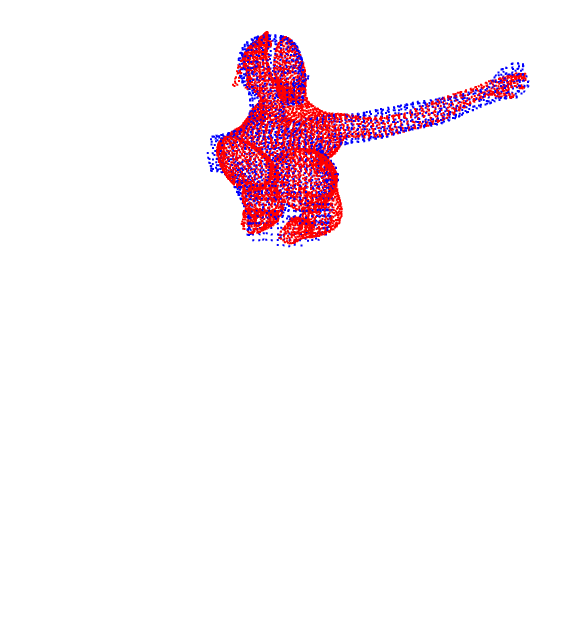}  \\
            \includegraphics[width=1\linewidth]{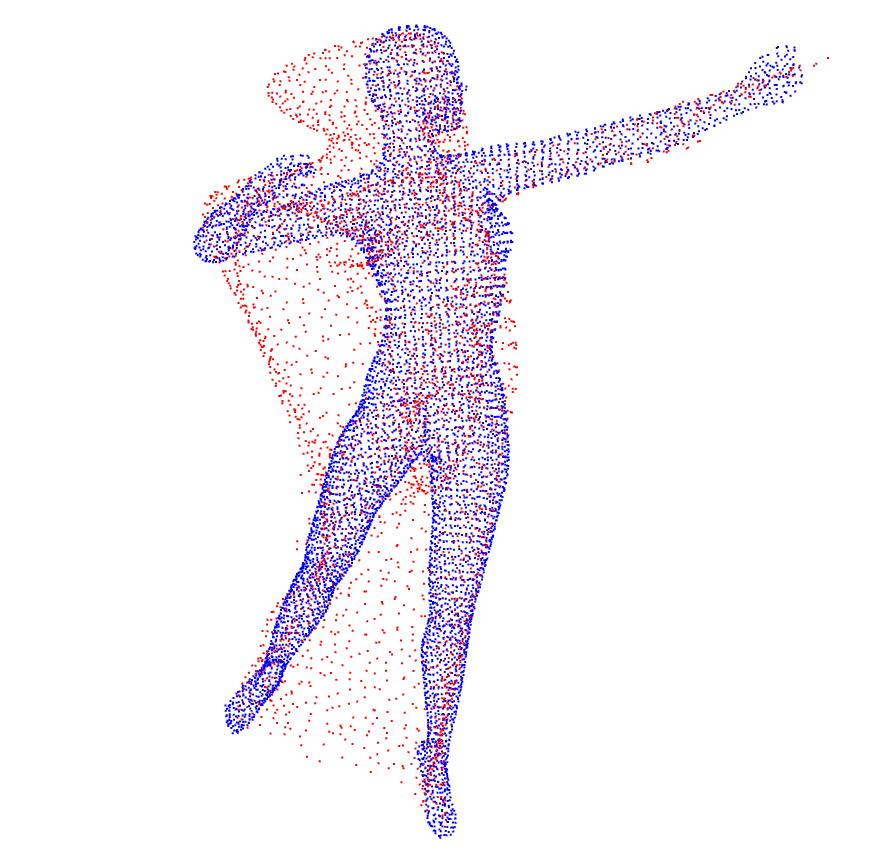} \\
            \includegraphics[width=1\linewidth]{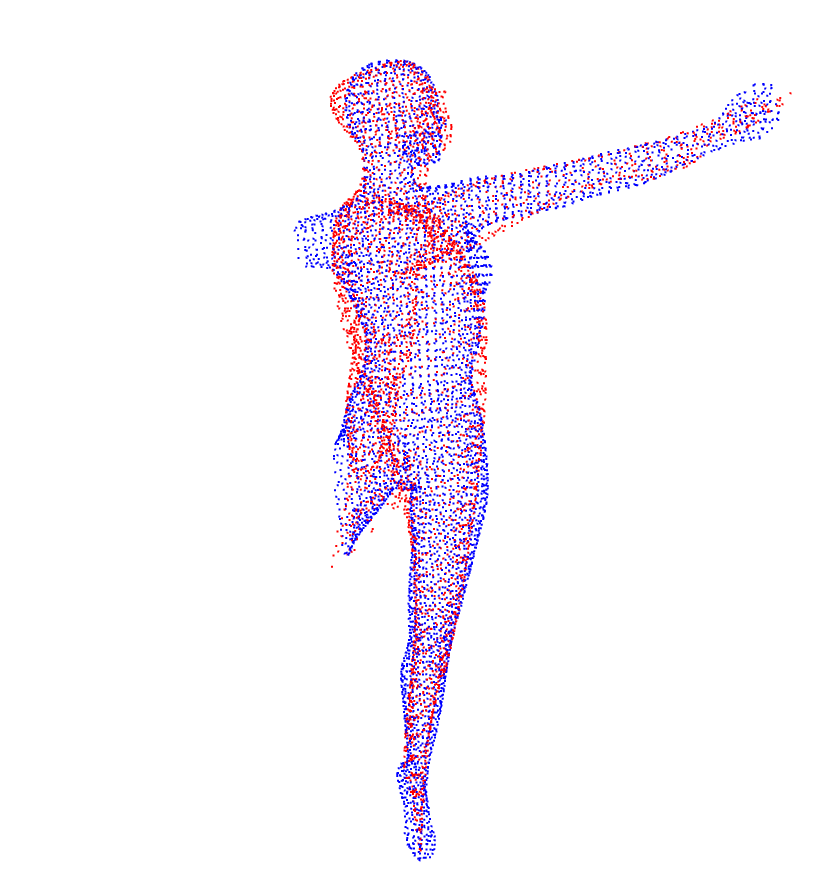} \\
        \end{minipage}
    }
    \vspace{-2mm}
    \caption{
    Registering incomplete point sets no.30 to no.31.
        We present the results of complete-to-incomplete (1-st row),
        incomplete-to-complete (2-nd row) and incomplete-to-incomplete (3-rd row) registration.
    }
    \vspace{-3mm}
    \label{real_human_partial_app}
\end{figure*}

\subsection{More Details in Sec.~\ref{Sec_rigid}}
\label{Sec_rigid_app}
The detailed quantitative results of rigid registration are presented in Tab.~\ref{app_rigid_quantitative},
and we additionally present some examples of the registration results in Fig.~\ref{rigid_app}.
It can be seen that both types of PWAN can accurately align all point sets,
while all baseline algorithms failed to handle the parasaurolophus shapes.

\begin{table*}[ht!]
	\begin{center}
	  \caption{Quantitative results of rigid registration. We report the median and standard deviation of rotation errors.}
    \label{app_rigid_quantitative}
      \begin{tabular}{c c c c c c c c} 
		  \hline
        & apartment & mountain & stair & wood-summer & parasaurolophus & T-rex \\
		\hline
	  \centering
    BCPD & 0.71 (13.4)	& 8.07 (7.4) &	0.24 (3.8)&	7.32 (9.4)&	0.16 (0.2)&	0.09 (0.2) \\
    CPD & 0.50 (1.7)	&6.39 (3.6)&	0.9 (39.2)&	1.99 (1.6)&	0.17 (60.8)&	0.08 (63.2)\\
    m-ICP & 1.86 (28.9)	&11.23 (13.9)	&0.32 (2.3)	&11.29 (7.9)	&1.94 (8.1)	&1.85 (2.5)	\\
		d-ICP & 		11.14 (23.2) &	13.44 (6.1)	&3.92 (6.4)	& 23.68 (8.2)&	21.29 (8.1)&	21.05 (4.2)\\
    m-PWAN  (Ours) & 0.32 (19.5)	&4.67 (4.1)	&0.23 (0.2)&	1.17 (0.8)&	0.11 (0.4)&	0.10 (0.1)	\\
    d-PWAN  (Ours) & 0.36 (29.9)&	4.7 (4.1)&	0.23 (0.2)&	1.20 (0.8)&	0.12 (2.3)	&0.07 (7.4)\\
    \hline
	  \end{tabular}
	\end{center}
\end{table*}

\begin{figure*}[htb!]
	\centering
  \vspace{-1mm}
    \subfigure[Initial sets]{
        \begin{minipage}[b]{0.12\linewidth}
            \centering
            \includegraphics[width=1\linewidth]{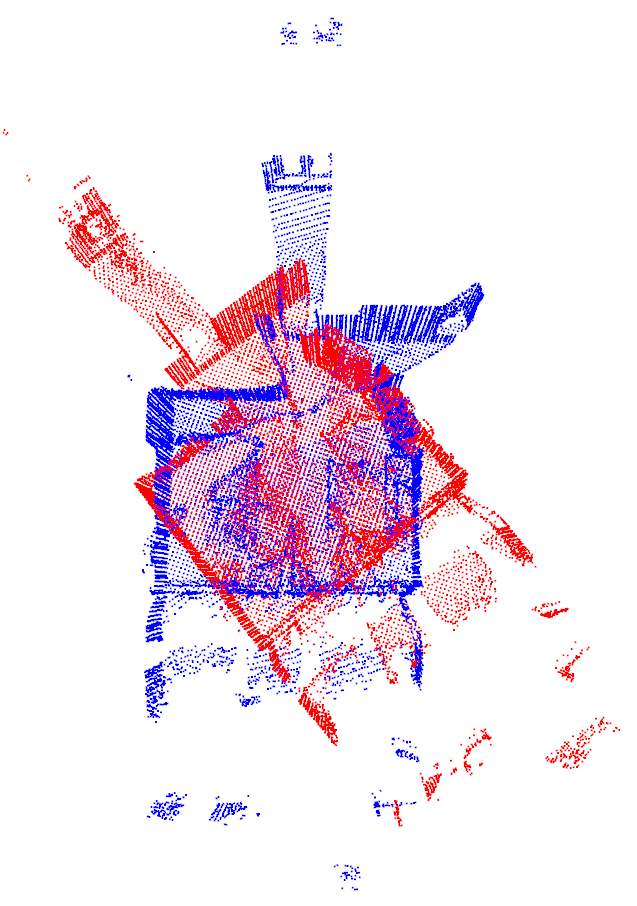}\\
            \includegraphics[width=1\linewidth]{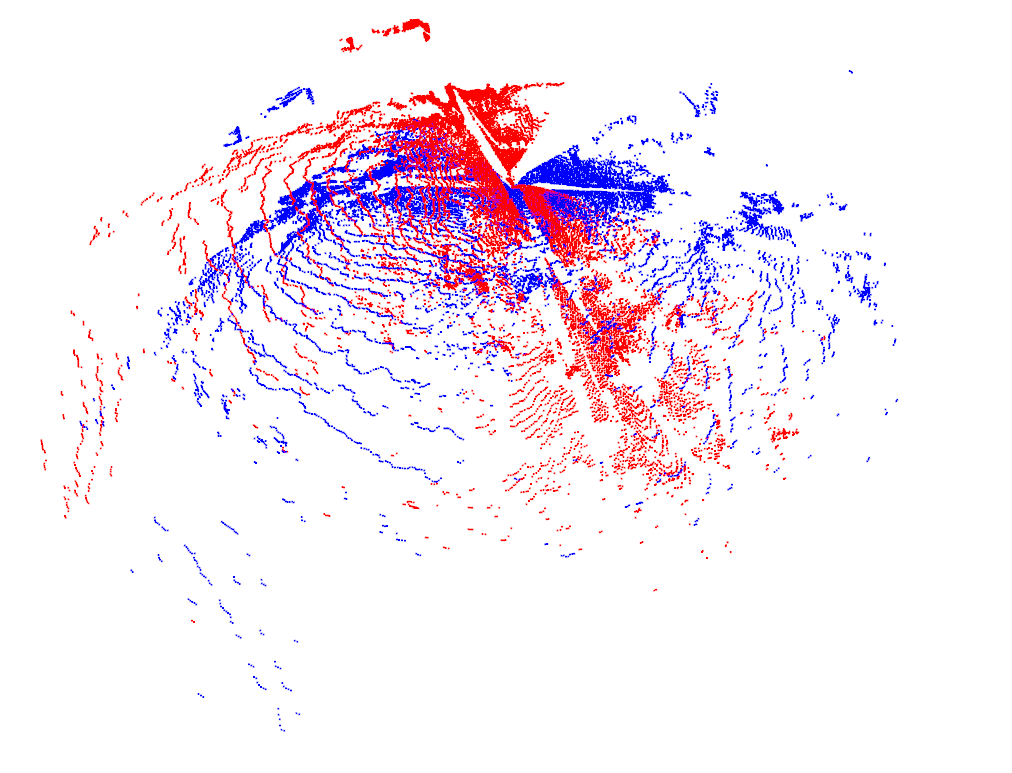} \\
            \includegraphics[width=1\linewidth]{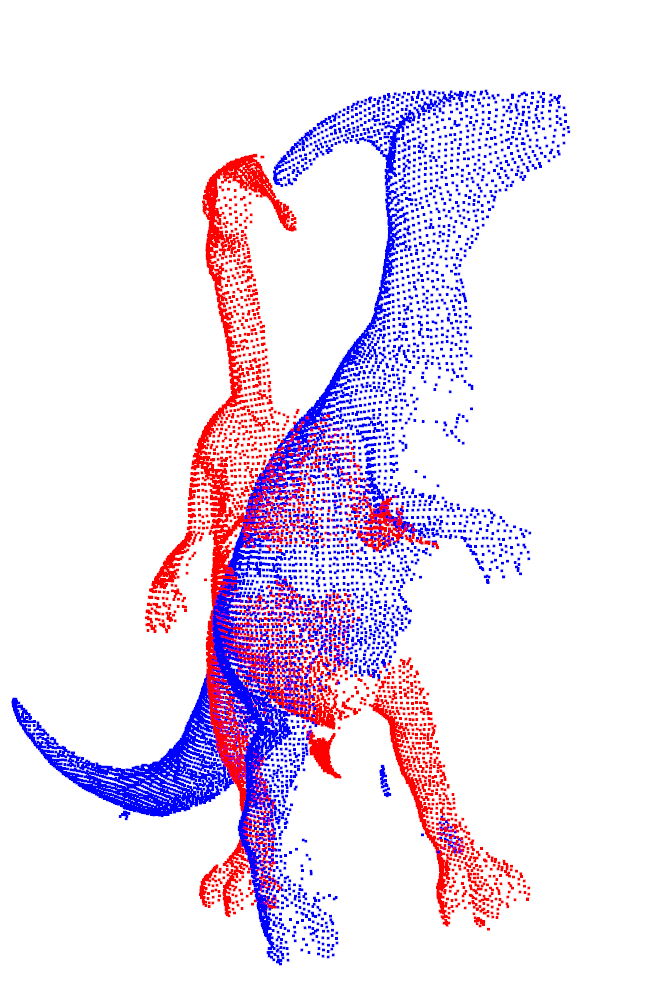} \\
        \end{minipage}
    }
    \subfigure[d-PWAN]{
        \begin{minipage}[b]{0.12\linewidth}
            \centering
            \includegraphics[width=1\linewidth]{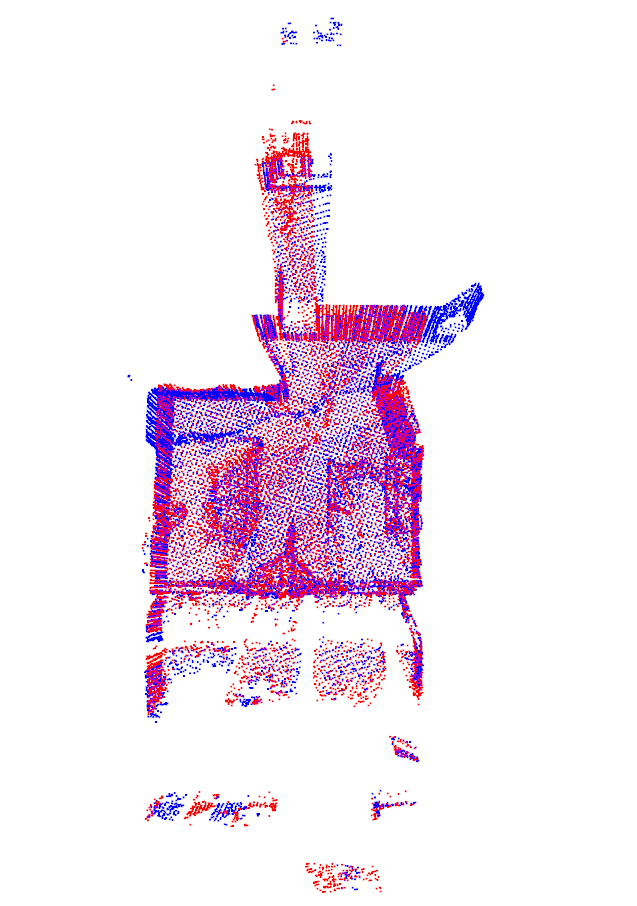} \\
            \includegraphics[width=1\linewidth]{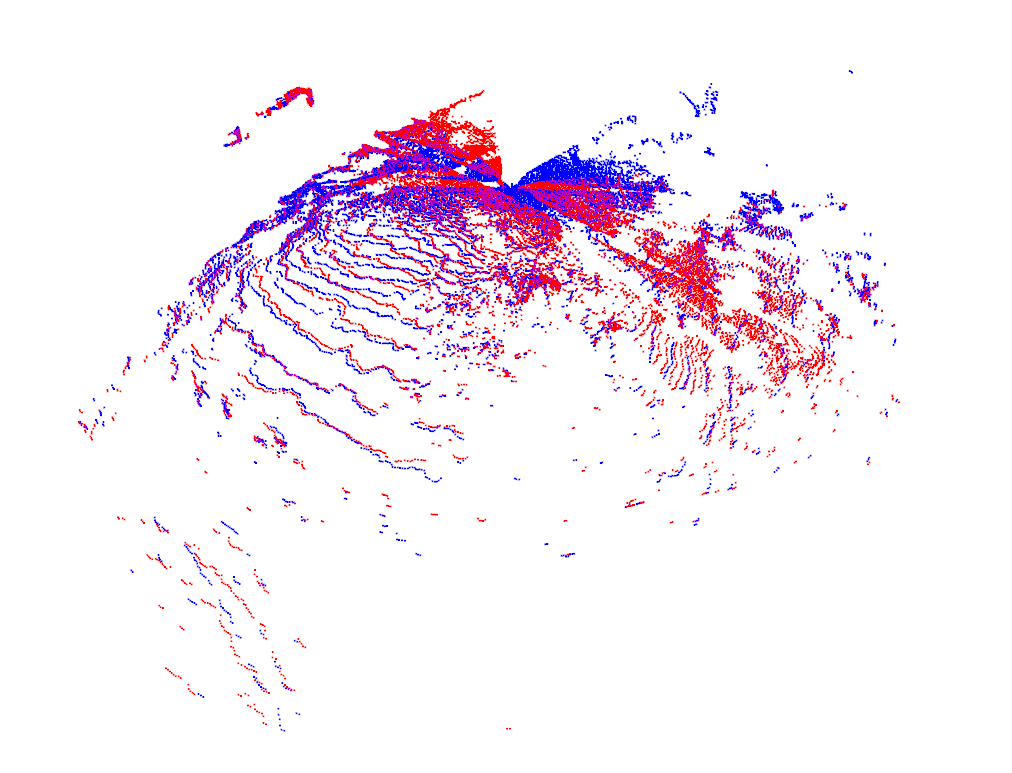} \\
            \includegraphics[width=1\linewidth]{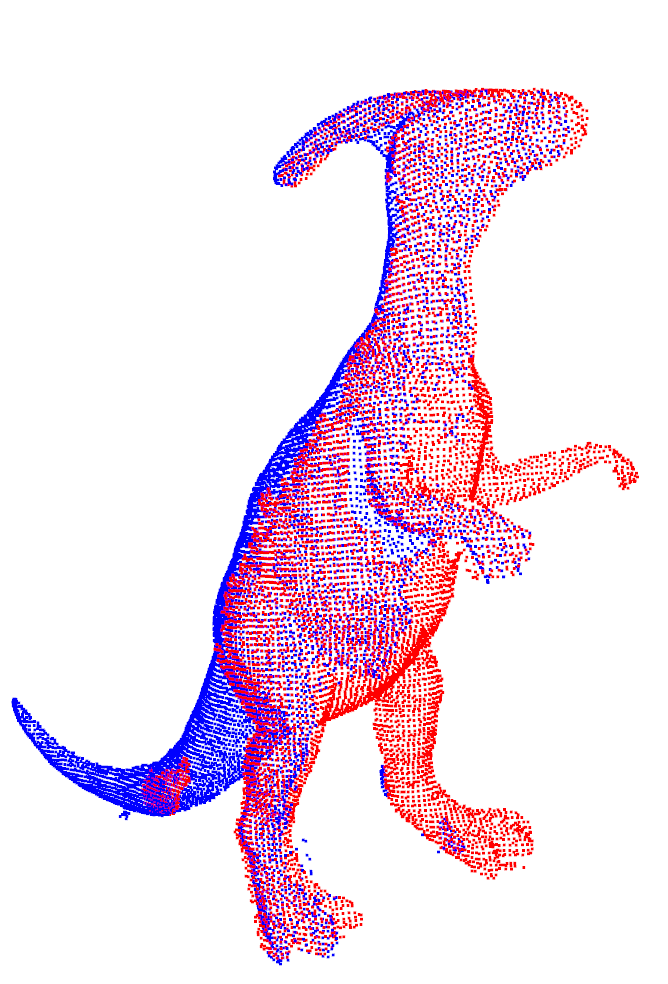} \\
        \end{minipage}
    }
    \subfigure[m-PWAN]{
        \begin{minipage}[b]{0.12\linewidth}
            \centering
            \includegraphics[width=1\linewidth]{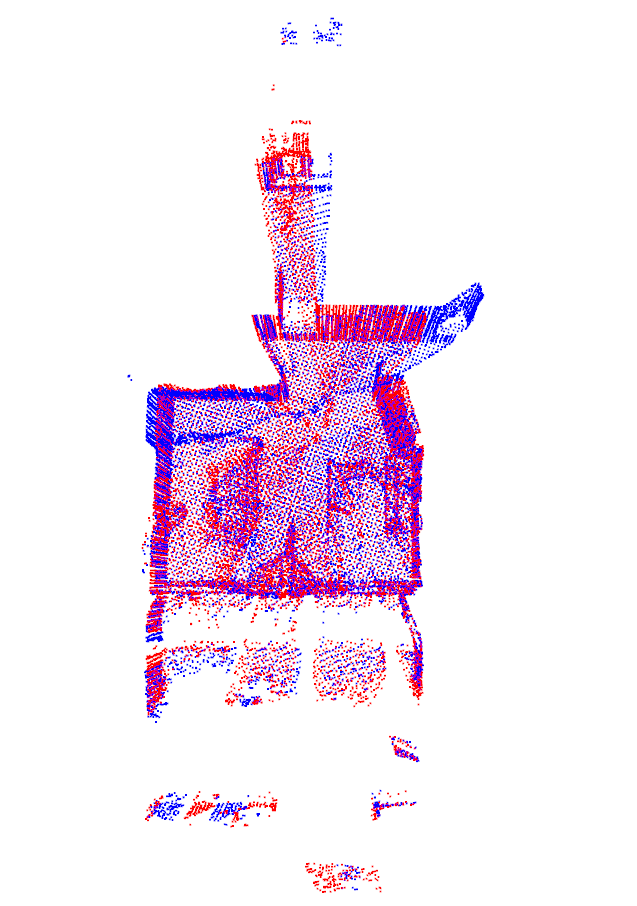} \\
            \includegraphics[width=1\linewidth]{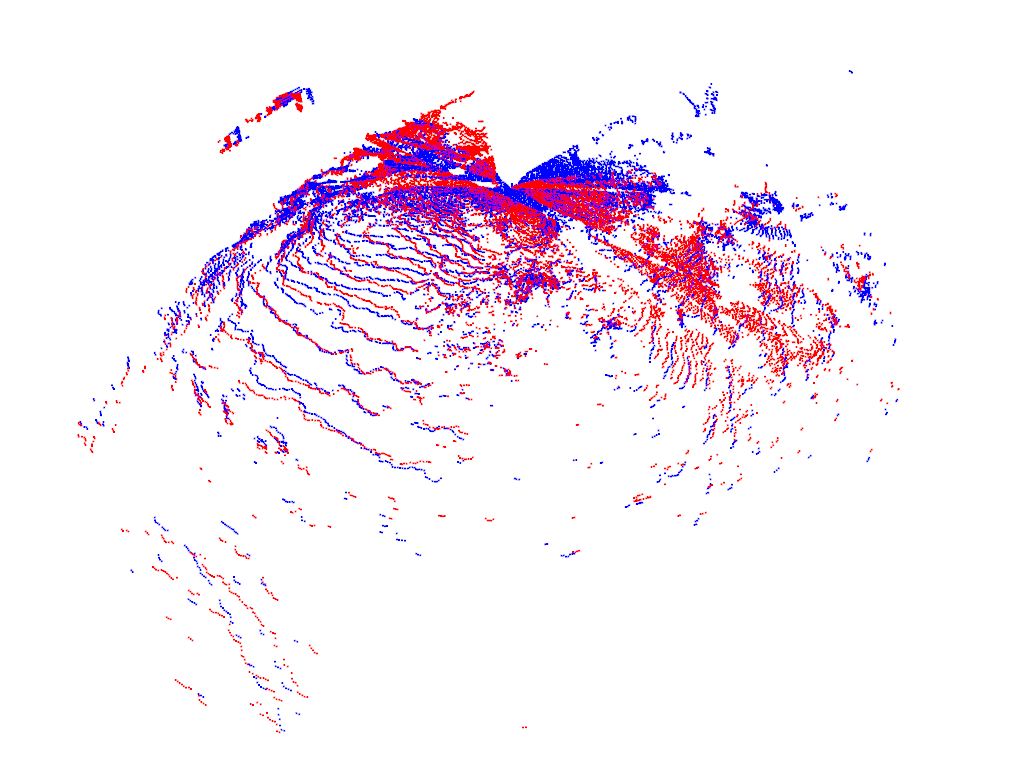} \\
            \includegraphics[width=1\linewidth]{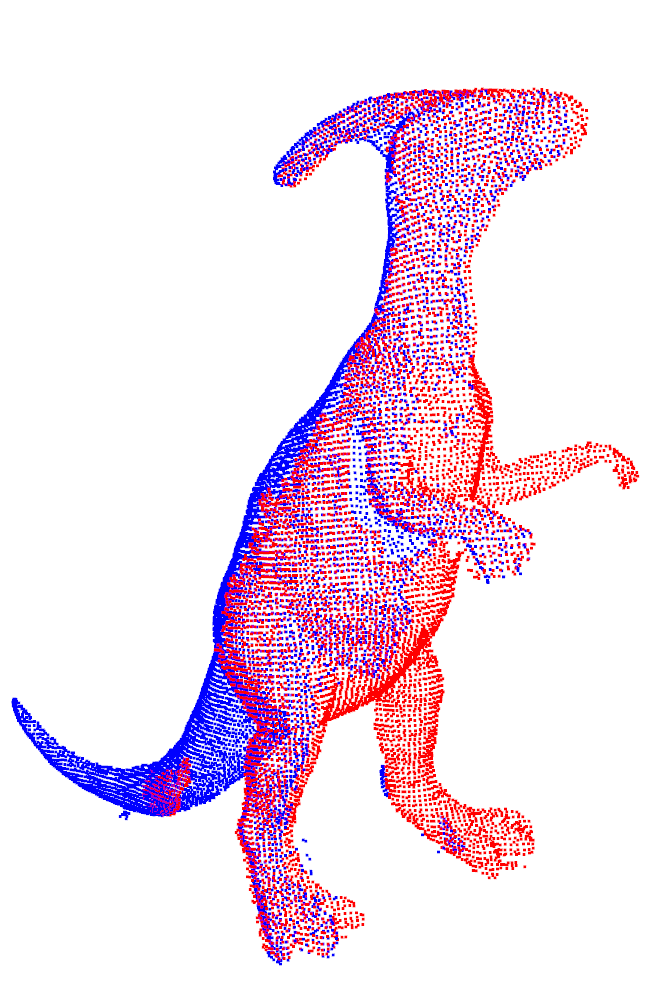} \\
        \end{minipage}
    }
    \subfigure[BCPD]{
        \begin{minipage}[b]{0.12\linewidth}
            \centering
            \includegraphics[width=1\linewidth]{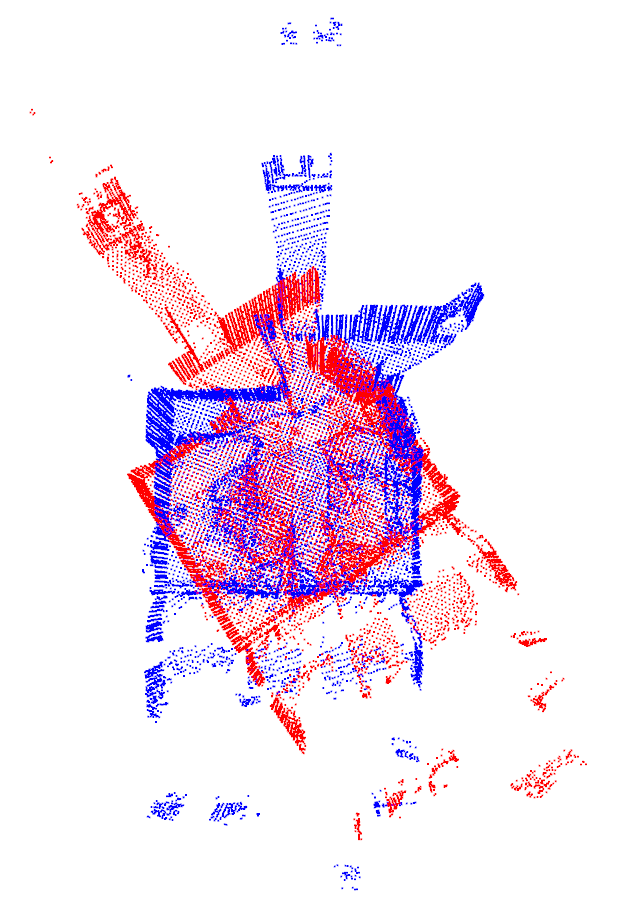} \\
            \includegraphics[width=1\linewidth]{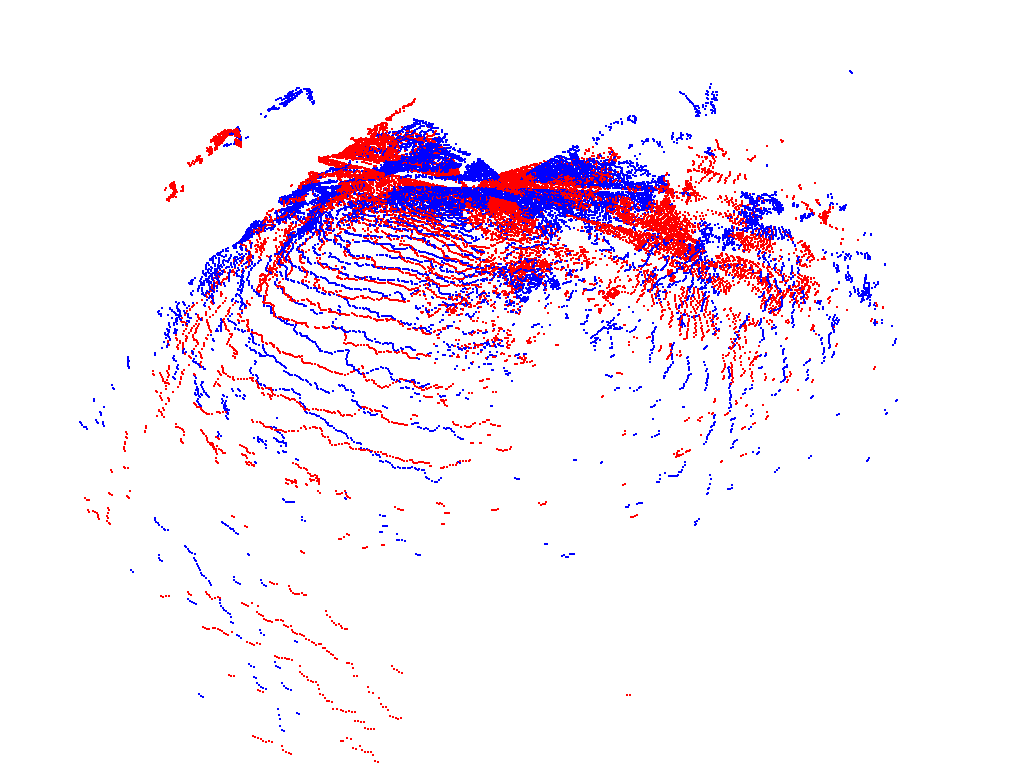} \\
            \includegraphics[width=1\linewidth]{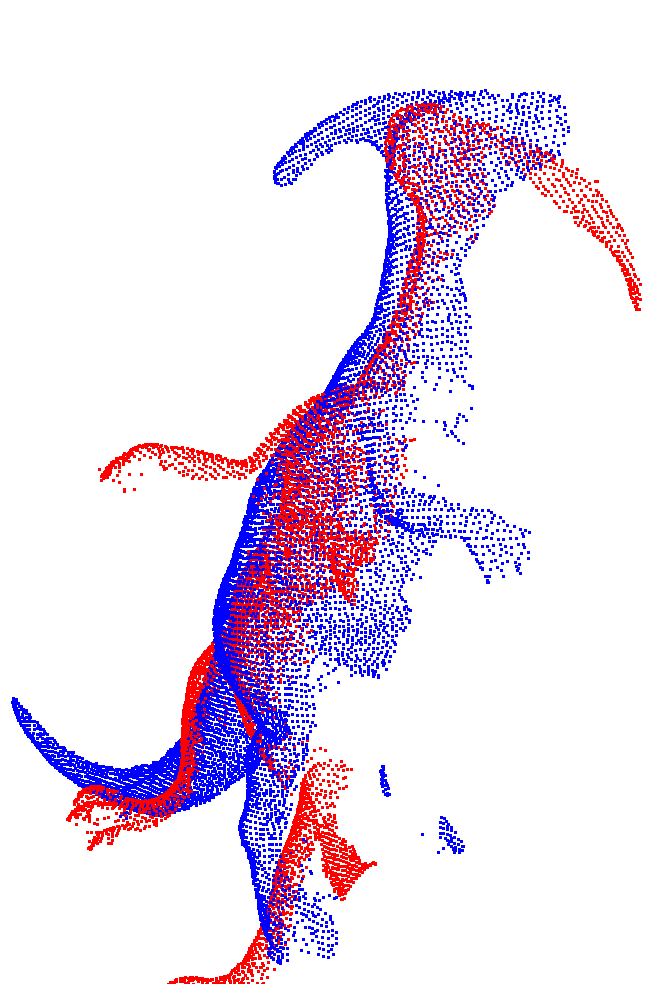} \\        
        \end{minipage}
    }
    \subfigure[CPD]{
        \begin{minipage}[b]{0.12\linewidth}
            \centering
            \includegraphics[width=1\linewidth]{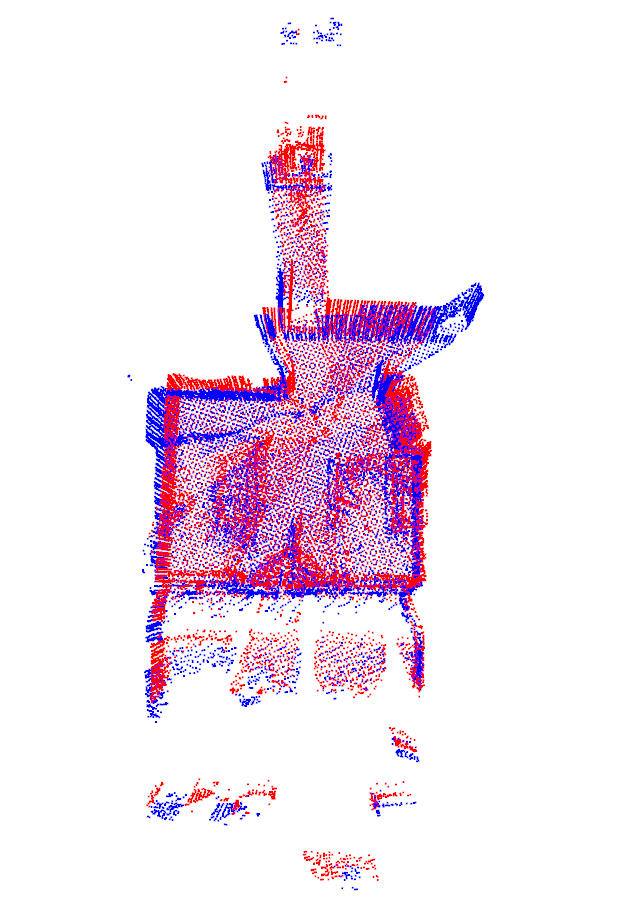}  \\
            \includegraphics[width=1\linewidth]{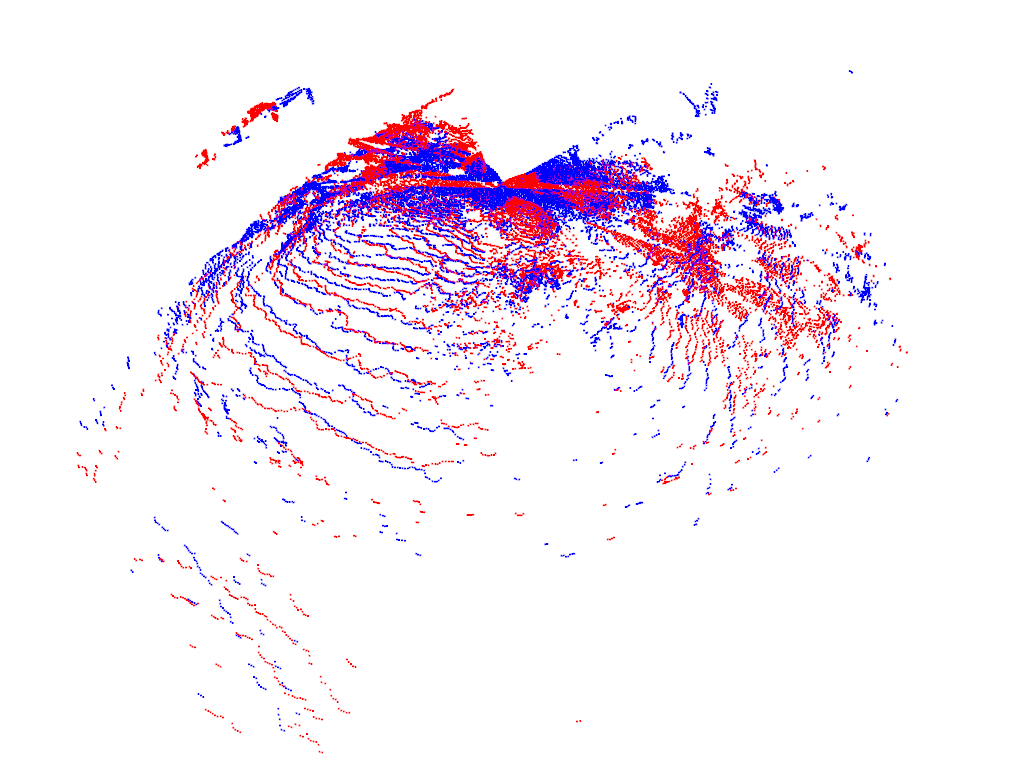} \\
            \includegraphics[width=1\linewidth]{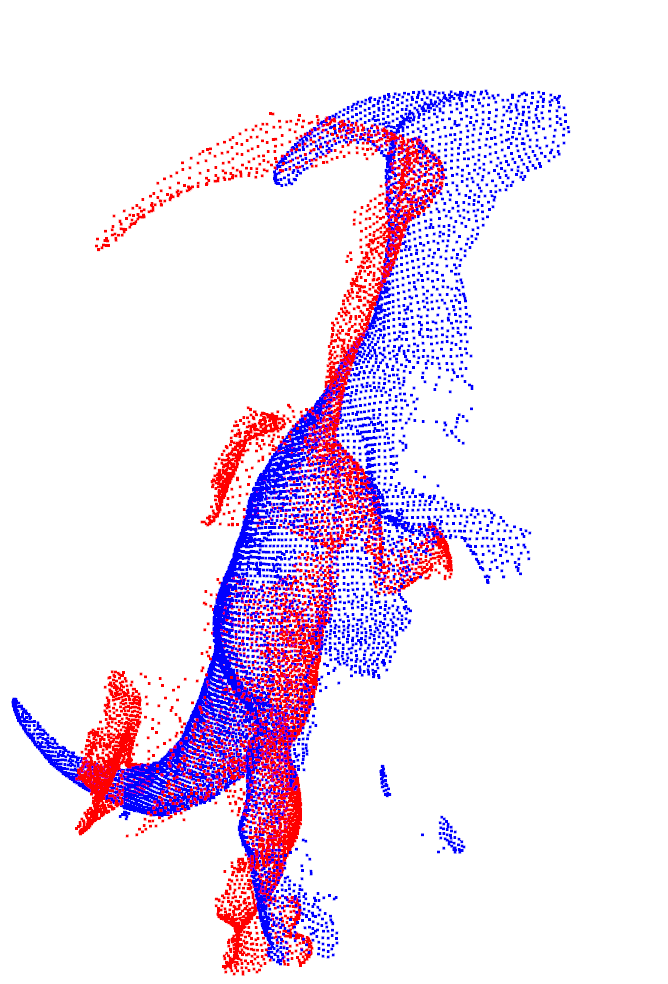} \\
        \end{minipage}
    }
    \subfigure[d-ICP]{
        \begin{minipage}[b]{0.12\linewidth}
            \centering
            \includegraphics[width=1\linewidth]{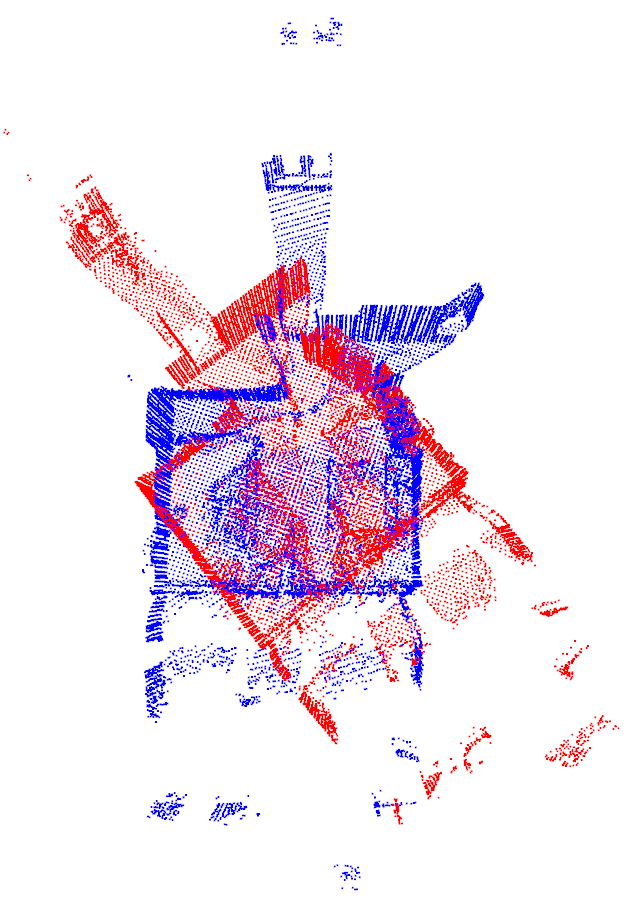}  \\
            \includegraphics[width=1\linewidth]{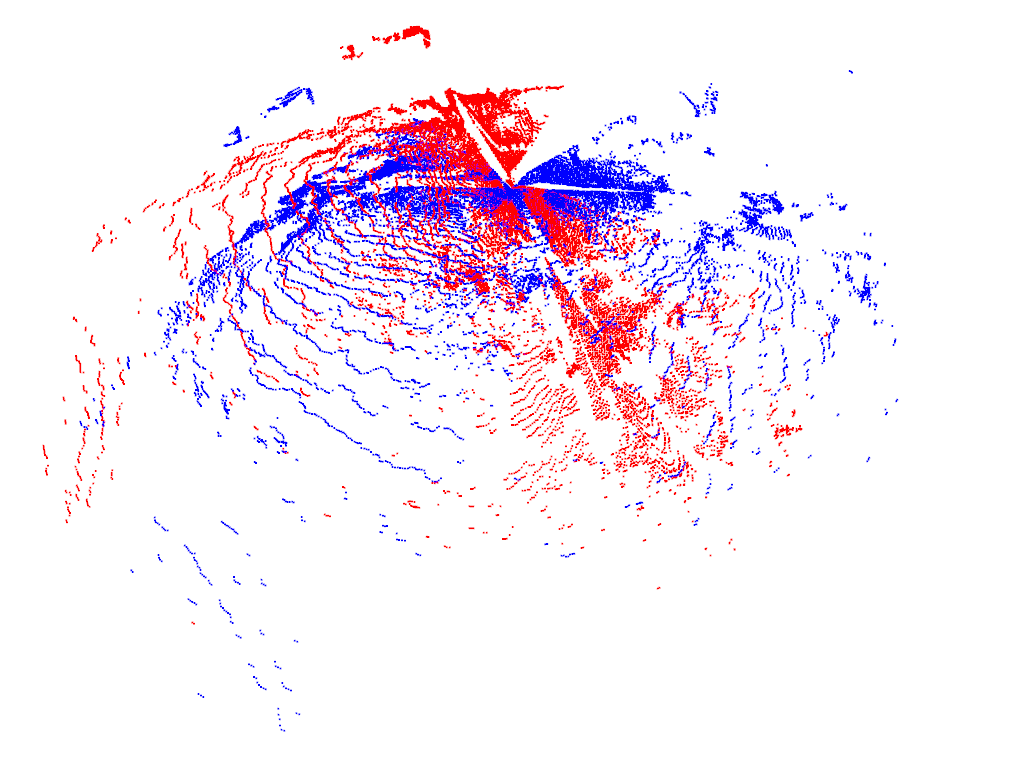} \\
            \includegraphics[width=1\linewidth]{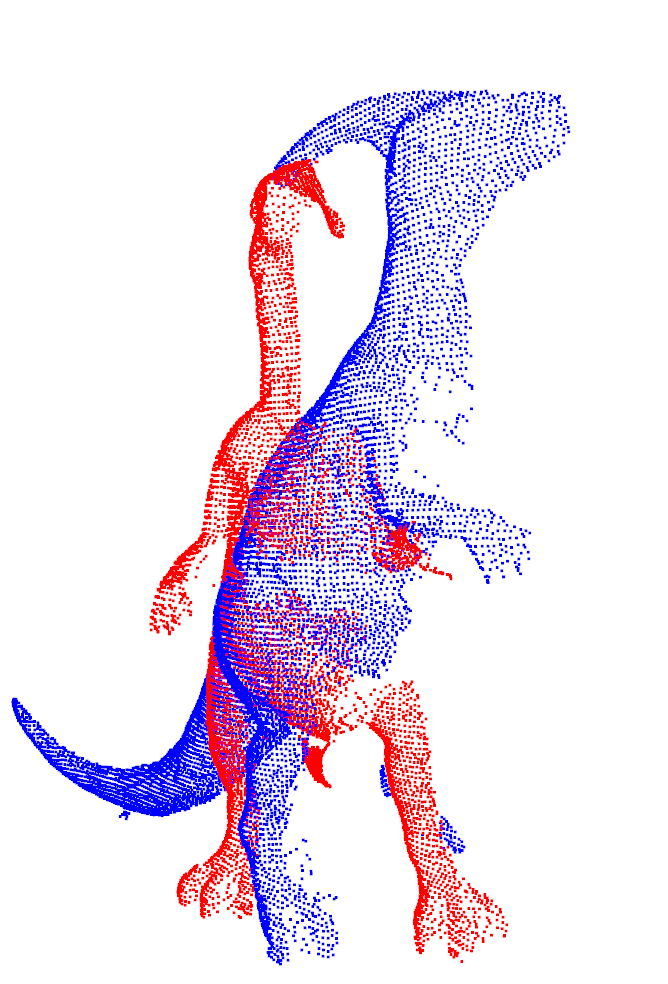} \\
        \end{minipage}
    }
    \subfigure[m-ICP]{
        \begin{minipage}[b]{0.12\linewidth}
            \centering
            \includegraphics[width=1\linewidth]{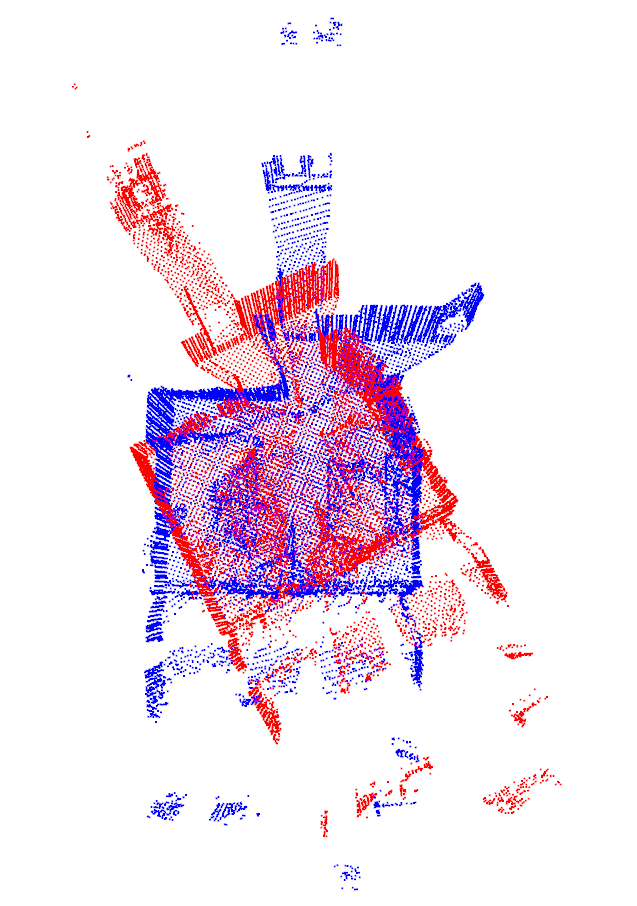}  \\
            \includegraphics[width=1\linewidth]{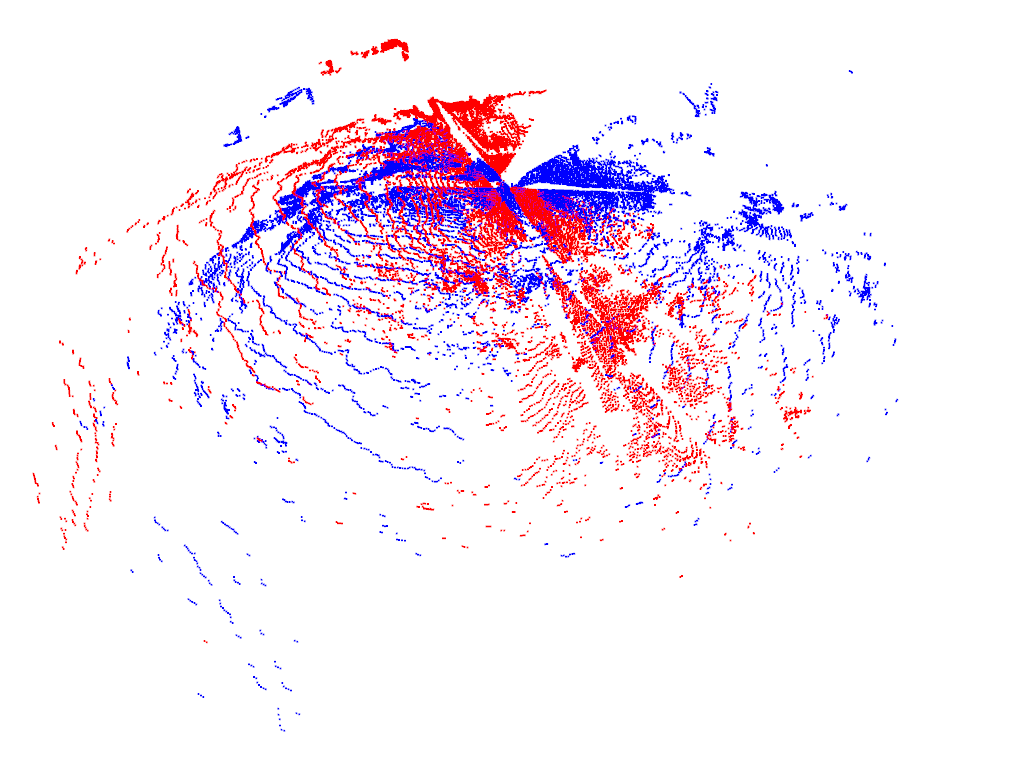} \\
            \includegraphics[width=1\linewidth]{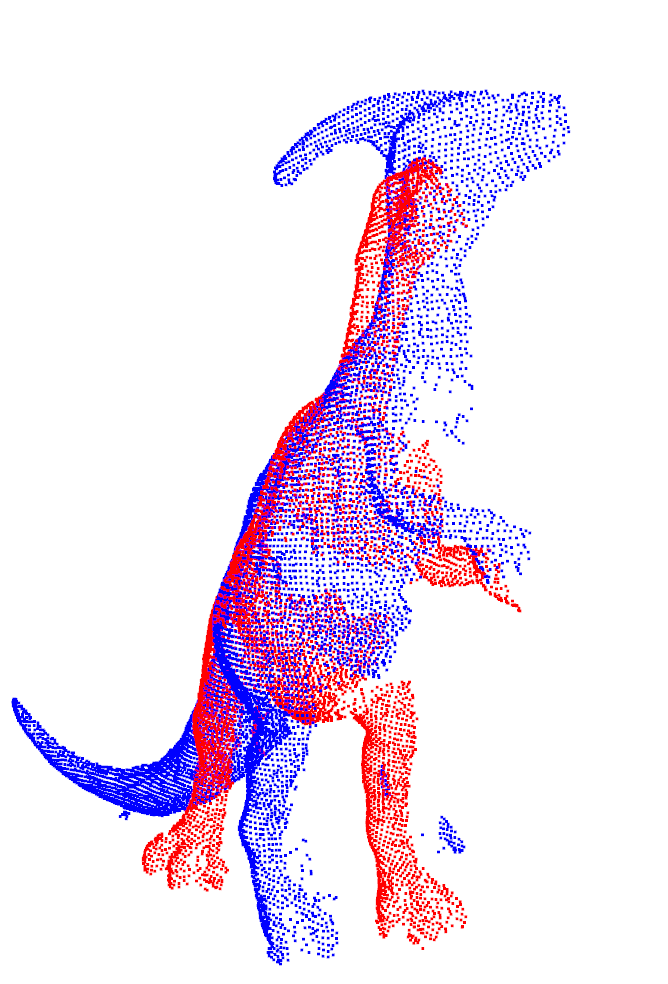} \\
        \end{minipage}
    }
    \vspace{-2mm}
    \caption{
    Examples of rigid registration results on a pair of apartment (indoor), mountain (outdoor) and parasaurolophus (object) point sets.
    }
    \vspace{-3mm}
    \label{rigid_app}
\end{figure*}

\subsection{Explanation of the Effect of $m_\alpha$}
\label{Sec_m_alpha_app}
To explain the effect of $m_\alpha$ in our formulation,
we present a toy example in Fig.~\ref{C_PDA},
where $\alpha$ and $\beta_\theta$ are uniform distributions supported on points.
We consider the correspondence given by $\mathcal{L}_{M,1}(\alpha, \beta_\theta)$,
where we fix $m_\beta=1$ and increase $m_\alpha$ from $1$ to infinity.
Fewer data points in $\alpha$ are aligned as $m_\alpha$ increases,
and the alignment gradually becomes the nearest neighborhood alignment when $m_\alpha$ is close to infinity.
\begin{figure}[htb!]
  \centering
  \subfigure[$m_\alpha=1$]{
    \includegraphics[width=0.32\linewidth]{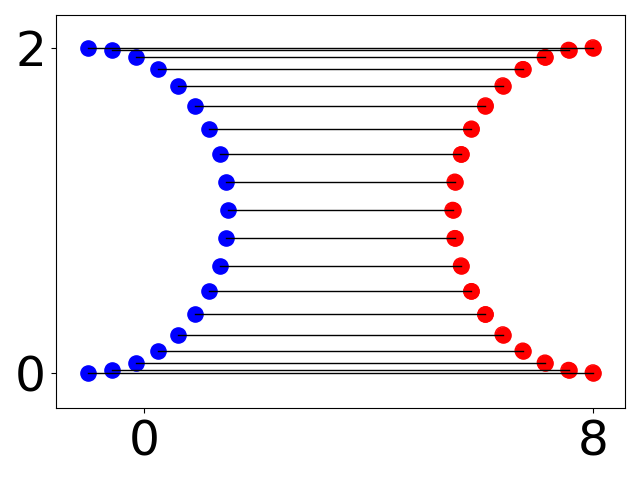}
  }
  \hspace{-4mm}
  \subfigure[$m_\alpha=3$]{
    \includegraphics[width=0.32\linewidth]{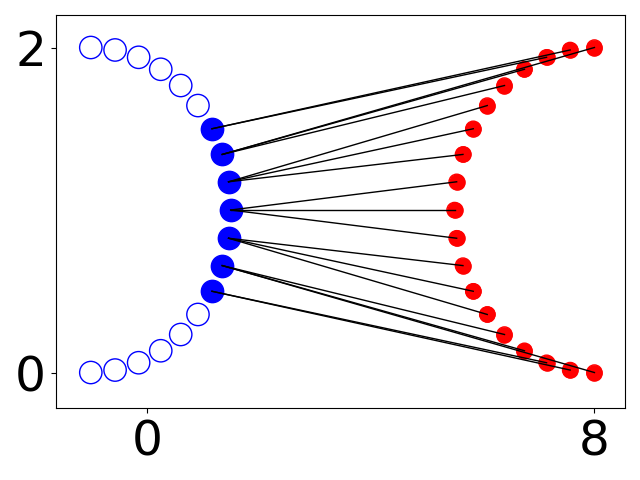}
  }
  \hspace{-4mm}
  \subfigure[$m_\alpha=20$]{
    \includegraphics[width=0.32\linewidth]{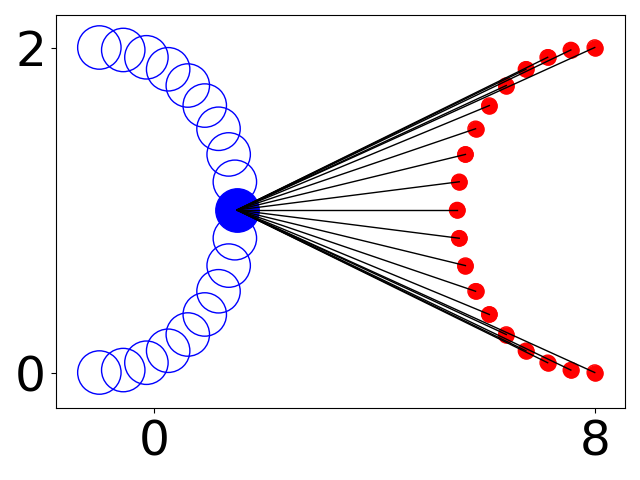}
  }
\caption{The alignment of $\alpha$ (blue) and $\beta_\theta$ (red) specified by $\mathcal{L}_{M,1}(\alpha, \beta_\theta)$ with $m_\beta=1$ and varying $m_\alpha$.
The size of each point is proportional to its mass.
Complete alignment ($m_\alpha=1$) gradually becomes the nearest neighborhood alignment ($m_\alpha = \infty $) as $m_\alpha$ increases.
}
\label{C_PDA}
\end{figure}

\subsection{Detailed Experimental Settings in Sec.~\ref{Sec_experiment_PDA}}
\label{Sec_experiment_PDA_app}
The parameters of PWAN are set as follows:
\begin{itemize}[leftmargin=3mm]
  \item[-]{Office-Home}:
  The batch size is $65$ and we use a stratified sampler as in~\cite{damodaran2018deepjdot},
  \ie, a mini-batch contains $1$ random sample from each reference class.
  We set $(T,u,s)=(5\times10^3, 5, 10^{3/10000})$.
  \item[-]{VisDa17}:
  We use batch size $60$ and a stratified sampler. $(T,u,s)=(10^4, 20, 10^{5/10000})$.
  \item[-]{ImageNet-Caltech}:
  We use batch size $100$. $(T,u,s)=(4.8\times10^4, 1, 10^{1/48000})$.
  \item[-]{DomainNet}:
  We use batch size $100$.   $(T,u,s)=(10^4, 1, 10^{6/100000})$.
  Following~\cite{gu2021adversarial},
  for this dataset,
  we do not use Relu activation function at the bottleneck layer.
\end{itemize}
In the above settings,
we select the largest possible batch size to fit in the GPU memory for each task,
and sufficiently long training steps are chosen to make sure PWAN converges.
We evaluate all methods by their test accuracy at the end of the training instead of the highest accuracy during the training~\cite{gu2021adversarial}. 
In our experiments,
the bottleneck size of ADV~\cite{gu2021adversarial} is reduced to $512$ from $2048$ for ImageNet-Caltech to prevent the out-of-memory error.

\subsection{More Details in Sec.~\ref{Sec_PDA_comparison}}
\label{Sec_PDA_comparison_app}

\begin{figure*}[tb!]
	\centering
    \hspace{3.8mm}
    \begin{minipage}{0.6\linewidth}
        \begin{minipage}[b]{1\linewidth}
            \subfigure[The result of PWAN with random seed $0$. (Accuracy=$94\%$)]{
                \begin{minipage}[b]{1\linewidth}
                    \includegraphics[width=0.3\linewidth]{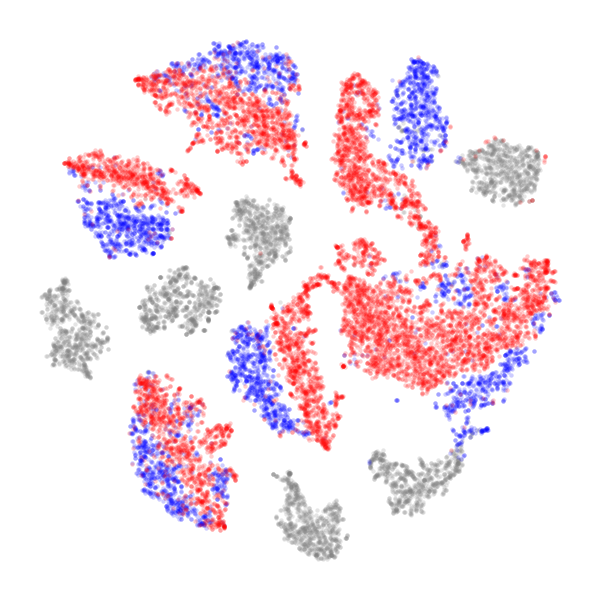}  
                    \includegraphics[width=0.6\linewidth]{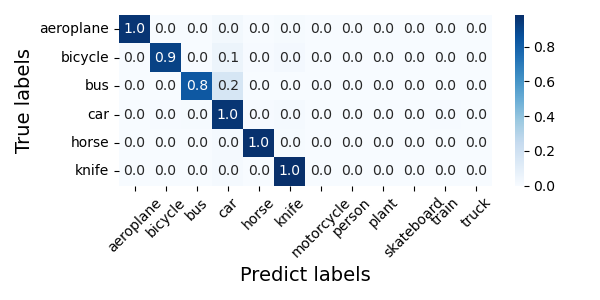}
                \end{minipage}
                \label{VisDa0}
            }
        \end{minipage}
        \begin{minipage}[b]{1\linewidth}
            \subfigure[The result of PWAN with random seed $1$. (Accuracy=$80\%$)]{
                \begin{minipage}[b]{1\linewidth}
                    \includegraphics[width=0.3\linewidth]{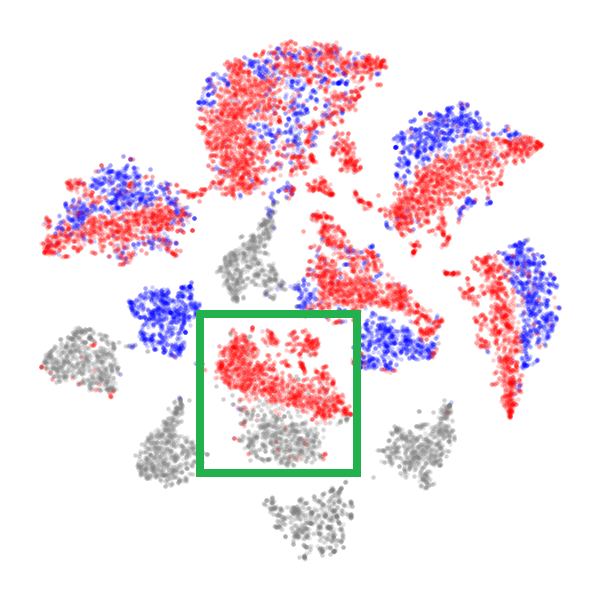}  
                    \includegraphics[width=0.6\linewidth]{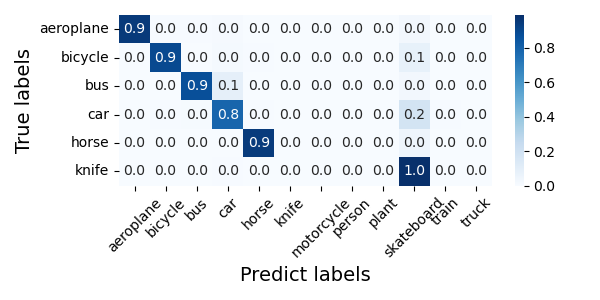} 
                \end{minipage}
                \label{VisDa1}
            }
        \end{minipage}
    \end{minipage}
    \begin{minipage}{0.3\linewidth}
        \subfigure[Random samples of the knife class in the source domain.]{
          \begin{minipage}[b]{1\linewidth}
            \includegraphics[width=0.1\linewidth]{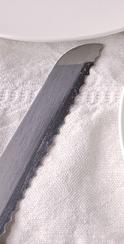}
            \includegraphics[width=0.2\linewidth]{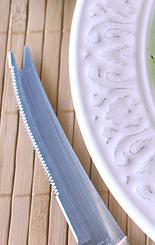}
            \includegraphics[width=0.2\linewidth]{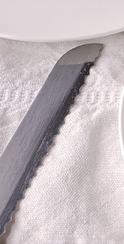}
            \includegraphics[width=0.2\linewidth]{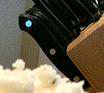}
            \includegraphics[width=0.2\linewidth]{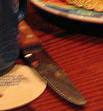}
          \end{minipage}
        \label{VisDa_knife}
        }
        \subfigure[Random samples of the skateboard class in the reference domain.]{
          \begin{minipage}[b]{1\linewidth}
            \includegraphics[width=0.3\linewidth]{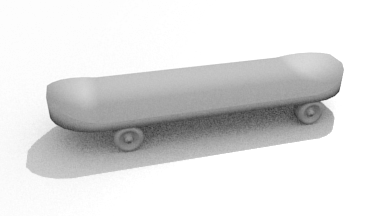}
            \includegraphics[width=0.3\linewidth]{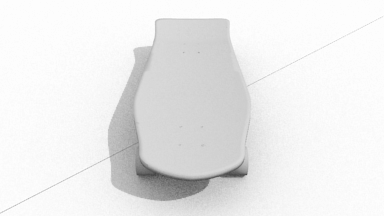}
            \includegraphics[width=0.3\linewidth]{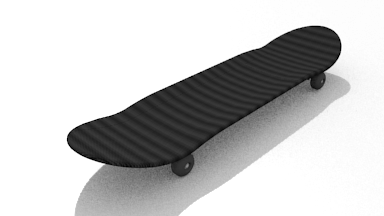} 

            \includegraphics[width=0.3\linewidth]{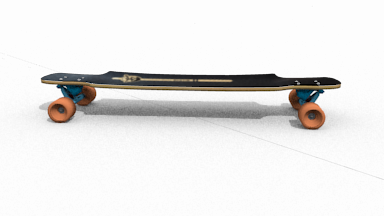}
            \includegraphics[width=0.3\linewidth]{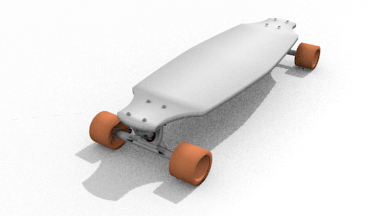}
            \includegraphics[width=0.3\linewidth]{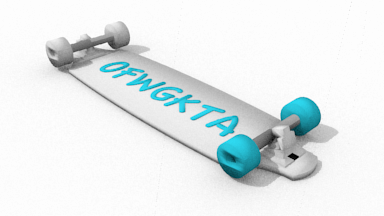}
          \end{minipage}
          \label{VisDa_skateboard}
        }
  \end{minipage}
\vspace{-2mm}
  \caption{
    The results of PWAN on VisDa17. 
    PWAN sometimes has difficulty discriminating the ``skateboard'' and the ``knife'' class due to their visual similarity.
    In practice,
    this issue can be easily solved by model selection using a small annotated validation set.
  }
	\label{ViDa}
\end{figure*}

\begin{figure}[t!]
    \centering
    \includegraphics[width=0.7\linewidth]{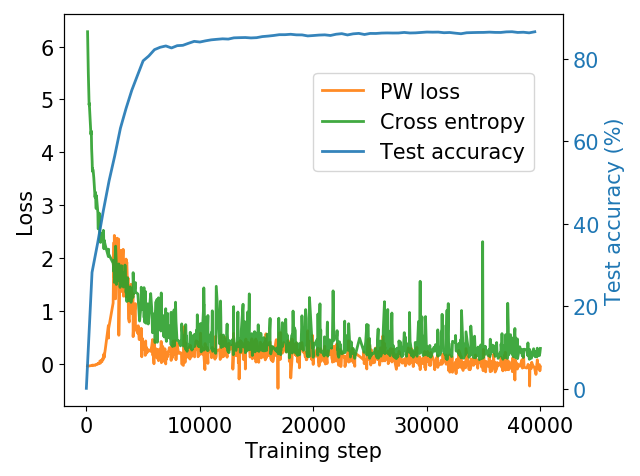}
  \caption{The training process of PWAN on ImageNet-Caltech.
  PW loss and cross-entropy loss gradually decrease until convergence.
  Meanwhile,
  test accuracy increases during the training process.
  }
  \label{PDA_training}
  \end{figure}  

\begin{figure}[t!]
	\centering
  \vspace{-1mm}
    \subfigure[CP-DomainNet]{
        \begin{minipage}[b]{0.45\linewidth}
            \centering
            \includegraphics[width=1\linewidth]{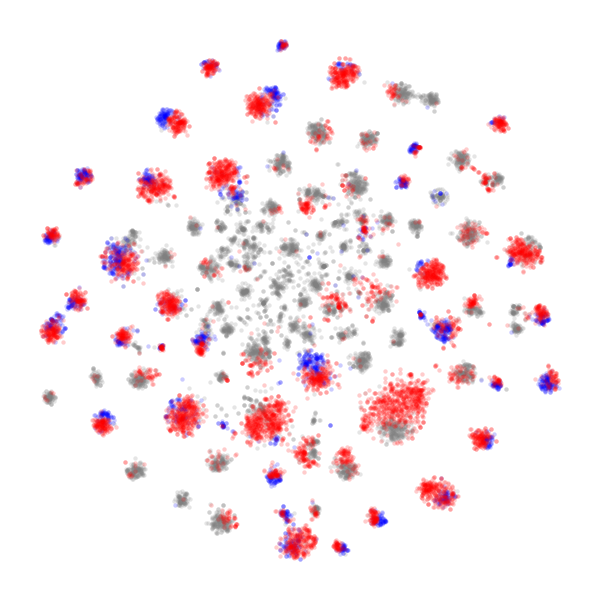}\\
        \end{minipage}
    }
    \subfigure[ImageNet-Caltech]{
        \begin{minipage}[b]{0.45\linewidth}
            \centering
            \includegraphics[width=1\linewidth]{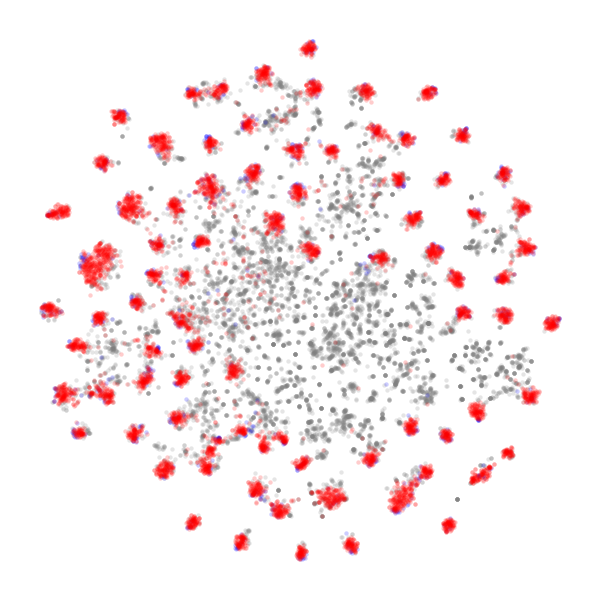} \\
        \end{minipage}
    }
    \vspace{-2mm}
    \caption{
    t-SNE visualization of the learned features of PWAN.
    Blue and gray points represent features in non-outlier and outlier classes in the reference domain respectively.
    Red points represent features in the source domain.
    }
    \vspace{-3mm}
    \label{vis_more}
\end{figure}

\begin{table*}[t!]
  \footnotesize
	\begin{center}
	  \caption{Ablation study on OfficeHome.
    We report the results of PWAN with label smoothing (L), 
    complement objective regularizer (C), 
    neighborhood reciprocity clustering (N),
    and $\alpha$-power (P).}
    \vspace{-3mm}
	  \setlength{\tabcolsep}{1.9pt}
    \label{Office_Home_tab_ablation}
      \begin{tabular}{c c c c c c c c c c c c c c} 
		  \hline
           & AC &	AP	& AR	& CA	& CP	& CR	& PA	& PC	& PR	& RA	& RC	& RP & Avg \\
		\hline
	  \centering
		PWAN  &63.5 (3.3)&	83.2 (2.7)&	89.3 (0.2)&	75.8 (1.5)&	75.5 (2.5)&	83.3 (0.1)&	77.0 (2.4)&	61.1 (1.5)&	86.9 (0.7)&	79.9 (0.6)&	65.0 (3.8)&	86.2 (0.2)	&77.2 (1.6)\\
    PWAN+L &63.9 (0.1)&	84.5 (3.3)	&90.3 (0.3)&	75.4 (0.9)&	75.4 (2.5)&	85.2 (1.2)	&78.0 (1.6)&	63.3 (1.4)	&87.5 (0.8)&	79.7 (0.7)&	66.3 (2.7)&	86.5 (0.8)&	78.0 (1.4)\\
    PWAN+LC  &65.2 (0.6)&	84.5 (3.0)	&89.9 (0.2)&	76.7 (0.6)&	76.8 (1.9)&	84.3 (1.8)&	78.7 (3.2)&	64.2 (0.9)&	87.3 (0.9)&	79.9 (1.1)&	68.0 (2.3)&	86.9 (0.9)&	78.5 (1.0) \\
    PWAN+LCN  & 66.2 (1.9) &	84.2 (2.7)&	90.0 (0.6)&	79.7 (0.9)&	78.4 (0.7)&	85.0 (2.2)&	80.4 (1.8)&	66.8 (0.4)&	88.4 (0.9)&	81.0 (1.1)&	67.7 (3.6)&	87.2 (0.4)&	79.6 (0.9) \\
    PWAN+LCNP  	& 65.4 (0.2) &	88.0 (1.8) &	89.9 (0.3) &	79.2 (1.1) &	78.0 (1.1) &	88.0 (0.5) &	80.5 (1.6)& 66.2 (0.5) &	88.6 (0.8)&	81.8 (0.5)&	70.2 (3.4) &	90.1 (0.6)&	80.5 (0.3)\\
    \hline
	  \end{tabular}
	\end{center}
\end{table*}

We provide more details of the result of PWAN on VisDa17 in Fig.~\ref{ViDa}.
We observe that PWAN generally performs well on all classes,
except that it sometimes recognizes the ``knife'' class as the ``skateboard'' classes as shown in Fig.~\ref{VisDa0} and Fig.~\ref{VisDa1}.
This error is somehow reasonable because these two classes are visually similar as shown in \ie, Fig.~\ref{VisDa_knife} and Fig.~\ref{VisDa_skateboard}.
We argue that this ambiguity can be easily addressed in real applications,
as we can choose the high-performance model by testing the model on a small annotated validation set.
In addition,
compared to the results in Fig.8 in~\cite{cao2022big},
PWAN achieves much better results,
as it can easily discriminate classes ``bicycle''/``motorcycle'' and ``bus''/``car''/``train''.

The training process on ImageNet-Caltech dataset is presented in Fig.~\ref{PDA_training},
where both cross-entropy and PW divergence converge during the training process.
In addition,
we provide a visualization of the learned features on ImageNet-Caltech and DomainNet in Fig.~\ref{vis_more},
where PWAN successfully aligns the source features to the reference features even when the reference data is dominated by outliers 
($68\%$ of classes of DomainNet and $92\%$ of classes of ImageNet-Caltech are outliers).


\end{document}